\documentclass[twoside]{article}

\usepackage[accepted]{aistats2024}
\usepackage[round]{natbib}

\usepackage[utf8]{inputenc} 
\usepackage[T1]{fontenc}    
\usepackage[usenames,dvipsnames]{xcolor}
\usepackage[pagebackref=true,colorlinks = true,urlcolor = Blue,linkcolor = Blue,citecolor = Blue,pdfborder={0 0 0}]{hyperref}      
\usepackage{cleveref}
\usepackage{url}            
\usepackage{booktabs}       
\usepackage{amsfonts, amsmath, amsthm}       
\usepackage{nicefrac}       
\usepackage{microtype}      
\usepackage{algorithm}
\usepackage[noend]{algorithmic}
\usepackage{graphicx}
\usepackage{dblfloatfix}
\usepackage{paralist}
\usepackage{subfigure}

\DeclareMathOperator*{\argmax}{arg\,max}
\DeclareMathOperator*{\argmin}{arg\,min}

\newtheorem{theorem}{Theorem}
\newtheorem{assumption}{Assumption}
\newtheorem{proposition}{Proposition}
\newtheorem{corollary}{Corollary}

\newtheorem{remark}{Remark}

\newtheorem{innercustomthm}{Theorem}
\newenvironment{customthm}[1]
  {\renewcommand\theinnercustomthm{#1}\innercustomthm}
  {\endinnercustomthm}

\begin{document}

%

%

\runningauthor{ Louis Sharrock, Daniel Dodd, Christopher Nemeth}

\twocolumn[

\aistatstitle{Tuning-Free Maximum Likelihood Training of Latent Variable Models via Coin Betting}

\aistatsauthor{ Louis Sharrock* \And Daniel Dodd* \And  Christopher Nemeth }

\aistatsaddress{ Department of Mathematics and Statistics, Lancaster University, UK } ]

\begin{abstract}
  We introduce two new particle-based algorithms for learning latent variable models via marginal maximum likelihood estimation, including one which is entirely tuning-free. Our methods are based on the perspective of marginal maximum likelihood estimation as an optimization problem: namely, as the minimization of a free energy functional. One way to solve this problem is via the discretization of a gradient flow associated with the free energy. We study one such approach, which resembles an extension of Stein variational gradient descent, establishing a descent lemma which guarantees that the free energy decreases at each iteration. This method, and any other obtained as the discretization of the gradient flow, necessarily depends on a learning rate which must be carefully tuned by the practitioner in order to ensure convergence at a suitable rate. With this in mind, we also propose another algorithm for optimizing the free energy which is entirely learning rate free, based on coin betting techniques from convex optimization. We validate the performance of our algorithms across several numerical experiments, including several high-dimensional settings. Our results are competitive with existing particle-based methods, without the need for any hyperparameter tuning.
\end{abstract}

\section{INTRODUCTION}
\label{sec:intro}

In statistics and machine learning, probabilistic latent variable models $p_\theta(z,x)$ comprising model \emph{parameters} $\smash{\theta \in \Theta \subseteq \mathbb{R}^{d_\theta}}$, unobserved \emph{latent variables} $\smash{z \in \mathcal{Z} \subseteq \mathbb{R}^{d_z}}$, and \emph{observations} $x \in \mathcal{X} \subseteq \mathbb{R}^{d_x}$, are widely used to capture the hidden structure of complex data such as images \citep{Bishop2006}, audio \citep{Smaragdis2006}, text \citep{Blei2003}, and graphs \citep{Hoff2002}. In this paper, we consider the task of estimating the parameters in such models by maximizing the marginal likelihood of the observed data, 
\begin{align}
\label{eq:marginal-likelihood}
    \theta_{*} &= \argmax_{\theta\in\Theta} p_{\theta}(x):= \argmax_{\theta\in\Theta} \int_{\mathcal{Z}}p_{\theta}(z,x)\mathrm{d}z,
\end{align}
and quantifying the uncertainty in the latent variables through the corresponding posterior $\smash{p_{\theta_{*}}(z|x) = {p_{\theta_{*}}(z,x)}/{p_{\theta_{*}}(x)}}$. This framework, which represents a pragmatic compromise between frequentist and Bayesian approaches, is known as the {empirical Bayes} paradigm \citep{Robbins1956, Casella1985}. 

A classical approach for solving the marginal maximum likelihood estimation (MMLE) problem in \eqref{eq:marginal-likelihood} is the {Expectation Maximization} (EM) algorithm \citep{Dempster1977}. This iterative method consists of two steps: an {expectation step} (E-step) and a {maximization step} ({M-step}). In the $t^{\text{th}}$ iteration, the E-step involves computing the expectation of the log-likelihood with respect to the current posterior distribution $\mu_{t} := p_{\theta_t}(\cdot|x)$ of the latent variables, viz, 
\begin{equation}
    Q_{t}(\theta) = \int_{\mathcal{Z}} \log \pi_{\theta}(z) \mu_{t}(z) \mathrm{d}z, \tag{E} \label{eq:E-step} 
\end{equation}
where $\pi_{\theta}(z):= p_{\theta}(z,x)$ denotes the joint density of $z$ and $x$, given fixed $x\in\mathbb{R}^{d_x}$. Meanwhile, the M-step involves optimizing this quantity with respect to the parameters, namely
\begin{equation}
    \theta_{t+1} := \argmax_{\theta\in\Theta}  Q_t(\theta). \tag{M} \label{eq:M-step} 
\end{equation}
 Under fairly general conditions, this procedure guarantees convergence of the parameters $\theta_t$ to a stationary point $\theta_{*}$ of the marginal likelihood, and convergence of the corresponding posterior approximations $p_{\theta_t}(\cdot | x)$ to $p_{\theta_{*}}(\cdot|x)$ \citep{Wu1983,Redner1984,McLachlan2007,Balakrishnan2017}. 
 In many applications of interest, neither the E-step nor the M-step is analytically tractable, in which case it is necessary to use approximations. In particular, the M-step can be approximated using numerical optimization routines \citep{Meng1993,Liu1994,Lange1995}. Meanwhile, the E-step can be approximated via Monte-Carlo methods, leading to the so-called {Monte Carlo EM} (MCEM) algorithm \citep{Wei1990,Tanner1993,Chan1995,Booth1999,Cappe1999,Sherman1999},
 or otherwise a Robbins-Monro  type stochastic approximation procedure \citep{Robbins1951,Kushner1978}, resulting in the {stochastic approximation EM} (SAEM) algorithm \citep{Delyon1999}. In practice, it is often not possible to sample exactly from the current posterior distribution $p_{\theta_t}(\cdot|x)$. In this setting, it is standard to use a Markov chain Monte Carlo (MCMC) approximation for the E-step \citep[e.g.,][]{McCulloch1997,Levine2001,Fort2003,Levine2004,Caffo2005,Fort2011,Atchade2017,Qiu2020, Nijkamp2020,DeBortoli2021}.

In this paper, we follow a different approach, based on an observation first made in \cite{Neal1998}, and recently revisited in \cite{Kuntz2023}, that the EM algorithm can be viewed as coordinate-descent on a free-energy functional $\mathcal{F}$ (see Sec. \ref{sec:free-energy}).
Leveraging this perspective, \cite{Kuntz2023} obtain a set of particle-based algorithms for optimizing the free energy, and thus for solving \eqref{eq:marginal-likelihood}. A similar approach has since also been studied by \cite{Akyldz2023}. We provide a more detailed discussion of these works in Section \ref{sec:related-work}.

Both in theory and in practice, the performance of these methods relies on a careful choice of hyperparameters such as the learning rate. In particular, the learning rate must be small enough to ensure stability of the parameter updates, whilst also large enough to guarantee convergence at a reasonable rate \citep[][Theorem 1]{Akyldz2023}. The situation is additionally complicated by the interdependence between the parameters and the latent variables, as well as the sensitivity of the learning rate to other hyperparameters such as the number of particles. As a result, it is typically necessary to make use of heuristics such as Adagrad \citep{Duchi2011}, Adam \citep{Kingma2015}, or RMSProp \citep{Tieleman2012}, or otherwise to resort to computationally expensive quasi-Newton updates \citep[App. C]{Kuntz2023} which exploit a second order approximation of the log-likelihood to better capture its local geometry. In this context, it is natural to ask whether we can obtain alternative algorithms which are more robust to different specifications of the learning rate, or even remove the dependence on the learning rate entirely.

\textbf{Our contributions.} In this paper, we propose two new particle-based algorithms for MMLE in latent variable models, including one which is completely tuning free. Our methods can be applied to a very broad class of latent variable models, namely, any for which the density $p_{\theta}(z,x)$ is differentiable in $\theta$ and $z$.
Inspired by \cite{Kuntz2023}, both of our algorithms are rooted in the viewpoint of MMLE as the minimization of the free energy functional. Our first approach, SVGD EM, is motivated by recent developments in the theory and application of gradient flows on the space of probability measures \cite[e.g.,][]{Otto2001,Ambrosio2008,Duncan2019}.
In particular, SVGD EM corresponds to the discretization of a particular gradient flow of the free energy $\mathcal{F}$ on $\Theta\times\mathcal{P}_2(\mathcal{Z})$, and resembles a generalization of the popular Stein variational gradient descent (SVGD) algorithm. Meanwhile, our second approach, Coin EM, is inspired by the parameter-free stochastic optimization methods developed by Orabona and coworkers \citep{Orabona2016,Orabona2017}, 
and their recent extension to the space of probability measures \citep{Sharrock2023}. In particular, Coin EM leverages a reduction of the minimization of the free energy to two interacting coin betting games. Unlike our first algorithm, or the recent schemes in \cite{Akyldz2023,Kuntz2023}, Coin EM does not correspond to the time-discretization of any gradient flow, and has no learning rates. It thus bears little resemblance to existing particle-based methods for training latent variable models.

After introducing our algorithms, we study the convergence properties of SVGD EM, establishing a descent lemma which guarantees that this algorithm decreases the free energy at each iteration. We then illustrate the performance of our algorithms on a wide range of examples, including a toy hierarchical model, two Bayesian logistic regression models, two Bayesian neural networks, and a latent space network model. Our results indicate that SVGD EM and Coin EM achieve comparable or superior performance to existing particle-based EM algorithms, with Coin EM removing entirely the need to tune any hyperparameters.

\section{MAXIMUM LIKELIHOOD TRAINING OF LATENT VARIABLE MODELS}
\label{sec:em-methodology}

\paragraph{Notation.} We will make use of the following notation. 
Let $\mathcal{P}_2(\mathcal{Z})$ denote the set of probability measures with finite second moment.
Given $\mu\in\mathcal{P}_2(\mathcal{Z})$, let $L^2(\mu)$ denote the space of functions $\smash{f:\mathcal{Z}\rightarrow\mathcal{Z}}$ such that $\smash{\int||f||^2\mathrm{d}\mu<\infty}$, with $\smash{||\cdot||_{L^2(\mu)}}$ and $\smash{\langle \cdot,\cdot\rangle_{L^2(\mu)}}$ the norm and inner product in this space. For 
 $\mu\in\mathcal{P}_2(\mathcal{Z})$ and measurable $T:\mathbb{R}^d\rightarrow\mathbb{R}^d$, we write $T_{\#}\mu$ for the pushforward of $\mu$ under $T$. For $\mu,\nu\in\mathcal{P}_2(\mathcal{Z})$, the quadratic Wasserstein distance between $\mu$ and $\nu$ is defined by $\smash{W_2^2(\mu,\nu) = \inf_{\gamma \in\Gamma(\mu,\nu)}  \int_{\mathcal{Z}\times\mathcal{Z}}||z_1-z_2||^2 \gamma(\mathrm{d}z_1,\mathrm{d}z_2)}$, where $\Gamma(\mu,\nu)$ is the set of couplings between $\mu$ and $\nu$. This is indeed a distance over $\mathcal{P}_2(\mathcal{X})$ \citep[Chapter 7.1]{Ambrosio2008}. Finally, we refer to $(\mathcal{P}_2(\mathcal{Z}), W_2)$ as the Wasserstein space. 
  
We will later also write $\smash{k:\mathbb{R}^d\times\mathbb{R}^d\rightarrow\mathbb{R}}$ for a positive semi-definite kernel, $\mathcal{H}_k$ for the reproducing kernel Hilbert space (RKHS) associated with this kernel, and $\smash{\mathcal{H}:=\mathcal{H}_k^{d_z}}$ for the product RKHS with elements $f = (f_1,\dots,f_{d_{z}})$, with $f_i\in\mathcal{H}_{k}$. We also write $\mathcal{S}_{\mu}: L^2(\mu)\rightarrow\mathcal{H}$ for the integral operator associated with $k$ and the measure $\mu$, defined as $\smash{\mathcal{S}_{\mu} f(z) = \int_{\mathbb{R}^d} k(z,w) f(w)\mu(\mathrm{d}w)}$, and $\smash{P_{\mu}:L^2(\mu) \rightarrow L^2(\mu)}$ for the operator $\smash{P_{\mu} = \iota S_{\mu}}$, where $\smash{\iota:\mathcal{H}\rightarrow L^2(\mu)}$ is the inclusion map, with adjoint $\iota^{*} = S_{\mu}$. This map differs from $S_{\mu}$ only in its range.
 
\subsection{The Free Energy}
\label{sec:free-energy}
Our approaches will leverage the viewpoint of marginal maximum likelihood estimation as minimization of the free-energy functional
$\mathcal{F}:\Theta\times \mathcal{P}(\mathcal{Z})\rightarrow \mathbb{R}$, defined according to
\begin{equation}
    \mathcal{F}(\theta,\mu) := \int \log \frac{\mu(z)}{\pi_{\theta}(z)} \mu(z)\mathrm{d}z. \label{eq:free-energy}
\end{equation}
To be precise, we build on the observation made in \cite{Neal1998}, and recently revisited in \cite{Kuntz2023}, that finding $\theta_{*} = \argmax_{\theta\in\Theta} p_{\theta}(x)$ and computing the corresponding posterior $\mu_{*} = p_{\theta_{*}}(\cdot|x)$ is equivalent to solving the joint minimization problem
\begin{equation}
    (\theta_{*},\mu_{*}) = \argmin_{(\theta,\mu)\in\Theta\times \mathcal{P}_2(\mathcal{Z})} \mathcal{F}(\theta,\mu). \label{eq:joint_minimization}
\end{equation} 
In particular, as noted in \cite{Neal1998}, the EM algorithm corresponds precisely to a coordinate descent scheme applied to $\mathcal{F}$: given an initial $\theta_0\in\Theta$, solve 
\begin{align}
    \mu_{t} &:= \argmin_{\mu\in \mathcal{P}_2(\mathcal{Z})} \mathcal{F}(\theta_t, \mu), \tag{E} \\
    \theta_{t+1} &:= \argmin_{\theta \in \Theta} \mathcal{F}(\theta,\mu_{t}), \tag{M}
\end{align} 
until convergence. Despite the simplicity of this approach, its practical use is limited for more complex models. In particular, it can only be applied when it is possible to solve the two optimization subroutines (i.e., compute the E-step and the M-step) exactly.

\subsection{Minimizing the Free Energy using a Stein-Variational Gradient Flow}
\label{sec:gradient-flow}
As pointed out in \cite{Kuntz2023}, instead of using \eqref{eq:E-step} and \eqref{eq:M-step} in order to solve \eqref{eq:free-energy}, a natural alternative is to use a discretization of a gradient flow associated with \eqref{eq:free-energy}. We are then faced with several questions. First, what is an appropriate notion of the gradient flow of the functional $\mathcal{F}(\theta,\mu)$? Second, how should we discretize this flow?  Regarding the first question, a natural way in which to construct a gradient flow for $\mathcal{F}(\theta,\mu)$ is to consider a Euclidean gradient flow w.r.t. the first argument, and a Wasserstein gradient flow w.r.t. the second argument \citep[e.g.,][Chapter 11]{Ambrosio2008}. In particular, we will say that $(\theta,\mu):[0,\infty)\rightarrow \Theta \times \mathcal{P}_2(\mathcal{Z})$ is a solution of a `Euclidean-Wasserstein' gradient flow of $\mathcal{F}$ if,
\begin{align}
\label{eq:gradient_flow}
    \partial_{t}\theta_t &= - \nabla_{\theta}\mathcal{F}(\theta_t,\mu_t) \\
    \partial_{t}\mu_t &= -\nabla_{\mu} \mathcal{F}(\theta_t,\mu_t), \label{eq:gradient_flow_2}
\end{align}
where $\nabla_{\theta}\mathcal{F}(\theta,\mu)$ is the Euclidean gradient of $\mathcal{F}(\cdot,\mu)$ at $\theta$, and $\smash{\nabla_{\mu} \mathcal{F}(\theta,\mu) = - \nabla \cdot \left( \mu \nabla_{W_2}\mathcal{F}(\theta,\mu)  \right)}$, with $\smash{\nabla_{W_2}\mathcal{F}(\mu,\theta)}$ denoting the Wasserstein gradient of $\mathcal{F}(\theta, \cdot)$ at $\mu$, which exists and is given by $\smash{\nabla_{W_2}\mathcal{F}(\mu,\theta) = \nabla_{z}\log(\frac{\mu}{\pi_{\theta}})}$ under mild regularity conditions on $\mu\in\mathcal{P}_2(\mathcal{Z})$ \cite[][Lemma 10.4.13]{Ambrosio2008}. Explicitly, the gradients in \eqref{eq:gradient_flow} - \eqref{eq:gradient_flow_2} are given by
\begin{align}
    \nabla_{\theta}\mathcal{F}(\theta,\mu)&= -\textstyle\int \nabla_{\theta}\log \pi_{\theta}(z)\mu(z)\mathrm{d}z, \\ 
    \nabla_{\mu}\mathcal{F}(\theta,\mu) &= - \nabla \cdot \left(\mu \nabla_{z}\log \left(\textstyle\frac{\mu}{\pi_{\theta}}\right)\right). \label{eq:flow_gradients}
\end{align}
To obtain an implementable discrete-time algorithm, an obvious choice is to consider an explicit Euler discretization of \eqref{eq:gradient_flow} - \eqref{eq:gradient_flow_2} which, for $t\in\mathbb{N}_{0}$, corresponds to\footnote{In a slight abuse of notation, we use $t$ to index both time in the continuous-time dynamics, and the iteration in the discrete-time algorithm. The appropriate meaning should always be clear from context.}
\begin{align}
    \theta_{t+1} &= \theta_{t} + \gamma \textstyle \int \nabla_{\theta}\log \pi_{\theta_{t}}(z)\mu_t(z)\mathrm{d}z, \label{eq:ideal_theta} \\
    \mu_{t+1} &= \left(\mathrm{id} - \gamma\nabla_{z}\log \left(\textstyle \frac{\mu_t}{
    \pi_{\theta_{t+1}}}\right)\right)_{\#}\mu_t, \label{eq:ideal_mu}
\end{align}
where $\gamma>0$ is a step size or learning rate, and $\mathrm{id}$ is the identity map. Unfortunately, implementing this scheme would require estimating the density of $\mu_t$ based on samples, which is rather challenging. 

\begin{algorithm*}[!htbp] %
	\caption{SVGD EM}
	\label{alg:svgdEM}
	\begin{algorithmic}[1]
		\STATE {\bf Input:} number of iterations $T$, number of particles $N$, initial particles $\{z_0^i\}_{i=1}^N \sim \mu_0$, initial $\theta_0$, target $\pi$, kernel $k$, learning rate $\gamma>0$.
		\FOR {$t =0, 1,\dots,T -1$}
            \STATE{
        \begin{align*}
        \\[-14mm]
        \theta_{t+1} &= \theta_t + \frac{\gamma}{N} \sum_{j=1}^N \nabla_\theta \log \pi_{\theta_t}(z_{t}^{j}) \\[-1mm]
        z_{t+1}^{i} &= z_{t}^{i} + \frac{\gamma}{N} \sum_{j=1}^N\left[k(z_t^{j},z_t^{i}) \nabla_z \log \pi_{\theta_{t+1}}(z_t^{j}) + \nabla_{z_t^{j}}k(z_t^{j},z_t^{i})\right],\quad i\in[N].
        \end{align*}
        \vspace{-6mm}
        }
		\ENDFOR
		\STATE {\bf return} $\theta_{T}$ and $\{z^i_{T}\}_{i=1}^{N}$.
	\end{algorithmic}
\end{algorithm*}

Inspired by Stein Variational Gradient Descent (SVGD) \citep{Liu2016a}, suppose that we replace the Wasserstein gradient $\nabla_{W_2}\mathcal{F}(\theta_{t},\mu_t)$ by its image $P_{\mu_{t}}\nabla_{W_2}\mathcal{F}(\theta_{t},\mu_{t})$, under the integral operator $P_{\mu_{t}}$. This essentially plays the role of the Wasserstein gradient in the RKHS $\mathcal{H}_k$. Under the assumption that $\lim_{||z||\rightarrow\infty}k(z,\cdot)\pi(z)=0$, we have, using integration by parts \citep[e.g.,][]{Liu2016a}, that 
\begin{align}
    &\mathcal{P}_{\mu}\nabla_{z}\log \big(\textstyle\frac{\mu}{\pi_{\theta}}\big)(\cdot)  \\
    &= - \int \left[\nabla_z \log \pi_{\theta}(z) k(z,\cdot)  + \nabla_{z}k(z,\cdot)\right] \mu(z)\mathrm{d}z. \nonumber
\end{align}
In this case, the update equations in \eqref{eq:ideal_theta} - \eqref{eq:ideal_mu} become
\begin{align}
    \theta_{t+1} &= \theta_{t} + \gamma \textstyle \int \nabla_{\theta}\log \pi_{\theta_{t}}(z)\mathrm{d}\mu_{t}(z),  \label{eq:svgd_population_discrete_time1} \\[2mm]
    \mu_{t+1} &= \left(\mathrm{id} + \gamma \textstyle \int \left[\nabla_z \log \pi_{\theta_{t+1}}(z) k(z,\cdot) \right.\right. \label{eq:svgd_population_discrete_time2} \\
    &\hspace{20mm} \left.\left.+  \nabla_{z}k(z,\cdot)\right] \mu_t(z)\mathrm{d}z\right)_{\#}\mu_{t}.  \nonumber
\end{align}
Finally, we can approximate the two intractable integrals using a set of $N$ interacting particles, $\{z_t^{i}\}_{i=1}^N$. Based on this approximation, we arrive at Alg. \ref{alg:svgdEM}.

Naturally, we would like to establish the convergence of this algorithm to the minimizer of $\mathcal{F}(\theta,\mu)$. In App. \ref{sec:continuous-time-results}, we establish that $\mathcal{F}(\theta_t,\mu_t)$ converges exponentially fast along the continuous-time SVGD EM dynamics, under a fairly natural gradient dominance condition. Here, we focus on the discrete-time case. We will require the following assumptions, which extend those introduced in \cite{Salim2022}. 

\begin{assumption} \label{assumption:bounded_k}
    There exists $B>0$ such that, for all $z\in\mathcal{Z}$, $||k(z,\cdot)||_{\mathcal{H}_k}, ||\nabla_{z}k(z,\cdot)||_{\mathcal{H}}\leq B$.
\end{assumption}

\begin{assumption}
\label{assumption:bounded_hessian0}
    The Hessian $H_{V}$ of $-\log \pi$ is well defined, and $\exists$ $M>0$ such that $||H_{V}||_{\mathrm{op}}\leq M$.
\end{assumption}

\begin{assumption}
    \label{assumption:bounded_I_stein}
    For all $\theta\in\Theta$, there exists $
    \lambda>0$ such that $p_{\theta}(\cdot|x)$ satisfies Talagrand's T1 inequality with constant $\lambda$. That is, for all $\mu\in\mathcal{P}_1(\mathcal{Z})$, $W_1(\mu,p_{\theta}(\cdot|x))\leq \sqrt{\tfrac{2\mathrm{KL}(\mu|| p_{\theta}(\cdot|x))}{\lambda}}$. 
\end{assumption}

\begin{remark}
    These assumptions are rather mild. For example, the T1 inequality holds if $p_{\theta}(\cdot|x)$ satisfies a logarithmic Sobolev inequality with constant $\lambda>0$ \cite[Theorem 22.17]{Villani2008}. In turn, this holds if $z\mapsto -\log \pi_{\theta}(z)$ is strongly convex.
\end{remark}

\begin{theorem}
\label{prop:descent_lemma}
    Assume that Assumptions \ref{assumption:bounded_k} - \ref{assumption:bounded_I_stein} hold. Suppose $\smash{0<\gamma\leq\gamma_{*}}$, where \begin{align}
        &\gamma_{*} = \frac{\alpha - 1}{\alpha B^2} \inf_{t\in\mathbb{N}} \bigg[ 1+ ||\nabla_{z} \log p_{\theta_t}(0|x)|| \\[-2mm]
        &~~~~~~~+ M \int_{\mathcal{Z}} ||z|| p_{\theta_t}(z|x)\mathrm{d}z + M \sqrt{\tfrac{2 \mathrm{KL}(\mu_0 || p_{\theta_t}(\cdot |x))}{\lambda}}\bigg]^{-1}\nonumber 
    \end{align}
    Then, for all $t\geq 0$, there exist positive constants $A_1,A_2>0$, given in Appendix \eqref{app:proof_descent_lemma}, such that 
    \begin{align}
        &\mathcal{F}(\theta_{t+1},\mu_{t+1}) - \mathcal{F}(\theta_{t},\mu_{t}) \leq   \\
        &- \gamma \left[
        A_1||\nabla_{\theta}\mathcal{F}(\theta_{t},\mu_{t})||_{\mathbb{R}^{d_{\theta}}}^2 +A_2||\mathcal{S}_{\mu_{t}}\nabla_{W_2}\mathcal{F}(\theta_{t+1},\mu_{t})||^2_{\mathcal{H}}\right]. \nonumber 
    \end{align}
\end{theorem}


We prove this Theorem in App. \ref{app:proof_descent_lemma}. An immediate consequence is the convergence of $\smash{||\nabla_{\theta}\mathcal{F}(\theta_{t},\mu_{t})||_{\mathbb{R}^{d_{\theta}}}^2}$ and $\smash{||S_{\mu_{t}}\nabla_{W_2}\mathcal{F}(\theta_{t+1},\mu_{t})||_{\mathcal{H}}^2}$ to zero, under appropriate conditions on the learning rate $\gamma$. We provide a precise statement of this corollary in App. \ref{sec:additional-theoretical-results}.

Naturally, the convergence of this algorithm depends on the learning rate  $\gamma>0$. Indeed, this is the case for any algorithm obtained as the discretization of a gradient flow associated with the free energy functional, including those recently studied in \cite{Akyldz2023,Kuntz2023}. If $\gamma$ is too large, the algorithm is unstable and the parameter $\theta$ and particles $\{z^i\}_{i=1}^N$ will diverge. If $\gamma$ is too small, the algorithm will converge slowly. Empirically, \cite{Kuntz2023} noted the significant challenges associated with appropriately setting the learning rate parameter simultaneously for the parameters and the latent variables. This is also observed in our own empirical results (see Sec. \ref{sec:numerical-experiments}). To address this issue, we now propose a tuning-free algorithm for minimizing the free energy, which removes entirely the need for user-chosen learning rates, and leads empirically to significantly faster convergence rates. 

\subsection{Minimizing the Free Energy using Coin Betting}
\label{sec:coin}

\begin{algorithm*}[!htbp] %
	\caption{Coin EM}
	\label{alg:coinEM}
	\begin{algorithmic}[1]
		\STATE {\bf Input:} number of iterations $T$, number of particles $N$, initial particles $\{z_0^i\}_{i=1}^N \sim \mu_0$, initial $\theta_0$, target $\pi$, kernel $k$.
		\FOR {$t =1,\dots,T$}
                \STATE{
                \begin{align*}
                \\[-10mm]
                \theta_{t} &= \theta_0 + \frac{{\sum_{s=1}^{t-1} \frac{1}{N}\sum_{j=1}^{N}\nabla_{\theta} \log \pi_{\theta_s}(z_{s}^{j})}}{t} \left( 1 + \textstyle\sum_{s=1}^{t-1} \langle \frac{1}{N }\textstyle\sum_{j=1}^{N}\nabla_{\theta} \log \pi_{\theta_s}(z_{s}^{j}), \theta_s  - \theta_0 \rangle \right) \\[2mm]
                z_{t}^{i} &= z_0^{i} +\frac{{\sum_{s=1}^{t-1} 
                \frac{1}{N} \textstyle\sum_{j=1}^N k(z_s^{j},z_s^{i}) \nabla_{z} \log \pi_{\theta_s}(z_s^{j}) +  \nabla_{z_s^j}k(z_s^{j},z_s^{i})}}{t} \\
                &\hspace{5mm}\times \left( 1 + \textstyle\sum_{s=1}^{t-1} \langle
                \frac{1}{N} \textstyle\sum_{j=1}^N k(z_s^{j},z_s^{i}) \nabla_{z} \log \pi_{\theta_s}(z_s^{j}) +  \nabla_{z_s^j}k(z_s^{j},z_s^{i}), z_s^{i} - z_0^{i}\rangle \right),\quad i\in[N].
                \end{align*}
                \vspace{-5mm}
                }
		\ENDFOR
		\STATE {\bf return} $\theta_{T}$ and $\{z^i_{T}\}_{i=1}^{N}$.
	\end{algorithmic}
\end{algorithm*}

Our approach is based on the learning rate free optimization techniques introduced in \cite{Orabona2016,Orabona2016a,Orabona2017}, and their recent extension to optimization problems on spaces of probability measures \citep{Sharrock2023}. In particular, our method can be viewed as a hybrid between the coin betting techniques from \cite{Orabona2016}, which we use for the optimization over $\Theta$, and one of the coin sampling algorithms from \cite{Sharrock2023}, which we use for the optimization over $\mathcal{P}_2(\mathcal{Z})$. The resulting algorithm, which we term Coin EM, represents an automatic, general-purpose EM algorithm, which does not require a user-defined learning rate, and outperforms existing algorithms across a range of applications (see Sec. \ref{sec:numerical-experiments}).

We begin by recalling the coin betting framework from \cite{Orabona2016}. Consider a gambler making repeated bets on the outcomes of a series of adversarial coin flips. The gambler's goal is to maximize their wealth, starting from some initial wealth, $w_0 = 1$. In each round of the game, $t\in\mathbb{N}$, the gambler bets on the outcome of a coin flip, without borrowing any additional money. The gambler's bet is encoded by $\theta_t \in\mathbb{R}$, where the sign indicates whether the bet is a heads or tails and the absolute value indicates the size of the bet. The gambler's wealth thus accumulates according to $\smash{w_t = w_0 + \sum_{s=1}^t c_s\theta_s}$, where $c_t\in\{-1,1\}$ denotes the outcome of the coin flip. The gambler's bets are restricted to satisfy $\theta_t = \beta_t w_{t-1}$, where $\beta_t\in[-1,1]$ denotes a betting fraction, meaning the gambler can only bet a fraction $\beta_t$ of their accumulated wealth up to time $t$, and cannot borrow any additional money.

One can use this coin betting game to solve convex optimization problems of the form $\theta_{*} = \argmin_{\theta\in\mathbb{R}^d}f(\theta)$, for some convex function $f:\mathbb{R}^d\rightarrow\mathbb{R}$. In particular, by considering a coin betting game with outcomes $c_t = -\nabla f(\theta_t)$, and replacing scalar multiplications with scalar products in the framework described above, \cite{Orabona2016} proved that, under certain assumptions on the betting strategy, $\smash{f(\frac{1}{t}\sum_{s=1}^t \theta_s)\rightarrow f(\theta_{*})}$ at a rate determined by this strategy. In the case that $|c_t|\leq 1$,\footnote{\label{footnote:adaptive}
If, instead, $|c_t|\leq L$, for some constant $L>0$, then one can replace $c_t$ by its normalized version. If this constant is unknown, then it can be replaced by an adaptive estimate. We provide details on this approach in App. \ref{sec:adaptive-coin}.} a standard choice for the betting fraction is $\smash{\beta_t = \frac{1}{t}\sum_{s=1}^{t-1} c_s}$, known as the Krichevsky-Trofimov (KT) betting strategy after \cite{Krichevsky1981}. This choice results in the sequence of bets
\begin{equation}
    \label{eq:betting-update}
\theta_t = \beta_t w_{t-1} = \frac{\sum_{s=1}^{t-1}c_s}{t}\big(1 + \textstyle\sum_{s=1}^{t-1} \langle c_s, \theta_s\rangle \big). 
\end{equation}
A similar approach can also be used to solve sampling problems, based on the perspective of sampling as an optimization problem on the space of probability measures, viz, $\smash{\mu_{*} = \argmin_{\mu\in\mathcal{P}_2(\mathcal{Z})}\mathcal{F}(\mu)}$. In this case, several modifications are required. First, in the $t^{\text{th}}$ round, one now bets $z_t - z_0$, rather than $z_t$, where $z_0$ is a random variable distributed according to some initial $\mu_0\in\mathcal{P}_2(\mathcal{Z})$. In this case, viewing $z_t:\mathcal{Z}\rightarrow\mathcal{Z}$ as a function that maps $z_0\mapsto z_t(z_0)$, one can define a series of `betting distributions' $\mu_{t}\in\mathcal{P}_2(\mathcal{Z})$  as the push-forwards of $\mu_0$ under $z_t$. This means that, given $z_0\sim\mu_0$, $z_t\sim \mu_t$. 

As shown in \cite{Sharrock2023}, this framework can be used to obtain a sampling algorithm by setting $\mathcal{F}(\mu) = \mathrm{KL}(\mu||\pi)$, and considering a betting game in which $c_t =-\mathcal{P}_{\mu_{t}}\nabla_{W_2}\mathcal{F}(\mu_{t})(z_t)$. The unknown $(\mu_t)_{t\in\mathbb{N}}$ are then approximated using a set of $N$ particles. This results in a learning-rate free analogue of SVGD, referred to as Coin SVGD. Empirically, this approach has demonstrated competitive performance with SVGD, with no need to tune a learning rate \citep{Sharrock2023,Sharrock2023a}.

Extending this approach, we can obtain a learning-rate free algorithm for MMLE in latent variable models. 
In particular, to minimize the variational free energy $\mathcal{F}(\theta,\mu)$, we consider a coin betting game involving two players, whose outcomes in the $t^{\text{th}}$ round are given by $\smash{c_t^{(1)} = - \nabla_{\theta}\mathcal{F}(\theta_t,\mu_t)}$ and $\smash{c_t^{(2)} = - \mathcal{P}_{\mu_t}\nabla_{W_2}\mathcal{F}(\theta_t,\mu_t)}$, respectively.\footnote{Alternatively, one can consider this as a betting game involving a single gambler, who bets on two sets of interacting outcomes given by $\smash{c_t^{(1)} = - \nabla_{\theta}\mathcal{F}(\theta_t,\mu_t)}$ and $\smash{c_t^{(2)} = - \mathcal{P}_{\mu_t}\nabla_{W_2}\mathcal{F}(\theta_t,\mu_t)}$.} 
By combining the original coin betting algorithm in \cite{Orabona2016,Orabona2016a} to optimize $\mathcal{F}(\theta,\mu)$ over $\theta\in\Theta$, and the Coin SVGD algorithm in \cite{Sharrock2023} to optimize over $\mu\in \mathcal{P}_2(\mathcal{Z})$, we now have a learning rate free method for solving \eqref{eq:free-energy}. This approach, which we refer to as Coin EM, is summarised in Alg. \ref{alg:coinEM}.

In Theorem \ref{prop:coin-em}, we provide a convergence result for a version of Coin EM in the population limit. Analyzing the convergence of this scheme is extremely challenging, and requires an additional technical assumption that it is unclear how to verify in practice. Nonetheless, we include this result in the interest of completeness. We will require the following assumptions.

\begin{assumption}
\label{coin-em-assumption-1}
    The function  $(\theta,z)\mapsto-\log\pi_{\theta}(z)$ is strictly convex. 
\end{assumption}

\begin{assumption}
\label{coin-em-assumption-2}
    There exists $L>0$ such that $||\nabla_{\theta}\log\pi_{\theta_t}(z)||, ||\nabla_{z_t}\log(\tfrac{\mu_t}{\pi_{\theta_t}})(z_t)|| \leq L$ for $t\in[T]$.
\end{assumption}

\begin{theorem}
\label{prop:coin-em}
    Assume that Assumptions \ref{coin-em-assumption-1}, \ref{coin-em-assumption-2}, and \ref{coin-em-assumption-3} (App. \ref{sec:coin-em-proof}) hold. Then the marginal likelihood $\theta\mapsto p_{\theta}(x)$ admits a unique maximizer $\theta^{*}$. In addition, writing $\mu^{*} = p_{\theta^{*}}(\cdot|x)$, it holds for all $T\geq 1$ that
    \begin{equation}
        \mathcal{F}\left(\frac{1}{T}\sum_{t=1}^T\theta_t,\frac{1}{T}\sum_{t=1}^T \mu_t\right) - \mathcal{F}(\theta^{*},\mu^{*}) = \tilde{O}\left(\frac{1}{\sqrt{T}}\right).
    \end{equation}
\end{theorem}

We provide a full statement of this result in App. \ref{sec:coin-results}, and its proof in App. \ref{sec:coin-em-proof}.

\textbf{Computational Complexity.} 
The computational complexity of both SVGD EM and Coin EM  is $\mathcal{O}(N^2T[\text{eval. cost of $(\nabla_{\theta} \log \pi, \nabla_{z}\log \pi)$}])$. To alleviate the cost w.r.t. the number of particles $N$, one can approximate the sum $\sum_{j=1}^N$ over the particles by subsampling \citep{Li2020}, or by using a random feature expansion of the kernel \citep{Rahimi2007}. Meanwhile, in big data settings where computing $(\nabla_{\theta} \log \pi, \nabla_{z}\log \pi)$ is expensive, these gradients can be replaced with stochastic estimates obtained by subsampling mini-batches of the data \citep[e.g.,][]{Welling2011,Liu2016a}. 

\section{RELATED WORK}
\label{sec:related-work}

\textbf{Comparison with \cite{Kuntz2023} and \cite{Akyldz2023}}. In a recent paper, \cite{Kuntz2023} revisited the perspective of MMLE as the minimization of the free energy, proposing the gradient flow in \eqref{eq:gradient_flow} - \eqref{eq:gradient_flow_2} to minimize this functional. In contrast to us, their algorithms are based on the observation that \eqref{eq:gradient_flow} - \eqref{eq:gradient_flow_2} is a Fokker-Planck equation satisfied by the law of
\begin{align}
    \mathrm{d}\theta_t &= \big[\textstyle\int \nabla_{\theta} \log \pi_{\theta_{t}}(z_t)\mathrm{d}\mu_t(z) \big]\mathrm{d}t, \\  \mathrm{d}z_t &= \nabla_{z} \log \pi_{\theta_t}(z_t)\mathrm{d}t + \sqrt{2}\mathrm{d}w_t. \label{eq:theta_kuntz} 
\end{align}
By approximating this SDE using a system of interacting particles $\smash{\{z_t^{j}\}_{j=1}^N}$, and discretizing in time, \cite{Kuntz2023} obtain the particle gradient descent (PGD) algorithm. \cite{Akyldz2023} have since analyzed an extension of this algorithm, which includes a 
noise term in the $\theta$ dynamics in \eqref{eq:theta_kuntz}, allowing them to obtain a non-asymptotic concentration bound for $\theta_{t}$. 

\cite{Kuntz2023} also develop two variants of PGD. The first is Particle Quasi-Newton, which includes a preconditioning term 
in the $\theta$ dynamics \citep[App. C]{Kuntz2023}. In principle, one could obtain a similar version of SVGD EM based on the techniques in \cite{Detommaso2018}. 
The second is Particle Marginal Gradient Descent, in which the $\theta$ update in \eqref{eq:theta_kuntz} is replaced by $\smash{\theta_t = \theta_{*}({\mu}_t^N)}$, where $\smash{\theta_{*}(\mu) = \argmin_{\mu\in\mathcal{P}(\mathcal{Z})}\mathcal{F}(\theta,\mu)}$, which can be applied whenever it is possible to compute $\mu\mapsto\theta_{*}(\mu)$ in closed form \citep[App. D]{Kuntz2023}. 
In a similar vein, we can obtain marginal variants of our algorithms. These are given in App. \ref{sec:marginal_algorithms}.

\textbf{Learning-rate free methods for optimization and sampling}. The idea of using coin betting for parameter free online learning was first introduced in \cite{Orabona2016,Orabona2016a} and has since been extensively developed by Orabona and coworkers \citep{Orabona2017,Jun2017,Cutkosky2018,Jun2019,Chen2022,Chen2022a}. On the other hand, aside from \citet{Sharrock2023,Sharrock2023a}, the literature on learning-rate free sampling methods is sparse. In practice, the standard technique when using gradient-based methods is to use a grid search, running the algorithm of choice for multiple learning rates, and selecting the value which minimizes an appropriately chosen metric. While there have been some efforts to automate the design of effective learning rate schedules \citep{Chen2016,Li2016a,Coullon2021,Kim2022}, 
typically these approaches still rely on various hyperparameters.

\section{NUMERICAL EXPERIMENTS}
\label{sec:numerical-experiments}

We now benchmark SVGD EM (Alg. \ref{alg:svgdEM}) and Coin EM (Alg. \ref{alg:coinEM}) against other comparable approaches. In all experiments, we use the RBF kernel $k(x,x') = \exp(-\frac{1}{h}||x-x'||^2)$, with the bandwidth $h$ chosen according to the median heuristic \citep{Liu2016a}. We provide additional experimental details in App. \ref{app:additional-experimental-details}, and additional numerical results in App. \ref{app:additional-numerical-results}. The code to reproduce our results can be found on GitHub at \url{https://github.com/chris-nemeth/coinem}.

\subsection{Toy hierarchical model}
\label{sec:toy_results}

{\setlength{\subfigcapskip}{-.5mm}
\begin{figure*}[t]
\vspace{-2mm}
  \centering
  \subfigure[$\mathrm{MSE}(\theta_{t})$ vs learning rate. \label{fig:toy_1a}]{\includegraphics[width=0.325\textwidth, trim = 0 0 0 0, clip]{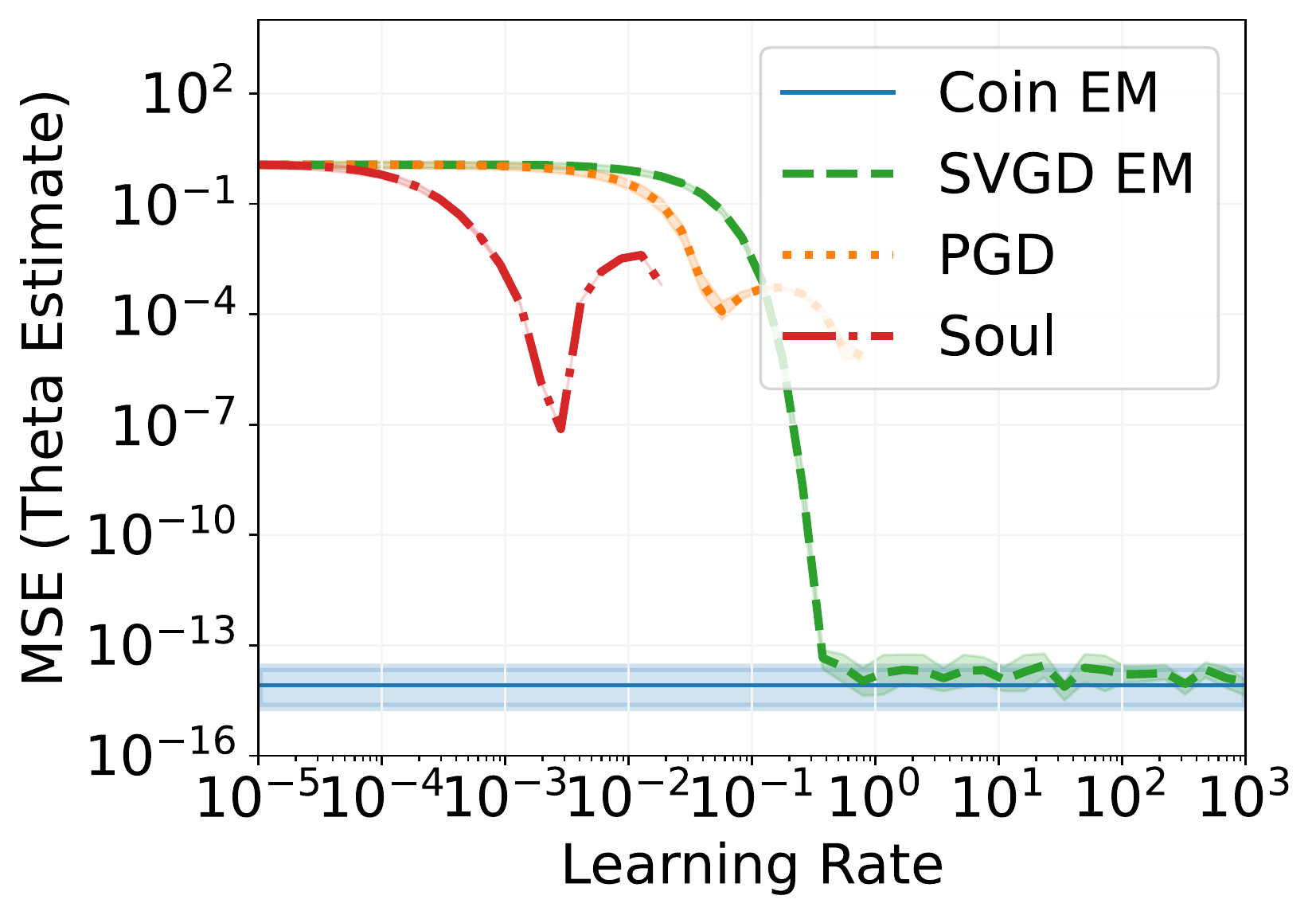}}
  \subfigure[$\mathrm{MSE}(\theta_{t})$ vs $t$. \label{fig:toy_1b}]{\includegraphics[width=0.32\textwidth, trim = 0 0 0 0, clip]{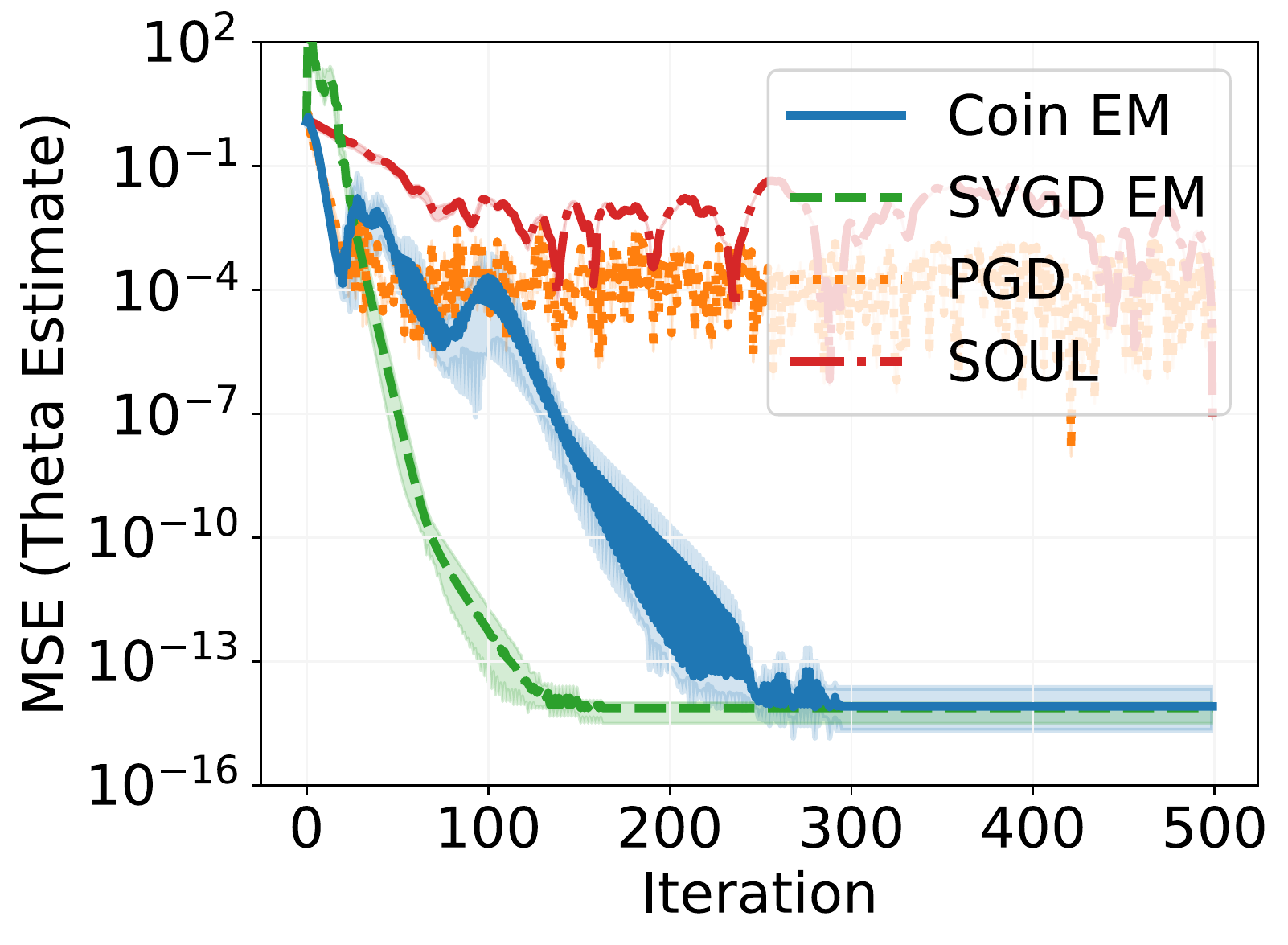}}
    \subfigure[$\mathrm{MSE}\left(\frac{1}{N}\sum_{i=1}^N z_t^{i}\right)$ vs $t$. \label{fig:toy_1c}]{\includegraphics[width=0.32\textwidth, trim = 0 0 0 0, clip]{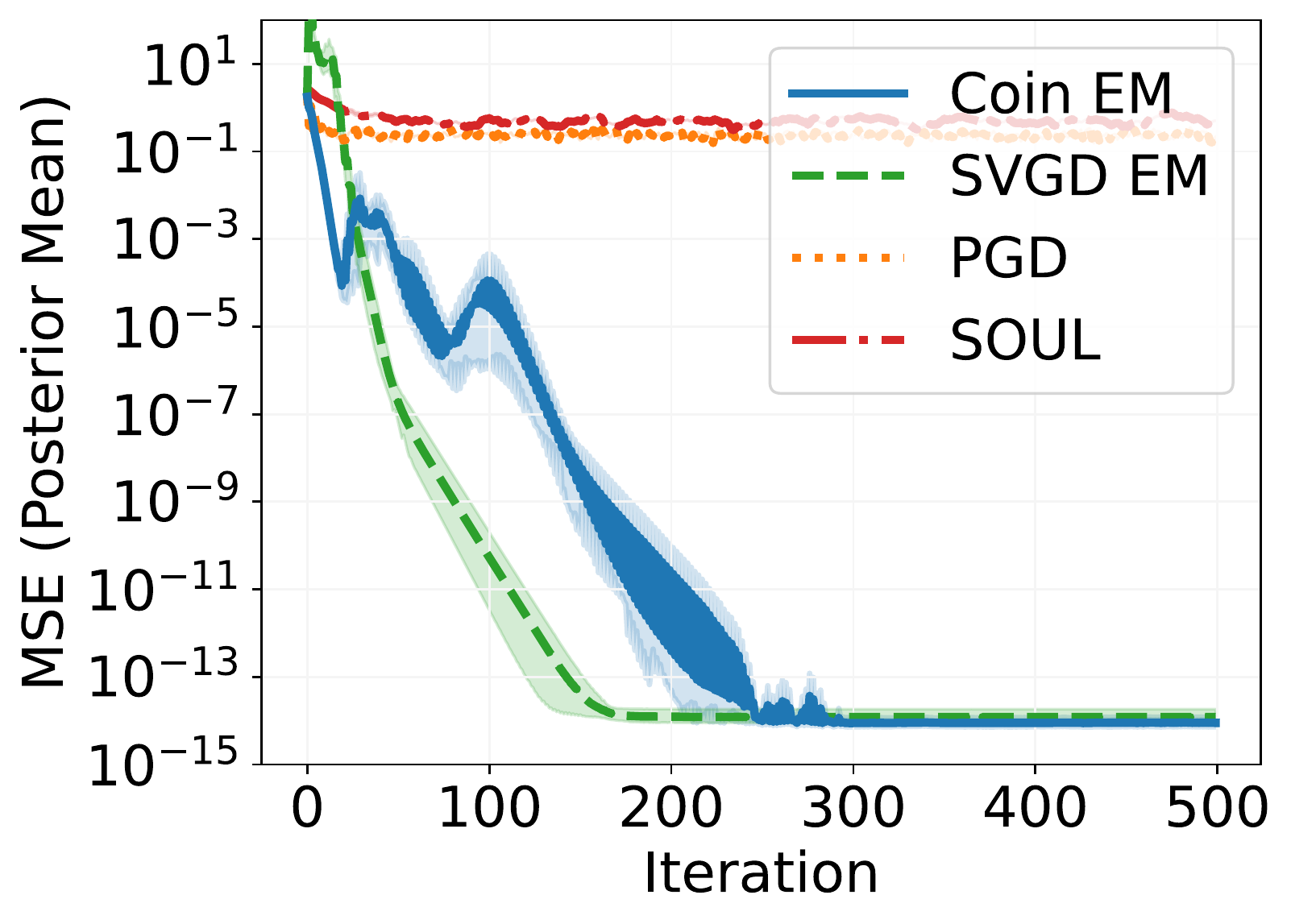}}
  \vspace{-4mm}
  \caption{\textbf{Results for the toy hierarchical model}. MSE of the parameter estimate $\theta_{t}$ as a function of the learning rate after $T=500$ iterations (a); and MSE of the parameter estimate (b) and the posterior mean (c) as a function of the number of iterations, using the optimal learning rate from (a).}
  \label{fig:toy_1}
\vspace{-4mm}
\end{figure*}

\begin{figure*}[b]
\vspace{-2mm}
  \centering
  \subfigure[Coin EM. \label{fig:toy_2a}]{\includegraphics[width=0.325\textwidth, trim = 0 0 0 0, clip]{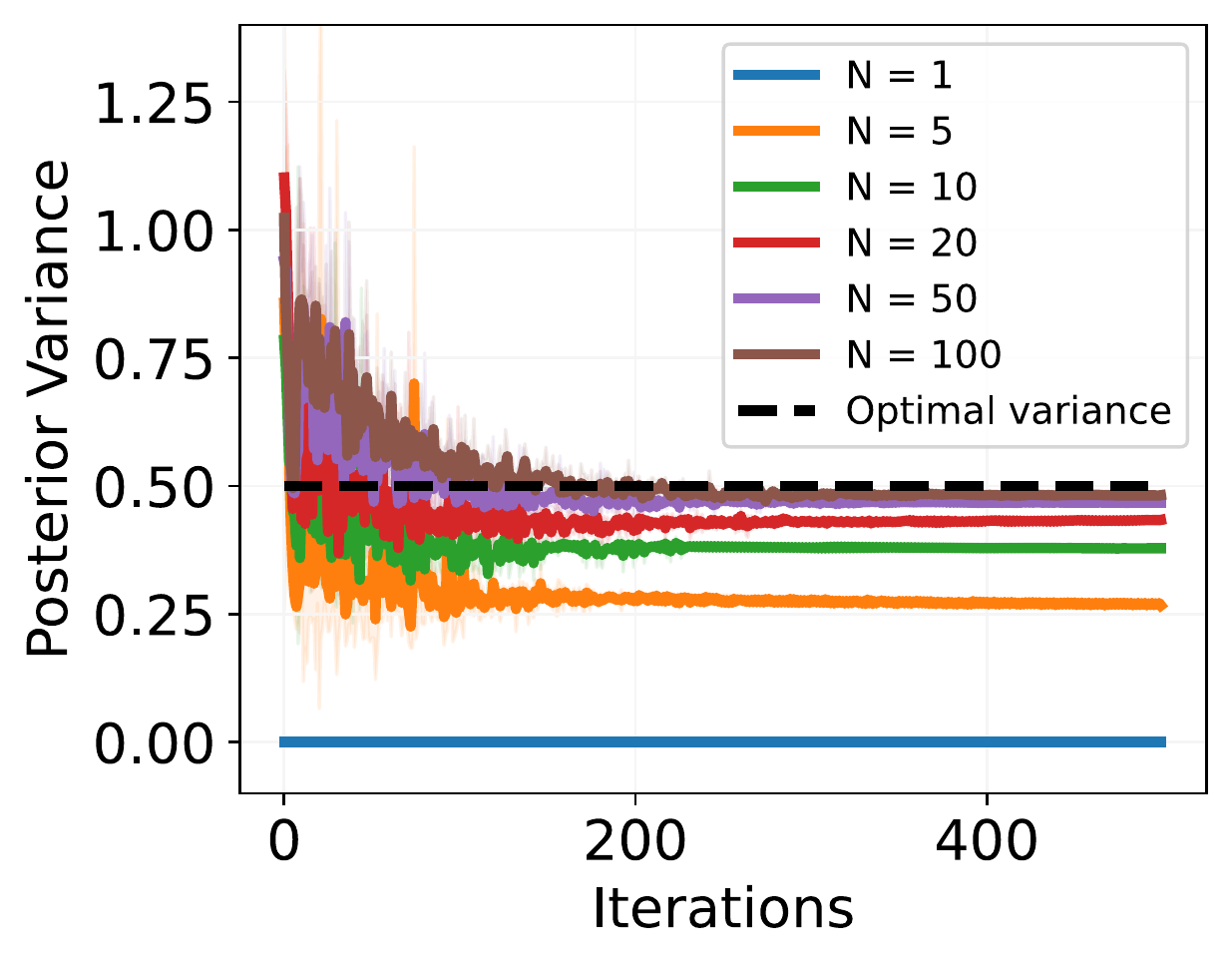}}
  \subfigure[SVGD EM.\label{fig:toy_2b}]{\includegraphics[width=0.325\textwidth, trim = 0 0 0 0, clip]{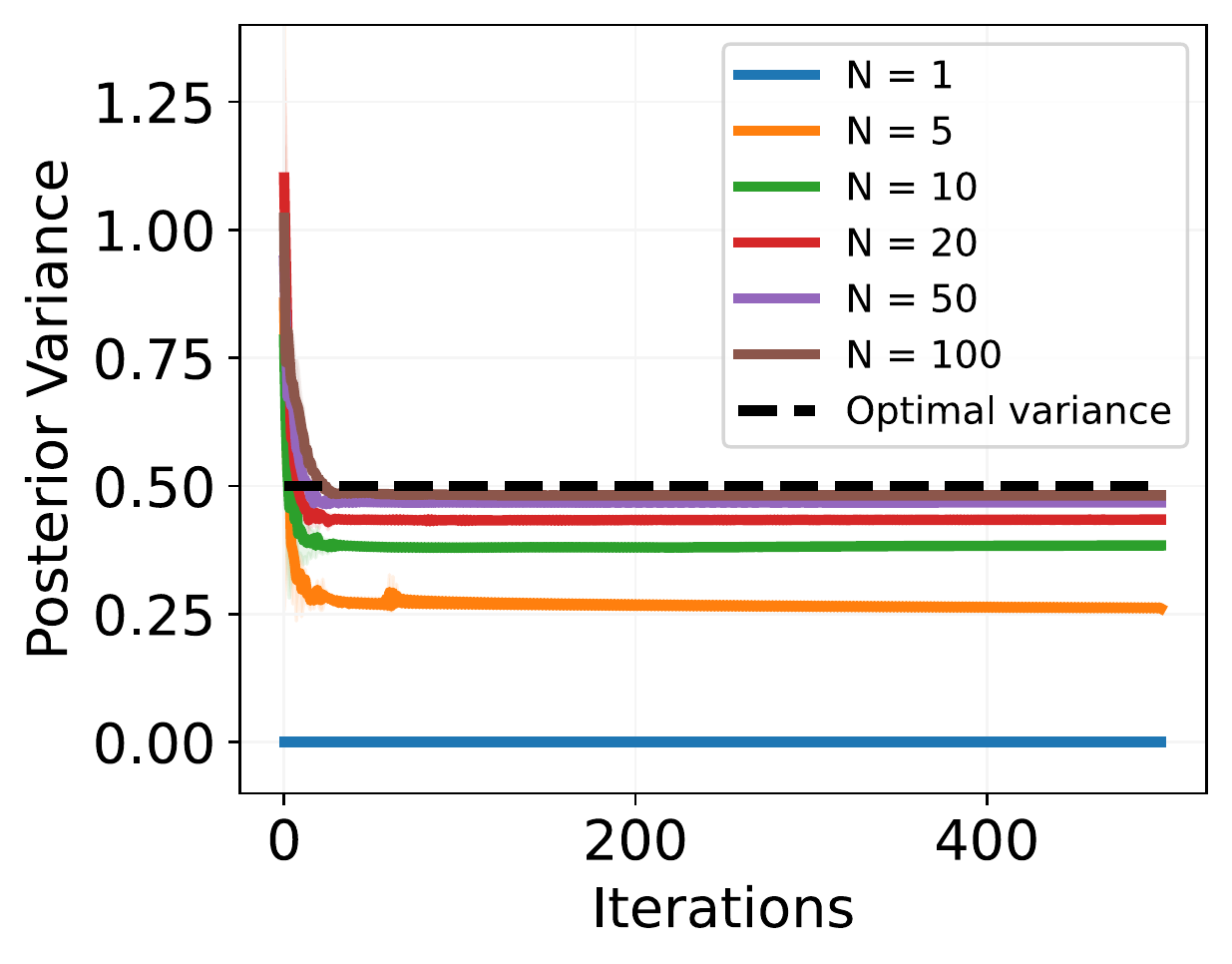}}
  \subfigure[Comparison.\label{fig:toy_2c}]{\includegraphics[width=0.325\textwidth, trim = 0 0 0 0, clip]{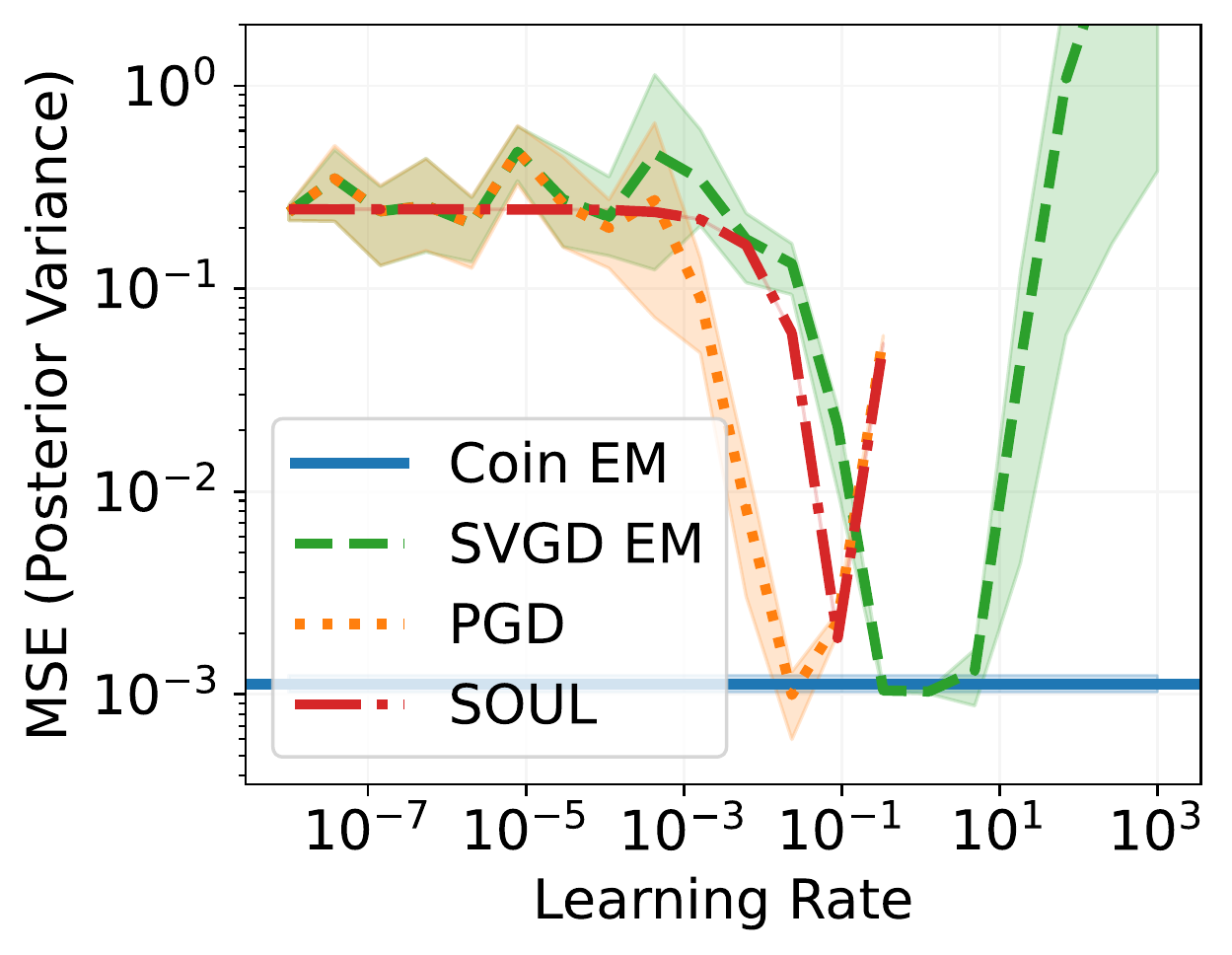}}
  \vspace{-3mm}
  \caption{\textbf{Additional results for the toy hierarchical model}. Estimates for the posterior variance in the case $d_z=1$ obtained using (a) Coin EM and (b) SVGD EM, as a function of the number of iterations. In (c), we plot the MSE of the posterior variance estimate as a function of the learning rate, for Coin EM, SVGD EM, PGD, and SOUL, after $T=250$ iterations and with $N=50$ particles.}
  \label{fig:toy_2}
\vspace{-3mm}
\end{figure*}
}

{\setlength{\subfigcapskip}{-.5mm}
\begin{figure*}[t]
\vspace{-2mm}
  \centering
  \subfigure[Parameter Estimates. \label{fig:bayes_lr_a}]{\includegraphics[width=0.338\textwidth, trim = 0 0 0 0, clip]{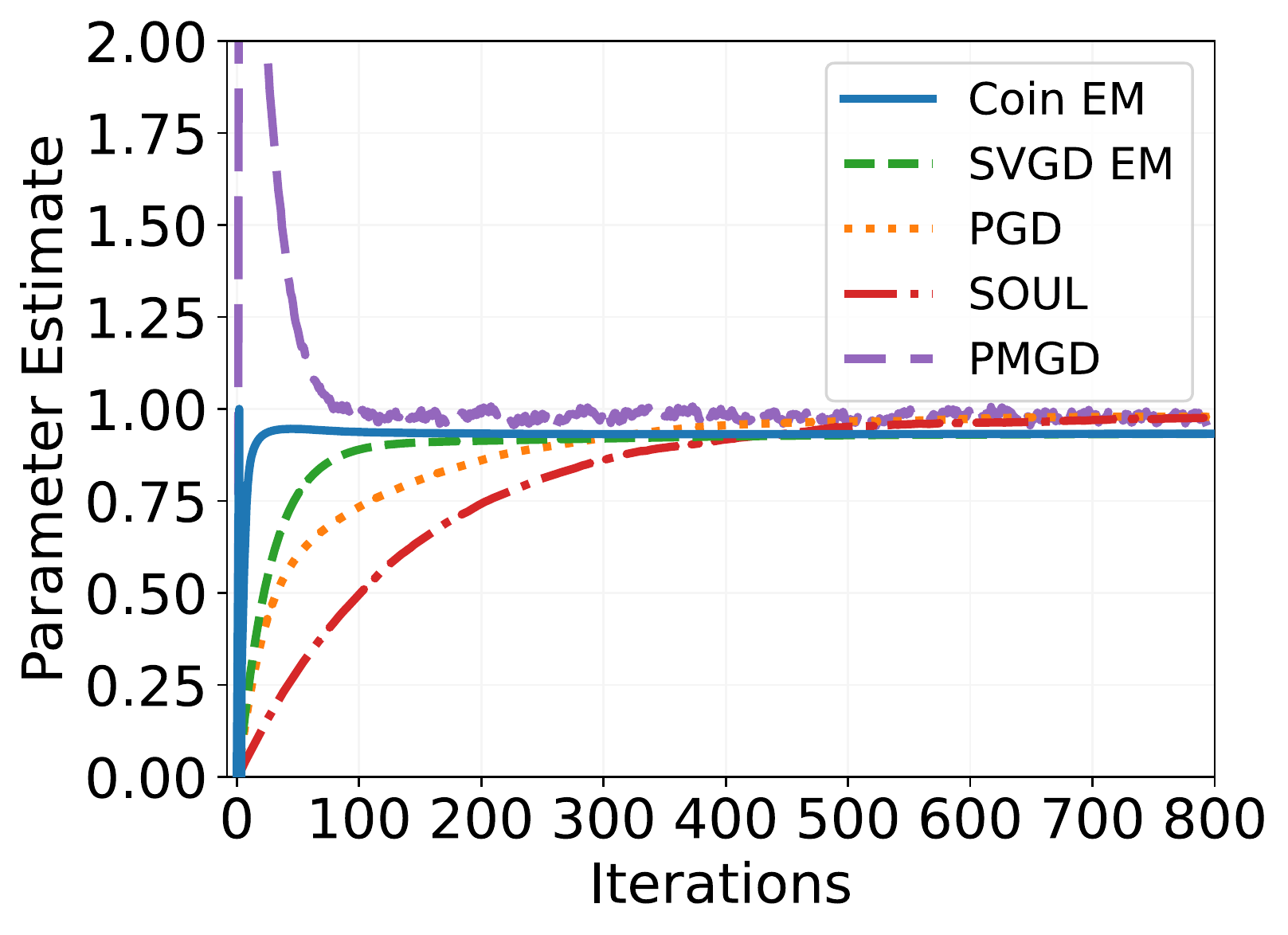}}
  \subfigure[Posterior Estimates. \label{fig:bayes_lr_b}]{\includegraphics[width=0.31\textwidth, trim = 0 0 0 0, clip]{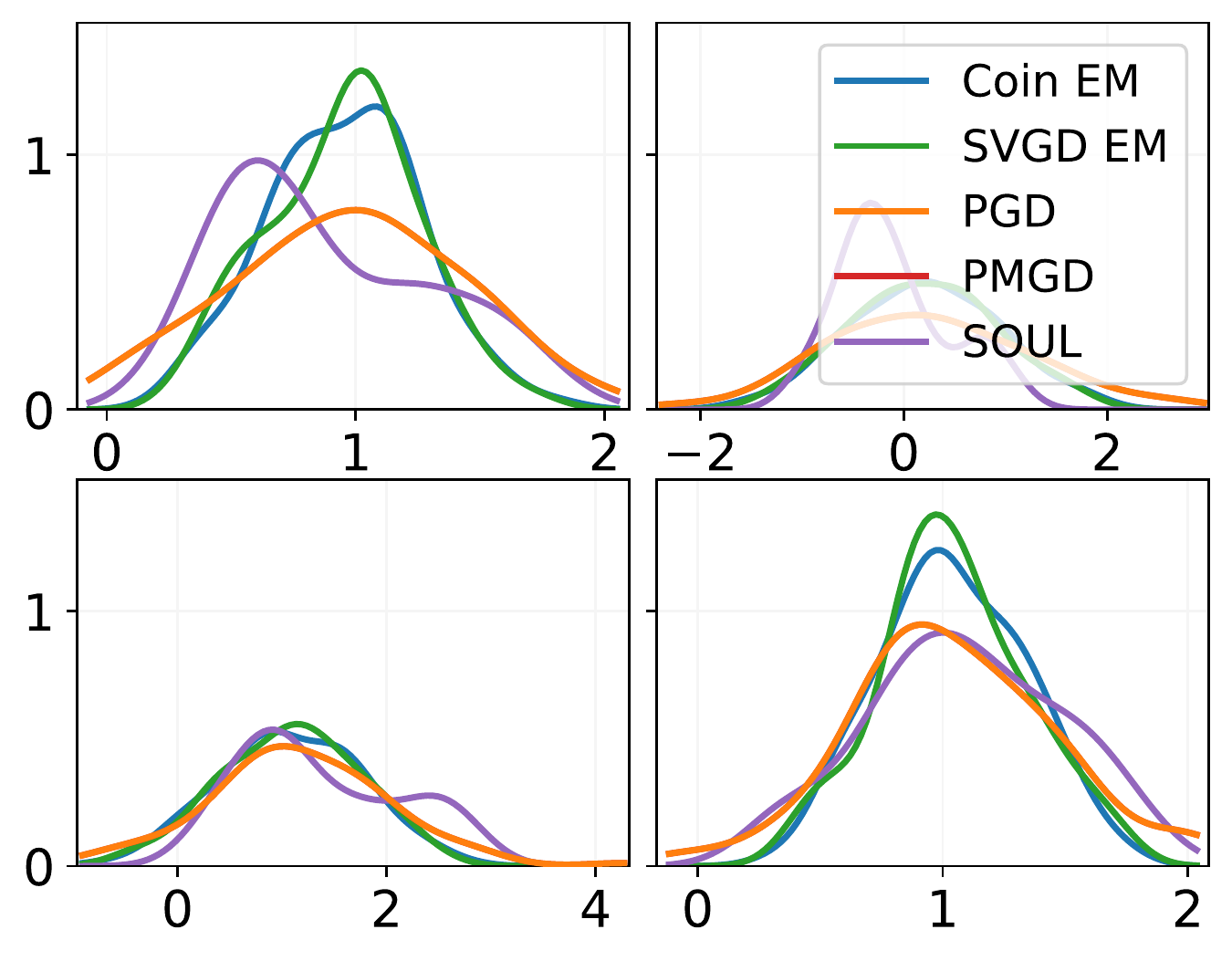}}
\subfigure[Test Error. \label{fig:bayes_lr_c}]{\includegraphics[width=0.317\textwidth, trim = 0 0 0 0, clip]{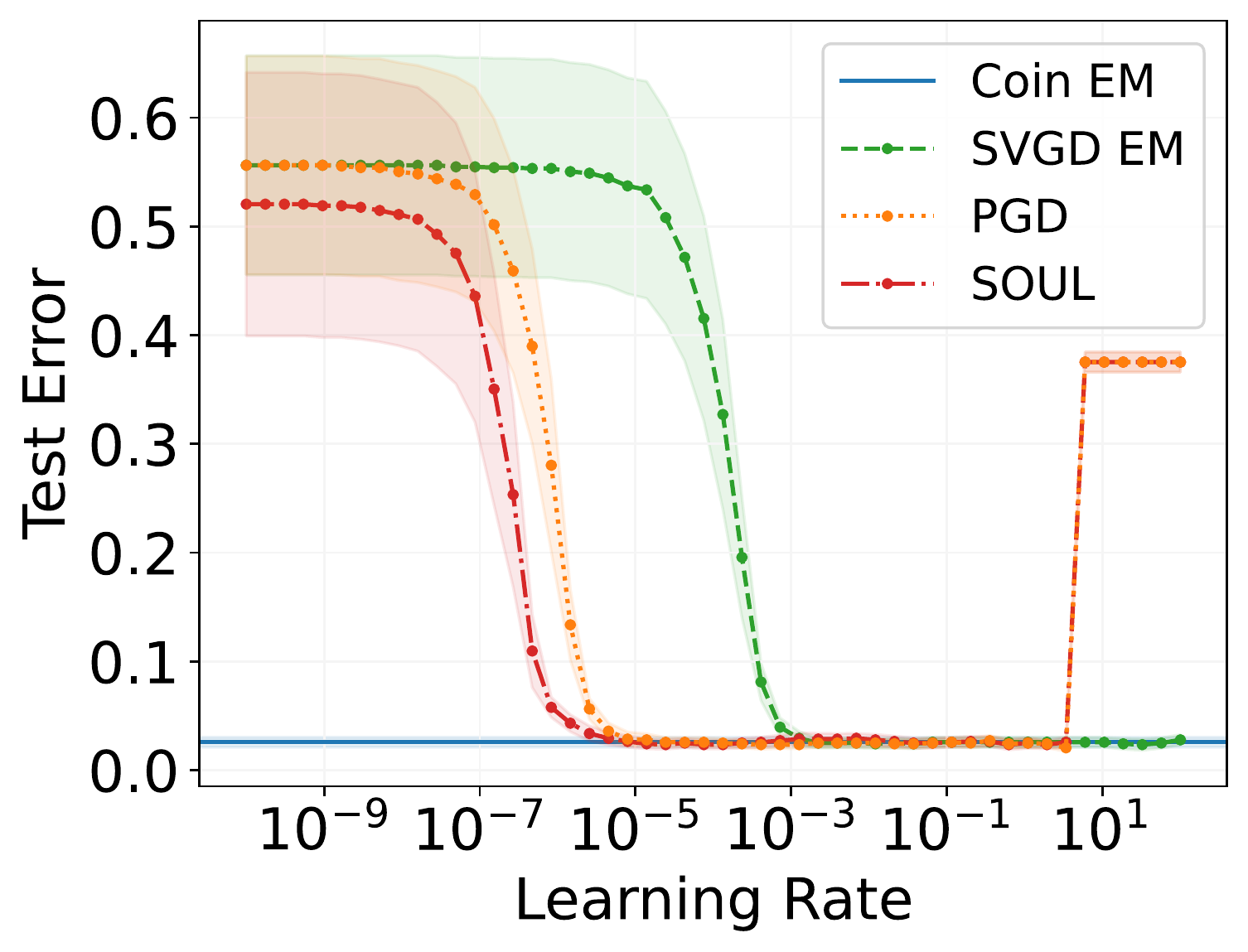}}
\vspace{-3mm}
  \caption{\textbf{Results for the Bayesian logistic regression}. Plots of (a) the sequence of parameter estimates $\theta_{t}$ initialized at zero, (b) the kernel density estimate of four components of the posterior approximation $\hat{\mu}_{800}^n = \frac{1}{n}\sum_{j=1}^n\delta_{z_{800}^{j}}$,  (c) the test error as a function of the learning rate.}
  \label{fig:bayes_lr}
  \vspace{-3mm}
\end{figure*}
}

We begin by considering a toy hierarchical model from \cite{Kuntz2023}. Suppose that we observe data $\smash{x=(x_1,\dots,x_{d_{z}})^{\top}}$, generated according to $\smash{x_i|z_i \stackrel{\mathrm{i.i.d.}}{\sim}\mathcal{N}(x_i|z_i,1)}$, where $\smash{z_i\stackrel{\mathrm{i.i.d.}}{\sim} \mathcal{N}(\theta,1)}$, for some parameter $\theta\in\mathbb{R}$. Our model is thus given by  
\begin{equation}
    p_{\theta}(z,x)= \prod_{i=1}^{d_z} \frac{1}{2\pi} \exp\left[-\frac{(z_i - \theta)^2}{2} - \frac{(x_i-z_i)^2}{2}\right].
\end{equation}
In this case, the marginal likelihood $\theta\mapsto p_{\theta}(x)$ has a unique maximum given by $\smash{\theta_{*} = d_z^{-1}\sum_{i=1}^{d_z}x_i}$, and one can obtain an explicit expression for the corresponding posterior $p_{\theta_{*}}(\cdot|x)$ \citep[][App. E.1]{Kuntz2023}.

In Fig. \ref{fig:toy_1}, we evaluate the performance of SVGD EM and Coin EM on this model, setting $d_z=100$ and $\theta=1$. We also include results for PGD \citep{Kuntz2023} and SOUL \citep{DeBortoli2021}. 
In this case, both of our methods generate parameters $\theta_t$ which converge rapidly to $\theta_{*}$, and particles $\{z_t^{i}\}_{i=1}^N$ whose mean converges to the corresponding posterior mean (Fig. \ref{fig:toy_1b} - \ref{fig:toy_1c}). Even in this toy example, it is clear that PGD, SOUL, and to a lesser extent SVGD EM, are very sensitive to the learning rate (Fig. \ref{fig:toy_1a}). If the learning rate is too small, then convergence is slow; if it is too large, then the parameter estimates are unstable and may fail to converge (see also Fig. \ref{fig:toy_1_different_LR} in App. \ref{sec:toy-add-results}). Coin EM circumvents this problem entirely, obtaining comparable or superior performance to the competing methods, with no need to tune a learning rate. 

In Fig. \ref{fig:toy_2}, we further investigate the performance of our methods, plotting the posterior variance estimates from Coin EM and SVGD EM in the case $d_z=1$. In this case, there is a significant bias when using a small number of particles, an observation which also holds true for other particle-based methods such as PGD \citep[][Figure 2]{Kuntz2023}. This should not be a surprise: even if the parameters were fixed, the SVGD updates in Alg. \ref{alg:svgdEM} would only converge to the true posterior in the mean-field limit \cite[e.g.,][]{Liu2017, Duncan2019,Korba2020,Salim2022}. Nonetheless, as is evident in Fig. \ref{fig:toy_2}, this bias can be all but eliminated by taking a sufficiently large number of particles. In any case, in terms of posterior variance, Coin EM and SVGD EM are once again comparable or superior to the optimal performance of PGD and SOUL (Fig. \ref{fig:toy_2c}). 
Additional results, illustrating the robustness of our method to changes in the number of particles (Fig. \ref{fig:toy_1_20_particles}, Fig. \ref{fig:toy_1_50_particles}) and the initialization (Fig. \ref{fig:toy_1_different_init}), can be found in App. \ref{sec:toy-add-results}.

\subsection{Bayesian logistic regression}
\label{sec:bayes_lr_results}
We next consider a standard Bayesian logistic regression with Gaussian priors, fit using the Wisconsin dataset \cite{Wolberg1990}. In this case, the latent variables are the regression weights. We place an isotropic Gaussian prior $\mathcal{N}(\theta\mathbf{1}_{d_z}, 51_{d_z})$ on the weights, and aim to estimate the unique maximizer $\theta_{*}$ of the marginal likelihood. We provide full details on this model in App. \ref{sec:bayes-lr-details}, and additional results for an alternative model in App. \ref{sec:bayes-lr-alt-results}.  

In Fig. \ref{fig:bayes_lr}, we compare the performance of our algorithms with PGD \citep{Kuntz2023}, PMGD \citep{Kuntz2023}, and SOUL \citep{DeBortoli2021}. We first plot an illustrative sequence of parameter estimates (Fig. \ref{fig:bayes_lr_a}) for each method, initialized at zero, using $N=100$ particles and $T=800$ iterations. All methods converge to approximately the same $\theta_{*}$. In this case, Coin EM converges noticeably faster than its competitors. The same is true when the parameter is initialized further from $\theta_{*}$, in which case the shorter transient period exhibited by Coin EM is even more pronounced (Fig. \ref{fig:bayes_lr_different_init} in App. \ref{sec:bayes-lr-add-results}). Interestingly, by increasing the learning rates of PGD, PMGD, and SOUL, one can improve their convergence rates, but this comes at the cost of a (significant) bias in the resulting parameter estimate (Fig. \ref{fig:bayes_lr_different_lr} in App. \ref{sec:bayes-lr-add-results}). 

Meanwhile, the posterior estimates obtained by each method are rather similar, as is their predictive performance. In fact, in this case the predictive power of the posterior approximations obtained by all of the methods are robust both to the choice of learning rate (Fig. \ref{fig:bayes_lr_c}) and the number of particles (see Fig. \ref{fig:bayes_lr_different_N} in App. \ref{sec:bayes-lr-add-results}). This can be attributed to the relatively simple nature of the posterior, which is both peaked and unimodal; see also \cite[Sec. 3.1]{Kuntz2023}, \cite[Sec. 4.1]{DeBortoli2021}.

\subsection{Bayesian neural network}
\label{sec:bayes_nn_results}
We now consider an example with a notably more complex posterior, namely, a Bayesian neural network (see App. \ref{sec:bayes-nn-alt-details} for full details). 
We consider the setting described in Section 5 in \cite{Yao2022a}; Section 3.2 in \cite{Kuntz2023}, using a simple two-layer neural network to classify MNIST images \citep{Lecun1998}.  The input layer consists of 40 nodes and 784 inputs, while the output layer has 2 nodes. The latent variables are the weights $w\in\mathbb{R}^{784\times 40}$ and $v\in\mathbb{R}^{2\times 40}$ of the output and input layers. We place zero mean isotropic Gaussian priors on the weights, with variances $e^{2\alpha}$ and $e^{2\beta}$, respectively. Rather than assigning hyperpriors to $\theta=(\alpha,\beta)$, we will learn these parameters from data.


{\setlength{\subfigcapskip}{-1.5mm}
\begin{figure*}[!htb]
\vspace{-2mm}
  \centering
  \subfigure[$N=5$. \label{fig:bnn_mnist_compare_methods_N5}]{\includegraphics[width=0.325\textwidth, trim=0 0 0 0, clip]{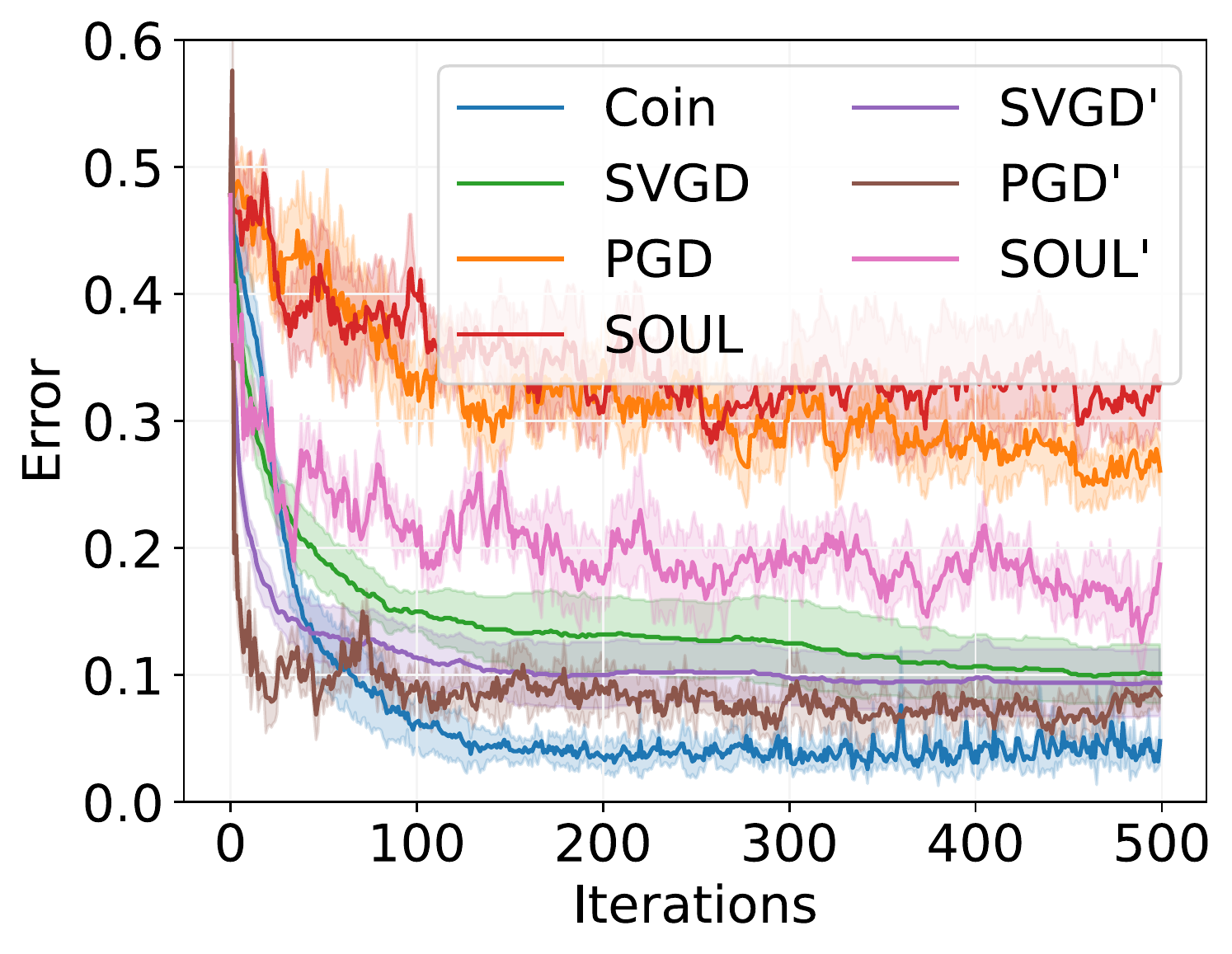}}
  \subfigure[$N=20$. \label{fig:bnn_mnist_compare_methods_N20}]{\includegraphics[width=0.325\textwidth, trim=0 0 0 0, clip]{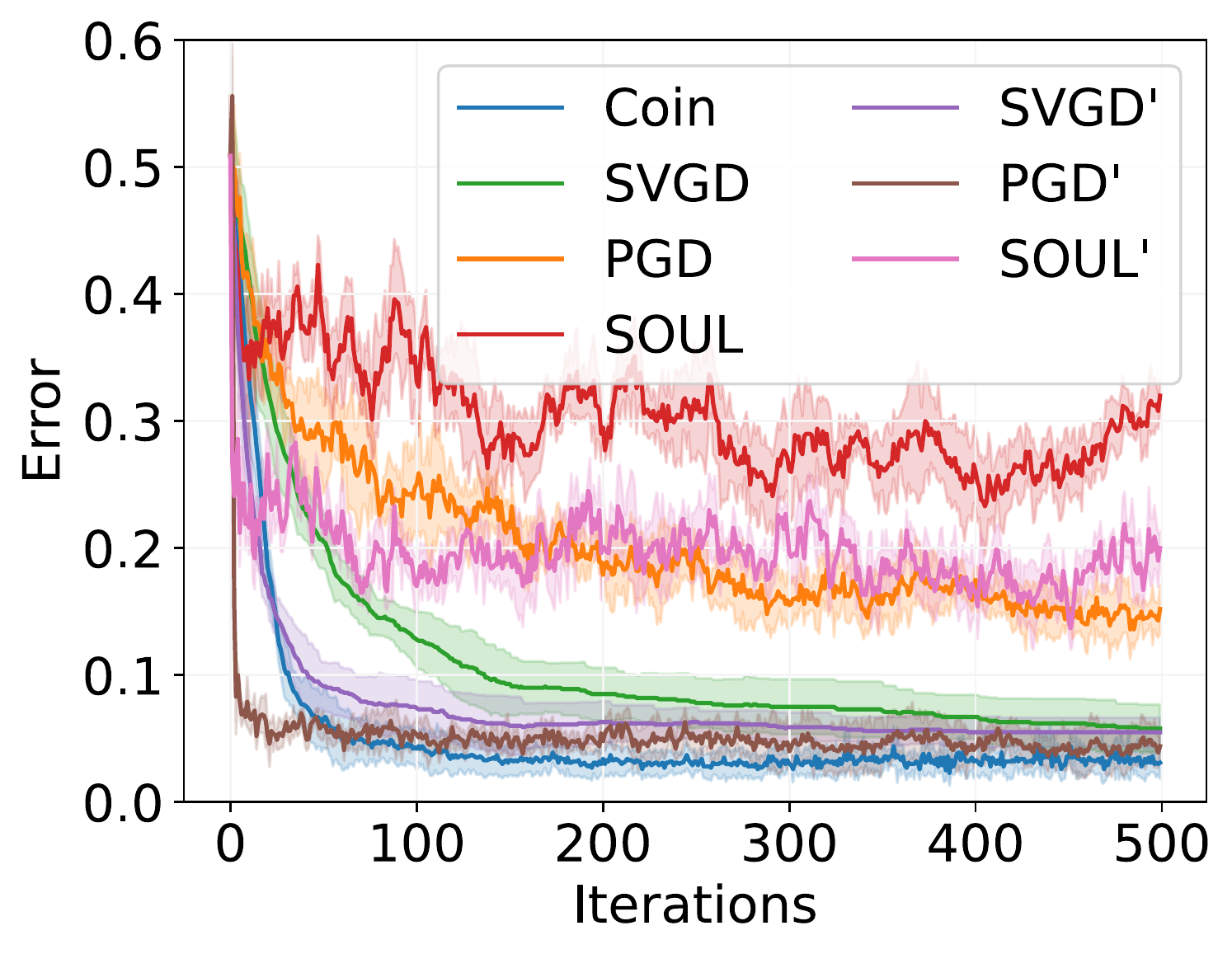}}
  \subfigure[$N=100$. \label{fig:bnn_mnist_compare_methods_N100}]{\includegraphics[width=0.325\textwidth, trim=0 0 0 0, clip]{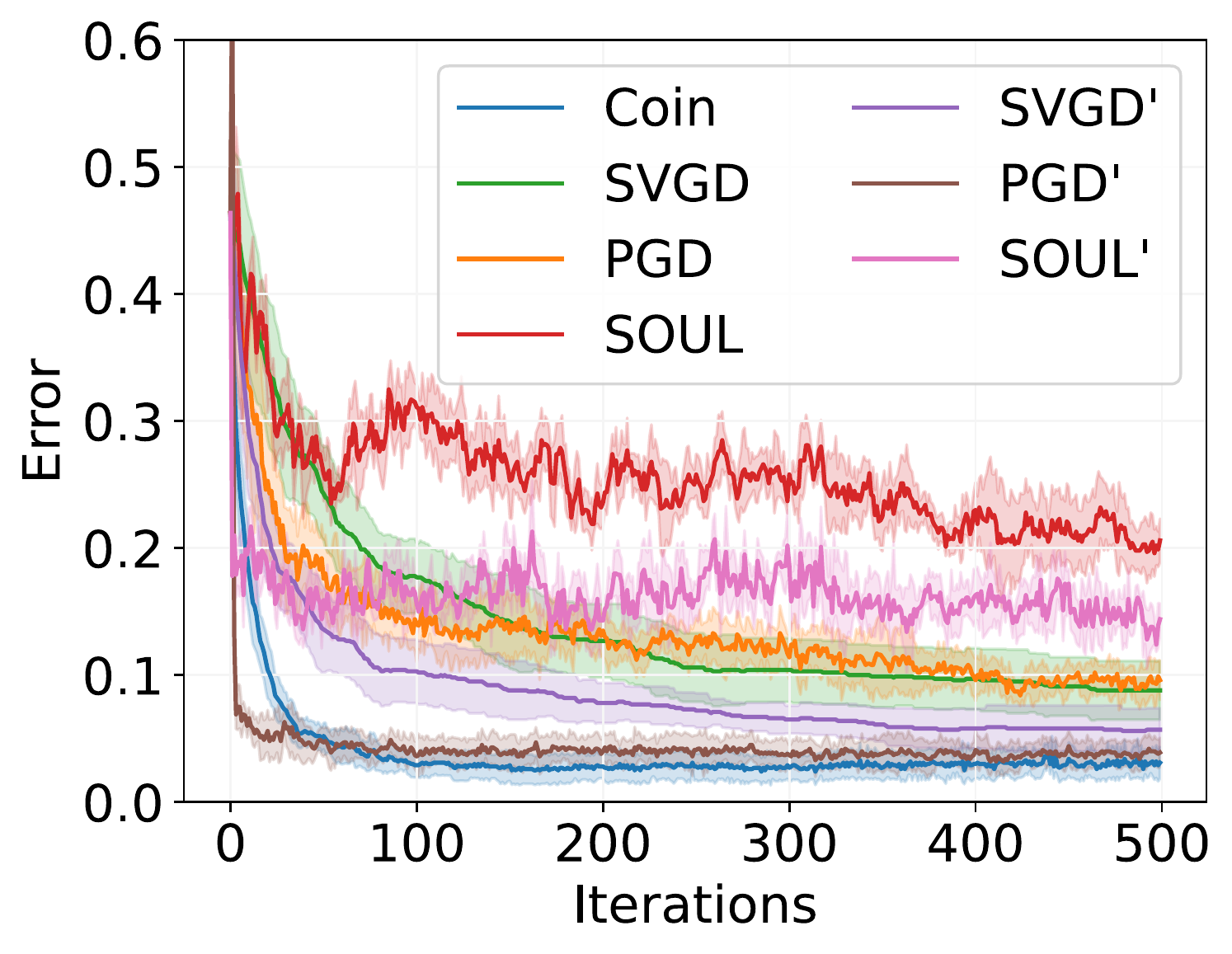}}
  \vspace{-3mm}
  \caption{\textbf{Results for the Bayesian neural network model.} Test error over $T=500$ training iterations, for different $N$. For all learning-rate dependent methods, we use the best learning rate as determined by Fig. \ref{fig:bnn_mnist_compare_lr}.}
  \label{fig:bnn_mnist_compare_methods}
  \vspace{-3mm}
\end{figure*}
}

{\setlength{\subfigcapskip}{-.5mm}
\begin{figure*}[!bp]
\vspace{-2mm}
  \centering
  \subfigure[Season 1]{\includegraphics[width=0.325\textwidth, trim = 20 20 20 40, clip]{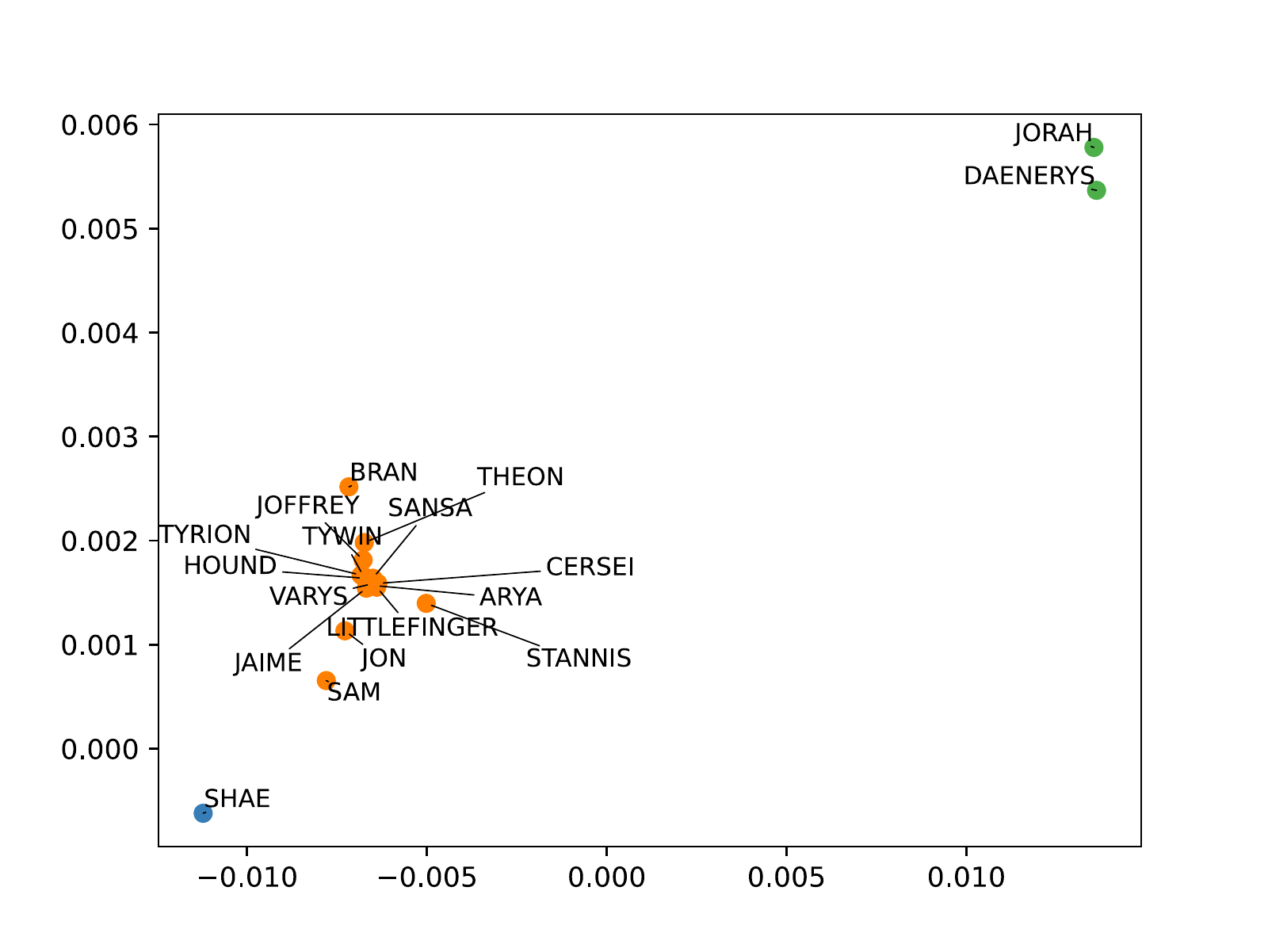}}
  \subfigure[Season 2]{\includegraphics[width=0.325\textwidth, trim = 20 20 20 40, clip]{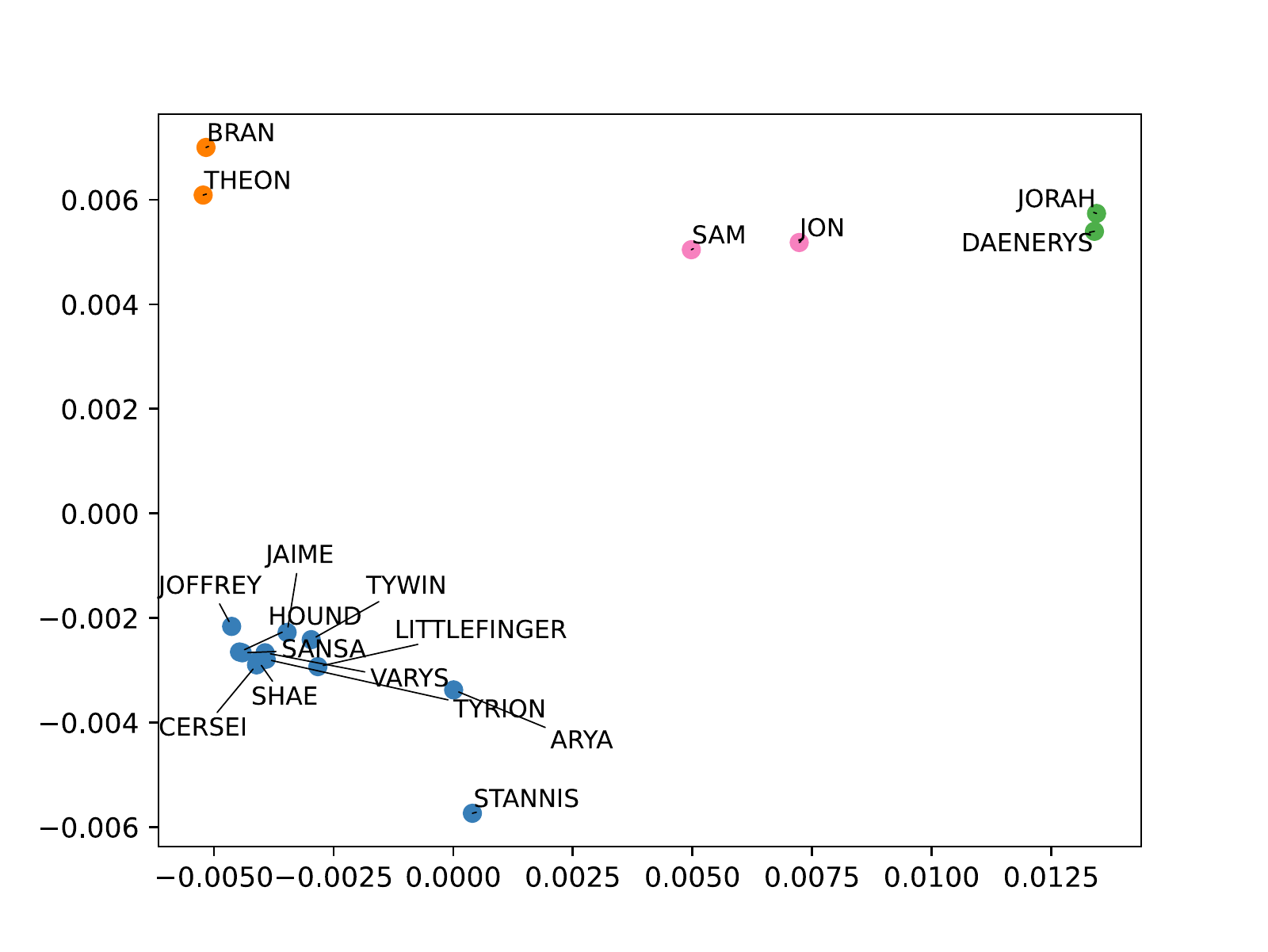}}
  \subfigure[Season 3]{\includegraphics[width=0.325\textwidth, trim = 20 20 20 40, clip]{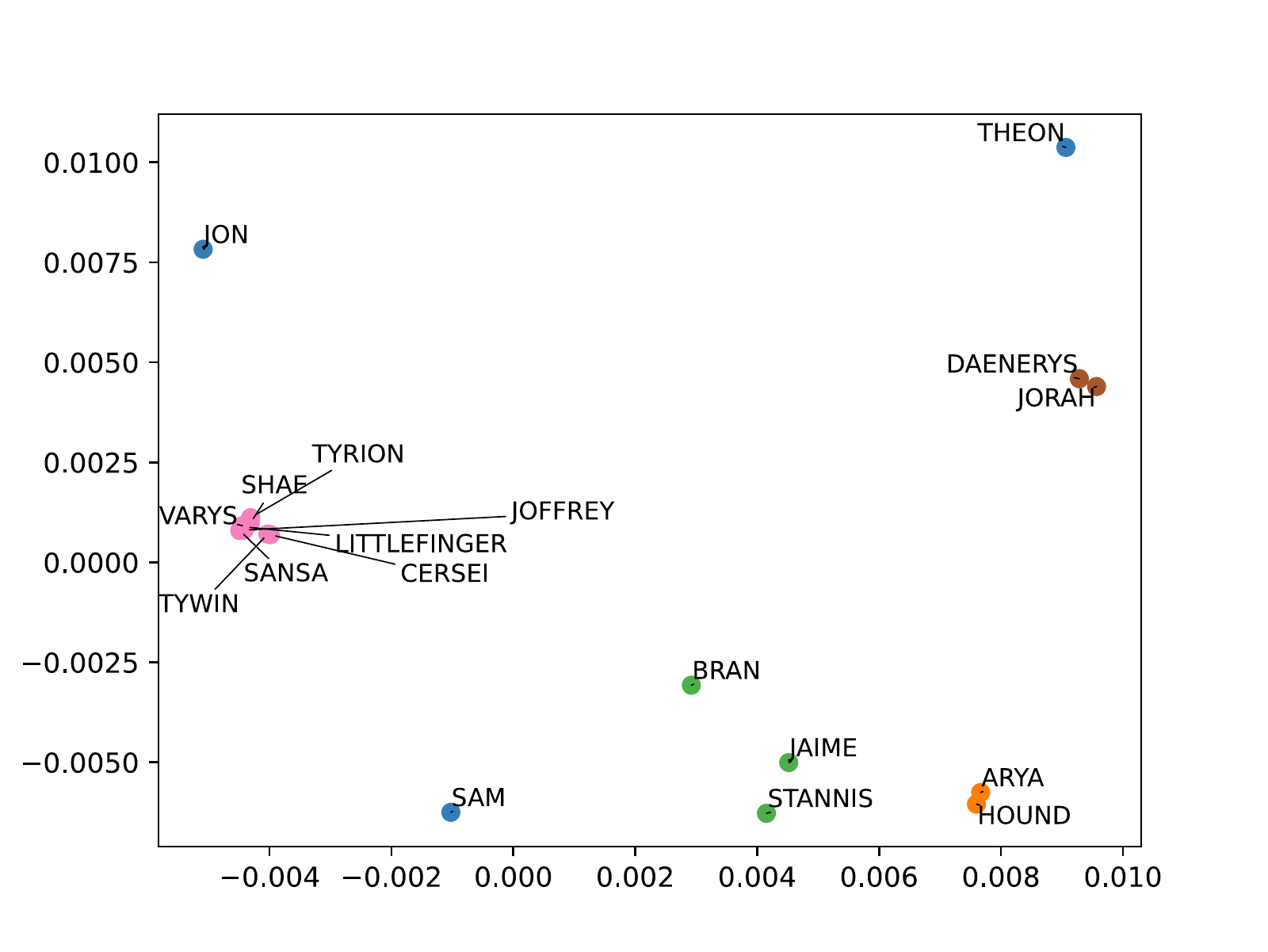}}
  \vspace{-3mm}
  \caption{\textbf{Results for the latent space network model}. Mean of the particles $\{z_{T}^i\}_{i=1}^N$ output by Coin EM afer $T=500$ iterations. Each node of the network represents a Game of Thrones character.}
  \label{fig:network_model}
\vspace{-3mm}
\end{figure*}
}

In Fig. \ref{fig:bnn_mnist_compare_methods}, we plot the test error achieved by Coin EM, SVGD EM, PGD, and SOUL for different numbers of particles.\footnote{We also include results for SVGD EM, PGD, and SOUL when using a scaling heuristic recommended in \cite[][Sec. 2]{Kuntz2023}, designed to stabilize the updates and avoid ill-conditioning (SVGD EM', PGD', and SOUL').}  For the learning-rate dependent methods, we use the best learning-rate as determined by a search over a fine grid (see Fig. \ref{fig:bnn_mnist_compare_lr} in App. \ref{sec:bayes-nn-alt-results}). In this example, Coin EM offers comparable or superior performance to the best competing method (PGD'). While the initial convergence of PGD' is typically slightly faster, the asymptotic error of Coin EM is consistently lower, particularly for small numbers of particles. In fact, our methods are generally more robust to changes in the number of particles (Fig. \ref{fig:bnn_mnist_compare_N} in App. \ref{sec:bayes-nn-alt-results}), the need for a scaling heuristic to avoid ill conditioning (Fig. \ref{fig:bnn_mnist_compare_lr} in App. \ref{sec:bayes-nn-alt-results}), and, of course, the learning rate (Fig. \ref{fig:bayes_nn_mnist} in App. \ref{sec:bayes-nn-alt-results}). In App. \ref{sec:bayes-nn-add-results}, we provide additional results for another Bayesian neural network (see App. \ref{sec:bayes-nn-details} for details), with similar conclusions.

\subsection{Latent space network model}
\label{sec:latent_space_model}
Finally, we consider a latent space network model \citep{Hoff2002,Loyal2023}. 
Such models assume that each node in a network has an unobserved latent position in a low-dimensional Euclidean space, and that the probability of a link between two nodes $i$ and $j$ depends on the distance between their latent variables $||z_{(i)}-z_{(j)}||$. Latent space network models can account for transitivity, homophily, and other network properties, and provide a natural way to analyze social network data. In particular, the estimated latent positions can be used to visualize the network, and also for downstream tasks such as node classification or community detection through clustering.

 Here, we consider fitting a latent space network model to a dataset curated by \cite{Beveridge2018}, which consists of a sequence of binary undirected networks indicating whether or not an interaction occurred between two characters in the TV series \emph{Game of Thrones}, with one network for each of the eight seasons (see App. \ref{sec:network-details} for further details). 
 
 In Fig. \ref{fig:network_model}, we plot the latent representation of the character interactions obtained using Coin EM. Plots for the other algorithms are given in App. \ref{sec:network-add-results}. In this case, we find that the latent representations obtained via Coin EM successfully capture several natural groupings of the characters, and how these groupings evolve through the first three seasons. For example, the latent representations of characters with significant interactions (e.g., Daenerys and Jorah in Season 1; 
 Arya and the Hound in Season 3) tend to appear in similar clusters. Meanwhile, this is not the case for the latent representations determined by PGD (Fig. \ref{fig:pgd_network_model} in App. \ref{sec:network-add-results}) or SOUL (Fig. \ref{fig:soul_network_model} in App. \ref{sec:network-add-results}), both of which fail to capture known relationships between characters, or how these evolve.

\section{DISCUSSION}
\vspace{-\baselineskip}
\textbf{Summary}. In this paper, we introduced two new algorithms for MMLE in latent variable models, including one which is entirely tuning-free. Our first algorithm, SVGD EM (Sec. \ref{sec:gradient-flow}), can be viewed as a particular form of gradient descent on the free-energy functional $\mathcal{F}$ over the product space $\Theta\times \mathcal{P}_2(\mathcal{Z})$. Our second algorithm, Coin EM (Sec. \ref{sec:coin}), is entirely different, and leverages coin betting ideas to remove any dependence on learning rates. 

\textbf{Limitations and Future Work}. We highlight several limitations, as well as some directions for future work. First, similar to SVGD, our algorithms have a $\mathcal{O}(N^2)$ cost per update. This compares unfavourably to the $\mathcal{O}(N)$ cost for PGD \citep{Kuntz2023}. Second, our convergence results for SVGD were derived in the population limit. In principle, one could obtain finite-particle convergence rates for SVGD EM using the techniques in \citet{Shi2022a}, or by considering `virtual particle' variants of SVGD EM, extending the approach in \citet{Das2023}.  Finally, even in the population limit, the convergence result for Coin EM was obtained under a technical sufficient condition, which it is unclear how to verify in practice. The problem of establishing the convergence of Coin EM without this assumption remains an interesting direction for future work.

\subsubsection*{Acknowledgments}
We thank the anonymous reviewers for their constructive feedback. We are grateful to Juan Kuntz for many insightful discussions. LS and CN were supported by the Engineering and Physical Sciences Research Council (EPSRC), grant number EP/V022636/1. CN acknowledges further support from the EPSRC, grant numbers EP/R01860X/1 and  EP/Y028783/1. DD was supported by the EPSRC-funded STOR-i Centre for Doctoral Training, grant number EP/S022252/1.

\bibliographystyle{apalike}
\bibliography{references}

\begin{thebibliography}{}

\bibitem[Akyıldız et~al., 2023]{Akyldz2023}
Akyıldız, {\"{O}}.~D., Crucinio, F.~R., Girolami, M., Johnston, T., and
  Sabanis, S. (2023).
\newblock {Interacting Particle Langevin Algorithm for Maximum Marginal
  Likelihood Estimation}.
\newblock {\em arXiv preprint}.

\bibitem[Ambrosio et~al., 2008]{Ambrosio2008}
Ambrosio, L., Gigli, N., and {Giuseppe Savar{\'{e}}} (2008).
\newblock {\em {Gradient Flows: In Metric Spaces and in the Space of
  Probability Measures}}.
\newblock Birkh{\"{a}}user, Basel.

\bibitem[Atchad{\'{e}} et~al., 2017]{Atchade2017}
Atchad{\'{e}}, Y.~F., Fort, G., and Moulines, E. (2017).
\newblock {On Perturbed Proximal Gradient Algorithms}.
\newblock {\em Journal of Machine Learning Research}, 18:1--33.

\bibitem[Balakrishnan et~al., 2017]{Balakrishnan2017}
Balakrishnan, S., Wainwright, M.~J., and Yu, B. (2017).
\newblock {Statistical guarantees for the EM algorithm: From population to
  sample-based analysis}.
\newblock {\em The Annals of Statistics}, 45(1):77--120.

\bibitem[Beveridge and Chemers, 2018]{Beveridge2018}
Beveridge, A. and Chemers, M. (2018).
\newblock {The Game of Game of Thrones: Networked Concordances and Fractal
  Dramaturgy}.
\newblock In {\em Reading Contemporary Serial Television Universes}, pages
  201--225. Routledge.

\bibitem[Bishop, 2006]{Bishop2006}
Bishop, C. (2006).
\newblock {\em {Pattern Recognition and Machine Learning}}.
\newblock Springer-Verlag, New York.

\bibitem[Blei et~al., 2003]{Blei2003}
Blei, D.~M., Ng, A.~Y., and Jordan, M.~I. (2003).
\newblock {Latent Dirichlet Allocation}.
\newblock {\em Journal of Machine Learning Research}, 3:993--1022.

\bibitem[Booth and Hobert, 1999]{Booth1999}
Booth, J.~G. and Hobert, J.~P. (1999).
\newblock {Maximizing generalized linear mixed model likelihoods with an
  automated Monte Carlo EM algorithm}.
\newblock {\em Journal of the Royal Statistical Society: Series B (Statistical
  Methodology)}, 61(1):265--285.

\bibitem[Bradbury et~al., 2018]{Bradbury2018}
Bradbury, J., Frostig, R., Hawkins, P., Johnson, M.~J., Leary, C., Maclaurin,
  D., Necula, G., Paszke, A., VanderPlas, J., Wanderman-Milne, S., and Zhang,
  Q. (2018).
\newblock {JAX: composable transformations of Python+NumPy programs}.

\bibitem[Brenier, 1991]{Brenier1991}
Brenier, Y. (1991).
\newblock {Polar factorization and monotone rearrangement of vector-valued
  functions}.
\newblock {\em Communications on Pure and Applied Mathematics}, 44(4):375--417.

\bibitem[Caffo et~al., 2005]{Caffo2005}
Caffo, B.~S., Jank, W., and Jones, G.~L. (2005).
\newblock {Ascent-based Monte Carlo expectation– maximization}.
\newblock {\em Journal of the Royal Statistical Society: Series B (Statistical
  Methodology)}, 67(2):235--251.

\bibitem[Capp{\'{e}} et~al., 1999]{Cappe1999}
Capp{\'{e}}, O., Doucet, A., Lavielle, M., and Moulines, E. (1999).
\newblock {Simulation-based methods for blind maximum-likelihood filter
  identification}.
\newblock {\em Signal Processing}, 73(1):3--25.

\bibitem[Casella, 1985]{Casella1985}
Casella, G. (1985).
\newblock {An Introduction to Empirical Bayes Data Analysis}.
\newblock {\em The American Statistician}, 39(2):83--87.

\bibitem[Chan and Ledolter, 1995]{Chan1995}
Chan, K.~S. and Ledolter, J. (1995).
\newblock {Monte Carlo EM Estimation for Time Series Models Involving Counts}.
\newblock {\em Journal of the American Statistical Association},
  90(429):242--252.

\bibitem[Chen et~al., 2016]{Chen2016}
Chen, C., Carlson, D., Gan, Z., Li, C., and Carin, L. (2016).
\newblock {Bridging the Gap between Stochastic Gradient MCMC and Stochastic
  Optimization}.
\newblock In {\em Proceedings of the 19th International Conference on
  Aritificial Intelligence and Statistics (AISTATS 2016)}, Cadiz, Spain.

\bibitem[Chen et~al., 2022a]{Chen2022}
Chen, K., Cutkosky, A., and Orabona, F. (2022a).
\newblock {Implicit Parameter-free Online Learning with Truncated Linear
  Models}.
\newblock In {\em Proceedings of the 33rd International Conference on
  Algorithmic Learning Theory (ALT 2022)}, Paris, France.

\bibitem[Chen et~al., 2022b]{Chen2022a}
Chen, K., Langford, J., and Orabona, F. (2022b).
\newblock {Better Parameter-Free Stochastic Optimization with ODE Updates for
  Coin-Betting}.
\newblock In {\em Proceedings of the Thirty-Sixth AAAI Conference on Artificial
  Intelligence (AAAI-22)}, Online.

\bibitem[Coullon et~al., 2021]{Coullon2021}
Coullon, J., South, L., and Nemeth, C. (2021).
\newblock {Efficient and Generalizable Tuning Strategies for Stochastic
  Gradient MCMC}.
\newblock {\em arXiv preprint}.

\bibitem[Cutkosky and Orabona, 2018]{Cutkosky2018}
Cutkosky, A. and Orabona, F. (2018).
\newblock {Black-Box Reductions for Parameter-free Online Learning in Banach
  Spaces}.
\newblock In {\em Proceedings of the 31st Annual Conference on Learning Theory
  (COLT 2018)}, Stockholm, Sweden.

\bibitem[Das and Nagaraj, 2023]{Das2023}
Das, A. and Nagaraj, D. (2023).
\newblock {Provably Fast Finite Particle Variants of SVGD via Virtual Particle
  Stochastic Approximation}.
\newblock In {\em Proceedings of the 37th Annual Conference on Neural
  Information Processing Systems (NeurIPS 2023)}, New Orleans, LA.

\bibitem[{De Bortoli} et~al., 2021]{DeBortoli2021}
{De Bortoli}, V., Durmus, A., Pereyra, M., and Vidal, A.~F. (2021).
\newblock {Efficient stochastic optimisation by unadjusted Langevin Monte
  Carlo}.
\newblock {\em Statistics and Computing}, 31(3):29.

\bibitem[Delyon et~al., 1999]{Delyon1999}
Delyon, B., Lavielle, M., and Moulines, E. (1999).
\newblock {Convergence of a stochastic approximation version of the EM
  algorithm}.
\newblock {\em The Annals of Statistics}, 27(1):94--128.

\bibitem[Dempster et~al., 1977]{Dempster1977}
Dempster, A.~P., Laird, N.~M., and Rubin, D.~B. (1977).
\newblock {Maximum Likelihood from Incomplete Data via the EM Algorithm}.
\newblock {\em Journal of the Royal Statistical Society: Series B (Statistical
  Methodology)}, 39(1):1--38.

\bibitem[Detommaso et~al., 2018]{Detommaso2018}
Detommaso, G., Cui, T., Spantini, A., Marzouk, Y., and Scheichl, R. (2018).
\newblock {A Stein variational Newton method}.
\newblock In {\em Proceedings of the 32nd Annual Conference on Neural
  Information Processing Systems (NIPS 2018)}, Montreal, Canada.

\bibitem[Dua and Graff, 2019]{Dua2019}
Dua, D. and Graff, C. (2019).
\newblock {UCI Machine Learning Repository}.
\newblock Technical report, University of California, Irvine, School of
  Information and Computer Sciences.

\bibitem[Duchi et~al., 2011]{Duchi2011}
Duchi, J., Hazan, E., and Singer, Y. (2011).
\newblock {Adaptive subgradient methods for online learning and stochastic
  optimization}.
\newblock {\em Journal of Machine Learning Research}, 12(61):2121--2159.

\bibitem[Duncan et~al., 2023]{Duncan2019}
Duncan, A., N{\"{u}}sken, N., and Szpruch, L. (2023).
\newblock {On the geometry of Stein variational gradient descent}.
\newblock {\em Journal of Machine Learning Research}, 24:1--40.

\bibitem[Ester et~al., 1996]{Ester1996}
Ester, M., Kriegel, H.-P., Sander, J., and Xu, X. (1996).
\newblock {A Density-Based Algorithm for Discovering Clusters in Large Spatial
  Databases with Noise}.
\newblock In {\em Proceedings of the Second International Conference on
  Knowledge Discovery and Data Mining}, pages 226--231. AAAI Press.

\bibitem[Fort and Moulines, 2003]{Fort2003}
Fort, G. and Moulines, E. (2003).
\newblock {Convergence of the Monte Carlo expectation maximization for curved
  exponential families}.
\newblock {\em The Annals of Statistics}, 31(4):1220--1259.

\bibitem[Fort et~al., 2011]{Fort2011}
Fort, G., Moulines, E., and Priouret, P. (2011).
\newblock {Convergence of adaptive and interacting Markov chain Monte Carlo
  algorithms}.
\newblock {\em The Annals of Statistics}, 39(6):3262--3289.

\bibitem[Gershman et~al., 2012]{Gershman2012}
Gershman, S.~J., Hoffman, M.~D., and Blei, D.~M. (2012).
\newblock {Nonparametric Variational Inference}.
\newblock In {\em Proceedings of the 29th International Conference on Machine
  Learning (ICML 2012)}, Edinburgh, UK.

\bibitem[Hernandez-Lobato and Adams, 2015]{Hernandez-Lobato2015}
Hernandez-Lobato, J.~M. and Adams, R.~P. (2015).
\newblock {Probabilistic Backpropagation for Scalable Learning of Bayesian
  Neural Networks}.
\newblock In {\em Proceedings of the 32nd International Conference on Machine
  Learning (ICML 2015)}, Lille, France.

\bibitem[Hoff et~al., 2002]{Hoff2002}
Hoff, P.~D., Raftery, A.~E., and Handcock, M.~S. (2002).
\newblock {Latent Space Approaches to Social Network Analysis}.
\newblock {\em Journal of the American Statistical Association},
  97(460):1090--1098.

\bibitem[Jun and Orabona, 2019]{Jun2019}
Jun, K.-S. and Orabona, F. (2019).
\newblock {Parameter-Free Online Convex Optimization with Sub-Exponential
  Noise}.
\newblock In {\em Proceedings of the 32nd Annual Conference on Learning Theory
  (COLT 2019)}, Phoenix, AZ.

\bibitem[Jun et~al., 2017]{Jun2017}
Jun, K.-S., Orabona, F., Wright, S., and Willett, R. (2017).
\newblock {Online Learning for Changing Environments using Coin Betting}.
\newblock {\em Electronic Journal of Statistics}, 11.

\bibitem[Kim et~al., 2022]{Kim2022}
Kim, S., Song, Q., and Liang, F. (2022).
\newblock {Stochastic gradient Langevin dynamics with adaptive drifts}.
\newblock {\em Journal of Statistical Computation and Simulation},
  92(2):318--336.

\bibitem[Kingma and Ba, 2015]{Kingma2015}
Kingma, D.~P. and Ba, J. (2015).
\newblock {Adam: a method for stochastic optimisation}.
\newblock In {\em Proceedings of the 3rd International Conference on Learning
  Representations (ICLR 2015)}, pages 1--13, San Diego, CA.

\bibitem[Korba et~al., 2020]{Korba2020}
Korba, A., Salim, A., Arbel, M., Luise, G., and Gretton, A. (2020).
\newblock {A Non-Asymptotic Analysis for Stein Variational Gradient Descent}.
\newblock In {\em Proceedings of the 34th Annual Conference on Neural
  Information Processing Systems (NeurIPS 2020)}, Vancouver, Canada.

\bibitem[Krichevsky and Trofimov, 1981]{Krichevsky1981}
Krichevsky, R.~E. and Trofimov, V.~K. (1981).
\newblock {The performance of universal encoding}.
\newblock {\em IEEE Transactions on Information Theory}, 27(2):199--207.

\bibitem[Kuntz et~al., 2023]{Kuntz2023}
Kuntz, J., Lim, J.~N., and Johansen, A.~M. (2023).
\newblock {Particle algorithms for maximum likelihood training of latent
  variable models}.
\newblock In {\em Proceedings of the 26th International Conference on
  Artificial Intelligence and Statistics (AISTATS 2023)}, Valencia, Spain.

\bibitem[Kushner and Clark, 1978]{Kushner1978}
Kushner, H.~J. and Clark, D.~S. (1978).
\newblock {\em {Stochastic approximation methods for constrained and
  unconstrained systems}}.
\newblock Springer-Verlag, New York.

\bibitem[Lange, 1995]{Lange1995}
Lange, K. (1995).
\newblock {A Gradient Algorithm Locally Equivalent to the EM Algorithm}.
\newblock {\em Journal of the Royal Statistical Society. Series B
  (Methodological)}, 57(2):425--437.

\bibitem[Lecun et~al., 1998]{Lecun1998}
Lecun, Y., Bottou, L., Bengio, Y., and Haffner, P. (1998).
\newblock {Gradient-based learning applied to document recognition}.
\newblock {\em Proceedings of the IEEE}, 86(11):2278--2324.

\bibitem[Levine and Casella, 2001]{Levine2001}
Levine, R.~A. and Casella, G. (2001).
\newblock {Implementations of the Monte Carlo EM Algorithm}.
\newblock {\em Journal of Computational and Graphical Statistics},
  10(3):422--439.

\bibitem[Levine and Fan, 2004]{Levine2004}
Levine, R.~A. and Fan, J. (2004).
\newblock {An automated (Markov chain) Monte Carlo EM algorithm}.
\newblock {\em Journal of Statistical Computation and Simulation},
  74(5):349--360.

\bibitem[Li et~al., 2016]{Li2016a}
Li, C., Chen, C., Carlson, D., and Carin, L. (2016).
\newblock {Preconditioned Stochastic Gradient Langevin Dynamics for Deep Neural
  Networks}.
\newblock In {\em Proceedings of the 30th AAAI Conference on Aritificial
  Intelligence (AAAI-16)}, Phoenix, AZ.

\bibitem[Li et~al., 2020]{Li2020}
Li, L., Li, Y., Liu, J.-G., Liu, Z., and Lu, J. (2020).
\newblock {A stochastic version of Stein variational gradient descent for
  efficient sampling}.
\newblock {\em Communications in Applied Mathematics and Computational
  Science}, 15(1):37--63.

\bibitem[Liu and Rubin, 1994]{Liu1994}
Liu, C. and Rubin, D.~B. (1994).
\newblock {The ECME Algorithm: A Simple Extension of EM and ECM with Faster
  Monotone Convergence}.
\newblock {\em Biometrika}, 81(4):633--648.

\bibitem[Liu, 2017]{Liu2017}
Liu, Q. (2017).
\newblock {Stein variational gradient descent as gradient flow}.
\newblock In {\em Proceedings of the 31st Annual Conference on Neural
  Information Processing Systems (NIPS 2017)}, pages 3118--3126, Red Hook, NY.

\bibitem[Liu and Wang, 2016]{Liu2016a}
Liu, Q. and Wang, D. (2016).
\newblock {Stein Variational Gradient Descent: A General Purpose Bayesian
  Inference Algorithm}.
\newblock In {\em Proceedings of the 30th Conference on Neural Information
  Processings Systems (NIPS 2016)}, Barcelona, Spain.

\bibitem[Loyal and Chen, 2023]{Loyal2023}
Loyal, J.~D. and Chen, Y. (2023).
\newblock {A Bayesian Nonparametric Latent Space Approach to Modeling Evolving
  Communities in Dynamic Networks}.
\newblock {\em Bayesian Analysis}, 18(1):49--77.

\bibitem[McCulloch, 1997]{McCulloch1997}
McCulloch, C.~E. (1997).
\newblock {Maximum Likelihood Algorithms for Generalized Linear Mixed Models}.
\newblock {\em Journal of the American Statistical Association},
  92(437):162--170.

\bibitem[McLachlan and Krishnan, 2007]{McLachlan2007}
McLachlan, G.~J. and Krishnan, T. (2007).
\newblock {\em {The EM Algorithm and Extensions}}.
\newblock John Wiley and Sons, 2nd edition.

\bibitem[Meng and Rubin, 1993]{Meng1993}
Meng, X.-L. and Rubin, D.~B. (1993).
\newblock {Maximum Likelihood Estimation via the ECM Algorithm: A General
  Framework}.
\newblock {\em Biometrika}, 80(2):267--278.

\bibitem[Neal and Hinton, 1998]{Neal1998}
Neal, R.~M. and Hinton, G.~E. (1998).
\newblock {A View of the Em Algorithm that Justifies Incremental, Sparse, and
  other Variants}.
\newblock In Jordan, M.~I., editor, {\em Learning in Graphical Models}, pages
  355--368. Springer Netherlands, Dordrecht.

\bibitem[Nijkamp et~al., 2020]{Nijkamp2020}
Nijkamp, E., Pang, B., Han, T., Zhu, S.-C., and Wu, Y.~N. (2020).
\newblock {Learning Multi-layer Latent Variable Model via Variational
  Optimization of Short Run MCMC for Approximate Inference}.
\newblock In {\em European Conference on Computer Vision}, pages 361--378,
  Online.

\bibitem[Orabona and Cutkosky, 2020]{Orabona2020}
Orabona, F. and Cutkosky, A. (2020).
\newblock {Tutorial on Parameter-Free Online Learning}.
\newblock In {\em Proceedings of the 37th International Conference on Machine
  Learning (ICML 2020)}, Online.

\bibitem[Orabona and Pal, 2016a]{Orabona2016}
Orabona, F. and Pal, D. (2016a).
\newblock {Coin Betting and Parameter-Free Online Learning}.
\newblock In {\em Proceedings of the 30th Conference on Neural Information
  Processings Systems (NIPS 2016)}, Barcelona, Spain.

\bibitem[Orabona and Pal, 2016b]{Orabona2016a}
Orabona, F. and Pal, D. (2016b).
\newblock {Parameter-Free Convex Learning Through Coin Betting}.
\newblock In {\em Proceedings of the 33rd International Conference on Machine
  Learning (ICML 2016): AutoML Workshop}, New York, NY.

\bibitem[Orabona and Tommasi, 2017]{Orabona2017}
Orabona, F. and Tommasi, T. (2017).
\newblock {Training Deep Networks without Learning Rates Through Coin Betting}.
\newblock In {\em Proceedings of the 31st Annual Conference on Neural
  Information Processing Systems (NIPS 2017)}, Long Beach, CA.

\bibitem[Otto, 2001]{Otto2001}
Otto, F. (2001).
\newblock {The Geometry of Dissipative Evolution Equations: The Porous Medium
  Equation}.
\newblock {\em Communications in Partial Differential Equations},
  26(1-2):101--174.

\bibitem[Qiu and Wang, 2020]{Qiu2020}
Qiu, Y. and Wang, X. (2020).
\newblock {Stochastic Approximate Gradient Descent via the Langevin Algorithm}.
\newblock In {\em Proceedings of the 34th AAAI Conference on Artificial
  Intelligence (AAAI-20)}, New York, NY.

\bibitem[Rahimi and Recht, 2007]{Rahimi2007}
Rahimi, A. and Recht, B. (2007).
\newblock {Random Features for Large-Scale Kernel Machines}.
\newblock In {\em Proceedings of the 21st Annual Conference on Neural
  Information Processing Systems (NIPS 2007)}, Vancouver, Canada.

\bibitem[Redner and Walker, 1984]{Redner1984}
Redner, R.~A. and Walker, H.~F. (1984).
\newblock {Mixture Densities, Maximum Likelihood and the EM Algorithm}.
\newblock {\em SIAM Review}, 26(2):195--239.

\bibitem[Robbins, 1956]{Robbins1956}
Robbins, H. (1956).
\newblock {An empirical Bayes approach to statistics}.
\newblock In {\em Proceedings of the Third Berkeley Symposium on Mathematical
  Statistics and Probability}, pages 157--164.

\bibitem[Robbins and Monro, 1951]{Robbins1951}
Robbins, H. and Monro, S. (1951).
\newblock {A stochastic approximation method}.
\newblock {\em The Annals of Mathematical Statistics}, 22(3):400--407.

\bibitem[Salim et~al., 2022]{Salim2022}
Salim, A., Sun, L., and {Peter Richt{\'{a}}rik} (2022).
\newblock {A Convergence Theory for SVGD in the Population Limit under
  Talagrand's Inequality T1}.
\newblock In {\em Proceedings of the 39th International Conference on Machine
  Learning (ICML 2022)}, Online.

\bibitem[Sharrock et~al., 2023]{Sharrock2023a}
Sharrock, L., Mackey, L., and Nemeth, C. (2023).
\newblock {Learning Rate Free Sampling in Constrained Domains}.
\newblock In {\em Proceedings of the 37th Annual Conference on Neural
  Information Processing Systems (NeurIPS 2023)}, New Orleans, LA.

\bibitem[Sharrock and Nemeth, 2023]{Sharrock2023}
Sharrock, L. and Nemeth, C. (2023).
\newblock {Coin Sampling: Gradient-Based Bayesian Inference without Learning
  Rates}.
\newblock In {\em Proceedings of the 40th International Conference on Machine
  Learning (ICML 2023)}, Honolulu, HI.

\bibitem[Sherman et~al., 1999]{Sherman1999}
Sherman, R.~P., Ho, Y.-Y.~K., and Dalal, S.~R. (1999).
\newblock {Conditions for convergence of Monte Carlo EM sequences with an
  application to product diffusion modeling}.
\newblock {\em The Econometrics Journal}, 2(2):248--267.

\bibitem[Shi and Mackey, 2023]{Shi2022a}
Shi, J. and Mackey, L. (2023).
\newblock {A Finite-Particle Convergence Rate for Stein Variational Gradient
  Descent}.
\newblock In {\em Proceedings of the 37th Annual Conference on Neural
  Information Processing Systems (NeurIPS 2023)}, New Orleans, LA.

\bibitem[Smaragdis et~al., 2006]{Smaragdis2006}
Smaragdis, P., Raj, B., and Shashanka, M. (2006).
\newblock {A Probabilistic Latent Variable Model for Acoustic Modeling}.
\newblock In {\em Proceedings of the 20th Annual Conference on Neural
  Information Processing Systems: Workshop on Advances in Models for Acoustic
  Processing (NIPS 2006)}, Vancouver, Canada.

\bibitem[Tanner, 1993]{Tanner1993}
Tanner, M.~A. (1993).
\newblock {\em {Tools for Statistical Inference}}.
\newblock Springer-Verlag, New York, NY, 2nd edition.

\bibitem[Tieleman and Hinton, 2012]{Tieleman2012}
Tieleman, T. and Hinton, G.~E. (2012).
\newblock {Lecture 6.5-rmsprop: divide the gradient by a running average of its
  recent magnitude.}
\newblock {\em COURSERA: Neural networks for machine learning}, 4(2):26--31.

\bibitem[Turnbull et~al., 2019]{Turnbull2019}
Turnbull, K., Lunagomez, S., Nemeth, C., and Airoldi, E. (2019).
\newblock {Latent space modelling of hypergraph data}.
\newblock {\em arXiv preprint}.

\bibitem[van Erven and Harremos, 2014]{Erven2014}
van Erven, T. and Harremos, P. (2014).
\newblock {R{\'{e}}nyi Divergence and Kullback-Leibler Divergence}.
\newblock {\em IEEE Transactions on Information Theory}, 60(7):3797--3820.

\bibitem[Villani, 2003]{Villani2003}
Villani, C. (2003).
\newblock {\em {Topics in Optimal Transportation}}.
\newblock American Mathematical Society, Providence, Rhode Island.

\bibitem[Villani, 2008]{Villani2008}
Villani, C. (2008).
\newblock {\em {Optimal Transport: Old and New}}.
\newblock Springer-Verlag, Berlin.

\bibitem[Wei and Tanner, 1990]{Wei1990}
Wei, G. C.~G. and Tanner, M.~A. (1990).
\newblock {A Monte Carlo Implementation of the EM Algorithm and the Poor Man's
  Data Augmentation Algorithms}.
\newblock {\em Journal of the American Statistical Association},
  85(411):699--704.

\bibitem[Welling and Teh, 2011]{Welling2011}
Welling, M. and Teh, Y.~W. (2011).
\newblock {Bayesian Learning via Stochastic Gradient Langevin Dynamics}.
\newblock In {\em Proceedings of the 28th International Conference on Machine
  Learning (ICML 2011)}, Bellevue, WA.

\bibitem[Wolberg and Mangasarian, 1990]{Wolberg1990}
Wolberg, W.~H. and Mangasarian, O.~L. (1990).
\newblock {Multisurface method of pattern separation for medical diagnosis
  applied to breast cytology.}
\newblock {\em Proceedings of the National Academy of Sciences},
  87(23):9193--9196.

\bibitem[Wu, 1983]{Wu1983}
Wu, C. (1983).
\newblock {On the convergence properties of the EM algorithm}.
\newblock {\em Annals of Statistics}, 11(1):95--103.

\bibitem[Yao et~al., 2022]{Yao2022a}
Yao, Y., Vehtari, A., and Gelman., A. (2022).
\newblock {Stacking for nonmixing Bayesian computations: The curse and blessing
  of multimodal posteriors}.
\newblock {\em Journal of Machine Learning Research}, 23(79):1--45.

\end{thebibliography}

\clearpage
\section*{Checklist}

 \begin{enumerate}

 \item For all models and algorithms presented, check if you include:
 \begin{enumerate}
   \item A clear description of the mathematical setting, assumptions, algorithm, and/or model. [\textbf{Yes}/No/Not Applicable]
   \item An analysis of the properties and complexity (time, space, sample size) of any algorithm. [\textbf{Yes}/No/Not Applicable]
   \item (Optional) Anonymized source code, with specification of all dependencies, including external libraries. [\textbf{Yes}/No/Not Applicable]
 \end{enumerate}

 \item For any theoretical claim, check if you include:
 \begin{enumerate}
   \item Statements of the full set of assumptions of all theoretical results. [\textbf{Yes}/No/Not Applicable]
   \item Complete proofs of all theoretical results. [\textbf{Yes}/No/Not Applicable]
   \item Clear explanations of any assumptions. [\textbf{Yes}/No/Not Applicable]     
 \end{enumerate}

 \item For all figures and tables that present empirical results, check if you include:
 \begin{enumerate}
   \item The code, data, and instructions needed to reproduce the main experimental results (either in the supplemental material or as a URL). [\textbf{Yes}/No/Not Applicable]
   \item All the training details (e.g., data splits, hyperparameters, how they were chosen). [\textbf{Yes}/No/Not Applicable]
         \item A clear definition of the specific measure or statistics and error bars (e.g., with respect to the random seed after running experiments multiple times). [\textbf{Yes}/No/Not Applicable]
         \item A description of the computing infrastructure used. (e.g., type of GPUs, internal cluster, or cloud provider). [\textbf{Yes}/No/Not Applicable]
 \end{enumerate}

 \item If you are using existing assets (e.g., code, data, models) or curating/releasing new assets, check if you include:
 \begin{enumerate}
   \item Citations of the creator If your work uses existing assets. [\textbf{Yes}/No/Not Applicable]
   \item The license information of the assets, if applicable. [\textbf{Yes}/No/Not Applicable]
   \item New assets either in the supplemental material or as a URL, if applicable. [\textbf{Yes}/No/Not Applicable]
   \item Information about consent from data providers/curators. [\textbf{Yes}/No/Not Applicable]
   \item Discussion of sensible content if applicable, e.g., personally identifiable information or offensive content. [Yes/No/\textbf{Not Applicable}]
 \end{enumerate}

 \item If you used crowdsourcing or conducted research with human subjects, check if you include:
 \begin{enumerate}
   \item The full text of instructions given to participants and screenshots. [Yes/No/\textbf{Not Applicable}]
   \item Descriptions of potential participant risks, with links to Institutional Review Board (IRB) approvals if applicable. [Yes/No/\textbf{Not Applicable}]
   \item The estimated hourly wage paid to participants and the total amount spent on participant compensation. [Yes/No/\textbf{Not Applicable}]
 \end{enumerate}
 \end{enumerate}

\clearpage
\onecolumn\appendix
\aistatstitle{Tuning-Free Maximum Likelihood Training of Latent Variable Models via Coin Betting: Supplementary Materials}

\section{ADDITIONAL THEORETICAL RESULTS}
\label{sec:additional-theoretical-results}

We will require the following additional notation. We will write Let $(\mathcal{U},||\cdot||_{\mathcal{U}})$ denote the Cartesian product of the Euclidean space $\smash{(\Theta,||\cdot||_{\mathbb{R}^{d_{{\theta}}}})}$ and the Hilbert product space $\smash{(\mathcal{H}_{k^{d_z}}, ||\cdot||_{\mathcal{H}_k^{d_z}})}$, with norm $||\cdot||_{\mathcal{U}}$ defined according to $\smash{||(\theta,f)||_{\mathcal{U}}^2 = ||\theta||_{\mathbb{R}^{d_{{\theta}}}}^2 + ||f||_{\mathcal{H}_k^{d_z}}^2}$ for $\smash{\theta\in\Theta}$ and $\smash{f\in\mathcal{H}_k^{d_z}}$. In addition, we will write $\smash{\nabla_{\mathcal{U}} \mathcal{F} = (\nabla_{\theta}\mathcal{F}, S_{\mu} \nabla_{W_2}\mathcal{F}) \in \mathcal{U}}$.

\subsection{SVGD EM: Continuous Time Results}
\label{sec:continuous-time-results}
In this section, we study the properties of the SVGD EM gradient flow, which we recall is given by 
\begin{alignat}{2}
\label{eq:svgd-em-continuous-1}
    \frac{\partial \theta_t}{\partial t} &= - \nabla_{\theta}\mathcal{F}(\theta_t,\mu_t),\quad &&\nabla_{\theta}\mathcal{F}(\theta_t,\mu_t) = -\int \nabla_{\theta} \log \pi_{\theta_t}(z)\mu_t(z)\mathrm{d}z, \\
    \label{eq:svgd-em-continuous-2}
    \frac{\partial\mu_t}{\partial t} &=  - \nabla_{\mu}\mathcal{F}(\theta_t,\mu_t)~,\qquad &&\nabla_{\mu}\mathcal{F}(\theta_t,\mu_t):= - \nabla \cdot \bigg( \mu_t  \bigg[P_{\mu_t} \nabla_{W_2}\mathcal{F}(\theta_t,\mu_t)\bigg]\bigg),
\end{alignat}
where $\smash{\nabla_{W_2}\mathcal{F}(\theta_t,\mu_t) = \nabla_{z}\log \left(\frac{\mu_t}{\pi_{\theta_t}}\right)}$ denotes the Wasserstein gradient of $\mathcal{F}(\theta_t,\cdot)$ at $\mu_t$. This gradient flow was obtained by replacing the Wasserstein gradient appearing in \eqref{eq:gradient_flow}, \eqref{eq:flow_gradients} with its kernelized version, $P_{\mu_t}\nabla_{W_2}\mathcal{F}(\theta_t,\mu_t)$.

We first provide a result which quantifies the dissipation of the free energy along the trajectory of the continuous-time SVGD EM dynamics.

\begin{proposition}
\label{prop:dissipation}
    The dissipation of the free energy along the SVGD EM gradient flow \eqref{eq:svgd-em-continuous-1} - \eqref{eq:svgd-em-continuous-2} is given by 
    \begin{equation}
        \frac{\mathrm{d}\mathcal{F}(\theta_t,\mu_t)}{\mathrm{d}t} = - \int \big|\big|\nabla_{\theta}\log \pi_{\theta_t}(z)\big|\big|^2_{\mathbb{R}^{d_{\theta}}}\mu_t(z)\mathrm{d}z - \big|\big|S_{\mu_t}\nabla_{z}\log \left(\frac{\mu_t}{p_{\theta_t}(\cdot|x)}\right)\big|\big|^2_{\mathcal{H}_k^{d}}. \label{eq:dissipation}
    \end{equation}
\end{proposition}

\begin{proof}
    See App. \ref{app:proof_exponential_convergence}.
\end{proof}

\begin{remark}
    We can identify the second term in \eqref{eq:dissipation} as the Stein Fisher information of $\mu_t$ relative to the posterior $p_{\theta_t}(\cdot|x)$  \cite{Duncan2019,Korba2020}, written
    \begin{equation}
       I_{\mathrm{Stein}}(\mu_t|p_{\theta_t}(\cdot|x)):= \big|\big|S_{\mu_t}\nabla_{z}\log \left(\frac{\mu_t}{p_{\theta_t}(\cdot|x)}\right)\big|\big|^2_{\mathcal{H}_k^{d}}.
    \end{equation}
    This quantity is sometimes also referred to as the squared kernel Stein discrepancy (KSD).
\end{remark}

Since both of the terms on the RHS in \eqref{eq:dissipation} are negative, Proposition \ref{prop:dissipation} shows that the free energy decreases along the SVGD EM gradient flow. Under some additional assumption, we can actually establish convergence of the two terms on the RHS to zero. First, we will require the following rather mild regularity condition on the marginal likelihood; see also \cite[][Assumption 1]{Kuntz2023}.

\begin{assumption}
\label{assumption:bounded_sets}
    The super-level sets of the marginal likelihood $p_{\theta}(x)$ are compact. That is, the set $\{\theta\in\Theta: p_{\theta}(x)\geq c\}$ is bounded for all $c\geq 0$.
\end{assumption}




\begin{proposition}
\label{prop:grad_convergence}
    Let $(\theta_t)_{t\geq 0}$ and $(\mu_t)_{t\geq 0}$ be solutions of the SVGD EM gradient flow \eqref{eq:svgd-em-continuous-1} - \eqref{eq:svgd-em-continuous-2}.  Suppose that Assumptions \ref{assumption:bounded_k},  \ref{assumption:bounded_hessian0} and \ref{assumption:bounded_sets}
    hold. In addition, suppose that there exists a positive constant $C>0$ such that $\int ||z||\mu_t(z)\mathrm{d}z<C$ for all $t\geq 0$. Then 
    \begin{equation}
        \lim_{t\rightarrow \infty}\left[ \big|\big|\nabla_{\theta}\mathcal{F}(\theta_t,\mu_t)\big|\big|^2 + \big|\big|S_{\mu_t}\nabla_{z}\log \left(\frac{\mu_t}{\pi_{\theta_t}}\right)\big|\big|^2_{\mathcal{H}_k^{d}}\right] = 0.
    \end{equation}
\end{proposition}

\begin{proof}
    See App. \ref{app:proof_exponential_convergence}.
\end{proof}
\pagebreak


\begin{remark}
    We note that, as an alternative to Assumption \ref{assumption:bounded_hessian0}, we could instead impose separate assumptions on the existence and boudedness of the Hessian $H_{V_z}$ of the function $V_{z}:\Theta\rightarrow\mathbb{R}$, which maps $\theta\mapsto -\log \pi_{\theta}(z)$ for fixed $z\in\mathcal{Z}$; the Hessian $H_{V_{\theta}}$ of the function $V_{\theta}:\mathcal{Z}\rightarrow\mathbb{R}$ which maps $z\mapsto-\log \pi_{\theta}(z)$ for fixed $\theta\in\Theta$; and the matrix $\smash{H_{V_{\theta,z}}:\mathbb{R}^{d_{z}}\rightarrow\mathbb{R}^{d_{\theta}}}$, which contains the mixed partial derivatives of $V$. Indeed, for certain results, e.g., Theorem \ref{prop:descent_lemma}, we only require boundedness of the `diagonal blocks' of the Hessian $H_{V}$ of the function $V:\Theta\times\mathcal{Z} \rightarrow \mathbb{R}$, $V:(\theta,z)\mapsto -\log \pi_{\theta}(z)$, which is strictly weaker than Assumption \ref{assumption:bounded_hessian0}. For ease of presentation, in this case we choose to work directly with the Hessian $H_V$ of $V$, as a function of both variables.
\end{remark}

In order to establish exponential convergence along the SVGD EM gradient flow, we will require an additional condition, which characterizes the properties of $\mathcal{F}$ around its equilibria. We assume a particular `gradient dominance' condition, which represents a natural extension of the corresponding conditions used for Euclidean gradients flows - the Polyak-Łojasiewicz inequality - and for the SVGD gradient flow - the Stein log-Sobolev inequality \cite{Duncan2019,Korba2020}.

\begin{assumption}
\label{assumption:grad_dominance}
     There exists $\lambda>0$ such that $\mathcal{F}$ satisfies the following gradient dominance condition
    \begin{align} 
    \mathcal{F}(\theta,\mu) - \min_{(\theta,\mu) \in \theta\times\mathcal{P}_2(\mathcal{Z})} \mathcal{F}(\theta,\mu) &\leq \frac{1}{2\lambda} ||\nabla_{\mathcal{U}} \mathcal{F}||^2_{\mathcal{U}}. 
    \end{align}
\end{assumption}

\begin{proposition} \label{prop:exponential_convergence}
    Assume that Assumption \ref{assumption:grad_dominance} holds. Then the free energy $\mathcal{F}(\theta,\mu)$ decreases exponentially fast along the SVGD EM gradient flow. In particular, 
    \begin{equation}
        \mathcal{F}(\theta_t,\mu_t) -  \min_{(\theta,\mu) \in \theta\times\mathcal{P}_2(\mathcal{Z})} \mathcal{F}(\theta,\mu)\leq e^{-2\lambda t} \big[\mathcal{F}(\theta_0,\mu_0) - \min_{(\theta,\mu) \in \theta\times\mathcal{P}_2(\mathcal{Z})} \mathcal{F}(\theta,\mu)\big].
    \end{equation}
    
\end{proposition}

\begin{proof}
    See App. \ref{app:proof_exponential_convergence}.
\end{proof}

\subsection{SVGD EM: Discrete Time Results}
We now consider the forward Euler discretisation of the dynamics in \eqref{eq:svgd-em-continuous-1} - \eqref{eq:svgd-em-continuous-2}, as given in \eqref{eq:svgd_population_discrete_time1} - \eqref{eq:svgd_population_discrete_time2}. For convenience, we recall these update equations again now:
\begin{align}
    \theta_{t+1} &= \theta_{t} + \gamma  \int \nabla_{\theta}\log \pi_{\theta_{t}}(z)\mathrm{d}\mu_{t}(z)  \label{eq:svgd_population_discrete_time1_v0} \\
    \mu_{t+1} &= \left(\mathrm{id} + \gamma  \int \left[\nabla_z \log \pi_{\theta_{t+1}}(z) k(z,\cdot)  +  \nabla_{z}k(z,\cdot)\right] \mu_t(z)\mathrm{d}z\right)_{\#}\mu_{t} .\label{eq:svgd_population_discrete_time2_v0}
\end{align}
The updates in \eqref{eq:svgd_population_discrete_time1_v0} - \eqref{eq:svgd_population_discrete_time2_v0} represent the population limit of SVGD EM (Alg. \ref{alg:svgdEM}). In Theorem \ref{prop:descent_lemma}, we established a descent lemma, guaranteeing that, given a suitable choice of the learning rate $\gamma$, the free energy decreases at each iteration of the SVGD EM algorithm. As a corollary to this result, we now obtain  a discrete time convergence rate for the average of $||\nabla_{\theta}\mathcal{F}(\theta_t,\mu_t)||_{\mathbb{R}^{d_{\theta}}}^2$ and $\smash{||\mathcal{S}_{\mu_t}\nabla_{W_2}\mathcal{F}(\theta_{t+1},\mu_t)||^2_{\mathcal{H}_k^{d_z}}}$. In particular, we have the following result.

\begin{corollary}
\label{corollary:average_convergence}
Let $\alpha>1$, and $\gamma\leq \min\left(\gamma^{*}, \tfrac{2}{M}, \tfrac{2}{(M+\alpha^2)B^2}\right)$. 
Then, defining $c_{\gamma} = \gamma\left(1 - \tfrac{[M + (M+\alpha^2)B^2]\gamma}{2}\right)$, the discrete-time, population-limit, SVGD EM updates in \eqref{eq:svgd_population_discrete_time1} - \eqref{eq:svgd_population_discrete_time2} satisfy 
\begin{align}
    &\min_{t=1,\dots,T} \left(\left|\left|\nabla_{\theta}\mathcal{F}(\theta_t,\mu_t)\right|\right|_{\mathbb{R}^{d_{\theta}}}^2+\left|\left|\mathcal{S}_{\mu_t}\nabla_{W_2}\mathcal{F}(\theta_{t+1},\mu_t)\right|\right|^2_{\mathcal{H}_k^{d_z}}\right) \\
    &\leq \frac{1}{T}\sum_{t=1}^T \left(\left|\left|\nabla_{\theta}\mathcal{F}(\theta_t,\mu_t)\right|\right|_{\mathbb{R}^{d_{\theta}}}^2+\left|\left|\mathcal{S}_{\mu_t}\nabla_{W_2}\mathcal{F}(\theta_{t+1},\mu_t)\right|\right|^2_{\mathcal{H}_k^{d_z}}\right) \leq \frac{\mathcal{F}(\theta_0,\mu_0) - \min_{(\theta,\mu) \in \theta\times\mathcal{P}_2(\mathcal{Z})} \mathcal{F}(\theta,\mu)}{c_{\gamma}T}.
\end{align}
\end{corollary}

\begin{proof}
    See App. \ref{app:proof_average_convergence}.
\end{proof}

\subsection{Coin EM}
\label{sec:coin-results}
In this section, we outline how to establish convergence of the population-limit of an `ideal' version of the Coin EM algorithm (Alg. \ref{alg:coinEM-ideal}). In particular, for this result, we will consider a version of the Coin EM algorithm in which the outcomes observed by the two gamblers are $\smash{c_t^{(1)} = -\nabla_{\theta}\mathcal{F}(\theta_t,\mu_t)}$ and $\smash{c_t^{(2)} = -\nabla_{W_2}\mathcal{F}(\theta_t,\mu_t)}$, rather than $\smash{c_t^{(1)} = -\nabla_{\theta}\mathcal{F}(\theta_t,\mu_t)}$ and $\smash{c_t^{(2)} = -\mathcal{P}_{\mu_t}\nabla_{W_2}\mathcal{F}(\theta_t,\mu_t)}$ as in the original Coin EM algorithm (Alg. \ref{alg:coinEM}). 

\begin{algorithm*}[!htbp] %
	\caption{Coin EM (Ideal Algorithm, Population Limit)}
	\label{alg:coinEM-ideal}
	\begin{algorithmic}[1]
		\STATE {\bf Input:} number of iterations $T$, number of particles $N$, initial particle $z_0 \sim \mu_0$, initial $\theta_0$, target $\pi$, kernel $k$.
		\FOR{$t =1,\dots,T$}
            \STATE{
        \begin{align}
        \theta_{t} &= \theta_0 - \frac{{\sum_{s=1}^{t-1} \nabla_{\theta}\mathcal{F}(\theta_s,\mu_s)}}{t}  \big( 1 -  \sum_{s=1}^{t-1} \langle \nabla_{\theta}\mathcal{F}(\theta_s,\mu_s), \theta_s  - \theta_0 \rangle \big), \label{eq:coin-svgd1-v2} \\[2mm]
       z_{t} &= z_0 -\frac{{\sum_{s=1}^{t-1} \mathcal{P}_{\mu_s}\nabla_{W_2}\mathcal{F}(\theta_s,\mu_s)(z_s)}}{t} \big( 1 -  \sum_{s=1}^{t-1} \langle \mathcal{P}_{\mu_s}\nabla_{W_2}\mathcal{F}(\theta_s,\mu_s)(z_s), z_s - z_0\rangle \big), \label{eq:coin-svgd2-v2}
        \end{align}
        }
		\ENDFOR
		\STATE {\bf return} $\theta_{T}$ and $z_T$. 
	\end{algorithmic}
\end{algorithm*}
In order to establish the convergence of this algorithm, we will require an additional technical assumption on the transport maps defined by \eqref{eq:coin-svgd2-v2} in Alg. \ref{alg:coinEM-ideal}. We emphasize that it remains a challenging open question how to verify this assumption in practice; see also the discussion in \cite{Sharrock2023}. We leave this question to future work.
\begin{assumption}
    \label{coin-em-assumption-3}
There exists a transport map $T_{\mu_0}^{p_{\theta^{*}}(\cdot|x)}:\mathbb{R}^{d_z}\rightarrow\mathbb{R}^{d_z}$ from $\mu_0$ to $p_{\theta^{*}}(\cdot|x)$, where $\theta^{*} = \argmax_{\theta\in\Theta}p_{\theta}(x)$, and a constant $K>0$, such that
\begin{align}
\sum_{t=1}^T \left|\left|t_{\mu_t}^{p_{\theta^{*}}(\cdot|x)}\circ z_t - T_{\mu_0}^{p_{\theta^{*}}(\cdot|x)}\right|\right|_{L^2(\mu_0)} \leq K\sqrt{T}\ln\left[\mathrm{poly}(T)\right], \label{K1}
\end{align}
where, for each $t\in[T]$, $t_{\mu_t}^{p_{\theta^{*}}(\cdot|x)}:\mathbb{R}^d\rightarrow\mathbb{R}^d$ denotes the optimal transport map from $\mu_t$ to $p_{\theta^{*}}(\cdot|x)$ \cite[e.g.,][]{Brenier1991}, and $z_t:\mathbb{R}^d\rightarrow\mathbb{R}^d$ denotes the transport map from $\mu_0$ to $\mu_t$ as defined by \eqref{eq:coin-svgd2-v2} in Alg. \ref{alg:coinEM-ideal}.
\end{assumption}
This assumption represents an extension of Assumption B.1' in \cite{Sharrock2023}.  It can, in some sense, be interpreted as a bound on the sum of the distances between some fixed transport map from $\smash{\mu_0}$ to $\smash{p_{\theta^{*}}(\cdot|x)}$ - e.g., the optimal transport map $\smash{t_{\mu_0}^{p_{\theta^{*}}(\cdot|x)}}$ - and the maps $\smash{(t_{\mu_t}^{\pi}\circ z_t)_{t\in[T]}}$ from $\smash{\mu_0}$ to $\smash{p_{\theta^{*}}(\cdot|x)}$, which first transport $\mu_0$ to $(\mu_t)_{t\in[T]}$ according to the transport maps $(z_t)_{t\in[T]}$ defined by Alg. \ref{alg:coinEM-ideal}, and then map $(\mu_t)_{t\in[T]}$ to $\pi$ via the optimal transport maps $\smash{(t_{\mu_t}^{p_{\theta^{*}}(\cdot|x)})_{t\in[T]}}$. 

We are now ready to provide a full statement of Theorem \ref{prop:coin-em}.
\begin{customthm}{2}
\label{prop:coin-em-v2}
    Assume that Assumptions \ref{coin-em-assumption-1}, \ref{coin-em-assumption-2}, and \ref{coin-em-assumption-3} hold. Then the marginal likelihood $\theta\mapsto p_{\theta}(x)$ admits a unique maximizer $\theta^{*}=\argmax_{\theta\in\Theta}p_{\theta}(x)$. In addition, writing $\mu^{*} = p_{\theta^{*}}(\cdot|x)$, it holds for all $T\geq 1$ that
    \begin{align}
        \mathcal{F}\left(\frac{1}{T}\sum_{t=1}^T\theta_t,\frac{1}{T}\sum_{t=1}^T \mu_t\right) - \mathcal{F}(\theta^{*},\mu^{*}) \leq \frac{L}{T} &\bigg[ 2 + K\sqrt{T}\ln\left[\mathrm{poly}(T)\right] + ||\theta^{*}-\theta_0|| \sqrt{T\ln \left(1+24T^2 ||\theta^{*} - \theta_0||^2\right)} \\
        &~~~+ \int_{\mathbb{R}^d} \big|\big|T_{\mu_0}^{p_{\theta^{*}}(\cdot|x)}(z) - z\big|\big| \sqrt{T \ln \big(1+ 24T^2\big|\big|T_{\mu_0}^{p_{\theta^{*}}(\cdot|x)}(z) - z\big|\big|^2\big)} \mu_0(z)\mathrm{d}z \bigg]. \nonumber
    \end{align}
\end{customthm}

\section{PROOFS OF THEORETICAL RESULTS}
\label{app:proofs}

\subsection{Proof of Propositions \ref{prop:dissipation}, \ref{prop:grad_convergence}, and \ref{prop:exponential_convergence}.}
\label{app:proof_exponential_convergence}

\begin{proof}[Proof of Proposition \ref{prop:dissipation}]
Using differential calculus in the product space $\Theta\times\mathcal{P}_2(\mathcal{Z})$, and the chain rule, we have that 
\begin{align}
    \frac{\mathrm{d}\mathcal{F}(\theta_t,\mu_t)}{\mathrm{d}t} &= \langle -\nabla_{\theta}\mathcal{F}(\theta_t,\mu_t), \nabla_{\theta}\mathcal{F}(\theta_t,\mu_t)\rangle_{\mathbb{R}^{d_{\theta}}} + \langle P_{\mu_t} \nabla_{W_2}\mathcal{F}(\theta_t,\mu_t), \nabla_{W_2}\mathcal{F} (\theta_t,\mu_t)\rangle_{L^2(\mu_t)} \nonumber \\
    &= -\big|\big|\nabla_{\theta}\mathcal{F}(\theta_t,\mu_t)\big|\big|_{\mathbb{R}^{d_\theta}}^2 - \big|\big|S_{\mu_t}\nabla_{W_2}\mathcal{F}(\theta_t,\mu_t)\big|\big|_{\mathcal{H}_k^{d_z}}^2 \label{eq:dissipation_proof_1}
\end{align}
where in the second line we have used the fact that, given $\mu\in\mathcal{P}_2(\mathcal{Z})$, and functions $\smash{f,g\in L^2(\mu),\mathcal{H}_k^{d_z}}$, it holds that $\smash{\langle f,\iota g\rangle_{L^2(\mu)} = \langle \iota^{*}f, g\rangle_{\mathcal{H}_k^{d_z}} = \langle S_{\mu} f, g\rangle_{\mathcal{H}_k^{d_z}}}$, since the adjoint of the inclusion $\iota:\mathcal{H}\rightarrow L^2(\mu)$ is $\iota^{*}=S_{\mu}$. To obtain the first term on the RHS of \eqref{eq:dissipation}, we can now just substitute the expression for $\nabla_{\theta}\mathcal{F}(\theta_t,\mu_t)$ from \eqref{eq:svgd-em-continuous-1} into the first term on the RHS of \eqref{eq:dissipation_proof_1}. For the remaining term, recalling the definition of $\pi_{\theta}(z)$, namely, $\pi_{\theta}(z) := p_{\theta}(z,x) = p_{\theta}(z|x)p_{\theta}(x)$, we have
    \begin{equation}
        \nabla_{z}\log\left(\frac{\mu_t}{\pi_{\theta_t}}\right) = \nabla_{z}\log\left(\frac{\mu_t}{p_{\theta_t}(.|x)p_{\theta_t}(x)}\right) = \nabla_{z} \log \left(\frac{\mu_t}{p_{\theta_t}(\cdot|x)}\right)
    \end{equation}
    where the final equality follows from the fact that $p_{\theta_t}(x)$ is independent of $z$. Substituting this into \eqref{eq:dissipation_proof_1} completes the proof.
\end{proof}

\begin{proof}[Proof of Proposition \ref{prop:grad_convergence}] 
We use similar arguments to those in the proofs of \cite[][Theorem 3]{Kuntz2023} and \cite[][Proposition 8]{Korba2020}, adapted appropriately to our setting. For notational convenience, let us define
\begin{equation}
I(\theta_t,\mu_t) := \left|\left|\nabla_{\theta}\mathcal{F}(\theta_t,\mu_t)\right|\right|_{\mathbb{R}^{d_{\theta}}}^2+\left|\left|\mathcal{S}_{\mu_t}\nabla_{W_2}\mathcal{F}(\theta_{t+1},\mu_t)\right|\right|^2_{\mathcal{H}_k^{d_z}}.
\end{equation}
We start by establishing that, under the stated assumptions, there exists a positive constant $\beta>0$ which guarantees that 
\begin{equation}
    \left|\frac{\mathrm{d}I(\theta_t,\mu_t)}{\mathrm{d}t}\right|\leq \beta I(\theta_t,\mu_t). \label{eq:I_dissipation}
\end{equation}
We will then show that this implies convergence of $I(\theta_t,\mu_t)$ to zero. To prove \eqref{eq:I_dissipation}, we first compute
\begin{equation}
    \frac{\mathrm{d}I(\theta_t,\mu_t)}{\mathrm{d}t} = \underbrace{\frac{\mathrm{d}}{\mathrm{d}t}\left|\left|\nabla_{\theta}\mathcal{F}(\theta_t,\mu_t)\right|\right|_{\mathbb{R}^{d_{\theta}}}^2}_{I_t^{(1)}(\theta_t,\mu_t)}+\underbrace{\frac{\mathrm{d}}{\mathrm{d}t}\left|\left|\mathcal{S}_{\mu_t}\nabla_{W_2}\mathcal{F}(\theta_{t},\mu_t)\right|\right|^2_{\mathcal{H}_k^{d_z}}}_{I_t^{(2)}(\theta_t,\mu_t)} \label{eq:i_first_step}
\end{equation}
We begin with $\smash{I_t^{(1)}(\theta_t,\mu_t)}$. Let us define $v_t = \nabla_{\theta}\mathcal{F}(\theta_t,\mu_t) = -\int \nabla_{\theta}\log \pi_{\theta_t}(z)\mu_t(z)\mathrm{d}z$ and $w_t = S_{\mu_t}\nabla_{W_2}\mathcal{F}(\theta_t,\mu_t) = S_{\mu_t}\nabla \log (\frac{\mu_t}{\pi_{\theta_t}})$.
We then have 
\begin{align}
    I_t^{(1)}(\theta_t,\mu_t) &= \frac{\mathrm{d}}{\mathrm{d}t}\left|\left|v_t\right|\right|_{\mathbb{R}^{d_{\theta}}}^2=2 \big\langle v_t, \frac{\mathrm{d}}{\mathrm{d}t}v_t\big\rangle_{\mathbb{R}^{d_{\theta}}}.
\end{align}
We now need to obtain $\frac{\mathrm{d}}{\mathrm{d}t}v_t$. Using the chain rule, and then integration by parts, we have
\begin{align}
 \frac{\mathrm{d}}{\mathrm{d}t}v_t&= -\int \frac{\mathrm{d}}{\mathrm{d}t}\left[\nabla_{\theta} \log \pi_{\theta_t}(z)\mu_t(z)\right]\mathrm{d}z \\
 &=-\int \frac{\mathrm{d}}{\mathrm{d}t}\left[\nabla_{\theta} \log \pi_{\theta_t}(z)\right] \mu_t(z)\mathrm{d}z - \int \nabla_{\theta} \log \pi_{\theta_t}(z) \frac{\partial\mu_t}{\partial t}(z)\mathrm{d}z \\
 &=-\int \nabla^2_{\theta} \log \pi_{\theta_t}(z) \mu_t(z) \frac{\partial\theta_t}{\partial t} \mathrm{d}z  - \int \nabla_{\theta} \log \pi_{\theta_t}(z) \nabla_{z} \cdot (w_t(z)\mu_t(z))\mathrm{d}z  \\
&= \int \nabla^2_{\theta} \log \pi_{\theta_t}(z) ~v_t ~\mu_t(z)\mathrm{d}z  + \int  \nabla_{\theta}\nabla_{z} \log \pi_{\theta_t}(z) ~w_t(z)\mu_t(z)\mathrm{d}z.
\end{align}
It follows, in particular, that 
\begin{align}
    I_t^{(1)}(\theta_t,\mu_t) = &-2 \sum_{i=1}^{d_{\theta}}\sum_{j=1}^{d_{\theta}}  \int v_t^{i} \left[\nabla^2_{\theta} \log \pi_{\theta_t}(z)\right]_{ij}~v_t^{j}~ \mu_t(z)\mathrm{d}z + 2 \sum_{i=1}^{d_{\theta}} \sum_{j=1}^{d_{z}} \int v_t^{i} \left[\nabla_{\theta}\nabla_{z} \log \pi_{\theta_t}(z)\right]_{ij} w_t^{j}(z)\mu_t(z)\mathrm{d}z. \label{eq:i1}
\end{align} 
For the second term, we will work component-wise. First, using the notation $w_t = (w_t^{1},\dots,w_t^{d_z})$, we have $||w_t||_{\mathcal{H}_k}^2= \sum_{j=1}^{d_z}||w_t^{j}||_{\mathcal{H}_k}^2$. Thus, in particular, $I_t^{(2)}(\theta_t,\mu_t) = \frac{\mathrm{d}}{\mathrm{d}t}\sum_{j=1}^{d_z}||w_t^{j}||_{\mathcal{H}_k}^2 = 2\sum_{j=1}^{d_z} \langle w_t^{j}, \frac{\mathrm{d}}{\mathrm{d}t}w_t^{j}\rangle_{\mathcal{H}_k}$. It remains to compute the components $\frac{\mathrm{d}}{\mathrm{d}t}w_t^{j}$, for $j\in[d_z]$. Starting from the definition, using integration by parts, the chain rule, and then integration by parts again, we have
\begin{align}
    \frac{\mathrm{d}}{\mathrm{d}t}w_t^{j}(z) &= \frac{\mathrm{d}}{\mathrm{d}t} \int  k(z',z) \partial_{z'_i} \log \left(\frac{\mu_t}{\pi_{\theta_t}}\right)(z') \mu_t(z')\mathrm{d}z'\\
    &= - \int  \frac{\mathrm{d}}{\mathrm{d}t} \left[ \left[ 
    \partial_{z'_j} \log \pi_{\theta_t}(z')k(z',z)  + \partial_{z'_{j}} k(z',z)\right]\mu_t(z')\right]\mathrm{d}z' \\
    &= -  \int  \big\langle \nabla_{\theta} \partial_{z'_j}  \log \pi_{\theta_t}(z') k(z',z) ,\frac{\partial\theta_t}{\partial t} \big\rangle_{\mathbb{R}^{d_{\theta}}} \mu_t(z')\mathrm{d}z' \nonumber \\
    &~~~~~~- \int \left[ \partial_{z'_j} \log \pi_{\theta_t}(z') k(z',z) + \partial_{z'_{j}} k(z',z) \right]\frac{\partial}{\partial t}\mu_t(z')\mathrm{d}z' 
    \\
    & =  \int \big\langle \nabla_{\theta} \partial_{z'_j} \log \pi_{\theta_t}(z') k(z',z),v_t \big\rangle_{\mathbb{R}^{d_{\theta}}} \mu_t(z')\mathrm{d}z' \nonumber \\
    &~~~~~~+ \int \big\langle \nabla_{z'}\left[ \partial_{z'_j}  \log \pi_{\theta_t}(z') k(z',z) + \partial_{z'_{j}} k(z',z) \right], w_t(z')\big\rangle_{\mathbb{R}^{d_z}}\mu_t(z')\mathrm{d}z' \\
    & =  \int \big\langle \nabla_{\theta} \partial_{z'_j} \log \pi_{\theta_t}(z') k(z',z),v_t \big\rangle_{\mathbb{R}^{d_{\theta}}} \mu_t(z')\mathrm{d}z' \nonumber \\
    &~~~~~~+ \int \sum_{i=1}^{d_z} \left[  \partial_{z'_i}\partial_{z'_j} \log \pi_{\theta_t}(z') k(z',z) +  \partial_{z'_j}\log \pi_{\theta_t}(z')  \partial_{z'_i} k(z',z) \right. \nonumber \\
    &\hspace{50mm}\left. + \partial_{z'_i}\partial_{z'_{j}} k(z',z) \right]  w_t^{j}(z')\mu_t(z')\mathrm{d}z'. 
\end{align}
Now, using the fact that each component of $w_t^{j}$ is an elements of the RKHS $\mathcal{H}_k$, and thus satisfies the reproducing property $w_t^{j}(z) = \langle w_t^{j}, k(z,\cdot)\rangle_{\mathcal{H}_k}$, we can rewrite the previous display as
\begin{align}
    I_t^{(2)}(\theta_t,\mu_t) & = 2\sum_{i=1}^{d_{\theta}} 
    \sum_{j=1}^{d_{z}} \int v_t^{i} \left[\nabla_{\theta}\nabla_{z} \log \pi_{\theta_t}(z)\right]_{ij} w_t^{j}(z) \mu_t(z)\mathrm{d}z \nonumber \\
    &+ 2\sum_{i=1}^d \sum_{j=1}^{d_z} \int \left[  \partial_{z_i}\partial_{z_j}\log \pi_{\theta_t}(z) w_t^{j}(z) +  \partial_{z_i} \log \pi_{\theta_t}(z) \partial_{z_j} w_t^{j}(z)  +  \partial_{z_j}\partial_{z_{i}} w_t^{j}(z) \right] w_t^{j}(z)\mu_t(z)\mathrm{d}z. \label{eq:i2}
\end{align}
Substituting \eqref{eq:i1} and \eqref{eq:i2} into \eqref{eq:i_first_step}, before once more making use of the reproducing property, we thus have 
\begin{align}
    \frac{\mathrm{d}I(\theta_t,\mu_t)}{\mathrm{d}t} &=2 \sum_{i=1}^{d_{\theta}}\sum_{j=1}^{d_{\theta}} \int v_t^{i} \left[-\nabla^2_{\theta} \log \pi_{\theta_t}(z)\right]_{ij}~v_t^{j}~ \mu_t(z)\mathrm{d}z  \nonumber \\
    &+ 2 \sum_{i=1}^{d_{\theta}} \sum_{j=1}^{d_{z}} \int v_t^{i} \left[\nabla_{\theta}\nabla_{z} \log \pi_{\theta_t}(z)\right]_{ij} w_t^{j}(z)\mu_t(z)\mathrm{d}z \nonumber \\
    &+ 2\sum_{i=1}^{d_z} \sum_{j=1}^{d_z} \int \left[  \partial_{z_i}\partial_{z_j}\log \pi_{\theta_t}(z) w_t^{j}(z) +  \partial_{z_i} \log \pi_{\theta_t}(z) \partial_{z_j} w_t^{j}(z)  + \partial_{z_j}\partial_{z_{i}} w_t^{j}(z) \right] w_t^{j}(z)\mu_t(z)\mathrm{d}z \\
    &= \sum_{i=1}^{d_{\theta}}\sum_{j=1}^{d_{\theta}} v_t^{i} A_t^{ij} v_t^{j} + \sum_{i=1}^{d_{\theta}} \sum_{j=1}^{d_{z}} v_t^{i} \langle B_t^{ij}, w_t^{j}  \rangle_{\mathcal{H}_k} + \sum_{i=1}^{d_{z}} \sum_{j=1}^{d_z} \langle w_t^{i}, C_t^{ij} w_t^{j}\rangle_{\mathcal{H}_k},
\end{align}
where in the final line we have defined 
\begin{align}
    A_{t}^{ij} &= 2 \int [-\partial_{\theta_i}\partial_{\theta_j} \log \pi_{\theta_t}(z)]\mu_t(\mathrm{d}z) \\
    B_{t}^{ij} &= 2 \int k(z,\cdot) \partial_{\theta_i}\partial_{z_{i}} \log \pi_{\theta_t}(z)\mu_t(z)\mathrm{d}z \\
    C_t^{ij} &= 2\int k(z,\cdot) \otimes k(z',\cdot)\partial_{z_{i}}\partial_{z_j} \log \pi(z) \mu_t(z)\mu_t(z')\mathrm{d}z\mathrm{d}z' \nonumber \\
    &~~~~+2\int \partial_{z_i} k(z,\cdot) \otimes k(z',\cdot) \partial_{z_i} \log \pi_{\theta_t}(z) \mu_t(z)\mathrm{d}z\mu_t(z')\mathrm{d}z'\nonumber  \\
    &~~~~+ 2\int \partial_{z_i} k(z,\cdot) \otimes \partial_{z_j} k(z',\cdot) \mu_t(z)\mu_t(z')\mathrm{d}z\mathrm{d}z'.
\end{align}
We now need to bound each of these terms in an appropriate sense. In particular, if we can show that $||A_t^{ij}||_{\mathrm{op}}\leq D_1$, $||B_t^{ij}||_{\mathcal{H}_k}\leq D_2$, and $||C_t^{ij}||_{\mathrm{HS}}\leq D_3$, for all $t\geq 0$, where $||\cdot||_{\mathrm{op}}$ denotes the operator norm, and $||\cdot||_{\mathrm{HS}}$ denotes the Hilbert-Schmidt norm, then it follows immediately that 
\begin{align}
    \left|\frac{\mathrm{d}I(\theta_t,\mu_t)}{\mathrm{d}t}\right| &\leq D_1 \sum_{i=1}^{d_{\theta}} ||v_t^{i}||^2 +  D_2 d_{\theta}^{\frac{1}{2}}d_{z}^{\frac{1}{2}} \big(\sum_{i=1}^{d_{\theta}}||v_t^{i}||^2 \big)^{\frac{1}{2}}
    \big(\sum_{i=1}^{d_{z}}||w_t^{j}||^2_{\mathcal{H}_k}\big)^{\frac{1}{2}} + D_3 d_z\sum_{i=1}^{d_z}||w_t^{i}||_{\mathcal{H}_k}^2 \nonumber \\
    &\leq D \left( ||v_t||^2_{\mathbb{R}^{d_{\theta}}} + 2||v_t||_{\mathbb{R}^{d_{\theta}}} ||w_t||_{\mathcal{H}_k^{d_z}} + ||w_t||_{\mathcal{H}_k^{d_z}}^2 \right) \leq 2D \left(||v_t||^2_{\mathbb{R}^{d_{\theta}}} + ||w_t||_{\mathcal{H}_k^{d_z}}^2 \right):= \beta I(\theta_t,\mu_t),\label{eq:54}
\end{align}
where in the first line we have used the Cauchy-Schwarz inequality, in the second line we have defined the constant $\smash{D=\max(D_1,\frac{1}{2}D_2d_{\theta}^{\frac{1}{2}}d_{z}^{\frac{1}{2}}, D_3d_z)}$, and in the final line we have used $a^2 + 2ab + b^2 = (a+b)^2 \leq 2(a^2 + b^2)$, and set $\beta=2D$. It remains to obtain the bounds on $A_t^{ij}$, $B_t^{ij}$, and $C_t^{ij}$.

For the first term, we just require the boundedness of the Hessian $H_{V_{z}}:\mathbb{R}^{d_{\theta}}\rightarrow\mathbb{R}^{d_{\theta}}$ of the function $\smash{V_z:\theta\mapsto -\log\pi_{\theta}(z)}$, for each $z\in\mathcal{Z}$, which is guaranteed by Assumption \ref{assumption:bounded_hessian0}. In particular, by Assumption \ref{assumption:bounded_hessian0}, we certainly have that 
$\smash{||H_{V_z}||_{\mathrm{op}}\leq M}$, 
from which it immediately follows that 
$||A_{t}^{ij}||_{\mathrm{op}} \leq 2 M$.
For the second term, once again using Assumption \ref{assumption:bounded_hessian0}, 
as well as the bound of the kernel (Assumption \ref{assumption:bounded_k}), we have 
\begin{equation}
    ||B_t^{ij}||_{\mathcal{H}_k} \leq 2 \int ||k(z,\cdot)||_{\mathcal{H}_k} |\partial_{\theta_i}\partial{z_i}\log \pi_{\theta_t}(z)|\mu_t(z)\mathrm{d}z\leq BM. \label{eq:B_bound}
\end{equation}
For the final term, we argue similarly to in \cite[][Proposition 8]{Korba2020}. In particular, using the boundedness of the Hessian $H_{V_{\theta}}:\mathbb{R}^{d_z}\rightarrow\mathbb{R}^{d_z}$ of the function $\smash{V_{\theta}: z\mapsto-\log \pi_{\theta}(z)}$ for each $\theta\in\Theta$, which is an immediate consequence of Assumption \ref{assumption:bounded_hessian0}, and the boundedness of the kernel and its gradient (Assumption \ref{assumption:bounded_k}), we have 
\begin{align}
    ||C_{t}^{ij}||_{\mathrm{HS}} &\leq 2 \int ||k(z,\cdot)||_{\mathcal{H}_k} |\partial_{z_i}\partial_{z_j}\log \pi_{\theta_t}(z)| \mu_t(z)\mathrm{d}z ||k(z',\cdot)||_{\mathcal{H}_k}\mu_t(z')\mathrm{d}z' \nonumber \\
    &~~~~+ 2 \int ||\partial_{z_i}k(z,\cdot)||_{\mathcal{H}_k}|\partial_{z_i} \log \pi_{\theta_t}(z)|\mu_t(z)\mathrm{d}z \int ||k(z',\cdot)||_{\mathcal{H}_k} \mu_t(z')\mathrm{d}z' \nonumber \\
    &~~~~+2  \left(\int ||k(z,\cdot)||_{\mathcal{H}_k}\mu_t(z)\mathrm{d}z\right)^2 \leq 2B^2 \left(M + \int\left|\partial_{z_i}\log \pi_{\theta_t}(z)\right|\mu_t(z)\mathrm{d}z + 1\right). \label{eq:C_bound}
\end{align}
To bound the remaining integral, we use the fact that $|\log \pi_{\theta}(z)|\leq |\partial_{z_i}\log \pi_{\theta}(0)| + M||z||$, which follows from Assumption \ref{assumption:bounded_hessian0}. 
It follows that
\begin{equation}
\int |\partial_{z_i}\log \pi_{\theta_t}(z)|\mu_t(\mathrm{d}z) \leq \int \left[|\partial_{z_i}\log \pi_{\theta}(0)|+M||z|| \right] \mu_t(z)\mathrm{d}z \leq |\partial_{z_i}\log \pi_{\theta}(0)| + MC. \label{eq:grad_z_bound}
\end{equation}
Substituting \eqref{eq:grad_z_bound} into \eqref{eq:C_bound}, we thus have that $||C_t^{ij}||_{\mathrm{HS}} \leq 2B^2(1+ M + CM + |\partial_{z_i}\log \pi_{\theta}(0)|)$. Thus, setting $D_1 = 2M$, $D_2 = BM$, $D_3 =2B^2(1+ M + CM + |\partial_{z_i}\log \pi_{\theta}(0)|)$, 
and using the argument in \eqref{eq:54}, we have established the result in \eqref{eq:I_dissipation}.

It remains to establish that $I(\theta_t,\mu_t)\rightarrow 0$. Once more, we will utilize some of the ideas from the proofs of \cite[][Theorem 3]{Kuntz2023} and \cite[][Proposition 8]{Korba2020}. We begin by observing that the free energy can be written as  
\begin{align}
    \mathcal{F}(\theta,\mu) &= \int \nabla_{z}\log\left(\frac{\mu}{\pi_{\theta}}\right) \mu(z)\mathrm{d}z= \int  \log\left(\frac{\mu}{p_{\theta}(.|x)p_{\theta}(x)}\right)\mu(z)\mathrm{d}z = \underbrace{\int \log \left(\frac{\mu_t}{p_{\theta_t}(\cdot|x)}\right)\mu(z)\mathrm{d}z}_{\mathrm{KL}(\mu_t|p_{\theta_t}(\cdot|x))} - \log p_{\theta}(x). \label{eq:free_energy_KL}
\end{align}
Rearranging \eqref{eq:free_energy_KL}, using the non-negativity of the KL divergence, and finally the fact that the free energy dissipates along the flow of the SVGD EM dynamics (Proposition \ref{prop:dissipation}), it then follows that
\begin{equation}
    \log p_{\theta_t}(x) = - \mathcal{F}(\theta_t,\mu_t) + \mathrm{KL}(\mu_t|p_{\theta_t}(\cdot|x)) \geq \mathcal{F}(\theta_t,\mu_t) \geq \mathcal{F}(\theta_0,\mu_0). \label{eq:sub_level}
\end{equation}
Under Assumption \ref{assumption:bounded_sets}, the set $\Theta_{\mathcal{F}_0} = \{\theta\in\Theta:\log p_{\theta}(x)\geq \mathcal{F}(\theta_0,\mu_0)\}$ is compact. This, along with \eqref{eq:sub_level}, immediately implies that $(\theta_t)_{t\geq 0}$ is relatively compact. We now return to \eqref{eq:free_energy_KL}. In particular, once more rearranging this equation, we have that, for each $t\geq 0$, 
\begin{align}
    \inf_{\theta\in\Theta_{\mathcal{F}_0}} \mathrm{KL}(\mu_t|p_{\theta}(\cdot|x)) &\leq \mathrm{KL}(\mu_t|p_{\theta_t}(\cdot|x)) = \mathcal{F}(\theta_t,\mu_t) + \log p_{\theta_t}(x) \\[-2mm]
    &\leq \mathcal{F}(\theta_0,\mu_0) + \log p_{\theta_t}(x) \leq \mathcal{F}(\theta_0,\mu_0) + \sup_{\theta\in\Theta_{\mathcal{F}_0}} \log p_{\theta}(x).
\end{align}
It thus follows, using also the fact that the KL divergence has compact sub-level sets in the weak topology \cite[][Theorem 20]{Erven2014}, that the family $(\mu_t)_{t\geq 0}$ is weakly relatively compact. Based on these two observations, and also the (weak) continuity of $I(\theta,\mu)$, we can conclude that $\smash{\sup_{t\geq 0} I(\theta_t,\mu_t)<\infty}$. Thus, using the bound \eqref{eq:I_dissipation} that we established earlier in the proof, we have that $\smash{|\frac{\mathrm{d}}{\mathrm{d}t} I(\theta_t,\mu_t)|\leq K}$ for some $K>0$. 

Finally, we are ready to show that $I(\theta_t,\mu_t)\rightarrow 0$. We follow closely the final part of the proof of \cite[][Proposition 8]{Korba2020}, which we can now apply almost verbatim. For the interested reader, we also provide the details here. We will argue by contradiction. In particular, suppose that $I(\theta_t,\mu_t)$, so that there exists a sequence $t_m\rightarrow\infty$ such that $I(\theta_{t_m},\mu_m)>\varepsilon>0$. In addition, since $|\frac{\mathrm{d}}{\mathrm{d}t} I(\theta_t,\theta_t)|$ is bounded, it is uniformly $K$ Lipschitz in time. Thus, there exists a sequence of intervals $J_m$ of length $\frac{m}{K}$, and centered at $t_m$, such that $I(\theta_t,\mu_t)\geq \frac{\varepsilon}{2}$ for all $t\in J_k$. Finally, integrating the dissipation over $s\in[0,t]$, we have 
\begin{align}
    \mathcal{F}(\theta_0,\mu_0) - \min_{(\theta,\mu)\in\Theta\times\mathcal{P}_2(\mathcal{Z})}\mathcal{F}(\theta,\mu) &\geq \mathcal{F}(\theta_0,\mu_0) - \mathcal{F}(\mu_t,\theta_t) = \int_0^t I(\theta_s,\mu_s)\mathrm{d}s \geq \sum_{m:t_m\leq t}\frac{\varepsilon^2}{2K},
\end{align}
which diverges as $t\rightarrow\infty$ since the subsequence $t_m\rightarrow\infty$. But this is contradiction, since $\mathcal{F}(\theta_0,\mu_0)<\infty$. This completes the proof.

\end{proof}

\begin{proof}[Proof of Proposition \ref{prop:exponential_convergence}]
    Suppose we write $\mathcal{F}^{*} = \min_{(\theta,\mu) \in \theta\times\mathcal{P}_2(\mathcal{Z})} \mathcal{F}(\theta,\mu)$. Proceeding almost identically to the start of the proof of Proposition \ref{prop:dissipation}, we have
    \begin{align}
        &\frac{\mathrm{d}}{\mathrm{d}t}\big(\mathcal{F}(\theta_t,\mu_t) - \mathcal{F}^{*}\big) \\
        &~~= \left\langle -\nabla_{\theta}\mathcal{F}(\theta_t,\mu_t), \nabla_{\theta}\mathcal{F}(\theta_t,\mu_t)\right\rangle_{\mathbb{R}^{d_{\theta}}} + \left\langle -P_{\mu_t} \nabla_{W_2}\mathcal{F}(\theta_t,\mu_t), \nabla_{W_2}\mathcal{F}(\theta_t,\mu_t)\right\rangle_{L^2(\mu_t)} \\
        &~~= -\left|\left|\nabla_{\theta}\mathcal{F}(\theta_t,\mu_t)\right|\right|_{\mathbb{R}^{d_{\theta}}}^2 - \left|\left|S_{\mu_t} \nabla_{W_2}\mathcal{F}(\theta_t,\mu_t)\right|\right|_{\mathcal{H}_k^{d_z}}^2 \\[1mm]
        &~~\leq -2\lambda \left(\mathcal{F}(\theta_t,\mu_t) - \mathcal{F}^{*}\right). 
    \end{align}
    where in the final line we have used the gradient dominance condition (Assumption \ref{assumption:grad_dominance}). The conclusion now follows straightforwardly via Gr\"omwall's inequality.
\end{proof}

\subsection{Proof of Theorem \ref{prop:descent_lemma}.}
\label{app:proof_descent_lemma}

\begin{proof}
    For fixed $t\in\mathbb{N}_{0}$, consider the following decomposition
    \begin{equation}
        \mathcal{F}(\theta_{t+1},\mu_{t+1}) - \mathcal{F}(\theta_t,\mu_t) =  \underbrace{\mathcal{F}(\theta_{t+1},\mu_{t+1}) - \mathcal{F}(\theta_{t+1},\mu_t)}_{\mathrm{I}} + \underbrace{\mathcal{F}(\theta_{t+1},\mu_{t}) - \mathcal{F}(\theta_t,\mu_{t})}_{\mathrm{II}}.
    \end{equation}
    We will deal with $\mathrm{(I)}$ and $\mathrm{(II)}$ in turn, beginning with (II). First, for each $\tau\geq 0$, and for fixed $t\in\mathbb{N}_{0}$, let $\theta_{\tau} = \theta_{t} - \tau \nabla_{\theta} \mathcal{F}(\theta_t,\mu_{t})$. We then have $\theta_0 = \theta_t$ and $\theta_{\gamma} = \theta_{t} - \gamma\nabla_{\theta}\mathcal{F}(\theta_t,\mu_t) = \theta_{t+1}$. We also have that $\smash{\dot{\theta}_{\tau} = -\nabla_{\theta}\mathcal{F}(\theta_t,\mu_{t+1})}$ and $\smash{\ddot{\theta}_{\tau} = 0}$, where here we emphasize that $\cdot$ and $\cdot\cdot$ denote derivatives with respect to the continuous $\tau\in\mathbb{R}_{+}$, rather than the discrete $t\in\mathbb{N}_{0}$. 
    
    In addition, suppose we let $g(\tau) = \mathcal{F}(\theta_{\tau},\mu_{t})$. We then have $g(0) = \mathcal{F}(\theta_t,\mu_{t})$ and $g(\gamma)= \mathcal{F}(\theta_{t+1},\mu_{t})$. In addition, using a Taylor expansion, we have that 
    \begin{equation}
        g(\gamma) = g(0) + \gamma g'(0) + \int_0^{\gamma} (\gamma - \tau) g''(\tau)\mathrm{d}\tau. \label{eq:g_taylor}
    \end{equation}
    Let us identify the remaining terms in this expansion. Using the chain rule, and in the second line also the fact that $\ddot{\theta}_{\tau}=0$, we have that 
    \begin{align}
        g'(\tau) &= \langle \dot{\theta}_{\tau}, \nabla_{\theta} \mathcal{F}(\theta_{\tau},\mu_{t})\rangle_{\mathbb{R}^{d_{\theta}}} = - \left\langle \nabla_{\theta}\mathcal{F}(\theta_t,\mu_{t}), \nabla_{\theta} \mathcal{F}(\theta_{\tau},\mu_{t}) \right\rangle_{\mathbb{R}^{d_{\theta}}} \label{eq:g'} \\
        g''(\tau)&= \langle \dot{\theta}_{\tau}, \nabla_{\theta}^2 \mathcal{F}(\theta_{\tau},\mu_{t})\dot{\theta}_{\tau} \rangle_{\mathbb{R}^{d_{\theta}}} = \left\langle \nabla_{\theta}\mathcal{F}(\theta_t,\mu_{t}), \mathrm{Hess}_{\theta}(\mathcal{F})(\theta_{\tau},\mu_{t}) \nabla_{\theta}\mathcal{F}(\theta_t,\mu_{t}) \right\rangle_{\mathbb{R}^{d_{\theta}}} \label{eq:g''} 
    \end{align}
    where $\mathrm{Hess}_{\theta}(\mathcal{F})$ denotes the Hessian matrix of $\mathcal{F}(\cdot,\mu_{t})$. Thus, $g'(0) =- ||\nabla_{\theta}\mathcal{F}(\theta_t,\mu_{t})||_{\mathbb{R}^{d_{\theta}}}^2$, and 
    $g''(\tau)\leq M ||\nabla_{\theta}\mathcal{F}(\theta_t,\mu_{t})||_{\mathbb{R}^{d_{\theta}}}^2$, 
    the latter by Assumption \ref{assumption:bounded_hessian0}. Putting everything together, we thus have that 
     \begin{align}
     \mathcal{F}(\theta_{t+1},\mu_{t}) &\leq  \mathcal{F}(\theta_{t},\mu_t) - \gamma ||\nabla_{\theta}\mathcal{F}(\theta_t,\mu_{t})||_{\mathbb{R}^{d_{\theta}}}^2 + M\int_{0}^{\gamma}(\gamma - \tau) ||\nabla_{\theta}\mathcal{F}(\theta_t,\mu_t)||_{\mathbb{R}^{d_{\theta}}}^2\mathrm{d}\tau \\
     &\leq \mathcal{F}(\theta_{t},\mu_t) - \gamma ||\nabla_{\theta}\mathcal{F}(\theta_t,\mu_{t})||_{\mathbb{R}^{d_{\theta}}}^2 + 
     \frac{M\gamma^2}{2}||\nabla_{\theta}\mathcal{F}(\theta_t,\mu_t)||_{\mathbb{R}^{d_{\theta}}}^2
     \end{align}
     which implies, in particular, that 
     \begin{align}
     \mathcal{F}(\theta_{t+1},\mu_{t+1}) - \mathcal{F}(\theta_t,\mu_t) &\leq - \gamma \left(1 - \frac{M\gamma}{2}\right)||\nabla_{\theta}\mathcal{F}(\theta_t,\mu_t)||_{\mathbb{R}^{d_{\theta}}}^2. \label{eq:theta_descent}
     \end{align}
    We will use a very similar argument to obtain an upper bound on $\mathrm{(I)}$. Here, we follow rather closely the proofs of \cite[Proposition 5]{Korba2020} and \cite[Proposition 3.1]{Salim2022}, adapted appropriately to our setting. In the interest of completeness, we provide this argument in full. 
    
    For $\tau\geq 0$, let $\psi_{\tau} = \mathrm{id} - \tau\xi$, where $\xi = \mathcal{P}_{\mu} g \in L^2(\mu)$ $\mu\in\mathcal{P}_2(\mathcal{Z})$, for some $g\in L^2(\mu)$, 
    and where $\mathrm{id}$ denotes the identity operator. In addition, let $\rho_{\tau} = (\psi_{\tau})_{\#}\mu_t$. By definition, we then have that $\rho_0 = \mu_t$ and $\rho_{\gamma} = \mu_{t+1}$. Finally, let $\xi' = S_{\mu} g$, with $\xi(z) = \xi'(z)$ for all $z\in\mathcal{Z}$.
    We begin by observing that, for each $(\theta,z)\in \Theta\times\mathcal{Z}$, the reproducing property and the Cauchy-Schwarz inequality yield the bounds
    \begin{align}
        ||\xi(z)||^2 &= \sum_{i=1}^{d_z} \langle k(z,\cdot), \xi'_i \rangle_{\mathcal{H}_k}^2 \leq ||k(z,\cdot)||^2_{\mathcal{H}_k} ||\xi'||_{\mathcal{H}_k^{d_z}} \leq B^2 ||\xi'||^2_{\mathcal{H}_k^{d_z}} \label{eq:bound1} \\
        ||J \xi(z)||^2_{\mathrm{HS}} 
        &= \sum_{i,j=1}^{d_z} \langle \partial_{z_j} k(z,\cdot), \xi'_i\rangle^2_{\mathcal{H}_k} \leq \sum_{i,j=1}^{d_z} ||\partial_{z_j}k(z,\cdot) ||^2_{\mathcal{H}_k} ||\xi'_i||^2_{\mathcal{H}_k} = ||\nabla k(z,\cdot)||^2_{\mathcal{H}_k}||\xi'||^2_{\mathcal{H}_k^{d_z}}\leq B^2 ||\xi'||^2_{\mathcal{H}_k^{d_z}}, \label{eq:bound2}
    \end{align}
    where in each case the final display follows from Assumption \ref{assumption:bounded_k}. Suppose, in addition, that for some $\alpha>1$, the step size $\gamma$ satisfies 
    \begin{equation}
        \gamma ||\xi'||_{\mathcal{H}_k^{d_z}} \leq \frac{\alpha - 1}{\alpha B}. \label{eq:step-size-condition}
    \end{equation} 
    It then follows from the previous bound that, for any $\tau\in[0,\gamma]$,  
    \begin{equation}
        ||\tau J \xi(z)||_{\mathrm{op}} \leq ||\tau J \xi(z)||_{\mathrm{HS}} \leq \gamma B ||\xi'||_{\mathcal{H}_{k}^{d_z}} \leq \frac{\alpha - 1}{\alpha}<1.
    \end{equation}
    where $J$ denotes the Jacobian matrix. This establishes that $J(\psi_{\tau})(z) = \mathrm{id} - \tau J \xi(z)$ is regular for all $z$, and that $\psi_{\tau}$ is a diffeomorphism for each $\tau\in[0,\gamma]$. In addition, we also have the bound
    \begin{equation}
        ||(J \psi_{\tau})^{-1}(z)||_{\mathrm{op}} \leq \sum_{i=0}^{\infty} ||\tau J \xi(z)||^{i}_{\mathrm{op}} \leq \sum_{i=0}^{\infty} \left(\frac{\alpha-1}{\alpha}\right)^{i} = \alpha. \label{eq:bound3}
    \end{equation}
    Now, let us define $h(\tau) = \mathcal{F}(\theta_{t+1},\rho_{\tau})$. Using the definition above, we have that $h(0) = \mathcal{F}(\theta_{t+1},\mu_t)$ and $h(\gamma) = \mathcal{F}(\theta_{t+1},\mu_{t+1})$. Now, via a Taylor expansion, 
    \begin{equation}
        h(\gamma) = h(0) + \gamma h'(0) + \int_0^{\gamma} (\gamma - \tau) h''(\tau)\mathrm{d}\tau. \label{eq:h_taylor}
    \end{equation}
    Similar to before, we need to identify the final two terms. We first recall, from \cite[Theorem 5.34]{Villani2003}, that the velocity field which rules the time evolution of $\rho_{\tau}$ is given by $w_{\tau}\in L^2(\rho_{\tau})$, defined according to $w_{\tau}(z)= -\xi (\psi_{\tau}^{-1}(z))$. Then, by an appropriate version of the chain rule \cite[e.g.][Section 8.2]{Villani2003}, one can show that 
    \begin{align}
        h'(0) &= -\langle \nabla_{W_2}\mathcal{F}(\theta_{t+1},\mu_t), \xi \rangle_{L^2(\mu_t)} = -\langle \mathcal{S}_{\mu_t}\nabla_{W_2}\mathcal{F}(\theta_{t+1},\mu_t), \xi' \rangle_{\mathcal{H}_k^{d_z}} \label{eq:h_taylor_1}\\
        h''(\tau) 
        &= \underbrace{\mathbb{E}_{z\sim \rho_{\tau}} \left[ \langle w_{\tau}(z), \mathrm{Hess}_{z}(-\log \pi_{\theta_{t+1}}(z))w_{\tau}(z)\right]}_{h_1(\tau)} + \underbrace{\mathbb{E}_{z\sim \rho_{\tau}} \left[ ||J(w_{\tau}(z))||^2_{\mathrm{HS}}\right]}_{h_2(\tau)}, \label{eq:h_taylor_2}
    \end{align}
     It remains to bound these two terms. Recalling that $\rho_{\tau} = (\psi_{\tau})_{\#}\mu_t$ and $w_{\tau}(z)= -\xi(\psi_{\tau}^{-1}(z))$, and using Assumption \ref{assumption:bounded_hessian0}, we have
    \begin{align}
        h_1(\tau) &=\mathbb{E}_{z\sim \rho_{\tau}} \left[ \langle w_{\tau}(z), \mathrm{Hess}(-\log \pi_{\theta_{t+1}}(z))w_{\tau}(z)\right] \\
        &=\mathbb{E}_{z\sim \rho_{\tau}} \left[ \langle \xi(\psi_{\tau}^{-1}(z)), \mathrm{Hess}(-\log \pi_{\theta_{t+1}}(z))\xi(\psi_{\tau}^{-1}(z))\right] \\
        &= \mathbb{E}_{z\sim\mu_t} \left[\left\langle \xi(z), \mathrm{Hess}(-\log \pi_{\theta_{t+1}}(z)) \xi(z) \right\rangle\right] \\
        &\leq M
        \mathbb{E}_{z\sim \mu_t}\left[||\xi(z)||^2\right]\leq M 
        B^2 \left|\left|\xi'\right|\right|_{\mathcal{H}_k^{d_z}}^2,  \label{eq:bound-h1}
    \end{align}
    where the final inequality follows from the  the bound in \eqref{eq:bound1}.
    For the remaining term, by the chain rule, we have $Jw_{\tau}(z) = J\xi(\psi^{-1}_{\tau}(z))(J\psi_{\tau})^{-1}(\psi_{\tau}^{-1}(z))$.
    It follows that
    \begin{align}
        h_2({\tau}) &= \mathbb{E}_{z\sim \rho_{\tau}} \left[ ||J(w_{\tau}(z))||^2_{\mathrm{HS}}\right] \\
        &=\mathbb{E}_{z\sim \rho_{\tau}} \left[ ||J\xi(\psi^{-1}_{\tau}(z))(J\psi_{\tau})^{-1}(\psi_{\tau}^{-1}(z)) ||^2_{\mathrm{HS}}\right] \\
        &=\mathbb{E}_{z\sim \mu_t}\left[ ||J \xi(z)(J\psi_{\tau})^{-1}(z)||_{\mathrm{HS}}^2\right]\\
        &\leq \mathbb{E}_{z\sim\mu_t}\left[ ||J\xi(z)||_{\mathrm{HS}}^2||(J\psi_{\tau})^{-1}(z)||_{\mathrm{op}}^2\right]
        \leq \alpha^2 B^2 \left|\left|\xi'\right|\right|_{\mathcal{H}_k^{d_z}}^2, \label{eq:bound-h2}
    \end{align}
    where the final display follows from the bounds in \eqref{eq:bound2} and \eqref{eq:bound3}. 
    Clearly, we would like to apply these results with $\xi = \xi_t = \mathcal{P}_{\mu_{t}}\nabla_{W_2}\mathcal{F}(\theta_{t+1},\mu_t)$, $\xi' = \xi'_t = \mathcal{S}_{\mu_t}\nabla_{W_2}\mathcal{F}(\theta_{t+1},\mu_t)$, which corresponds to the SVGD EM update. In this case, substituting \eqref{eq:h_taylor_1} and  \eqref{eq:h_taylor_2} back into \eqref{eq:h_taylor}, and using \eqref{eq:bound-h1} and \eqref{eq:bound-h2}, we have that
    \begin{align}
        \mathcal{F}(\theta_{t+1},\mu_{t+1}) \leq \mathcal{F}(\theta_{t+1},\mu_t) &- \gamma ||\mathcal{S}_{\mu_t}\nabla_{W_2}\mathcal{F}(\theta_{t+1},\mu_t)||^2_{\mathcal{H}_k^{d_z}} \\
        &+\left(M + 
        \alpha^2\right)B^2\int_0^{\gamma} (\gamma -\tau )||\mathcal{S}_{\mu_t}\nabla_{W_2}\mathcal{F}(\theta_{t+1},\mu_t)||^2_{\mathcal{H}_k^{d_z}}\mathrm{d}\tau \\
        \leq \mathcal{F}(\theta_{t+1},\mu_t) &- \gamma ||\mathcal{S}_{\mu_t}\nabla_{W_2}\mathcal{F}(\theta_{t+1},\mu_t)||^2_{\mathcal{H}_k^{d_z}} \\
        &+\left(\frac{M + \alpha^2}{2}\right)B^2\gamma^2||\mathcal{S}_{\mu_t}\nabla_{W_2}\mathcal{F}(\theta_{t+1},\mu_t)||^2_{\mathcal{H}_k^{d_z}},
    \end{align}
    and thus, in particular, that 
    \begin{equation}
        \mathcal{F}(\theta_{t+1},\mu_{t+1})-\mathcal{F}(\theta_{t+1},\mu_t) \leq -\gamma\left(1- 
        \frac{(M+\alpha^2)B^2\gamma}{2}\right)||\mathcal{S}_{\mu_t}\nabla_{W_2}\mathcal{F}(\theta_{t+1},\mu_t)||^2_{\mathcal{H}_k^{d_z}}. \label{eq:mu_descent}
    \end{equation}
    In order to apply these results to obtain \eqref{eq:mu_descent}, we need to show that the sequence $(\xi_t)_{t\in\mathbb{N}}$ satisfies the required condition on the step size in \eqref{eq:step-size-condition}. In particular, we need to establish that, for each $t\in\mathbb{N}$, under the assumption that $\gamma<\gamma_{*}$, it holds that 
    \begin{equation}
        \gamma ||\mathcal{S}_{\mu_{t}}\nabla_{W_2}\mathcal{F}(\theta_{t+1},\mu_t)||_{\mathcal{H}_k^{d_z}} \leq \frac{\alpha - 1}{\alpha B}. \label{eq:step-size-bound}
    \end{equation}
    In fact, we will establish this bound by proving that, under our assumptions, a slightly stronger result holds, namely, that for all $\theta\in\tilde{\Theta}$, where $\tilde{\Theta} = \{\theta_t\}_{t\in\mathbb{N}}$, it holds that
    \begin{equation}
        \gamma ||\mathcal{S}_{\mu_t}\nabla_{W_2}\mathcal{F}(\theta,\mu_t)||_{\mathcal{H}_k^{d_z}}\leq \frac{\alpha - 1}{\alpha B}. \label{eq:sup-step-size-bound}
    \end{equation}
    Similar to the proof of \cite[Theorem 3.2]{Salim2022}, we prove this via induction. First note, arguing similarly to the proof of \cite[Lemma C.1]{Salim2022}, that for each $\theta\in\Theta$, 
    \begin{align}
        ||\mathcal{S}_{\mu}\nabla_{W_2}\mathcal{F}(\theta,\mu)||_{\mathcal{H}_k^{d_z}} &= \left|\left| \mathbb{E}_{z\sim \mu} \left[\nabla_{z}\log \pi_{\theta}(z) k(z,\cdot) + \nabla_{z} k(z,\cdot)\right] \right|\right|_{\mathcal{H}_k^{d_z}} \\
        &\leq \mathbb{E}_{z\sim \mu} \left|\left|\nabla_{z}\log \pi_{\theta}(z)k(z,\cdot)\right|\right|_{\mathcal{H}_k^{d_z}} + \mathbb{E}_{z\sim \mu} \left|\left|\nabla_{z} k(z,\cdot) \right|\right|_{\mathcal{H}_k^{d_z}} \label{eq:94} \\[2mm]
        & \leq B\left[1 + \left|\left|\nabla_{z}\log p_{\theta}(0|x)\right|\right| + M\mathbb{E}_{z\sim \mu}\left[||z||\right]\right] \label{eq:95} \\[2mm]
        &\leq B\left[1 + \left|\left|\nabla_{z}\log p_{\theta}(0|x)\right|\right| + M\left[W_1(p_{\theta}(\cdot|x),\delta_0) + W_1(\mu,p_{\theta}(\cdot|x)\right]\right] \label{eq:96} \\
        &\leq B \left[1 + \left|\left|\nabla_{z}\log p_{\theta}(0|x)\right|\right| + M\int_{\mathcal{Z}}||z||p_{\theta}(z|x)\mathrm{d}z + M\sqrt{\frac{2\mathrm{KL}(\mu||p_{\theta}(\cdot|x)}{\lambda}}\right] \label{eq:97} \\
        &=BC_{\theta} + BM\sqrt{\frac{2\mathrm{KL}(\mu||p_{\theta}(\cdot|x))}{\lambda}} \label{eq:final-ineq}
    \end{align}
    where \eqref{eq:94} follows from the triangle inequality, \eqref{eq:95} follows from Assumption \ref{assumption:bounded_k} and Assumption \ref{assumption:bounded_hessian0}, \eqref{eq:96} follows from the triangle inequality for the metric $W_1$, \eqref{eq:97} follows from Assumption \ref{assumption:bounded_I_stein}, and in the final line we have defined 
    \begin{equation}
        C_{\theta}= 1 + \left|\left|\nabla_{z}\log p_{\theta}(0|x)\right|\right| + M\int_{\mathcal{Z}}||z||p_{\theta}(z|x)\mathrm{d}z. \label{eq:gamma-bound}
    \end{equation}
    Now, suppose that $0<\gamma<\gamma^{*}$. Then, using the definition of $\gamma^{*}$ given in Theorem \ref{prop:descent_lemma}, and the bound in \eqref{eq:final-ineq}, it follows straightforwardly that, for any $\theta\in\tilde{\Theta}$,   
    \begin{align}
        \gamma ||\mathcal{S}_{\mu_0}\nabla_{W_2}\mathcal{F}(\theta,\mu_0)||_{\mathcal{H}_k^{d_z}} &\leq \gamma\left[BC_{\theta} + BM\sqrt{\frac{2\mathrm{KL}(\mu_0||p_{\theta}(\cdot|x)}{\lambda}}\right] \\
        &\leq \frac{\alpha - 1}{\alpha B^2} \left[C_{\theta} + M\sqrt{\frac{2\mathrm{KL}(\mu_0||p_{\theta}(\cdot|x)}{\lambda}}\right]^{-1} \left[BC_{\theta} + BM\sqrt{\frac{2\mathrm{KL}(\mu_0||p_{\theta}(\cdot|x)}{\lambda}}\right]\\
        &\leq \frac{\alpha - 1}{\alpha B}.  \label{eq:step-size-arg}
    \end{align}
    This establishes that \eqref{eq:sup-step-size-bound} holds for $t=0$. We now prove the inductive step. Suppose that, for each $t\in[T-1]$, $0<\gamma\leq \gamma^{*}$ implies that, for each $\theta\in\tilde{\Theta}$, 
    \begin{equation}
        \gamma ||\mathcal{S}_{\mu_{t}}\nabla_{W_2}\mathcal{F}(\theta,\mu_t)||_{\mathcal{H}_k^{d_z}} \leq \frac{\alpha - 1}{\alpha B}. \label{eq:step-size-bound-v2}
    \end{equation}
    Using this assumption, we can substitute \eqref{eq:h_taylor_1} and  \eqref{eq:h_taylor_2} into \eqref{eq:h_taylor}, and use \eqref{eq:bound-h1} and \eqref{eq:bound-h2}, but now with $\xi = P_{\mu_t}\nabla_{W_2}\mathcal{F}(\theta,\mu_{t})$. This implies that, for all $t\in[T-1]$, and for all $\theta\in\tilde{\Theta}$, 
    \begin{equation}
        \mathcal{F}(\theta,\mu_{t+1})-\mathcal{F}(\theta,\mu_t) \leq -\gamma\left(1- 
        \frac{(M+\alpha^2)B^2\gamma}{2}\right)||\mathcal{S}_{\mu_t}\nabla_{W_2}\mathcal{F}(\theta,\mu_t)||^2_{\mathcal{H}_k^{d_z}}. \label{eq:mu_descent_v2}
    \end{equation}
    Thus, in particular, $\mathcal{F}(\theta,\mu_T)\leq \mathcal{F}(\theta,\mu_0)$ for all $\theta\in\tilde{\Theta}$. It follows straightforwardly that
    \begin{align}
        BC_{\theta} + BM \sqrt{\frac{2\mathrm{KL}(\mu_T||p_{\theta}(\cdot|x))}{\lambda}} 
        &= BC_{\theta} + BM \sqrt{\frac{2\left[\mathcal{F}(\theta,\mu_T) + \log p_{\theta}(x)\right]}{\lambda}} \\
        &\leq BC_{\theta} + BM \sqrt{\frac{2\left[\mathcal{F}(\theta,\mu_0) + \log p_{\theta}(x)\right]}{\lambda}} =BC_{\theta} + BM \sqrt{\frac{2\mathrm{KL}(\mu_0||p_{\theta}(\cdot|x))}{\lambda}}.
    \end{align}
    By combining this display with the bound in \eqref{eq:final-ineq}, and arguing as in \eqref{eq:step-size-arg}, we then have the required result in \eqref{eq:sup-step-size-bound}. In particular, for all $\theta\in\tilde{\Theta}$, we have
    \begin{equation}
        \gamma ||\mathcal{S}_{\mu_{T}}\nabla_{W_2}\mathcal{F}(\theta,\mu_T)||_{\mathcal{H}_k^{d_z}} \leq \gamma \left[BC_{\theta} + BM \sqrt{\frac{2\mathrm{KL}(\mu_T||p_{\theta}(\cdot|x))}{\lambda}} \right]\leq \gamma \left[BC_{\theta} + BM \sqrt{\frac{2\mathrm{KL}(\mu_0||p_{\theta}(\cdot|x))}{\lambda}}\right]\leq \frac{\alpha - 1}{\alpha B}.
    \end{equation}
    This completes the proof of \eqref{eq:sup-step-size-bound}, which in turn implies \eqref{eq:step-size-bound}, and thus \eqref{eq:mu_descent}. Finally, combining \eqref{eq:theta_descent} and \eqref{eq:mu_descent}, we have the desired bound, 
    \begin{align}
        \mathcal{F}(\theta_{t+1},\mu_{t+1}) - \mathcal{F}(\theta_t,\mu_t) \leq - \gamma &\left[\left(1 - 
        \frac{M\gamma}
        {2}\right)||\nabla_{\theta}\mathcal{F}(\theta_t,\mu_t)||_{\mathbb{R}^{d_{\theta}}}^2 + \left(1- 
        \frac{(M+\alpha^2)B^2\gamma}{2}\right)||\mathcal{S}_{\mu_t}\nabla_{W_2}\mathcal{F}(\theta_{t+1},\mu_t)||^2_{\mathcal{H}_k^{d_z}}\right].
    \end{align}
    This completes the proof of Theorem \ref{prop:descent_lemma}, with constants
    \begin{equation}
        A_1 = 1- \frac{M\gamma}{2},\quad A_2 = 1- \frac{(M+\alpha^2)B^2\gamma}{2}.
    \end{equation}
\end{proof}

\subsection{Proof of Corollary \ref{corollary:average_convergence}.}
\label{app:proof_average_convergence}
\begin{proof}
     Using Theorem \ref{prop:descent_lemma}, iterating, and using the definition of $c_{\gamma}$ given in Corrolary \ref{corollary:average_convergence}, we have that 
    \begin{align}
        \mathcal{F}(\theta_{T+1},\mu_{T+1}) - \mathcal{F}(\theta_0,\mu_0) \leq - c_{\gamma}\sum_{t=0}^{T}\bigg(\left|\left|\nabla_{\theta}\mathcal{F}(\theta_t,\mu_t)\right|\right|_{\mathbb{R}^{d_{\theta}}}^2+\left|\left|\mathcal{S}_{\mu_t}\nabla_{W_2}\mathcal{F}(\theta_{t+1},\mu_t)\right|\right|^2_{\mathcal{H}_k^{d_z}}\bigg),
    \end{align}
    Thus, rearranging, and using the fact that $\min_{(\theta,\mu)\in\Theta\times\mathcal{P}_2(\mathcal{Z})}\leq \mathcal{F}(\theta_{T+1},\mu_{T+1})$, we have as required that
    \begin{align}
        \frac{1}{T}\sum_{t=0}^{T}\bigg(\left|\left|\nabla_{\theta}\mathcal{F}(\theta_t,\mu_t)\right|\right|_{\mathbb{R}^{d_{\theta}}}^2+\left|\left|\mathcal{S}_{\mu_t}\nabla_{W_2}\mathcal{F}(\theta_{t+1},\mu_t)\right|\right|^2_{\mathcal{H}_k^{d_z}}\bigg) 
        &\leq \frac{\mathcal{F}(\theta_0,\mu_0) - \min_{(\theta,\mu)\in\Theta\times\mathcal{P}_2(\mathcal{Z})}\mathcal{F}(\theta,\mu)}{c_{\gamma}T}.
    \end{align}
\end{proof}

\subsection{Proof of Theorem \ref{prop:coin-em}}
\label{sec:coin-em-proof}

\begin{proof}
    Under Assumption \ref{coin-em-assumption-1}, the marginal likelihood $\theta\mapsto p_{\theta}(x)$ admits a unique maximizer $\theta^{*} = \argmax_{\theta\in\Theta}p_{\theta}(x)$ \cite[Theorem 4]{Kuntz2023}. Thus, in particular, the free energy $(\theta,\mu)\mapsto \mathcal{F}(\theta,\mu)$ admits a unique minimizer $(\theta^{*},\mu^{*})$, with $\mu^{*} = p_{\theta^{*}}(\cdot|x)$ \cite[Theorem 1]{Kuntz2023}. 
    
    Next, once more appealing to Assumption \ref{coin-em-assumption-1}, the function $(\theta,\mu)\mapsto \mathcal{F}(\theta,\mu)$ is (geodesically) convex on the product space $\Theta \times \mathcal{P}_2(\mathcal{Z})$. In particular, convexity w.r.t. the first argument follows directly from the definition of the free-energy, while geodesic convexity w.r.t. the second argument follows from \cite[Proposition 9.3.2, Theorem 9.4.12]{Ambrosio2008}. We thus have, using Jensen's inequality for \eqref{eq:jensen} and the first-order characterization of a (geodesically) convex function \cite[e.g.,][Chapter 10]{Ambrosio2008} for \eqref{eq:geo-convex}, that
    \begin{align}
        \mathcal{F}\left(\frac{1}{T}\sum_{t=1}^T\theta_t, \frac{1}{T}\sum_{t=1}^T \mu_t\right) - \mathcal{F}(\theta^{*},\mu^{*}) &\leq \frac{1}{T}\sum_{t=1}^T \left[ \mathcal{F}(\theta_t,\mu_t) - \mathcal{F}(\theta^{*},\mu^{*})\right] \label{eq:jensen} \\
        &\leq \frac{1}{T}\sum_{t=1}^T \langle -\nabla_{\theta} \mathcal{F}(\theta_t,\mu_t),\theta_t - \theta^{*}\rangle_{\mathbb{R}^{d_{\theta}}} \nonumber \\
        &~~~+ \frac{1}{T}\sum_{t=1}^T \int_{\mathbb{R}^{d_z}} \langle - \nabla_{W_2}\mathcal{F}(\theta_t,\mu_t)(z), t_{\mu_t}^{\mu_{*}}(z) - z\rangle_{\mathbb{R}^{d_z}} \mu_t(z)\mathrm{d}z. \label{eq:geo-convex}
    \end{align}
    The required bound given in Theorem \ref{prop:coin-em-v2} (see App. \ref{sec:coin-results}) now follows from existing results, arguing as in the proof of \cite[Corollary 5]{Orabona2016} with $\varepsilon=1$ to bound the first term,\footnote{For an informal overview of this approach, see also Part II of \cite{Orabona2020}.} and arguing as in the proof of \cite[Proposition 3.3']{Sharrock2023} with $w_0=1$ and replacing $\nabla_{W_2}\mathcal{F}(\mu_t) \mapsto \nabla_{W_2}\mathcal{F}(\theta_t,\mu_t)$ to bound the second term.
\end{proof}

\vfill

\pagebreak

\section{MARGINAL SVGD EM AND MARGINAL COIN EM}
\label{sec:marginal_algorithms}
In certain models, the M step is tractable. In other words, it is possible to derive a tractable expression for 
$\smash{\theta_{*}(\mu) = \argmax_{\mu\in\mathcal{P}_2(\mathcal{Z})}\mathcal{F}(\theta,\mu)}$. Thus, in particular, given the empirical measure $\smash{\mu^{N} = \frac{1}{N}\sum_{j=1}^N \delta_{z^{j}}}$, we can compute $\theta_{*}(z^{1:N}):=\theta_{*}(\mu^N)$, where $z^{1:N} = (z^{1},\dots,z^{N})\in\mathcal{Z}^N$. In such cases, instead of SVGD EM (Alg. \ref{alg:svgdEM}) or Coin EM (Alg. \ref{alg:coinEM}), we can use marginal variants of these algorithms, in which the $\theta$ updates are now exact.

\begin{algorithm}[ht] %
	\caption{Marginal SVGD EM}
	\label{alg:marginalsvgdEM}
	\begin{algorithmic}[1]
		\STATE {\bf input:} number of iterations $T$, number of particles $N$, initial particles $\{z_0^i\}_{i=1}^N \sim \mu_0$, initial $\theta_0$, target density $\pi$, kernel $k$, function $\mu\mapsto \theta_{*}(\mu)$, learning rate $\gamma$.
		\FOR{$t = 1,\dots,T -1$}
		\STATE For $i \in [N]$ 
        \begin{align*}
        \hspace{18mm}z_{n+1}^{i} &= z_{n}^{i} + \frac{\gamma}{N} \textstyle\sum_{j=1}^N\left[k(z_t^{j},z_t^{i}) \nabla_z \log \pi_{\theta_{*}\left(z_t^{1:N}\right)}(z_t^{j}) + \nabla_{z_t^{j}}k(z_t^{j},z_t^{i})\right]
        \end{align*}
		\ENDFOR
		\STATE {\bf return} $\theta_{T}$ and $\{z^i_{T}\}_{i=1}^{N}$.
	\end{algorithmic}
    \end{algorithm}

\begin{algorithm}[ht] %
	\caption{Marginal Coin EM}
	\label{alg:mcoinEM}
	\begin{algorithmic}[1]
		\STATE {\bf input:} number of iterations $T$, number of particles $N$, initial particles $\{z_0^i\}_{i=1}^N \sim \mu_0$, initial $\theta_0$, target density $\pi$, kernel $k$, function $\mu\mapsto \theta_{*}(\mu)$.
		\FOR{$t =1,\dots,T$}
		\STATE For $i \in [N]$ 
        \begin{align*}
        z_{t}^{i} = z_0^{i} &+\frac{{\textstyle\sum_{s=1}^{t-1} 
        \frac{1}{N} \textstyle\sum_{j=1}^N k(z_s^{j},z_s^{i}) \nabla_{z} \log \pi_{\theta_{*}\left(z_{s}^{1:N}\right)}(z_s^{j}) +  \nabla_{z^{j}}k(z_s^{j},z_s^{i})}}{t} \\
        &\times \big( 1 + \textstyle\sum_{s=1}^{t-1} \langle
        \frac{1}{N} \textstyle\sum_{j=1}^N k(z_s^{i},z_s^{j}) \nabla_{z} \log \pi_{\theta_{*}\left(z_{s}^{1:N}\right)}(z_s^{j}) +  \nabla_{z^{j}}k(z_s^{j},z_s^{i}), z_s^{i} - z_0^{i}\rangle \big)
        \end{align*}
		\ENDFOR
		\STATE {\bf return} $\theta_{T}$ and $\{z^i_{T}\}_{i=1}^{N}$.
	\end{algorithmic}
    \end{algorithm}

\section{ADAPTIVE COIN EM}
\label{sec:adaptive-coin}
In Sec. \ref{sec:coin}, the update equation given in \eqref{eq:betting-update}, and thus the update equations in Alg. \ref{alg:coinEM}, were given under the assumption that the sequence of outcomes, in this case $(c_t^{\theta})_{t\in[T]}= (\nabla_{\theta}\mathcal{F}(\theta_t,\mu_t^N))_{t\in[T]}$ and $(c_t^{\mu})_{t\in[T]}= (\mathcal{P}_{\mu_t}\nabla_{W_2}\mathcal{F}(\theta_t,\mu_t^N))_{t\in[T]}$, were bounded above by $1$ (see the remark in Footnote \ref{footnote:adaptive}). In practice, of course, this may not be the case. If, instead, there exists a known constant which upper bounds $(c_t^{\theta})_{t\in[T]}$ and $(c_t^{\mu})_{t\in[T]}$, then one can simply replace these quantities by their normalized versions in Alg. \ref{alg:coinEM}. In the more unlikely scenario that such a constant is unknown, we can instead use an adaptive version of Coin EM, in which an empirical estimate of this constant is updated as the algorithm progresses; see \cite[][Sec. 6]{Orabona2017} and \cite[][App. D]{Sharrock2023} for precedents. This algorithm, which we recommend and use in all experiments, is summarized in Alg. \ref{alg:coinEM-adaptive}.

In addition, once more following \cite[][Sec. 6]{Orabona2017}, whenever we apply Coin EM to a Bayesian neural network (see Sec. \ref{sec:bayes_nn_results} and App. \ref{sec:bayes-nn-alt-details}), we further alter the parameter update equation in Alg. \ref{alg:coinEM-adaptive} to read as 
\begin{equation}
\theta_{t,j} = \theta_{0,j} + \frac{\sum_{s=1}^{t-1} c_{s,j}^{\theta}}{\max(G_{t,j}^{\theta} + L_{t,j}^{\theta},100 L_{t,j}^{\theta})} (1 + \frac{R_{t,j}^{\theta}}{L_{t,j}^{\theta}}), 
\end{equation}
with an analogous modification for the particle update equation. It is worth noting that both of these adaptive versions of the Coin EM algorithm remain entirely tuning free.

\begin{algorithm}[H]
   \caption{Adaptive Coin EM}
   \label{alg:coinEM-adaptive}
\begin{algorithmic}[1]
   \STATE {\bfseries input:} number of iterations $T$, number of particles $N$, initial particles $\{z_0^{i}\}_{i=1}^N \sim \mu_0$, initial $\theta_0$, target density $\pi$, kernel $k$ 
   \STATE{\bfseries initialize:} for $i\in[N]$: $L^{\theta}_{0,j}=0$, $G_{0,j}^{\theta}=0$, $R_{0,j}^{\theta}=0$; for $i\in[N]$ and $j\in[d]$, $L^{z,i}_{0,j}=0$, $G_{0,j}^{z,i}=0$, $R_{0,j}^{z,i}=0$.
   \FOR{$t=1,\dots,T$}
   \STATE Compute
   \begin{equation*}
       c_{t-1}^{\theta} = \frac{1}{N}\sum_{j=1}^N \nabla_{\theta} \log \pi_{\theta_{t-1}}(z_{t-1}^{j}) \tag{parameter gradient}
   \end{equation*}
   \FOR{$j=1,\dots,d_{\theta}$}
   \STATE Compute
   \begin{align*}
       L_{t,j}^{\theta} &= \mathrm{max}(L_{t-1,j}^{\theta}, |c_{t,j}^{\theta}|)  \tag{max. observed scale} \\
       G_{t,j}^{\theta} &= G_{t-1,j}^{\theta} + |c_{t-1,j}^{\theta}| \tag{sum of absolute value of gradients} \\ %
       R_{t,j}^{\theta} &= \max(R_{t-1,j}^{\theta} + \langle c_{t-1,j}^{\theta}, \theta_{t-1,j} - \theta_{0,j} \rangle, 0) \tag{total reward} \\
       {\theta}_{t,j} &= {\theta}_{0,j} + \frac{\sum_{s=1}^{t-1} c_{s,j}^{\theta}}{G_{t,j}^{\theta} + L_{t,j}^{\theta}} (1 + \frac{R_{t,j}^{\theta}}{L_{t,j}^{\theta}}). \tag{parameter update}
    \end{align*}
    \ENDFOR
    \FOR {$i=1,\dots,N$}
    \STATE Compute 
    \begin{equation*}
        c_{t-1}^{z,i} = \frac{1}{N}\sum_{j=1}^N k(z_{t-1}^{j},z_{t-1}^{i}) \nabla_{z^j} \log \pi_{\theta_t}(z_{t-1}^{j}) +  \nabla_{z^{j}}k(z_{t-1}^{j},z_{t-1}^{i}) \tag{particles gradient}
    \end{equation*}
    \FOR{$j=1,\dots,d_{z}$}
    \STATE Compute 
    \begin{align*}
       L_{t,j}^{z,i} &= \mathrm{max}(L_{t-1,j}^{z,i}, |c_{t-1,j}^{i}|)  \tag{max. observed scale} \\
       G_{t,j}^{z,i} &= G_{t-1,j}^{z,i} + |c_{t-1,j}^{z,i}| \tag{sum of absolute value of gradients} \\ %
       R_{t,j}^{z,i} &= \max(R_{t-1,j}^{z,i} + \langle c_{t-1,j}^{z,i}, z^{i}_{t-1,j} - z^{i}_{0,j} \rangle, 0) \tag{total reward} \\
       {z}^{i}_{t,j} &= {z}^{i}_{0,j} + \frac{\sum_{s=1}^{t-1} c_{s,j}^{z,i}}{G_{t,j}^{z,i} + L_{t,j}^{z,i}} (1 + \frac{R_{t,j}^{z,i}}{L_{t,j}^{z,i}}). \tag{particles update}
    \end{align*}
    \ENDFOR
    \ENDFOR
    \ENDFOR
   \STATE {\bfseries output:} $\theta_T$ and $(z_T^{i})_{i=1}^N$.
\end{algorithmic}
\end{algorithm}

\section{ADDITIONAL EXPERIMENTAL DETAILS}
\label{app:additional-experimental-details}
We implement all of the algorithms using Python 3 and JAX \citep{Bradbury2018}. We perform all experiments using a MacBook Pro 16" (2021) laptop with Apple M1 Pro chip and 16GB of RAM.

\subsection{Toy hierarchical model}
\label{sec:toy-model-details}
\textbf{Implementation Details}. For the results in Fig. \ref{fig:toy_1}, we initialize all methods with $\theta_0\sim \mathcal{N}(0,0.1^2)$ and $z_0^{i}\sim\mathcal{N}(0,1)$. We use $N=10$ particles and run each algorithm for $T=500$ iterations. Additional results for larger numbers of particles are given in App. \ref{sec:toy-add-results}. In Fig. \ref{fig:toy_1a}, we run each algorithm over a grid of 50 logarithmically spaced learning rates $\gamma\in[10^{-5}, 10^{3}]$. We use Adagrad \cite{Duchi2011} to automatically adjust the learning rates of SVGD EM and PGD. In Fig. \ref{fig:toy_1b} and Fig. \ref{fig:toy_1c}, we run each algorithm using the learning rate which obtained the lowest final iterate MSE in Fig. \ref{fig:toy_1a}. 

For Fig. \ref{fig:toy_2a} and Fig. \ref{fig:toy_2b}, we run the algorithms for $T=5000$ iterations, and report the empirical variance of the latent particles $(z_t^{i})_{i=1}^N$ for each $t\in[0,T]$. In Fig. \ref{fig:toy_2c}, we run each algorithm for $T=250$ iterations with $N=50$ particles. For Coin EM and SVGD EM, we compute the MSE between the empirical variance of the final particles $(z_T^{i})_{i=1}^N$, and the true posterior variance, which in this case is just given by 0.5 \cite[][App. E.1]{Kuntz2023}. For PGD and SOUL, we use the empirical variance of  time-averaged particles, due to the noise used in the particle updates. All results are averaged over 10 random seeds (5 random seeds for Fig. \ref{fig:toy_2c}), and we report 95\% bootstrap CIs as well as the mean.

\subsection{Bayesian logistic regression}
\label{sec:bayes-lr-details}
\textbf{Model}. We consider the setup in \cite[][Sec. 4.1]{DeBortoli2021}. The observed data is $\mathcal{D} = \{x_{i}, y_{i}\}_{i=1}^{N}$, where $x_i\in\mathbb{R}^{d_x}$ are $d_x$-dimensional real-valued covariates, and $y_i\in\{0,1\}$ are binary class labels. We assume the labels $y_i$ are conditionally independent given the features $x_i$ and the regression weights $z\in\mathbb{R}^{d_z}$, with $p(\{y_i, x_i\}| z) = \sigma(x_i^T z)^{y_i}[1-\sigma(x_i^T z)]^{1-y_i}$ for $i\in[N]$, where $\sigma(u):=e^u/(1+e^{u})$ denotes the standard logistic function. We place a Gaussian prior on the latent weights, $p(z) =\mathcal{N}(z|\theta \mathbf{1}_{d_z}, 5\mathbf{1}_{dz})$, for some real parameter $\theta\in\mathbb{R}$ which we would like to estimate. The model is thus
\begin{equation}
    p_{\theta}(z,\mathcal{D}) = \mathcal{N}(z|\theta 
    \mathbf{1}_{d_z},5\mathbf{1}_{d_z}) \prod_{i=1}^N \sigma(x_i^T z)^{y_i}[1-\sigma(x_i^T z)]^{1-y_i}.
\end{equation}
\textbf{Data}. We fit this model using the Wisconsin Breast Cancer dataset \cite{Wolberg1990}. This dataset contains 683 datapoints, each with nine features $x_i\in\mathbb{R}^{9}$ extracted from a digitized image of a fine needle aspirate of a
breast mass, and a label $y_i\in\{0,1\}$ indicating whether the mass is benign ($y_i=0$) or malign ($y_i = 1$). We normalize the features so that each has mean zero and unit standard deviation across the dataset, and use a random 80-20 train-test split.

\textbf{Implementation Details}. For the results in Fig. \ref{fig:bayes_lr}, following \cite{DeBortoli2021,Kuntz2023}, we initialize $\theta_0 = 0$ and $z_0^{i}\sim\mathcal{N}(0,1)$. In Fig. \ref{fig:bayes_lr_a} and Fig. \ref{fig:bayes_lr_b}, we use a learning rate of $\gamma=0.02$ for PGD, PMGD, and SOUL, and a learning rate of $\gamma=0.2$ for SVGD EM, which typically obtains its best performance for a higher learning rate than the other methods (see, e.g., Fig. \ref{fig:toy_1a} or Fig. \ref{fig:toy_2c}). We run each algorithm with $N=100$ particles for $T=800$ iterations. In Fig. \ref{fig:bayes_lr_c}, we run each algorithm using a grid of 50 logarithmically spaced learning rates $\gamma\in[10^{-10}, 10^{2}]$, using $N=20$ particles and $T=400$ iterations. We repeat each experiment over 10 random test-train splits of the data.

\subsection{Bayesian logistic regression (alternative model)}
\label{sec:bayes-lr-alt-details}
In App. \ref{sec:bayes-lr-alt-results}, we present results for an alternative Bayesian logistic regression model, which were omitted from the main text due to space constraints. Here, we provide full details of this model.

\textbf{Model}. We now follow the setup in \cite{Gershman2012}. Similar to before, the observed data is given by $\mathcal{D} = \{x_{i}, y_{i}\}_{i=1}^{N}$, where $x_i\in\mathbb{R}^{d_x}$ are $d_x$-dimensional real-valued covariates, and $y_i\in\{0,1\}$ are binary class labels. Meanwhile, the latent variables are now given by $z = \{w,  \log \alpha\}$, consisting of a $d_x$-dimensional real-valued regression coefficients $w_k\in\mathbb{R}^{d_x}$, and a positive precision parameter $\alpha\in\mathbb{R}_{+}$. Finally, the parameters are given by $\theta = (\log a, \log b)$, for positive $a,b\in\mathbb{R}_{+}$. 

We assume that the labels $y_i$ are conditionally independent given the features $x_i$ and the regression weights $z\in\mathbb{R}^{d_z}$, with $p(\{y_i, x_i\}| z) = \sigma(x_i^T z)^{y_i}[1-\sigma(x_i^T z)]^{1-y_i}$ for $i\in[N]$, where $\sigma(u):=e^u/(1+e^{u})$ denotes the standard logistic function. We now also assume that $p(w_k | a) = N(w_k; 0, \alpha^{-1})$ for $k\in[d_x]$, and place a Gamma prior on $\alpha$, so $p(\alpha) = \Gamma(\alpha| a, b)$. In this case, the model is given by 
\begin{equation}
    p_{\theta}(z,x) = \operatorname{Gamma}(\alpha| a, b) \prod_{k=1}^{d_x} N(w_k; 0, \alpha^{-1}) \prod_{i=1}^N \sigma(x_i^T z)^{y_i}[1-\sigma(x_i^T z)]^{1-y_i}.
\end{equation}

\textbf{Data}. We fit this model to several datasets from the UCI Machine Learning repository \cite{Dua2019}: the Covertype dataset, which contains 581012 datapoints and 54 attributes, the Banknote dataset, which contains 1372 datapoints and 4 features, and the Cleveland heart disease dataset, which contains 303 datapoints and 13 features. We split the dataset into train-test sets using a 70-30 train-test split.

\textbf{Implementation Details}. For the results in Fig. \ref{fig:logistic_regression_2_auc}, we initialize $a_0 = 1$, $b_0 = 0.01$, $\alpha_0^{i} \sim \mathrm{Gamma}(a_0,b_0)$, and $\smash{w_{0}^{i} \sim \mathcal{N}(0,{1}{/\alpha_{0}^{i}})}$. We run each algorithm using $N=10$ particles for $T=1000$ iterations, over a grid of logarithmically spaced learning rates $\gamma\in[10^{-9},10^{1}]$. We repeat each experiment over 10 random train-test splits of the data.

\subsection{Bayesian neural network}
\label{sec:bayes-nn-alt-details}

\textbf{Model}. We consider the setting described in \citep[Setion 6.5]{Yao2022a} and \citep[Sec. 3.2]{Kuntz2023}, applying a neural network to classify MNIST images \citep{Lecun1998}. In particular, we use a Bayesian two-layer neural network with tanh activation functions, a softmax output layer, and Gaussian priors on the weights. 

In this case, the observed data is of the form $\mathcal{D} = \{x_i,y_i\}_{i=1}^N$, where $x_i\in\mathbb{R}^{784}$ are $28\times 28$ grayscale images of handwritten digits, and $y_i\in\{0,1\}$ are labels denoting whether the image is a 4 or 9. We normalise the 784 features so that each has mean zero and unit standard deviation across the dataset.  We assume the labels $y_i$ are conditionally independent given the features $x_i$ and the network weights $z = (w,v)$, where $w\in\mathbb{R}^{d_w=40\times 784}$ and $v\in\mathbb{R}^{d_v=2\times 40}$, and that 
\begin{equation}
    p(y_i|x_i,w) =\exp \textstyle \left(\sum_{j=1}^{40} v_{ij} \tanh\left(\sum_{k=1}^{784} w_{jk}x_{ik}\right) \right)
\end{equation}
for $i\in[n]$. We then place Gaussian priors on the weights of input layer and the output layer, viz $p(w_k)=\mathcal{N}(w_k|0, e^{2\alpha})$, for $k\in[784]$ and $p(v_j)=\mathcal{N}(v_j|0,e^{2\beta})$ for $j\in[40]$. In this case, rather than assigning a hyperprior to $\alpha$ and $\beta$, we will learn them from data. Thus, our model takes the form
\begin{equation}
    p_{\theta}(z,\mathcal{D}) = \prod_{j=1}^{40} \mathcal{N}(w_k|0, e^{2\alpha}) \prod_{k=1}^{784}\mathcal{N}(v_k|0,e^{2\beta})\prod_{i=1}^N p(y_i|x_i,w)
\end{equation}
where $z=(w, v) \in \mathbb{R}^{40\times 784 + 2\times 40}$, and $\theta = (\alpha,\beta)$.

\textbf{Data}. We use the MNIST dataset \cite{Lecun1998}, which contains 70'000 $28 \times 28$ grayscale images $x_i\in\mathbb{R}^{784}$ of handwritten digits between 0 and 9, each with an accompanying label $y_i\in \{0,1,\dots,9\}$. Following \cite{Kuntz2023,Yao2022a}, we subsample 1000 points with labels 4 and 9. We normalize the features so that each has mean zero and unit standard deviation across the dataset, and split the data using a random 80-20 train-test split.

\textbf{Implementation Details}. For the results in Fig. \ref{fig:bnn_mnist_compare_lr} - \ref{fig:bayes_nn_mnist}, we initialize $\alpha_0=0$, $\beta_0=0$, and $w_k^{i}\sim\mathcal{N}(0,e^{2\alpha_0})$, $v_j\sim\mathcal{N}(0,e^{2\beta_0})$, and run each algorithm for $T=500$ iterations. In addition, in Fig. \ref{fig:bayes_nn_mnist}, we use $N=10$ particles, and run each algorithm for 50 logarithmically spaced learning rates  $\gamma\in[10^{-9},10^{1}]$. We repeat each experiment over 5 (Fig. \ref{fig:bnn_mnist_compare_lr} - \ref{fig:bnn_mnist_compare_methods}) or 10 (Fig. \ref{fig:bayes_nn_mnist}) random train-test splits of the data.

\subsection{Bayesian neural network (alternative model)}
\label{sec:bayes-nn-details}
In App. \ref{sec:bayes-nn-add-results}, we present results for an alternative Bayesian neural network model, which were omitted from the main text due to space constraints. Here, we provide full details of this model.

\textbf{Model}. We consider the setup described in \cite{Hernandez-Lobato2015, Liu2016a}. The observed data is of the form $\mathcal{D} = \{x_i,y_i\}_{i=1}^N$, where $x_i\in\mathbb{R}^{d_x}$ are $d_x$-dimensional real-valued covariates, and $y_i\in\mathbb{R}$ are real-valued responses. We assume that the responses $y_i$ are conditionally independent given the covariates $x_i$ and the network weights $w\in\mathbb{R}^{d_w}$, and that $p(y_i|x_i,w) = \mathcal{N}(y_i | f(x_i,w), \gamma^{-1})$ for $i\in[n]$, where $f(x_i,w)$ denotes the output of the neural network. We place a Gaussian prior on each of the neural network weights, namely $p(w_j)=\mathcal{N}(w_j|0, \lambda^{-1})$, for $j\in[d_w]$. To complete our model, we assume a Gamma prior for the inverse covariance $\gamma\in\mathbb{R}_{+}$, and a Gamma hyperprior on the inverse covariance $\lambda\in\mathbb{R}_{+}$, so $p(\gamma) = \mathrm{Gamma}(\gamma|a_{\gamma},b_{\gamma})$ $p(\lambda) = \mathrm{Gamma}(\lambda|a_{\lambda},b_{\lambda})$. Our model thus takes the form
\begin{equation}
    p_{\theta}(z,\mathcal{D}) = \mathrm{Gamma}(\gamma|a_{\gamma},b_{\gamma}) \mathrm{Gamma}(\lambda|a_{\lambda},b_{\lambda}) \prod_{i=1}^N \prod_{j=1}^{d_z} \mathcal{N}(y_i | f(x_i,w), \gamma^{-1}) \mathcal{N}(w_j|0, \lambda^{-1})
\end{equation}
where $z=(w, \log \lambda, \log \gamma) \in \mathbb{R}^{d_w+1+1}$, and $\theta = (\log a_{\gamma}, \log b_{\gamma}, \log a_{\lambda}, \log b_{\lambda})$. Rather than fixing the parameters as in \cite{Hernandez-Lobato2015,Liu2016a}, we allow these parameters to be learned from the data.

\textbf{Data}. We fit the Bayesian neural network using several UCI datasets \cite{Dua2019}. The number of datapoints varies from 506 in the Boston Housing (\emph{Boston}) dataset, to 11934 in the Naval Propulsion (\emph{Naval}) dataset. The number of features is between 4 in the Combined Cycle Power Plant (\emph{Power}) to 16 in the Naval Propulsion dataset (\emph{Naval}). In each case, we normalize the features so that each has mean zero and unit standard deviation across the dataset, and split the data using a random 90-10 train-test split.

\textbf{Implementation Details}. For the results in Fig. \ref{fig:bayes_nn_additional}, we use a neural network with a single hidden layer and with 50 hidden units, and a $\mathrm{Relu}$ activation function. We use Adagrad \cite{Duchi2011} to automatically adjust the learning rates of SVGD EM and PGD. We initialize the parameters by setting $a_{\gamma_0}=1$, $b_{\gamma_0}=0.1$, $a_{\lambda_0}=1$, $b_{\lambda_0}=0.1$. Meanwhile, we initialize the particles for the precisions as $\lambda\sim\mathrm{Gamma}(a_{\lambda_0},b_{\lambda_0})$ and $\gamma\sim\mathrm{Gamma}(a_{\gamma_0},b_{\gamma_0})$, and for the weights $w$ according to zero-mean Gaussians with variance given by the reciprocal of the input dimension (the latent dimension for the first layer, and the number of hidden units for the second). We run each algorithm using $N=20$ particles for $T=1000$ iterations, over a logarithmically spaced grid of 30 learning rates $\gamma\in[10^{-9}, 10^{1}]$. We repeat each experiment over 10 random train-test splits of the data.

\subsection{Latent space network model}
\label{sec:network-details}

\textbf{Model}. We consider the latent space network model first introduced in \cite{Hoff2002}, where, in the context of social networks, each entity is represented by a node of the network. Edges between nodes indicate connections between entities. We restricted ourselves to binary undirected graphs, where an edge in the network between nodes $i$ and $j$ is represented by a binary indicator, $y_{i,j}=\{0,1\}$.

The latent space network model \cite{Hoff2002,Turnbull2019} is popular network model to simplify complex network data by embedding the network's nodes in a lower-dimensional space $z$. A general form of this model is,
\begin{align}
\label{eq:network-model}
Y_{i,j} &\sim \text{Bernoulli}(p_{ij}) \quad i \neq j; \quad i,j=1,\ldots, n \\
\text{logit}(p_{ij}) &= \theta + d(z_i,z_j), \nonumber
\end{align}
where $d(\cdot,\cdot)$ is a function of the similarity between the latent variables. In this paper, we let $d(z_i,z_j) = ||z_i-z_j||$ be the Euclidean distance between $z_i$ and $z_j$, although other similarity metrics which preserve the triangle inequality are also possible.

\textbf{Data}. We use a dataset curated by \cite{Beveridge2018} from fan scripts on the website Genius. The dataset contains the frequency of interactions among characters in the first four season of the TV series Game of Thrones. Each season is represented by a binary undirected network, where an edge indicates that two characters interacted at least 10 times in that season. For the purpose of visualization, we filter out minor characters with fewer than 10 interactions per season. The dataset covers four seasons of the TV series, one network per season,  with $n = 165$ nodes each season.

\textbf{Implementation Details}.  We follow a similar inference scheme as \cite{Hoff2002}. Firstly, we find the maximum likelihood estimates (MLEs) for $(\theta,z)$ by maximizing the likelihood function \eqref{eq:network-model}. Using the MLEs $\smash{(\hat\theta, \hat{z})}$, we initialize PGD, SOUL and Coin EM with $\smash{\theta_0=\hat\theta}$ and for $i=1,\ldots,N$, set $\smash{z_0^{i} \sim \mathcal{N}(\hat{z},0.1)}$. One of the well-known challenges of using latent space network models is that a set of points in Euclidean space will have the same distances regardless of how they are rotated, reflected, or translated. This means that the likelihood \eqref{eq:network-model} is invariant to transformation of the latent positions z. This issue can be resolved by using a Procrustes transformation of the latent variables, relative to a fixed point, which we choose to be the MLE, $\hat{z}$. Therefore, for all $z_t^{i}$ in all algorithms we store and plot $\tilde{z}_t^{i}= \text{argmin}_{Tz} \ \text{tr}(\hat{z} - T z_{t}^{i})^{\top}(\hat{z}-T z_{t}^{i})$. For the SOUL and PGD algorithms we use learning rate parameter $\gamma=0.001$. For all algorithms we use $N=10$ particles at $T=500$ iterations.

\section{ADDITIONAL NUMERICAL RESULTS}
\label{app:additional-numerical-results}
\subsection{Toy hierarchical model}
\label{sec:toy-add-results}

In this section, we provide additional results for the toy hierarchical model studied in Sec. \ref{sec:toy_results}. 

\textbf{Additional results for a different number of particles.} We first consider the impact of varying the number of particles. In Fig. \ref{fig:toy_1_20_particles} and Fig. \ref{fig:toy_1_50_particles}, we replicate the results in Fig. \ref{fig:toy_1}, now using $N=20$ and $N=50$ particles, rather than $N=10$. There is little to distinguish between the two sets of results. In particular, all methods still converge to $\theta_{*}$; and Coin EM and SVGD EM still tend to achieve a lower MSE than PGD and SOUL. We note that the appearance of wide blue bands in Fig. \ref{fig:toy_1_20_particles} and Fig. \ref{fig:toy_1_50_particles} (and Fig. \ref{fig:toy_1}) is an artifact of the plotting resolution, and the logarithmic scale. In truth, the Coin EM parameter estimates oscillate rapidly as they convergence towards the marginal maximum likelihood estimate, giving the appearance of these bands at this resolution. This type of oscillatory convergence is common for coin betting algorithms; see, e.g., \cite[][pg. 36]{Orabona2020} for a very simple example.

\begin{figure}[htb]
  \centering
  \subfigure[$\mathrm{MSE}(\theta_{t})$ vs learning rate. \label{fig:toy_1a_100_particles}]{\includegraphics[width=0.328\textwidth]{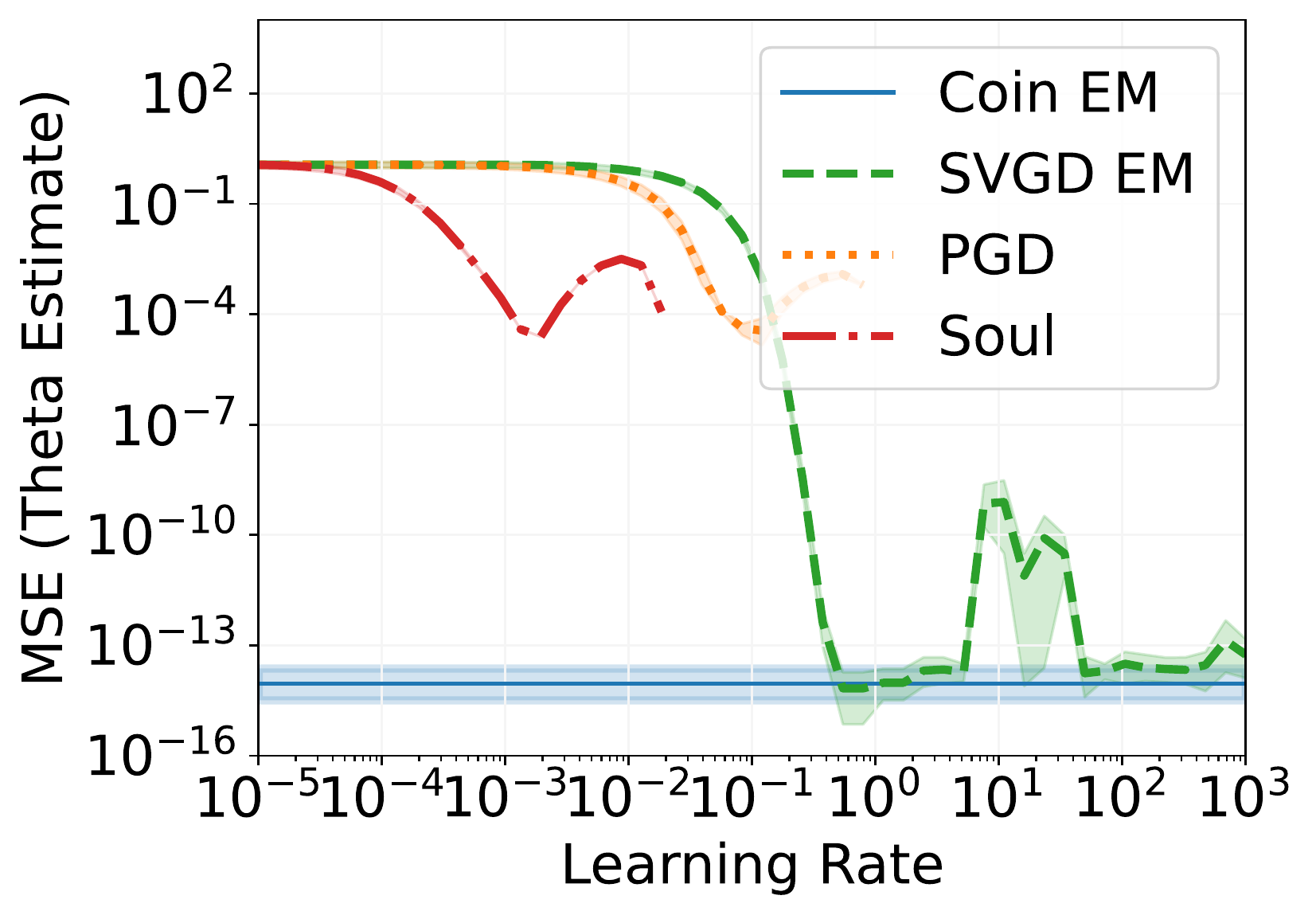}}
  \subfigure[$\mathrm{MSE}(\theta_{t})$ vs $t$.]{\includegraphics[width=0.315\textwidth]{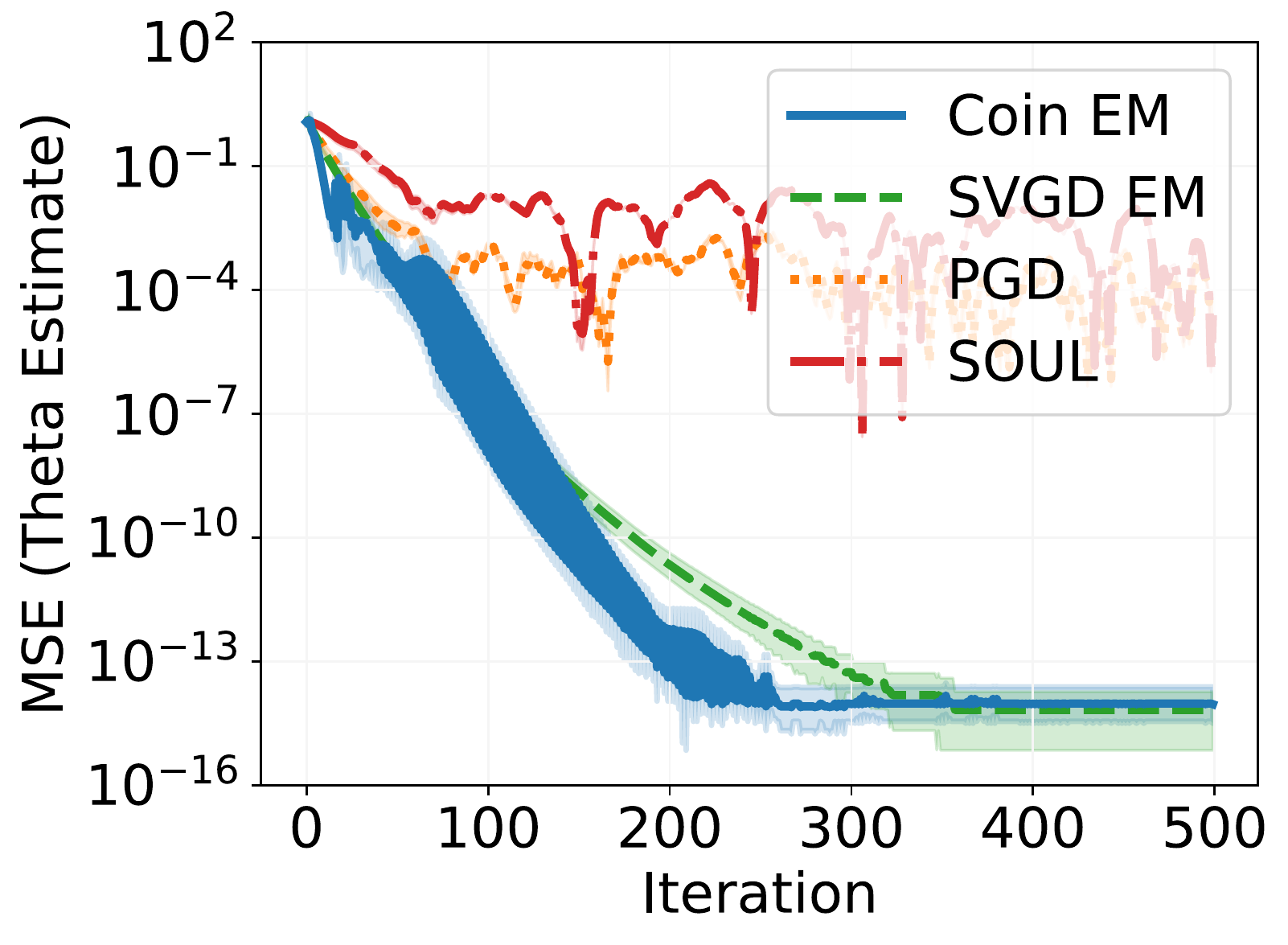}}
  \subfigure[$\mathrm{MSE}\left(\frac{1}{N}\sum_{i=1}^N z_t^{i}\right)$ vs $t$.]{\includegraphics[width=0.315\textwidth]{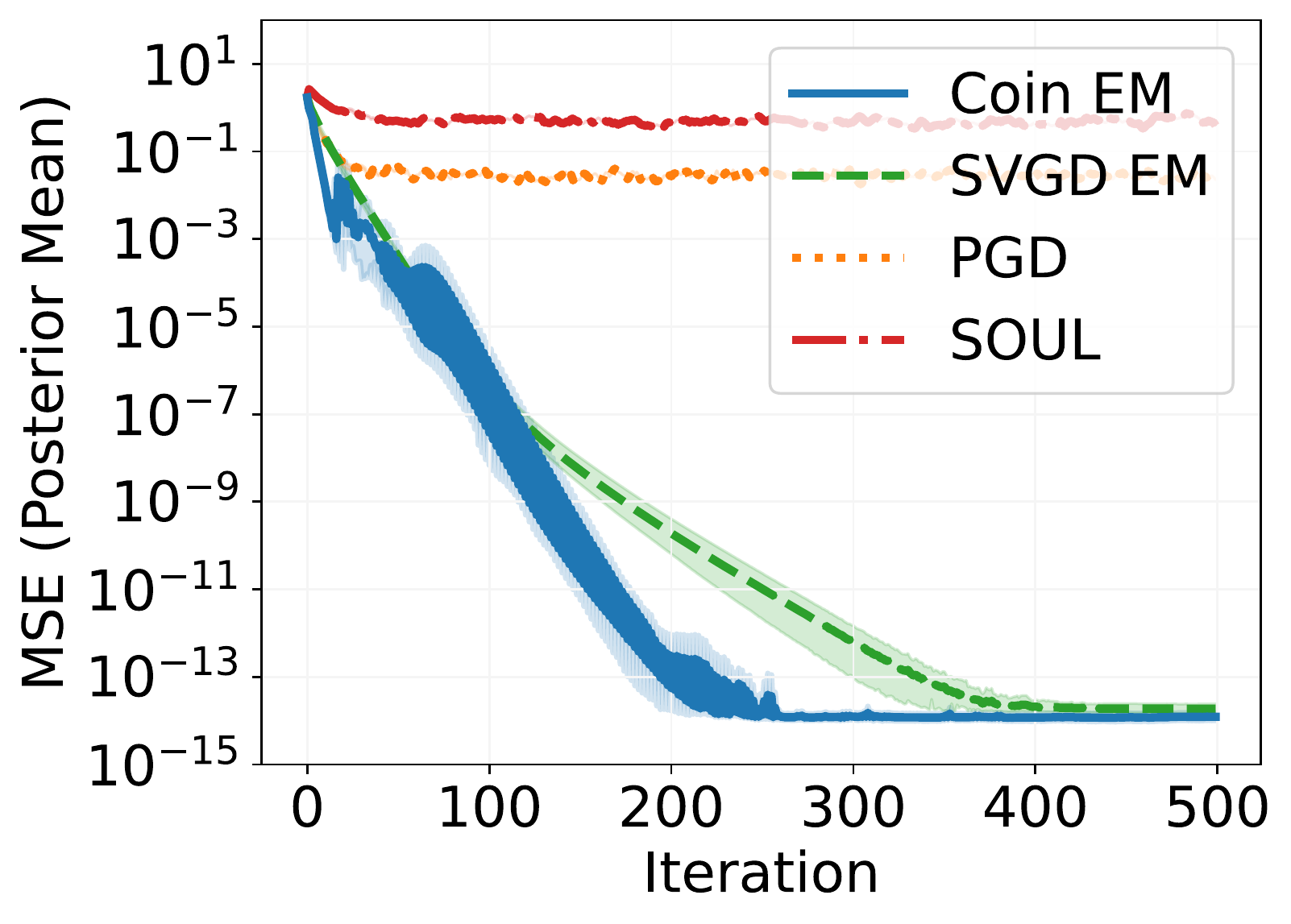}}
  \caption{\textbf{Additional results for the toy hierarchical model with $N=20$ particles.} MSE of the parameter estimate $\theta_{t}$ as a function of the learning rate after $T=500$ iterations (a); and MSE of the parameter estimate (b) and the posterior mean (c) as a function of the number of iterations, using the optimal learning rate from (a).} 
  \label{fig:toy_1_20_particles}
\end{figure}

\begin{figure}[htb]
  \centering
  \subfigure[$\mathrm{MSE}(\theta_{t})$ vs learning rate.\label{fig:toy_1a_50_particles}]{\includegraphics[width=0.325\textwidth]{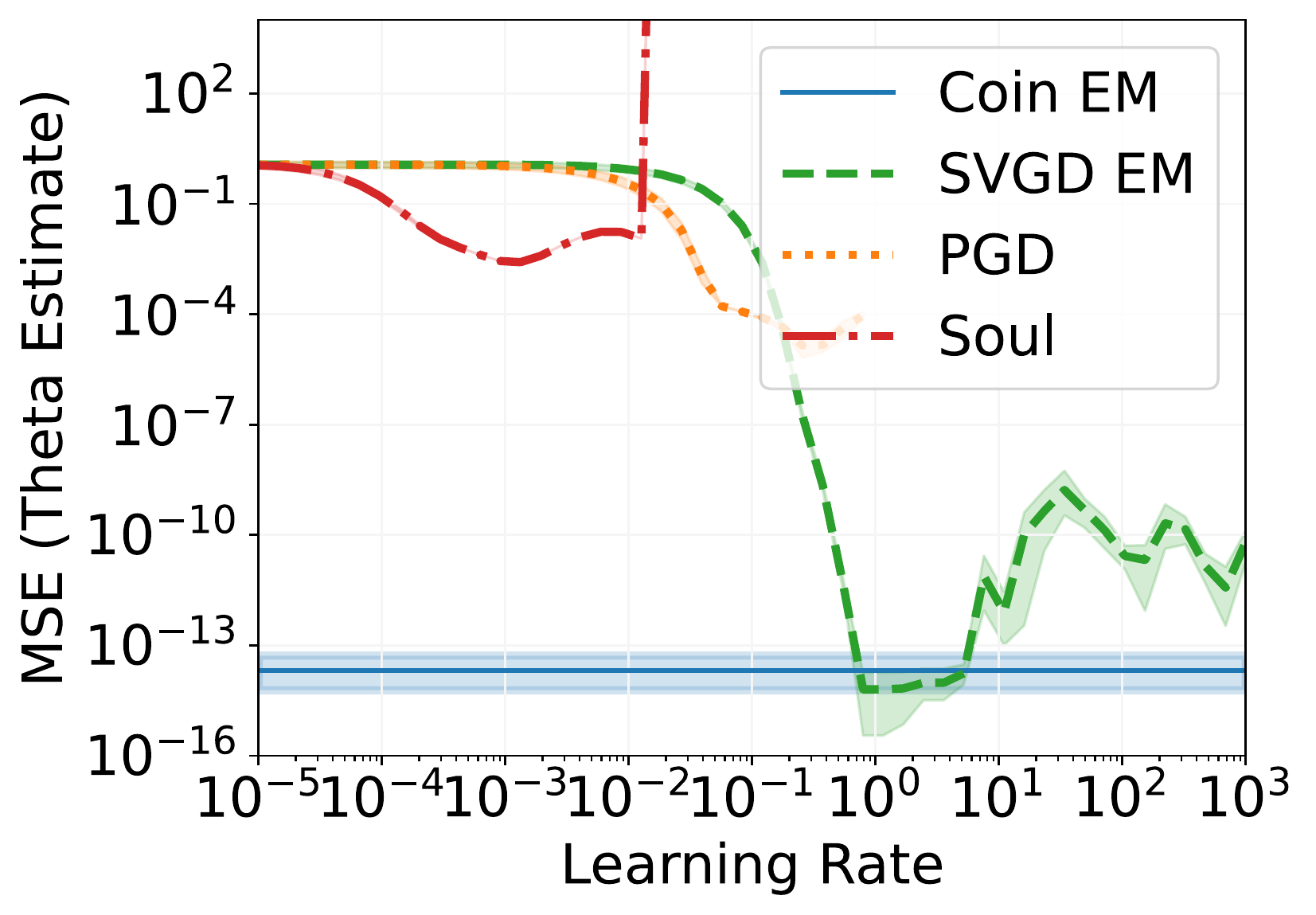}}
  \subfigure[$\mathrm{MSE}(\theta_{t})$ vs $t$.]{\includegraphics[width=0.325\textwidth]{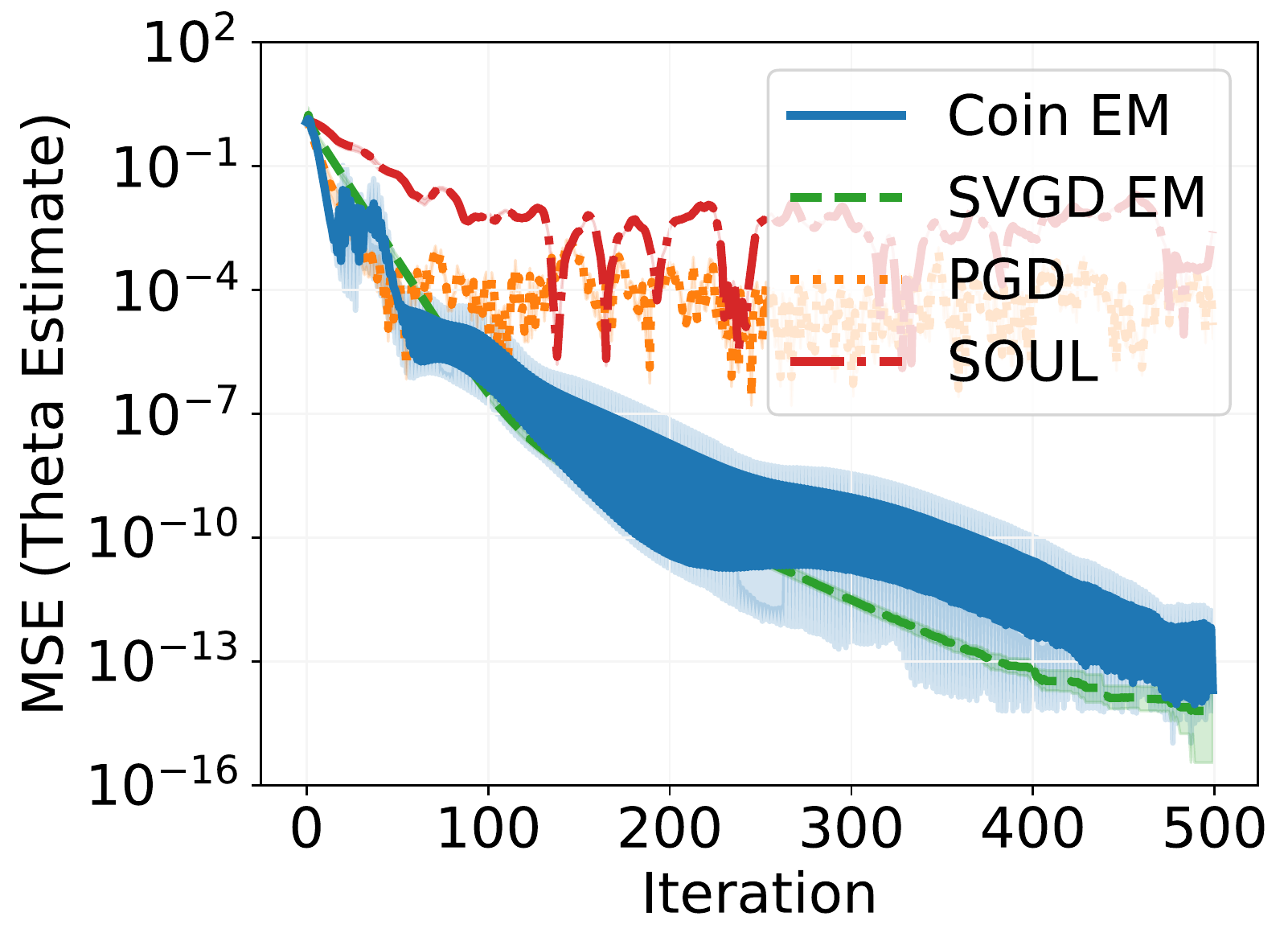}}
  \subfigure[$\mathrm{MSE}\left(\frac{1}{N}\sum_{i=1}^N z_t^{i}\right)$ vs $t$.]{\includegraphics[width=0.325\textwidth]{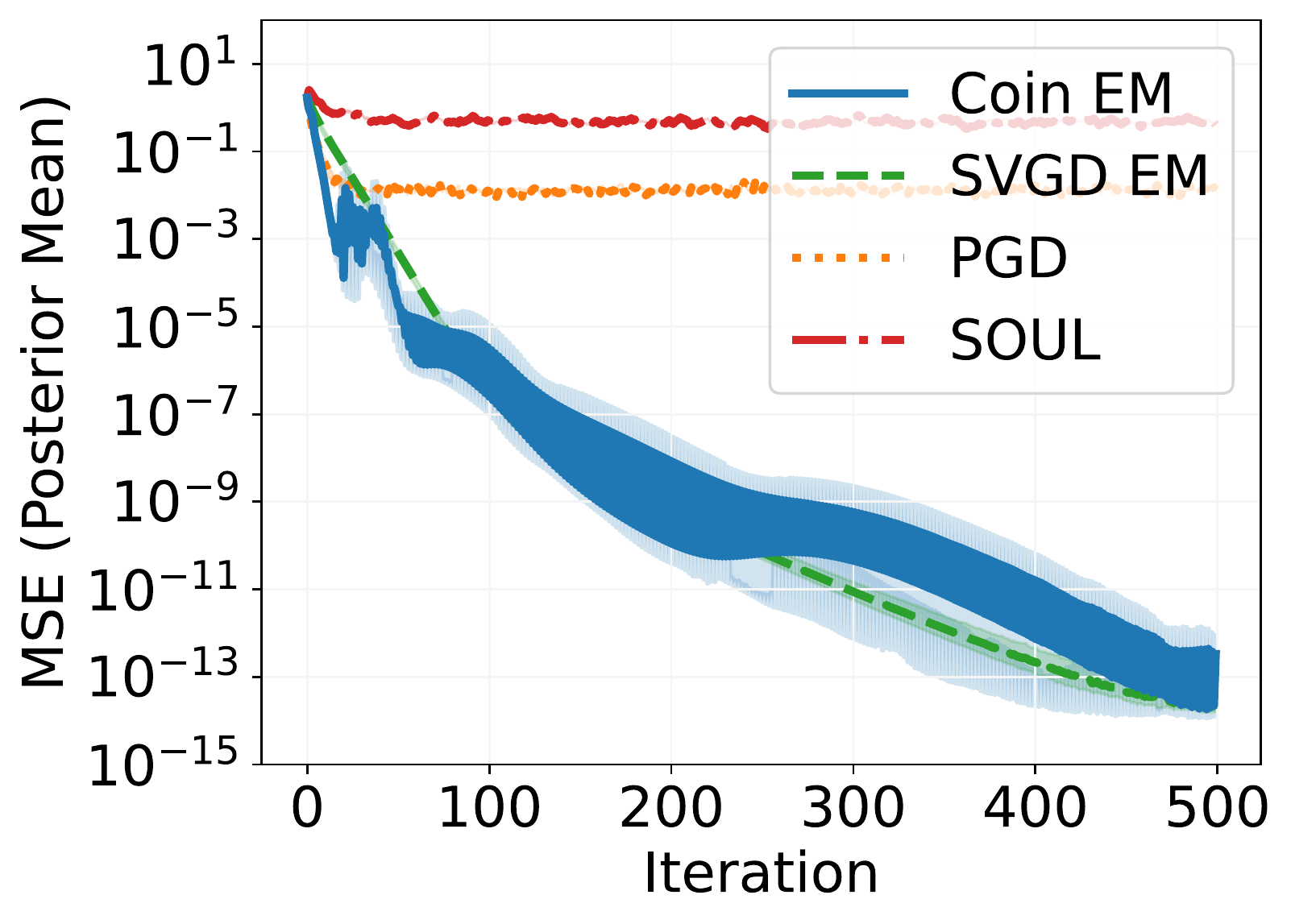}}
  \caption{\textbf{Additional results for the toy hierarchical model with $N=50$ particles.} MSE of the parameter estimate $\theta_{t}$ as a function of the learning rate after $T=500$ iterations (a); and MSE of the parameter estimate (b) and the posterior mean (c) as a function of the number of iterations, using the optimal learning rate from (a).} 
  \label{fig:toy_1_50_particles}
\end{figure}

\textbf{Additional results for different learning rates}. Next, we provide an additional demonstration of how the choice of learning rate can affect the parameter estimates generated by PGD, SOUL, and SVGD EM. In particular, in Fig. \ref{fig:toy_1_different_LR}, we plot the parameter estimates generated by SVGD EM, PGD, and SOUL for three different learning rates: the optimal learning rate as determined by Fig. \ref{fig:toy_1a_50_particles}, a smaller learning rate, and a larger learning rate `at the edge of stability'. The specific learning rates $\{\gamma_{\mathrm{opt}}, \gamma_{\mathrm{small}}, \gamma_{\mathrm{large}}\}$ used in this figure $\{0.79, 0.039, 100.00\}$ for SVGD EM,\footnote{We note that, when SVGD EM is implemented with an adaptive method such as Adagrad \cite{Duchi2011}, as is the case here, it remains stable even for very large values of the learning rate.} $\{0.26,0.013,1.15\}$ for PGD, and $\{0.0013, 0.000066, 0.018\}$ for SOUL. 

\begin{figure}[!htb]
  \centering
  \subfigure[SVGD EM \label{fig:toy_different_lr_SVGD}]{\includegraphics[width=0.325\textwidth]{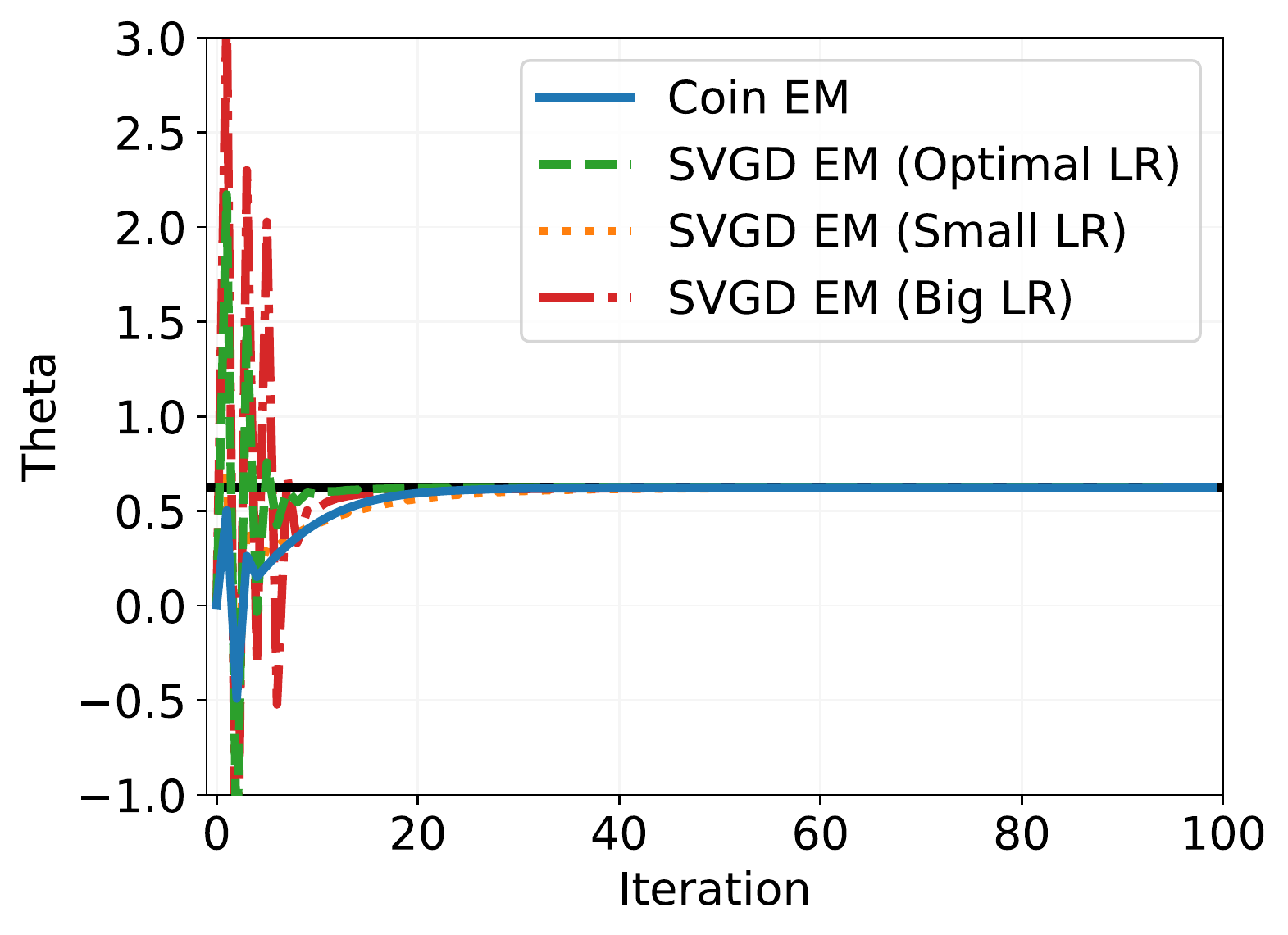}}
  \subfigure[PGD.]{\includegraphics[width=0.325\textwidth]{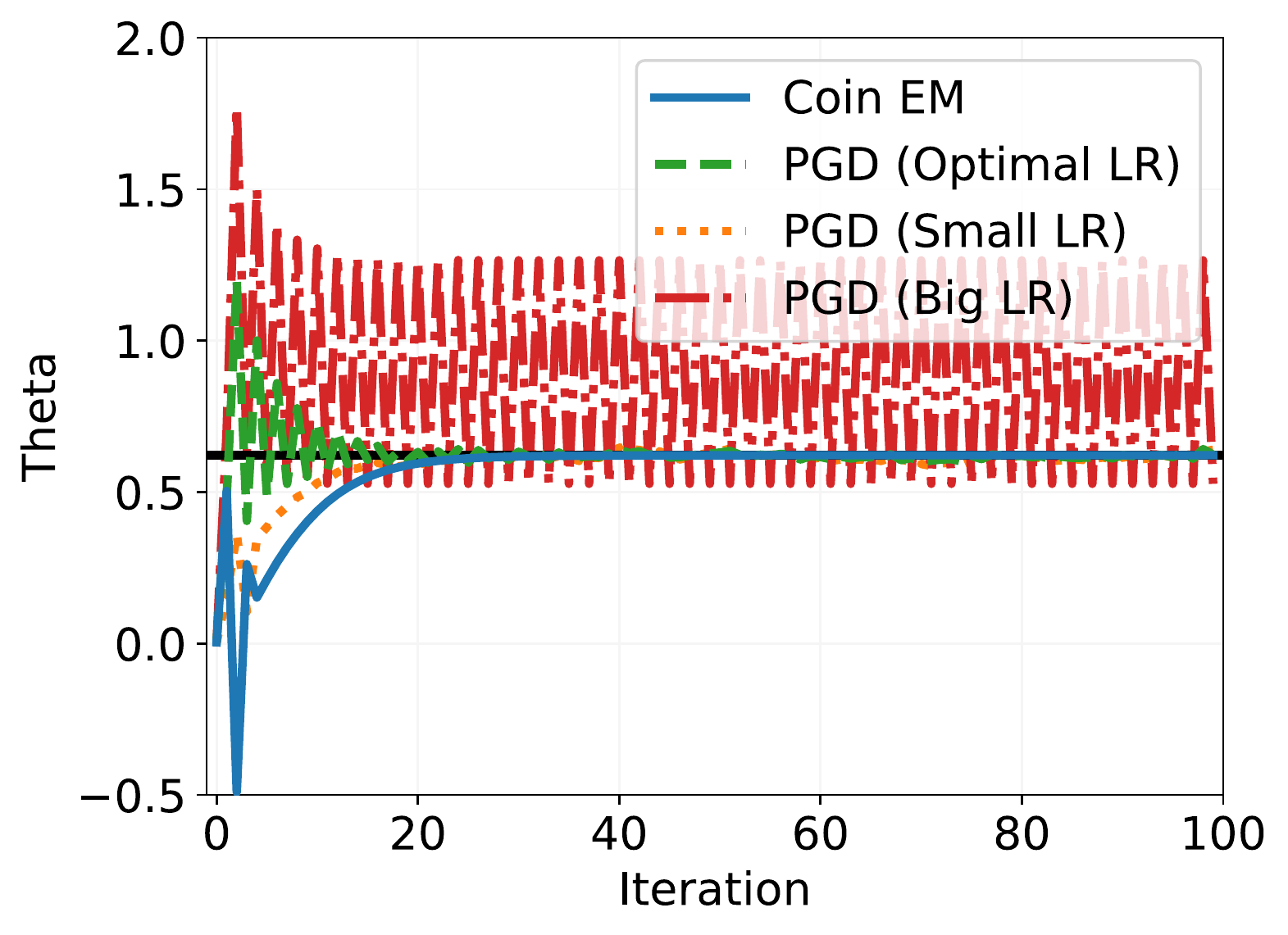}}
  \subfigure[SOUL.]{\includegraphics[width=0.325\textwidth]{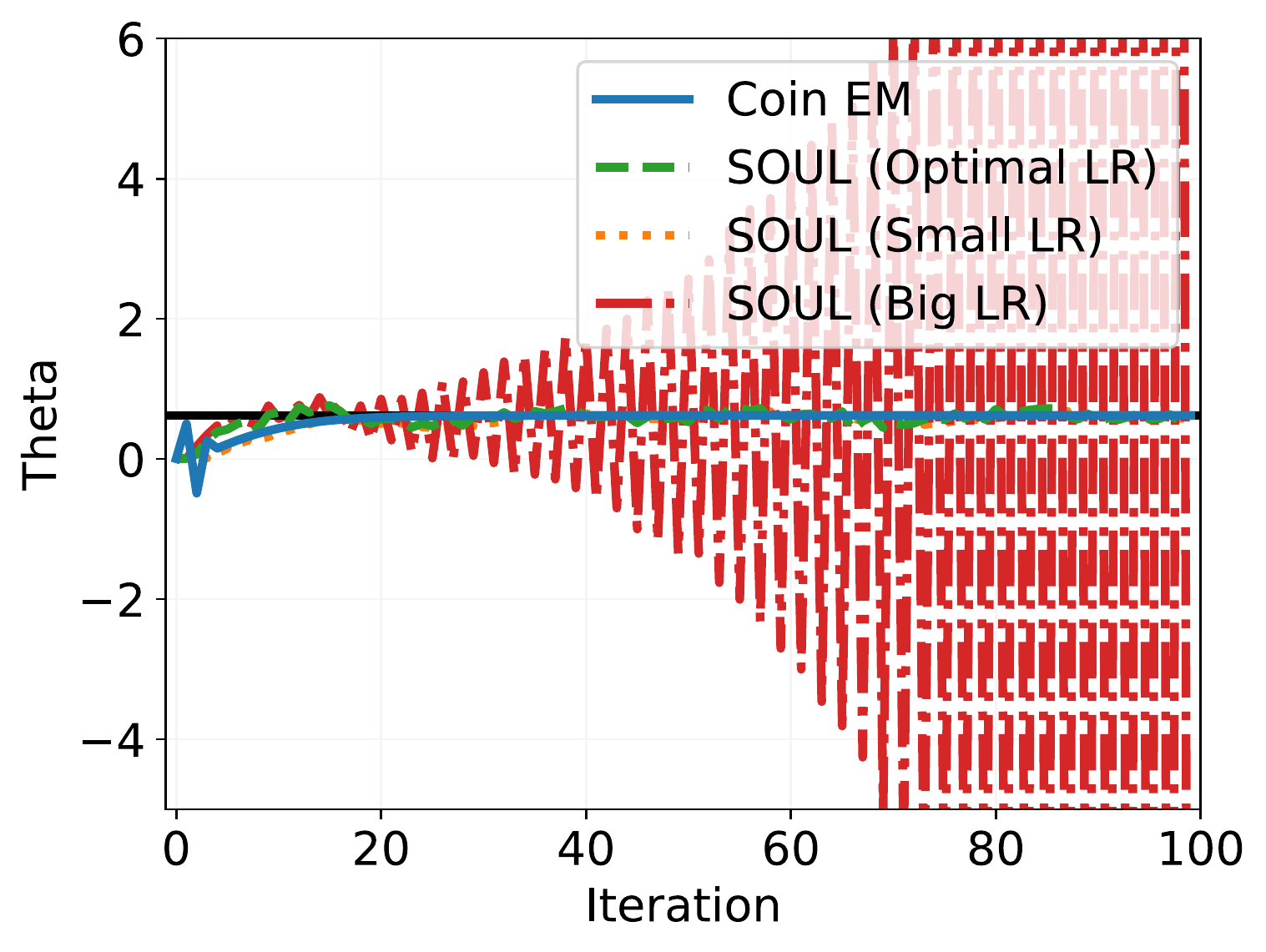}}
  \caption{\textbf{Additional results for the toy hierarchical model with different learning rates.} The sequence of parameter estimates generated by SVGD EM, PGD, and SOUL, for three different learning rates: the `optimal' learning rate from Fig. \ref{fig:toy_1a_50_particles}, a smaller learning rate, and a larger learning rate at the edge of stability.}
  \label{fig:toy_1_different_LR}
\end{figure}

\textbf{Additional results for a different initialization.} We now consider the impact of varying the initialization. In Fig. \ref{fig:toy_1_different_init}, we repeat the experiment in Sec. \ref{sec:toy_results}, but now using an initialization far away from the true parameter $\theta=1$. In particular, we now initialize the parameters using $\theta_0\sim\mathcal{N}(10,0.1^2)$. The results are rather similar to before. In particular, Coin EM and SVGD EM both converge rapidly to the true parameter, and tend to outperform the competitors. 

\begin{figure}[H]
  \centering
  \subfigure[$\mathrm{MSE}(\theta_{t})$ vs learning rate.]{\includegraphics[width=0.326\textwidth]{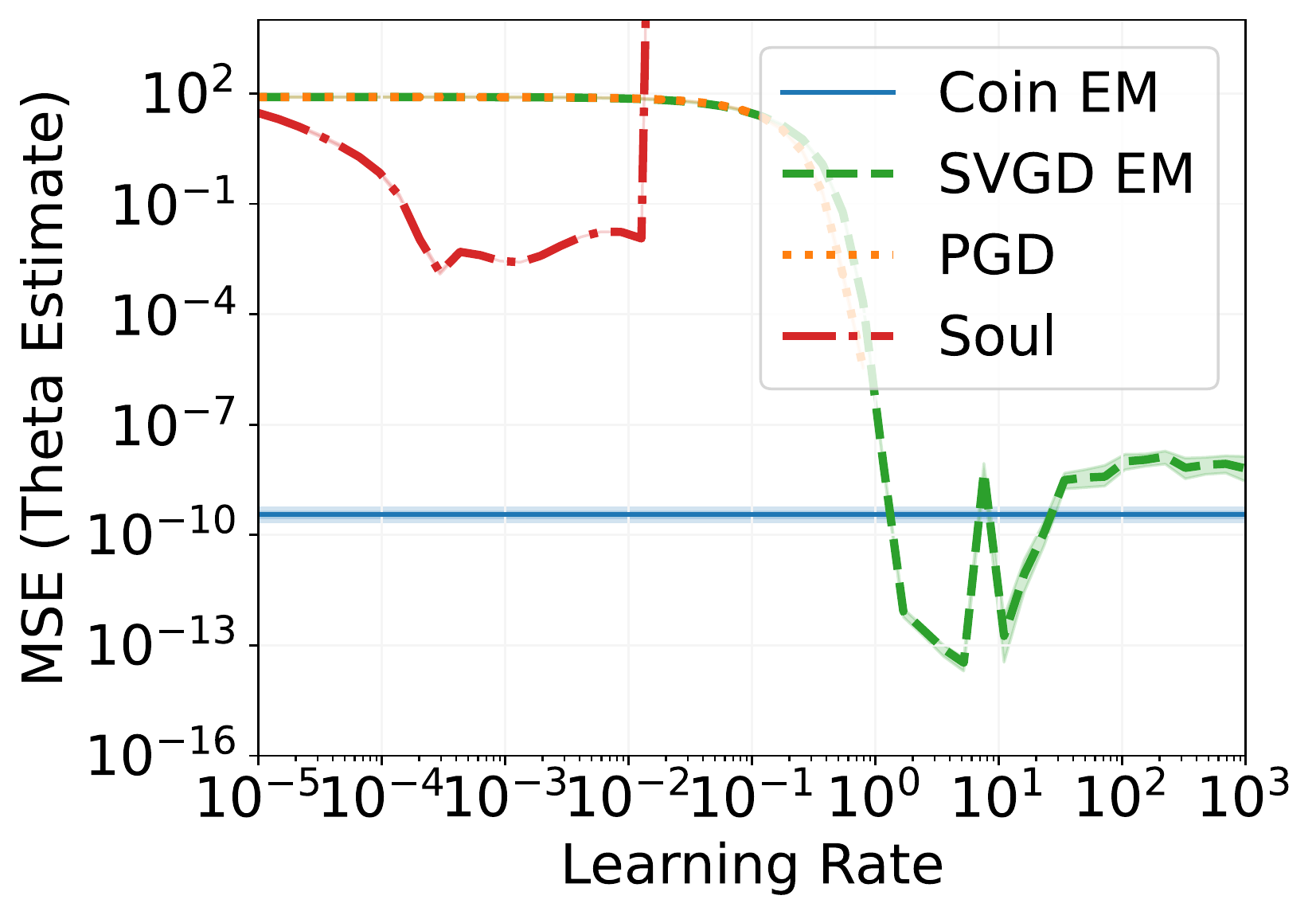}}
  \subfigure[$\mathrm{MSE}(\theta_{t})$ vs $t$.]{\includegraphics[width=0.318\textwidth]{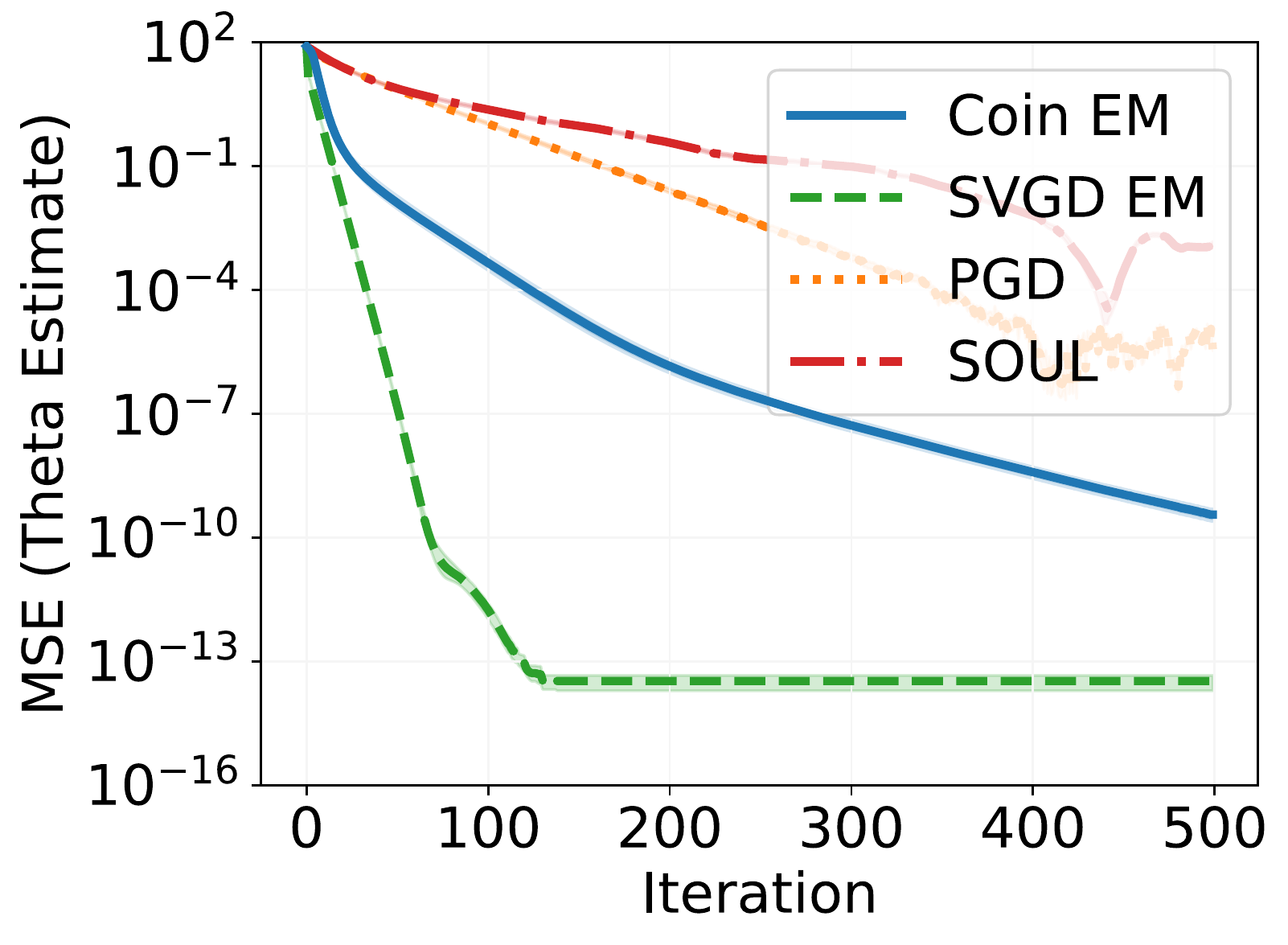}}
  \subfigure[$\mathrm{MSE}\left(\frac{1}{N}\sum_{i=1}^N z_t^{i}\right)$ vs $t$.]{\includegraphics[width=0.318\textwidth]{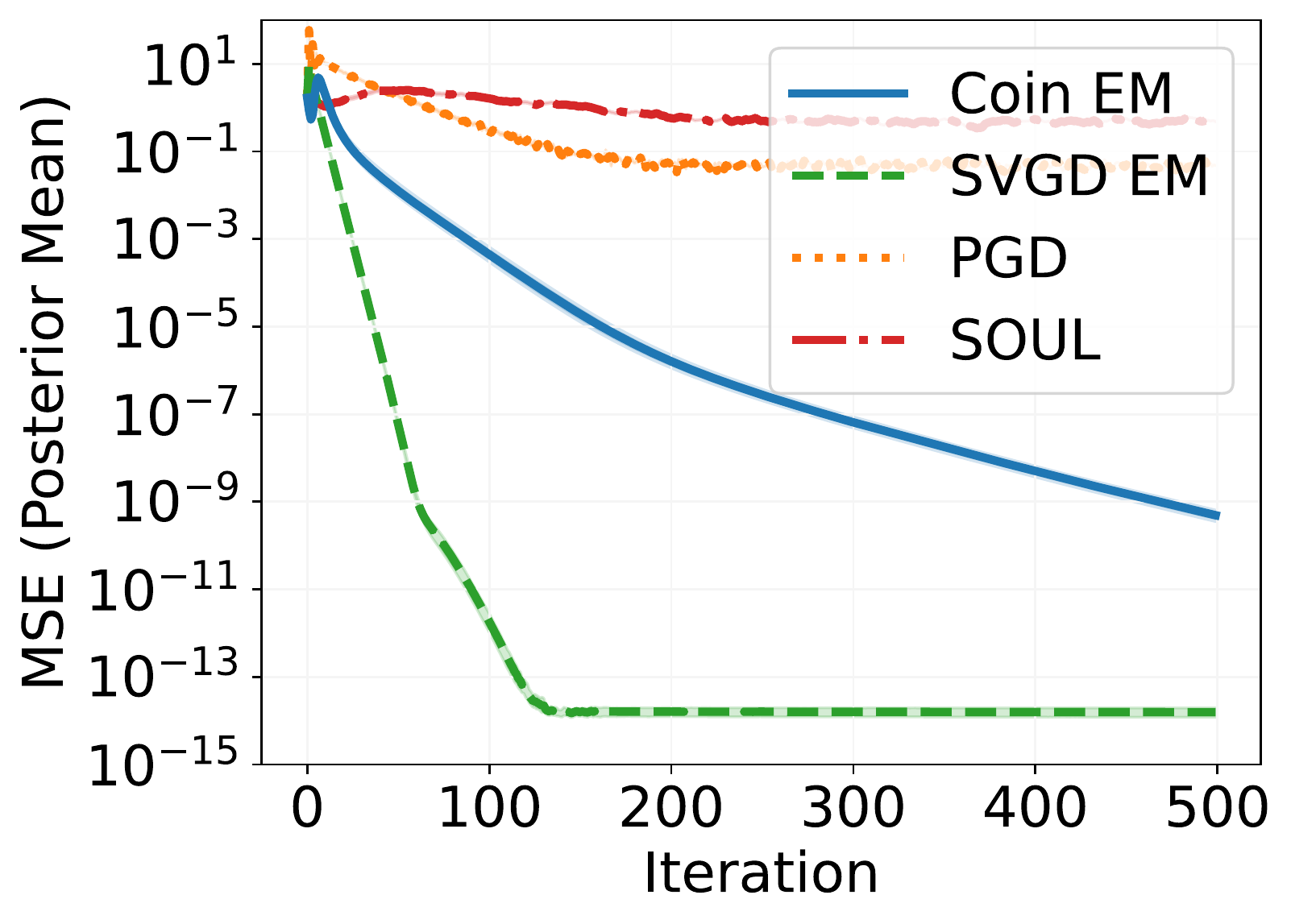}}
  \caption{\textbf{Additional results for the toy hierarchical model with initialization $\theta_0\sim\mathcal{N}(10,0.1)$.} MSE of the parameter estimate $\theta_{t}$ as a function of the learning rate after $T=500$ iterations (a); and MSE of the parameter estimate (b) and the posterior mean (c) as a function of the number of iterations, using the optimal learning rate from (a).}
  \label{fig:toy_1_different_init}
\end{figure}

\subsection{Bayesian logistic regression}
\label{sec:bayes-lr-add-results}
In this section, we present additional numerical results for the Bayesian logistic regression model described in Sec. \ref{sec:bayes_lr_results}. 

\textbf{Additional results for different initializations}. In Fig. \ref{fig:bayes_lr_different_init}, we plot the sequence of parameter estimates ouput by Coin EM, SVGD EM, PGD, PMGD, and SOUL, for different parameter intializations. In particular, we now initialize the parameter at $\theta_0 = 10$ or $\theta_0 = -10$, compared to $\theta_0=0$ in Fig. \ref{fig:bayes_lr_a}. 

In both cases, all of methods converge to a similar value. SOUL is known to obtain accurate estimates of $\theta_{*}$ in this example \cite{DeBortoli2021}, provided the learning rate is suitably small, and thus we use this as a benchmark. In these examples, where the parameter estimate is initialized far from the true parameter, the shorter transient exhibited by Coin EM relative to the other methods is even more evident.
\begin{figure}[htb]
  \centering
  \subfigure[$\theta_0 = 10$.]{\includegraphics[width=0.45\textwidth, trim = 10 10 20 20, clip]{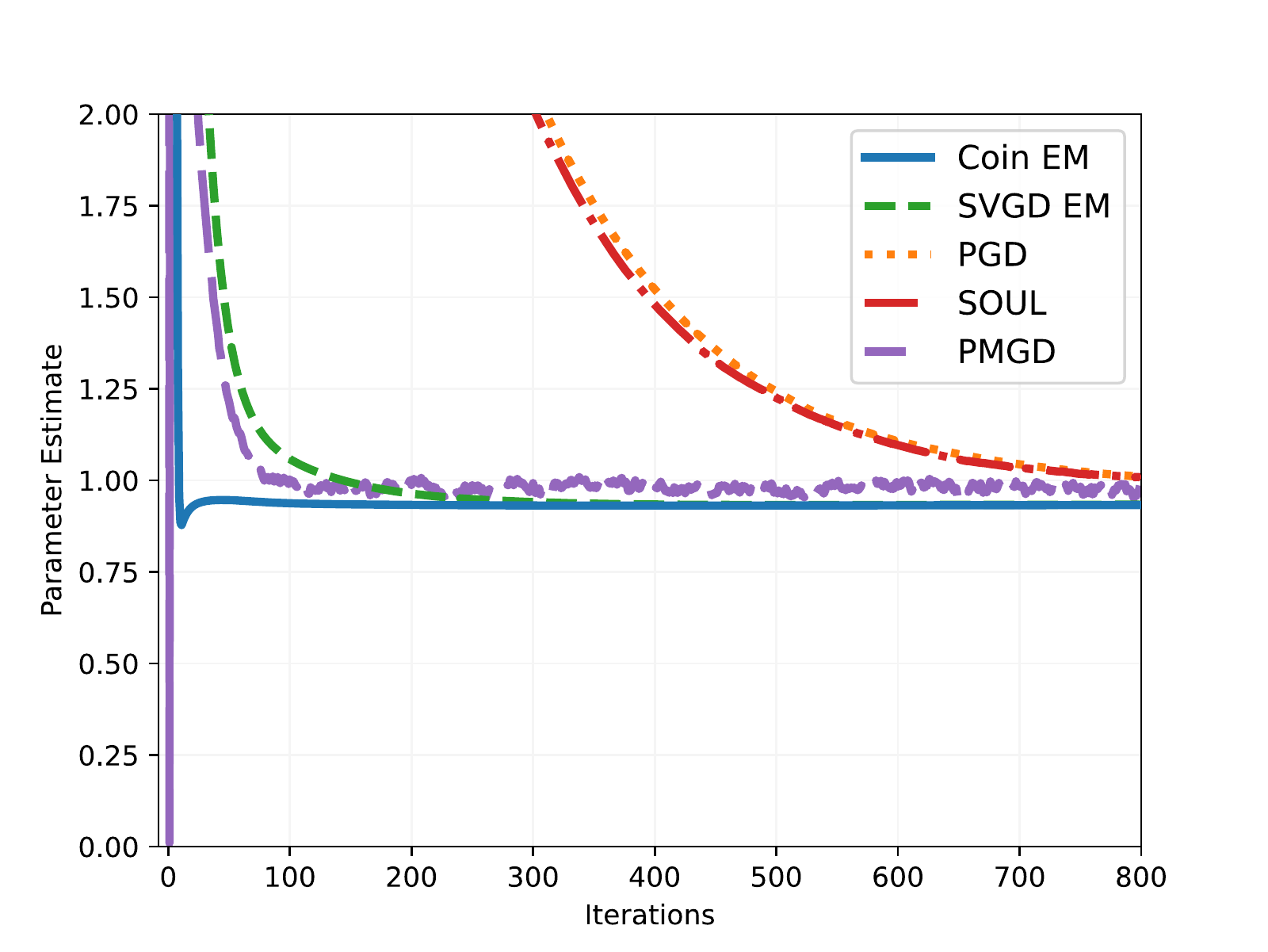}}
  \subfigure[$\theta_0 = -10$.]{\includegraphics[width=0.45\textwidth, trim = 10 10 20 20, clip]{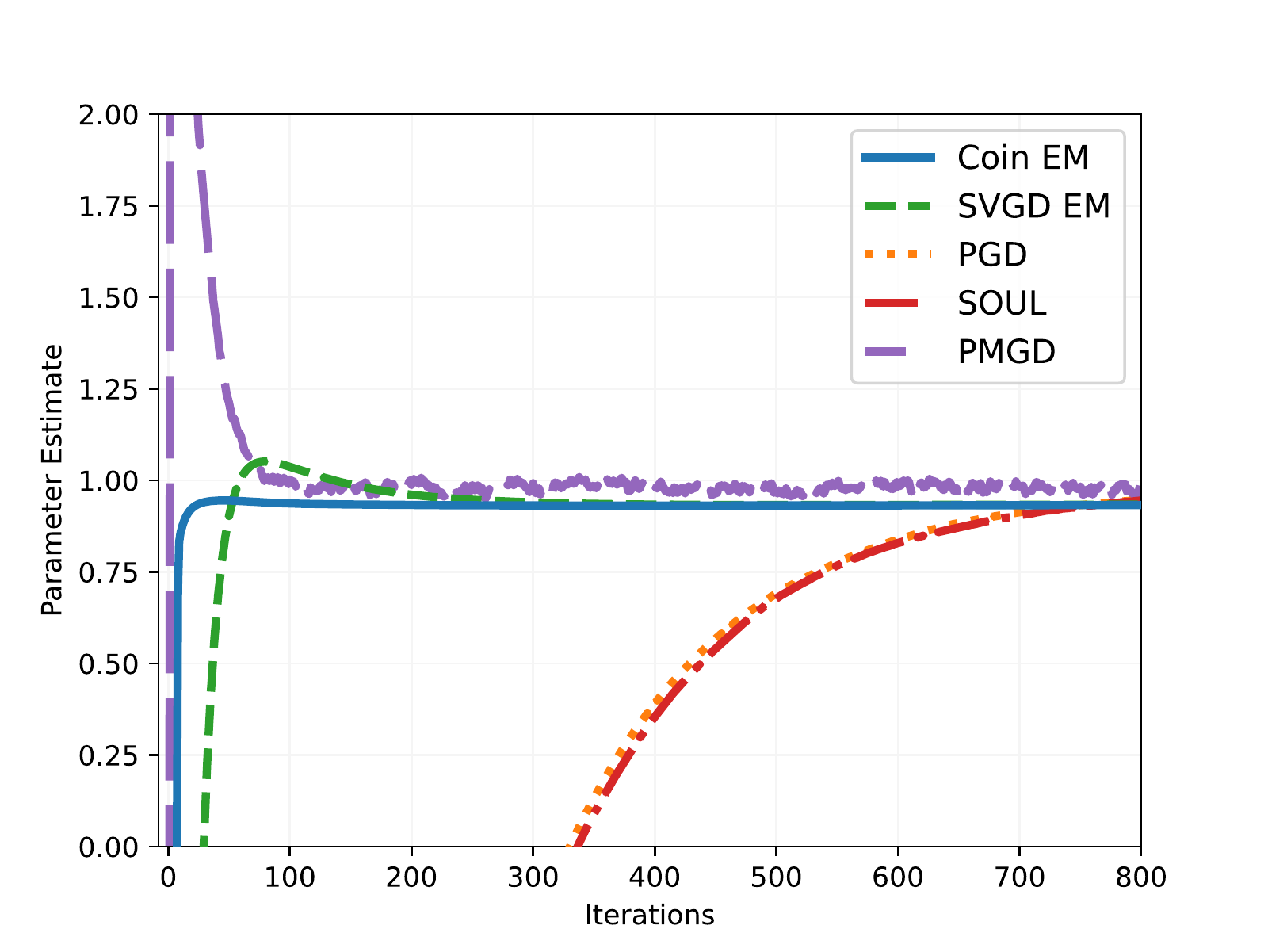}}
  \caption{\textbf{Additional results for the Bayesian logistic regression in Sec. \ref{sec:bayes_lr_results}}. Parameter estimates $\theta_{t}$ from Coin EM, SVGD EM, PGD, SOUL, and PMGD, initialized at (a) $\theta_0 = 10$ and (b)   $\theta_0 = -10$.}
  \label{fig:bayes_lr_different_init}
\end{figure}

\textbf{Additional results for different learning rates}. In Fig. \ref{fig:bayes_lr_different_lr}, we plot the sequence of parameter estimates ouput by Coin EM, SVGD EM, PGD, PMGD, and SOUL, for different learning rates. In particular, we now use smaller learning rates of $\gamma = 0.001$ (PGD, PMGD, SOUL) or $\gamma = 0.005$ (SVGD EM), or larger learning rates of $\gamma = 0.1$ (PGD, PMGD, SOUL) or $\gamma = 0.5$ (SVGD EM). For reference, we also include the results for the original learning rates of $\gamma = 0.02$ (PGD, PMGD, SOUL) or $\gamma = 0.2$ (SVGD EM) in Fig. \ref{fig:bayes_lr_a} (see Sec. \ref{sec:bayes-lr-details}). 

\begin{figure}[htb]
  \centering
  \subfigure[Original Learning Rate.\label{fig:bayes_lr_different_lr_a}]{\includegraphics[width=0.325\textwidth, trim = 0 0 0 0, clip]{figs/type1/lr/experiment_1/param_plot.pdf}}
  \subfigure[Smaller Learning Rate. \label{fig:bayes_lr_different_lr_b}]{\includegraphics[width=0.325\textwidth, trim = 0 0 0 0, clip]{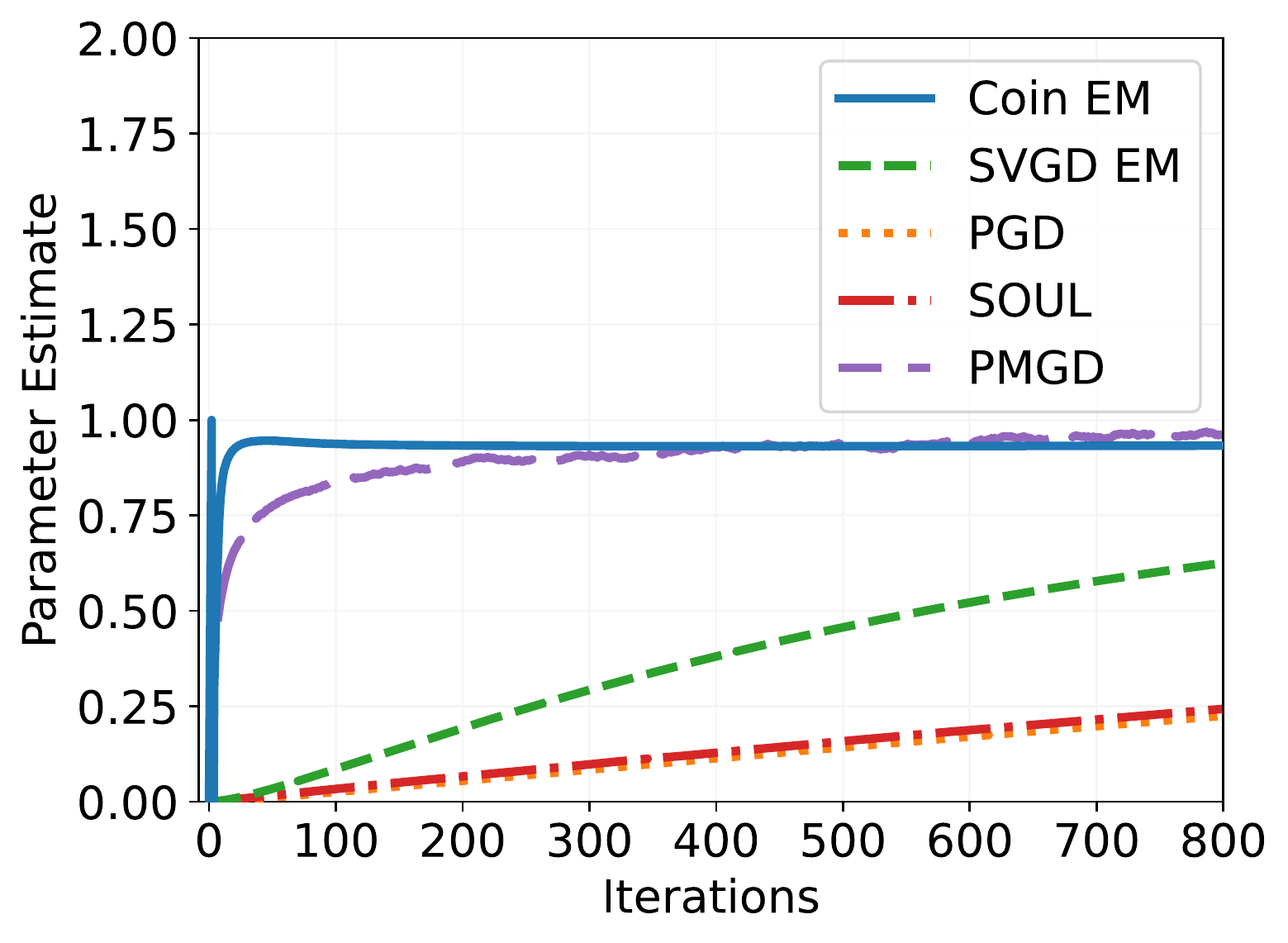}}
  \subfigure[Larger Learning Rate.\label{fig:bayes_lr_different_lr_c}]{\includegraphics[width=0.325\textwidth, trim = 0 0 0 0, clip]{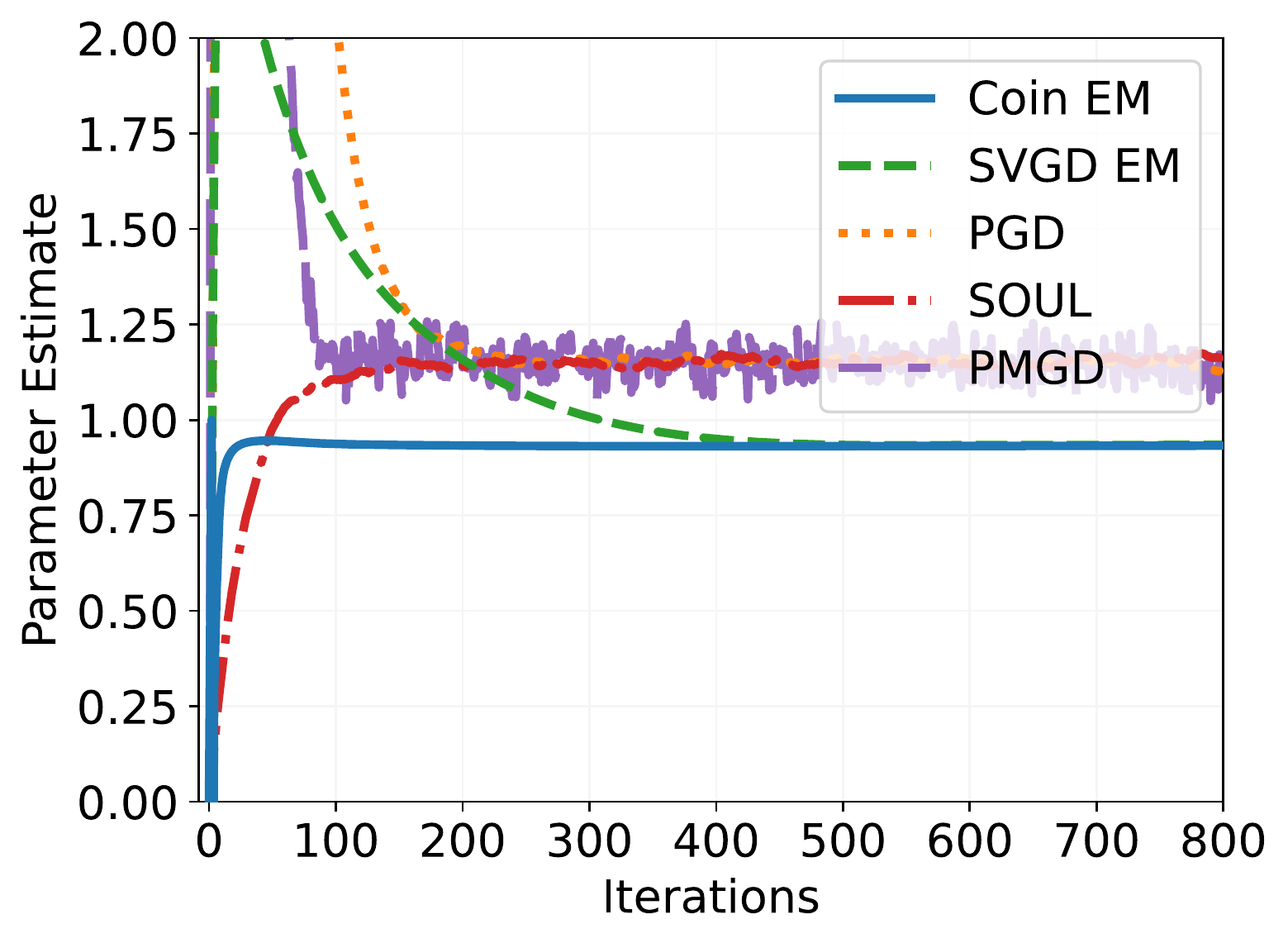}}
  \caption{\textbf{Additional results for the Bayesian logistic regression in Sec. \ref{sec:bayes_lr_results}}. Parameter estimates $\theta_{t}$ from Coin EM, SVGD EM, PGD, SOUL, and PMGD, using (a) smaller learning rates $\gamma = 0.001$ (PGD, SOUL, PMGD) and $\gamma=0.005$ (SVGD EM); and (b) larger learning rates $\gamma = 0.1$ (PGD, SOUL, PMGD) and $\gamma=0.5$ (SVGD EM).}
  \label{fig:bayes_lr_different_lr}
\end{figure}

These figures illustrate the difficulties associated with tuning the learning rate for PGD, PMGD, SOUL and, to a lesser extent, SVGD EM. In particular, if the learning rate is chosen too small (Fig. \ref{fig:bayes_lr_different_lr_b}), then convergence to the true parameter $\theta_{*}$ is painfully slow. On the other hand, if one selects a larger learning rate (Fig. \ref{fig:bayes_lr_different_lr_c}), convergence is more rapid, but one incurs a (significant) bias in the resulting parameter estimates. This is evident if one compares the asymptotic parameter estimates obtained in Fig. \ref{fig:bayes_lr_different_lr_a} and in Fig. \ref{fig:bayes_lr_different_lr_c}. For PGD, PMGD, and SOUL, this bias originates in the bias associated with an Euler-Maruyama discretization of the Langevin dynamics; see, e.g., the discussion in \cite[][Sec. 2]{Kuntz2023}. Interestingly, SVGD EM does not seem to incur this bias to the same extent, despite the error associated with the finite-particle SVGD dynamics. Coin EM, of course, has no dependence on the learning rate, and is consistent across these experiments.

\textbf{Additional results for different numbers of particles}. Finally, we consider the impact of changing the number of particles uses in the latent variable updates on the predictive performance. In Fig. \ref{fig:bayes_lr_different_N}, we repeat the experiment used to generate Fig. \ref{fig:bayes_lr_c}, but now using $N=5,20,100$ particles. In this example, we see that there is little to be gained from increasing the number of particles in terms of the predictive performance, other than a minor increase in the performance of the learning-rate dependent methods (SVGD EM, PGD, SOUL) for sub-optimal choices of the learning rate. In this example, the posteriors are peaked and unimodal - see \cite[][Fig. 2]{DeBortoli2021} - and can be approximated well using even a single particle in the vicinity of the modes.

\begin{figure}[htb]
\vspace{-3mm}
  \centering
  \subfigure[$N=5$.\label{fig:bayes_lr_different_N_a}]{\includegraphics[width=0.325\textwidth, trim = 0 0 0 0, clip]{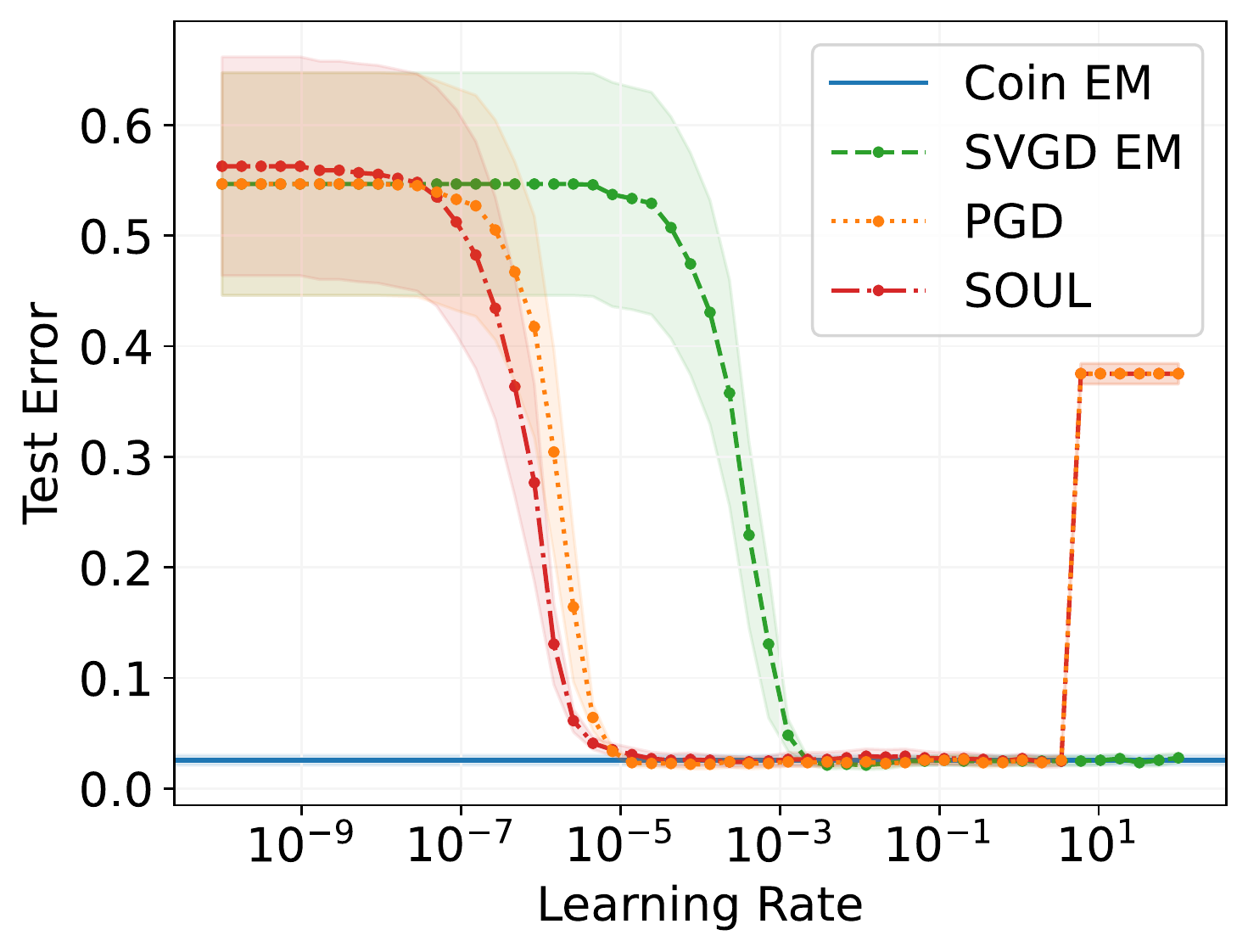}}
  \subfigure[$N=20$.\label{fig:bayes_lr_different_N_b}]{\includegraphics[width=0.325\textwidth, trim = 0 0 0 0, clip]{figs/type1/lr/experiment_1/test_error_vs_lr.pdf}}
  \subfigure[$N=100$.\label{fig:bayes_lr_different_N_c}]{\includegraphics[width=0.325\textwidth, trim = 0 0 0 0, clip]{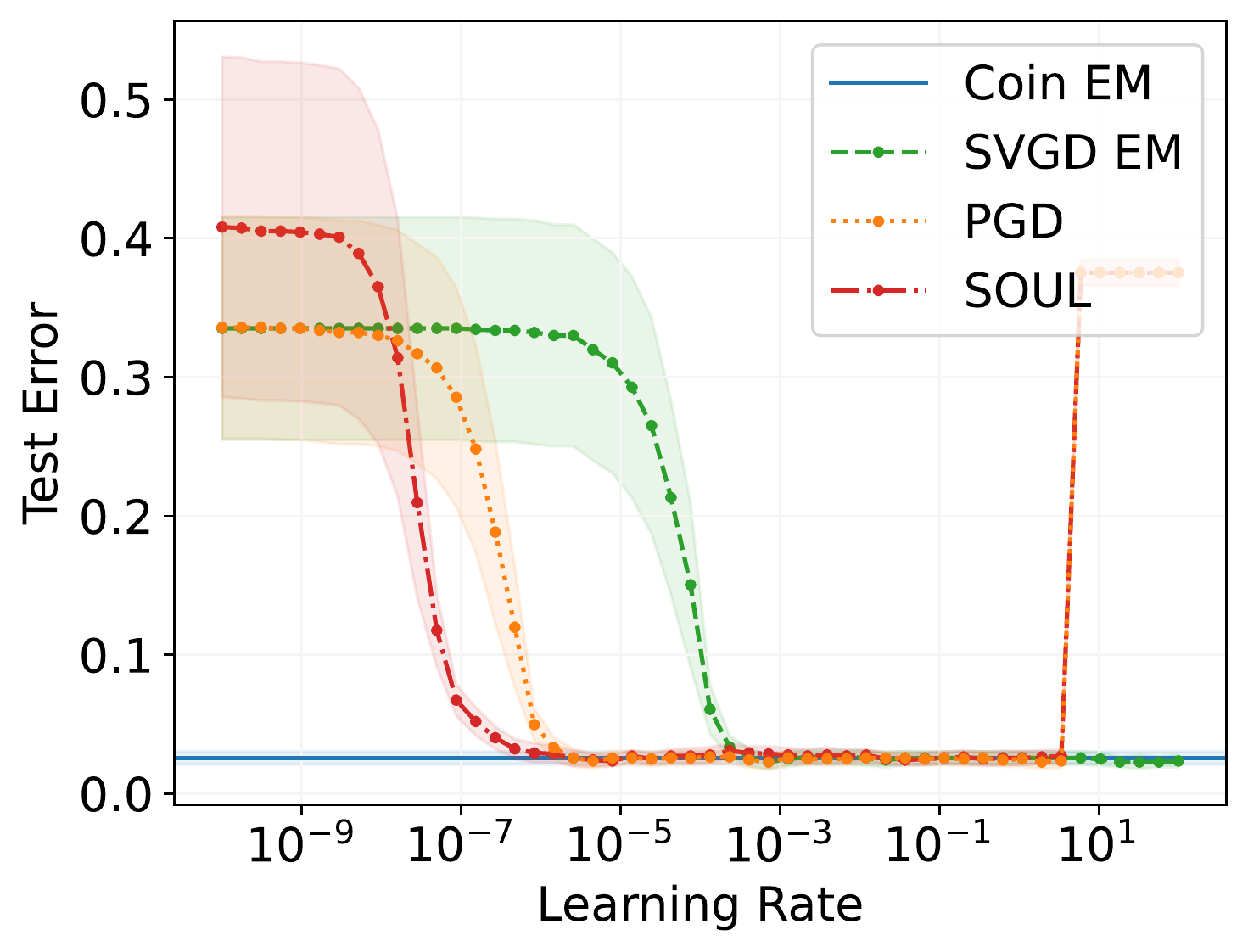}}
  \caption{\textbf{Additional results for the Bayesian logistic regression in Sec. \ref{sec:bayes_lr_results}}. Test error as a function of the learning rate, for different numbers of particles.}
  \label{fig:bayes_lr_different_N}
  \vspace{-3mm}
\end{figure}

\subsection{Bayesian logistic regression (alternative model)}
\label{sec:bayes-lr-alt-results}
We now present numerical results for the alternative Bayesian logistic regression model described in App. \ref{sec:bayes-lr-alt-details}. In Fig. \ref{fig:logistic_regression_2_auc}, we compare SVGD EM and Coin EM against PGD and SOUL, plotting the area under the receiver operator characteristic curve (AUC) as a function of the learning rate, after running each algorithm with 10 particles for 1000 iterations. For each of the datasets considered, the predictive performance of Coin EM is similar to the performance of SVGD EM, PGD and SOUL with well tuned learning rates. In comparison to the Bayesian logistic regression studied in Section \ref{sec:bayes_lr_results}, here the predictive performance of the SVGD EM, PGD, and SOUL is rather more sensitive to the learning rate, particularly for the Covertype dataset. In particular, in this case there is a much smaller range of values for which these methods exhibit performance on-par with, or superior to, Coin EM.

\begin{figure}[htb]
  \centering
  \subfigure[Covertype]{\includegraphics[width=0.325\textwidth]{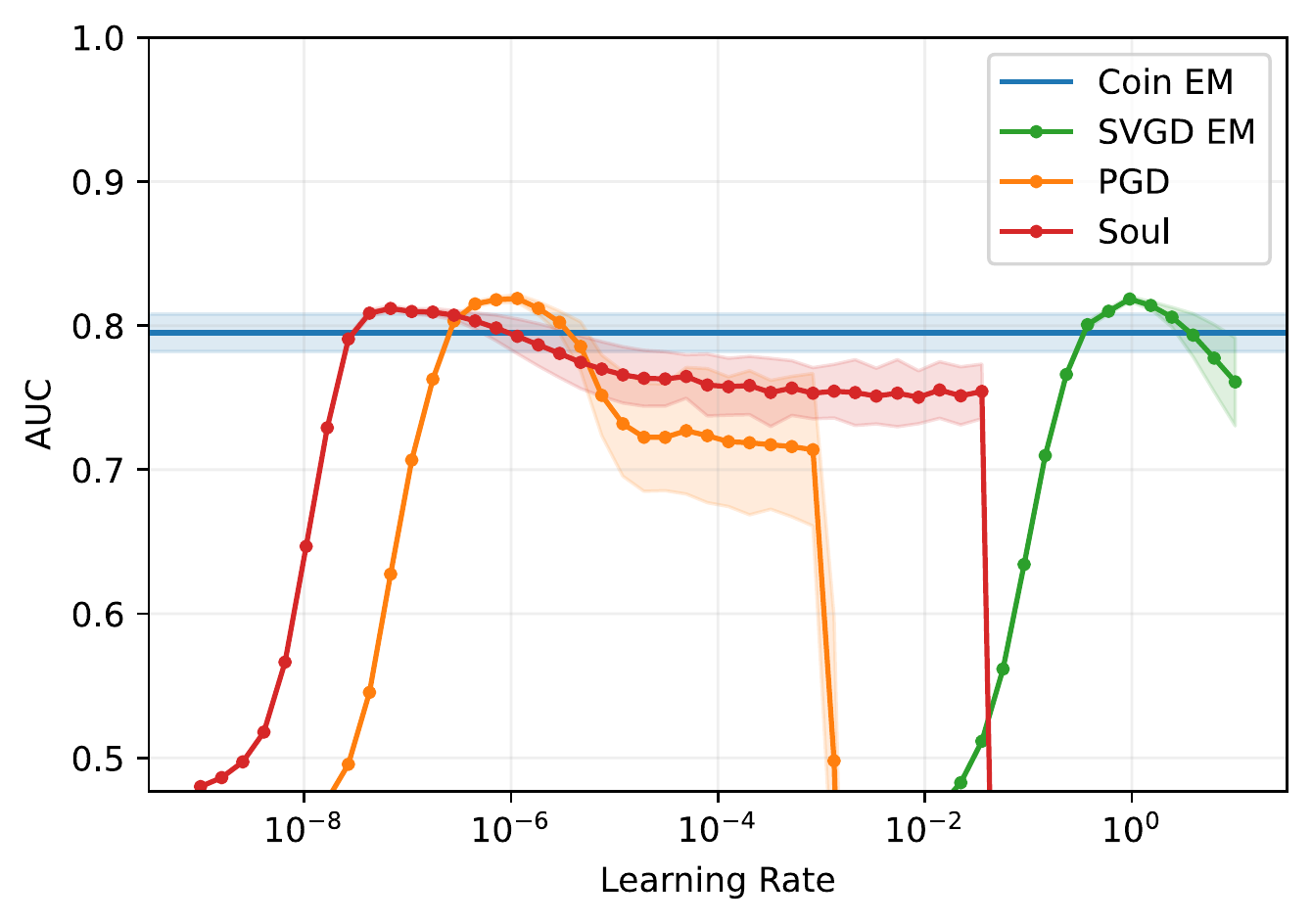}}
  \subfigure[Banknote]{\includegraphics[width=0.325\textwidth]{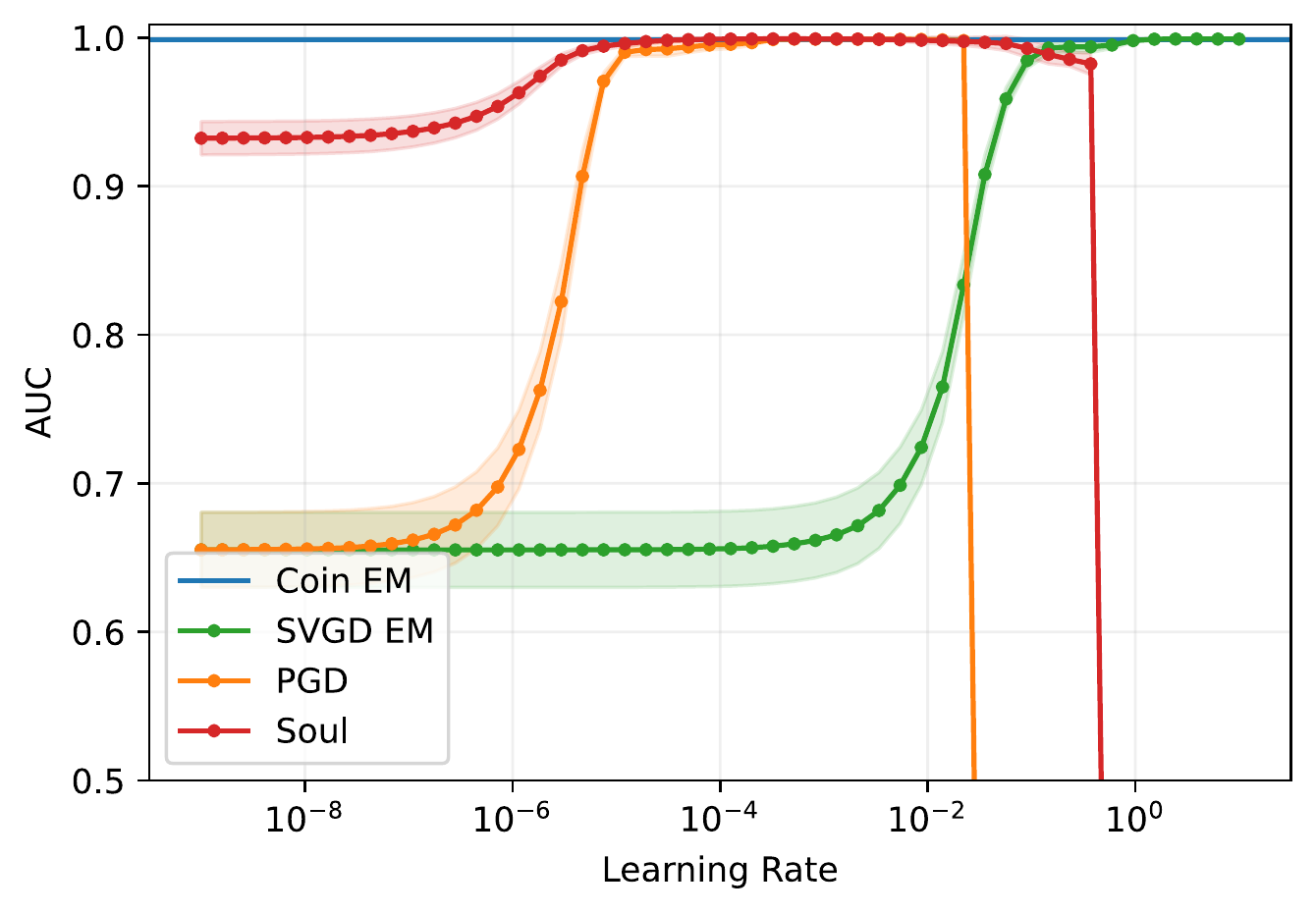}}
  \subfigure[Cleveland]{\includegraphics[width=0.325\textwidth]{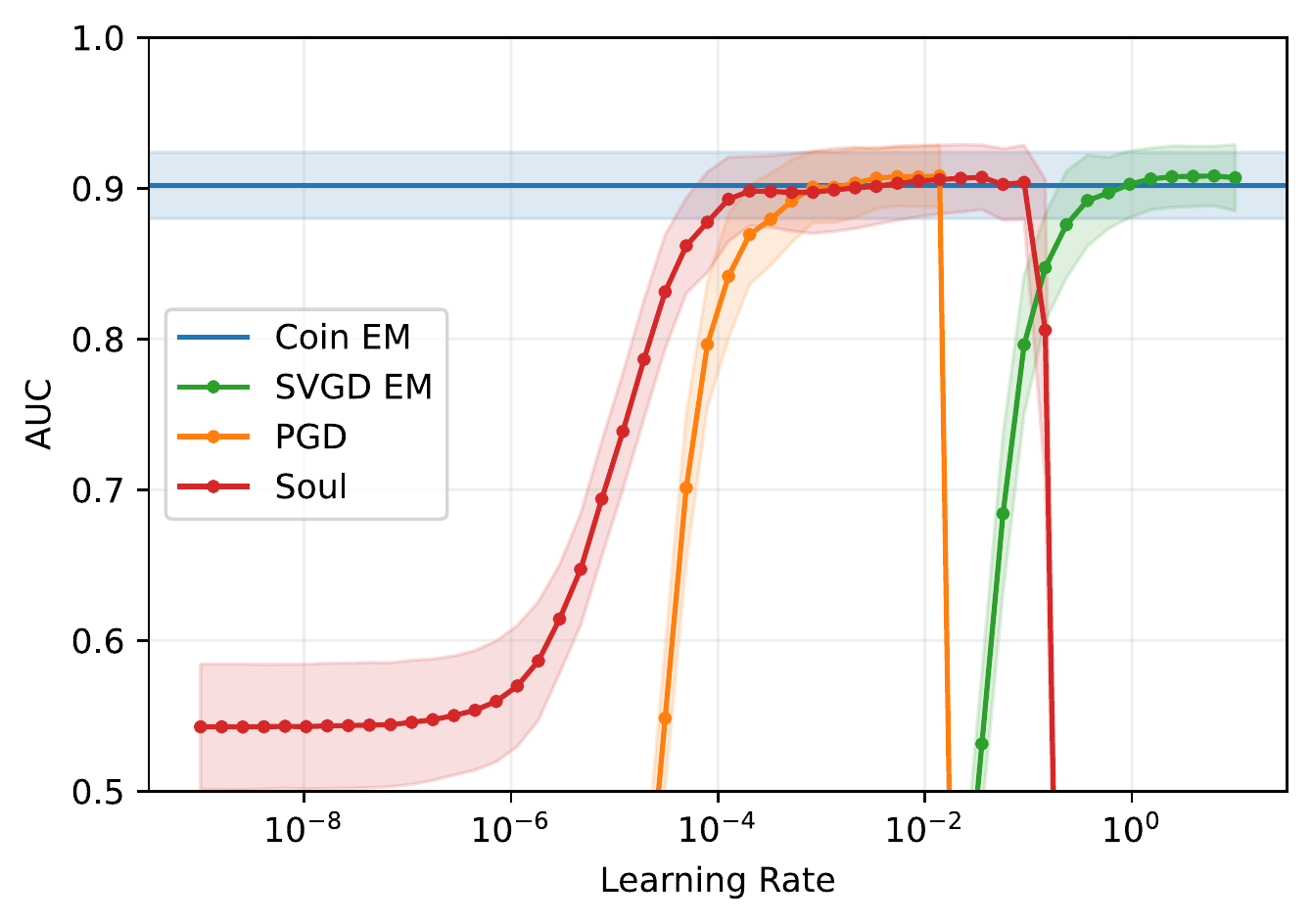}}
  \caption{\textbf{Results for the alternative Bayesian logistic regression in App. \ref{sec:bayes-lr-alt-details}}. AUC as a function of the learning rate, averaged over 10 random test-train splits.}
  \label{fig:logistic_regression_2_auc}
\end{figure}

\subsection{Bayesian neural network}
\label{sec:bayes-nn-alt-results}

We now present numerical results for the alternative Bayesian neural network model described in App. \ref{sec:bayes-nn-alt-details}. 

In Fig. \ref{fig:bnn_mnist_compare_lr}, we plot the test error achieved by SVGD EM, PGD, and SOUL for several different choices of learning rate; and for several different choices of the number of particles. For comparison, we also plot the test error achieved by Coin EM. We also include results for SVGD EM, PGD, and SOUL when using a scaling heuristic recommended in \cite[][Sec. 2]{Kuntz2023}, which is designed to stabilize the updates and avoid ill-conditioning. 
We refer to these methods as SVGD EM', PGD', and SOUL'. 

Unsurprisingly, all methods other than Coin EM are highly dependent on the choice of learning rate. While, in all cases, we observe that the convergence rate can be improved by increasing the learning rate, this approach can only go so far. In particular, if the learning rate is increased much beyond the largest values considered in Figure \ref{fig:bnn_mnist_compare_lr}, then the updates are likely to become unstable (see, e.g., Fig. \ref{fig:bnn_mnist_compare_lr_soul_N5} - \ref{fig:bnn_mnist_compare_lr_soul_N100}).

{\setlength{\subfigcapskip}{-1.2mm}
\begin{figure}[H]
  \centering
  \vspace{-3mm}
  \subfigure[PGD ($N=5$). \label{fig:bnn_mnist_compare_lr_pgd_N5}]{\includegraphics[width=0.25\textwidth, trim=0 0 0 0, clip]{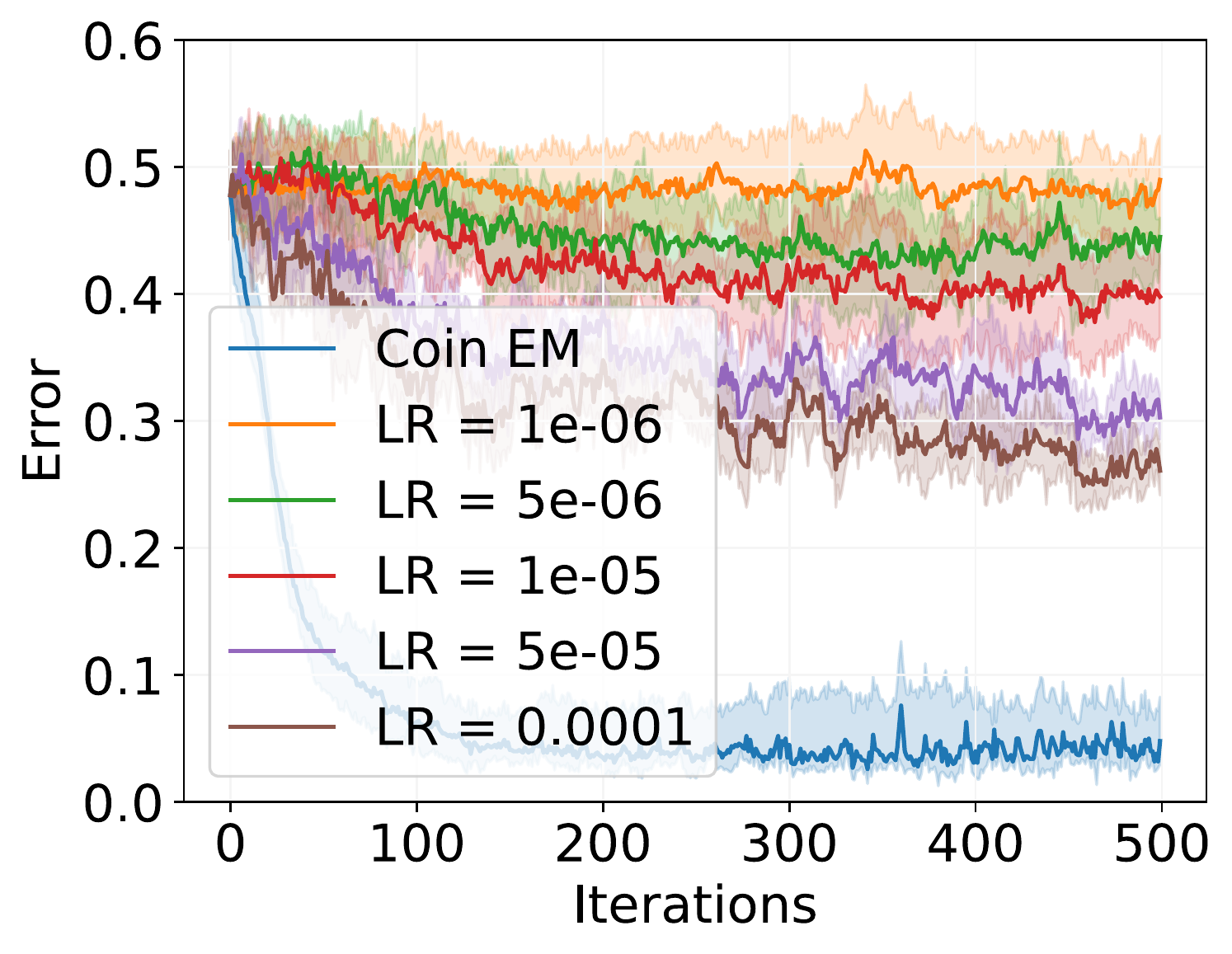}}
  \subfigure[PGD ($N=20$). \label{fig:bnn_mnist_compare_lr_pgd_N20}]{\includegraphics[width=0.25\textwidth, trim=0 0 0 0, clip]{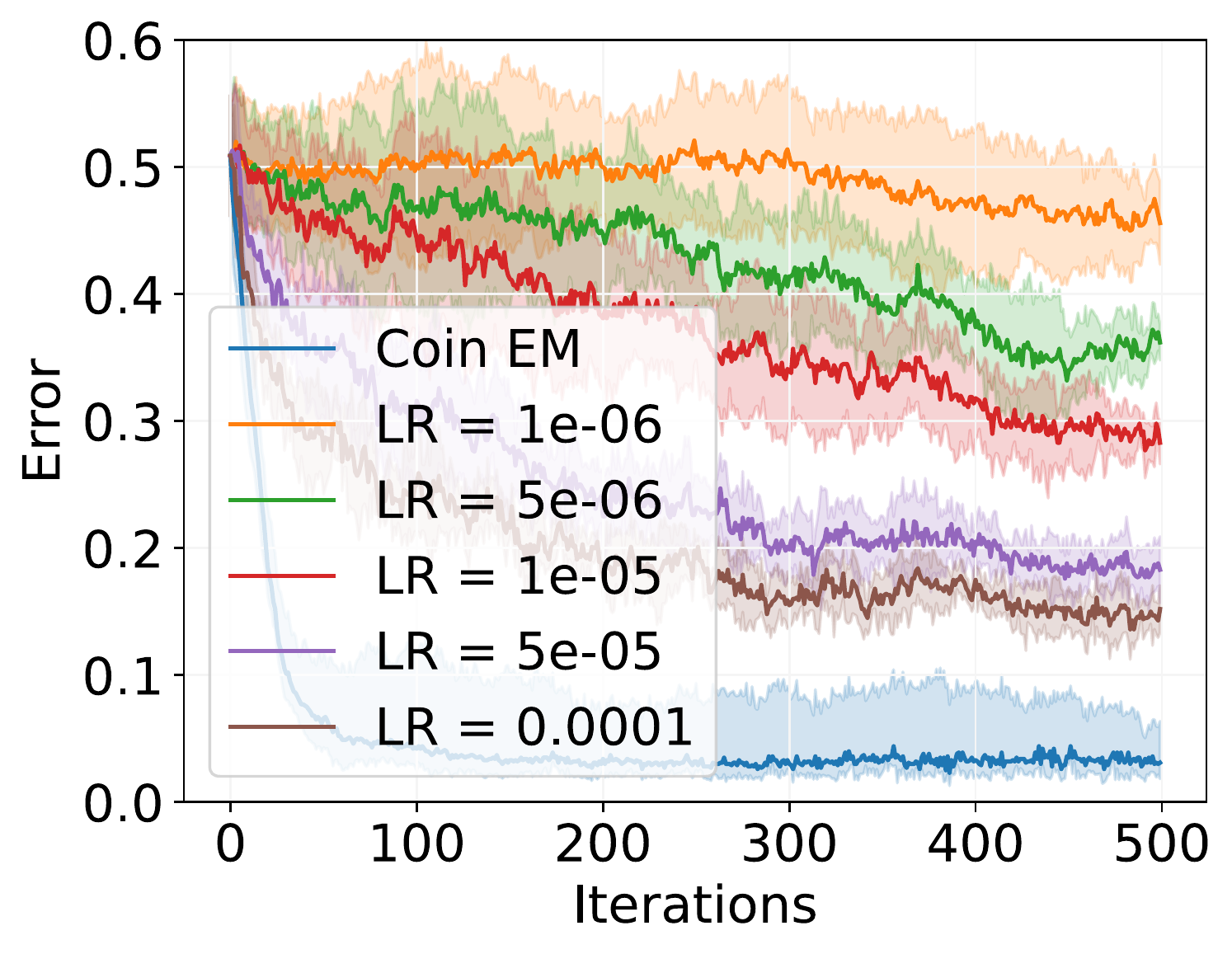}}
  \subfigure[PGD ($N=100$). \label{fig:bnn_mnist_compare_lr_pgd_N100}]{\includegraphics[width=0.25\textwidth, trim=0 0 0 0, clip]{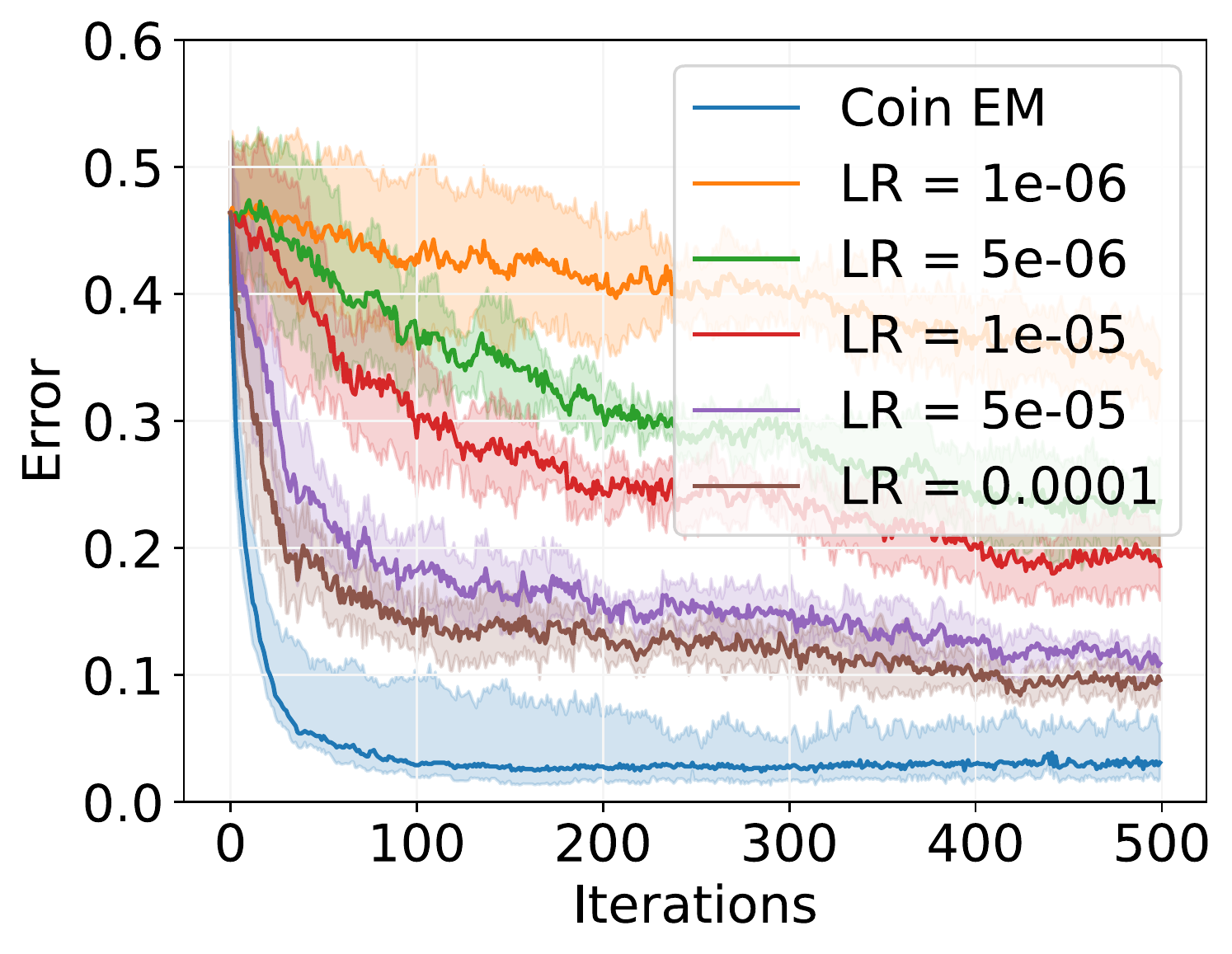}}
  \vspace{-1mm}
  \vspace{-1mm}
  \subfigure[PGD' ($N=5$). \label{fig:bnn_mnist_compare_lr_pgd_h_N5}]{\includegraphics[width=0.25\textwidth, trim=0 0 0 0, clip]{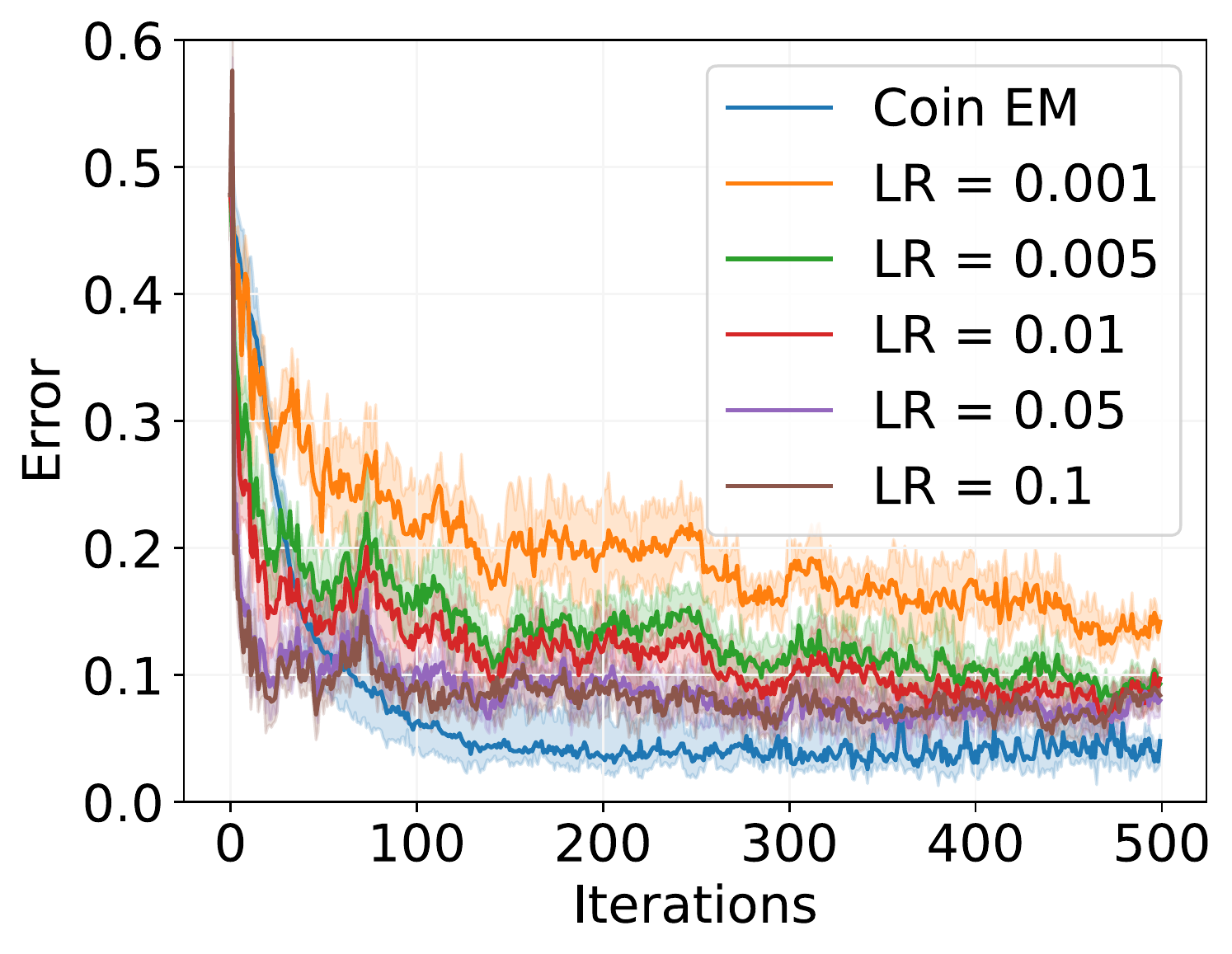}}
  \subfigure[PGD' ($N=20$). \label{fig:bnn_mnist_compare_lr_pgd_h_N20}]{\includegraphics[width=0.25\textwidth, trim=0 0 0 0, clip]{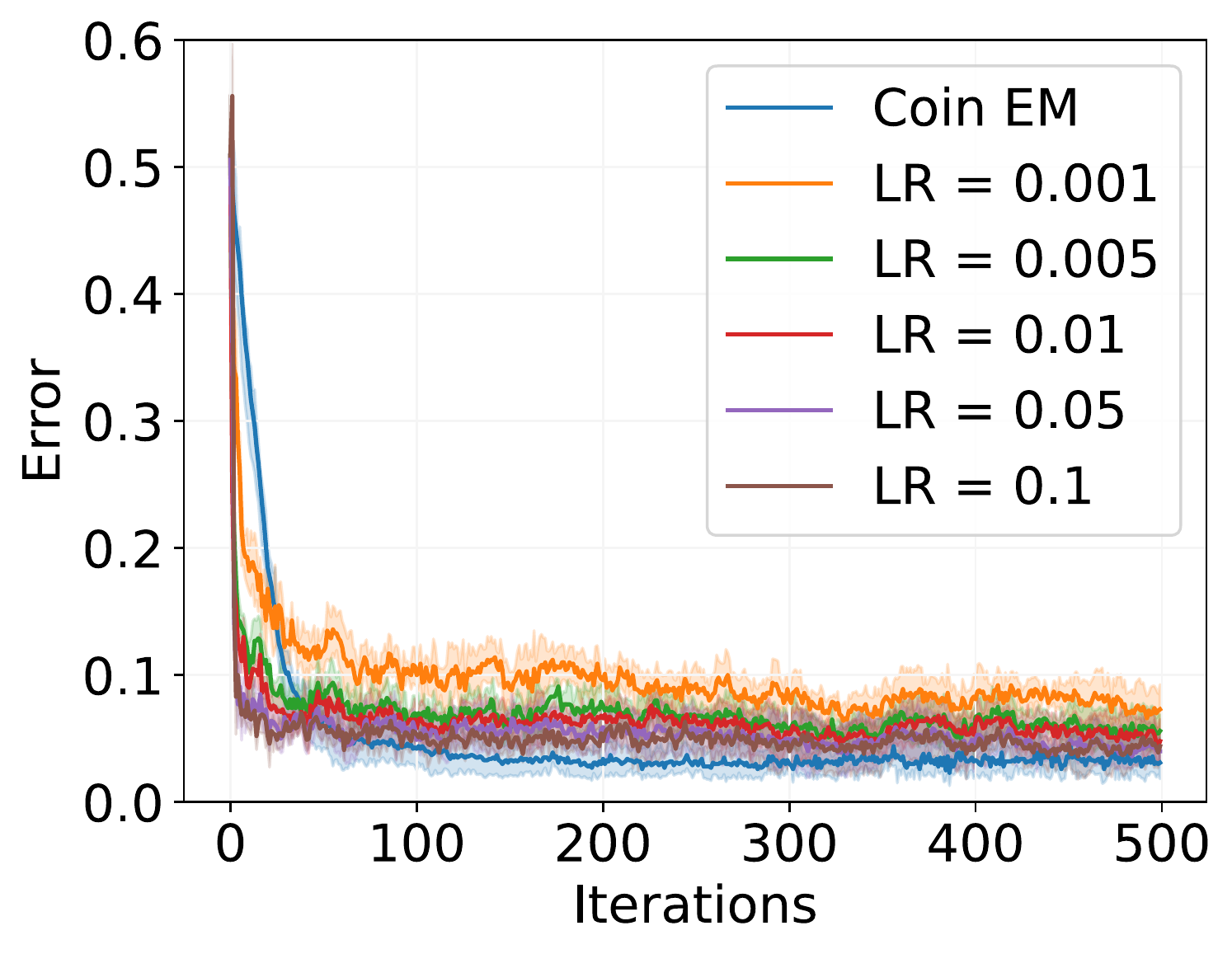}}
  \subfigure[PGD' ($N=100$). \label{fig:bnn_mnist_compare_lr_pgd_h_N100}]{\includegraphics[width=0.25\textwidth, trim=0 0 0 0, clip]{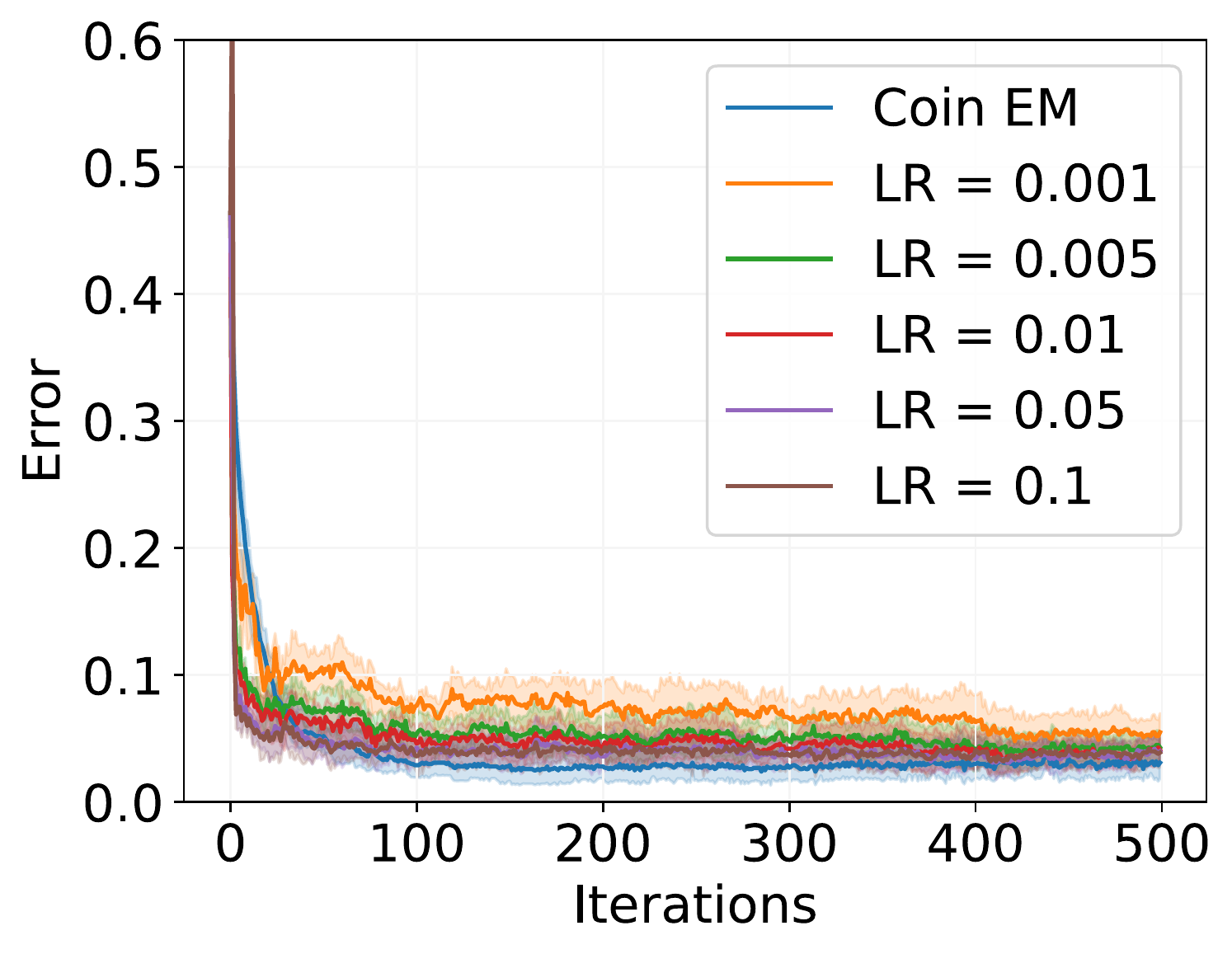}}
  \vspace{-1mm}
  \vspace{-1mm}
  \subfigure[SVGD EM ($N=5$). \label{fig:bnn_mnist_compare_lr_svgd_N5}]{\includegraphics[width=0.25\textwidth, trim=0 0 0 0, clip]{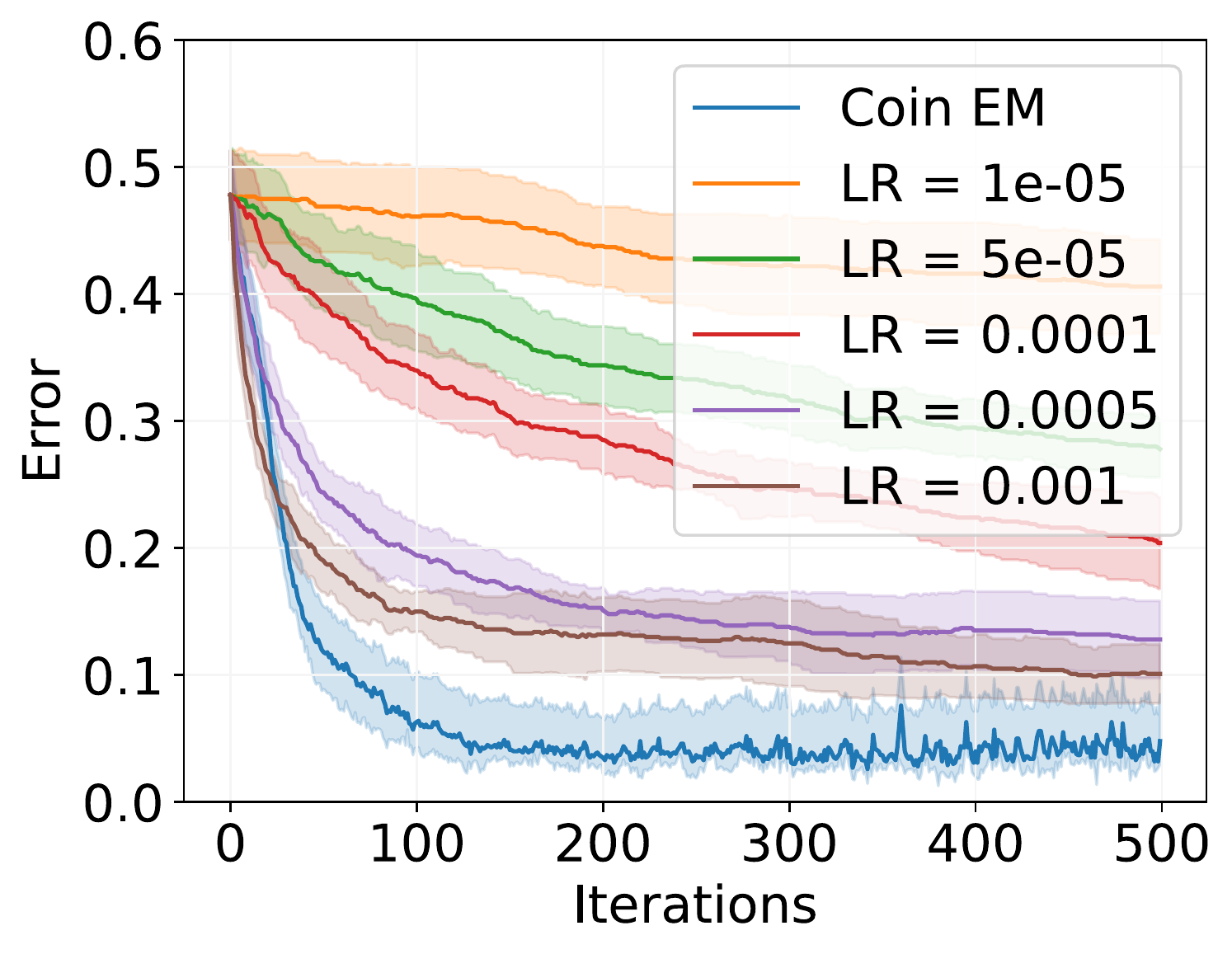}}
  \subfigure[SVGD EM ($N=20$). \label{fig:bnn_mnist_compare_lr_svgd_N20}]{\includegraphics[width=0.25\textwidth, trim=0 0 0 0, clip]{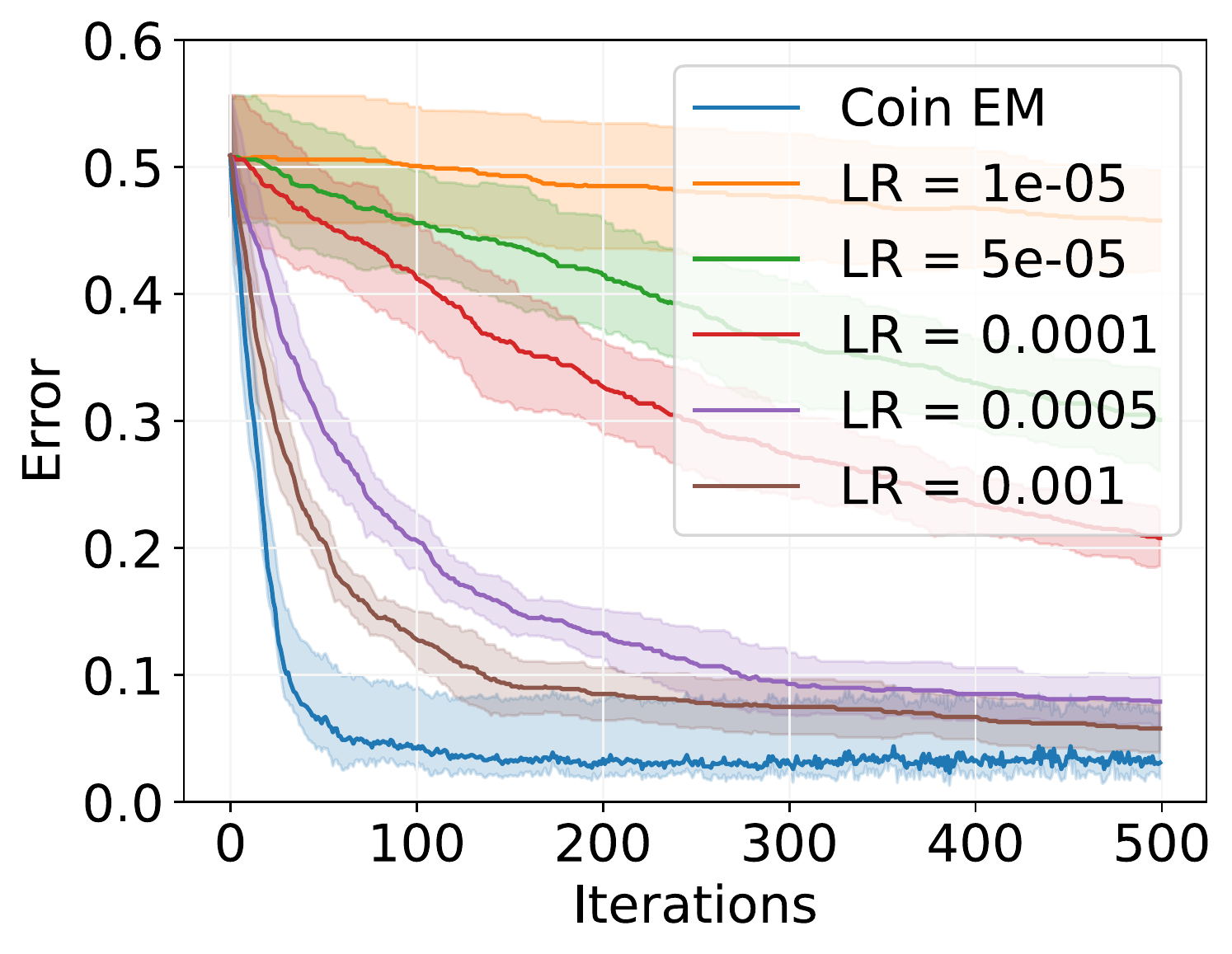}}
  \subfigure[SVGD EM ($N=100$). \label{fig:bnn_mnist_compare_lr_svgd_N100}]{\includegraphics[width=0.25\textwidth, trim=0 0 0 0, clip]{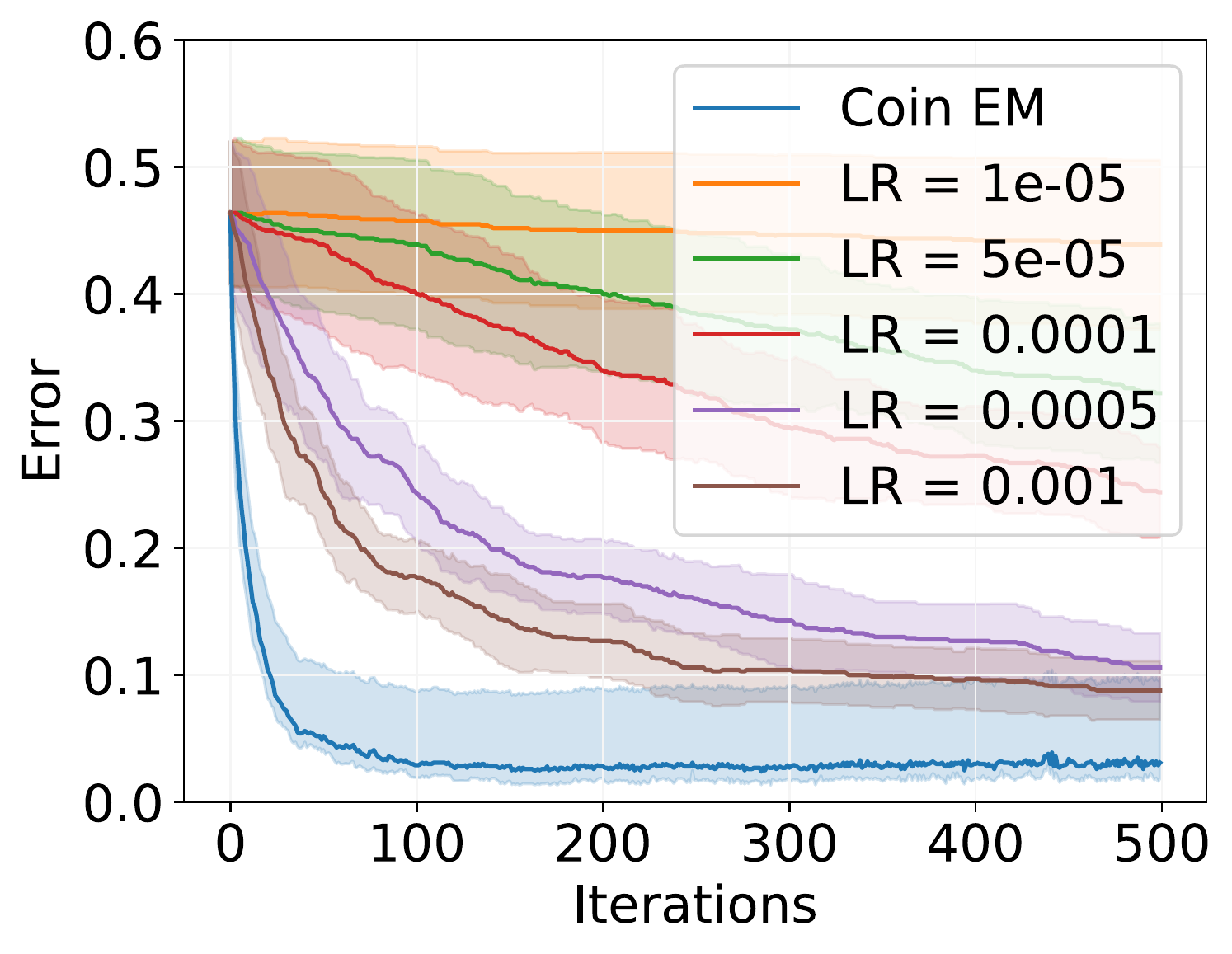}}
  \vspace{-1mm}
  \vspace{-1mm}
  \subfigure[SVGD EM' ($N=5$). \label{fig:bnn_mnist_compare_lr_svgd_h_N5}]{\includegraphics[width=0.25\textwidth, trim=0 0 0 0, clip]{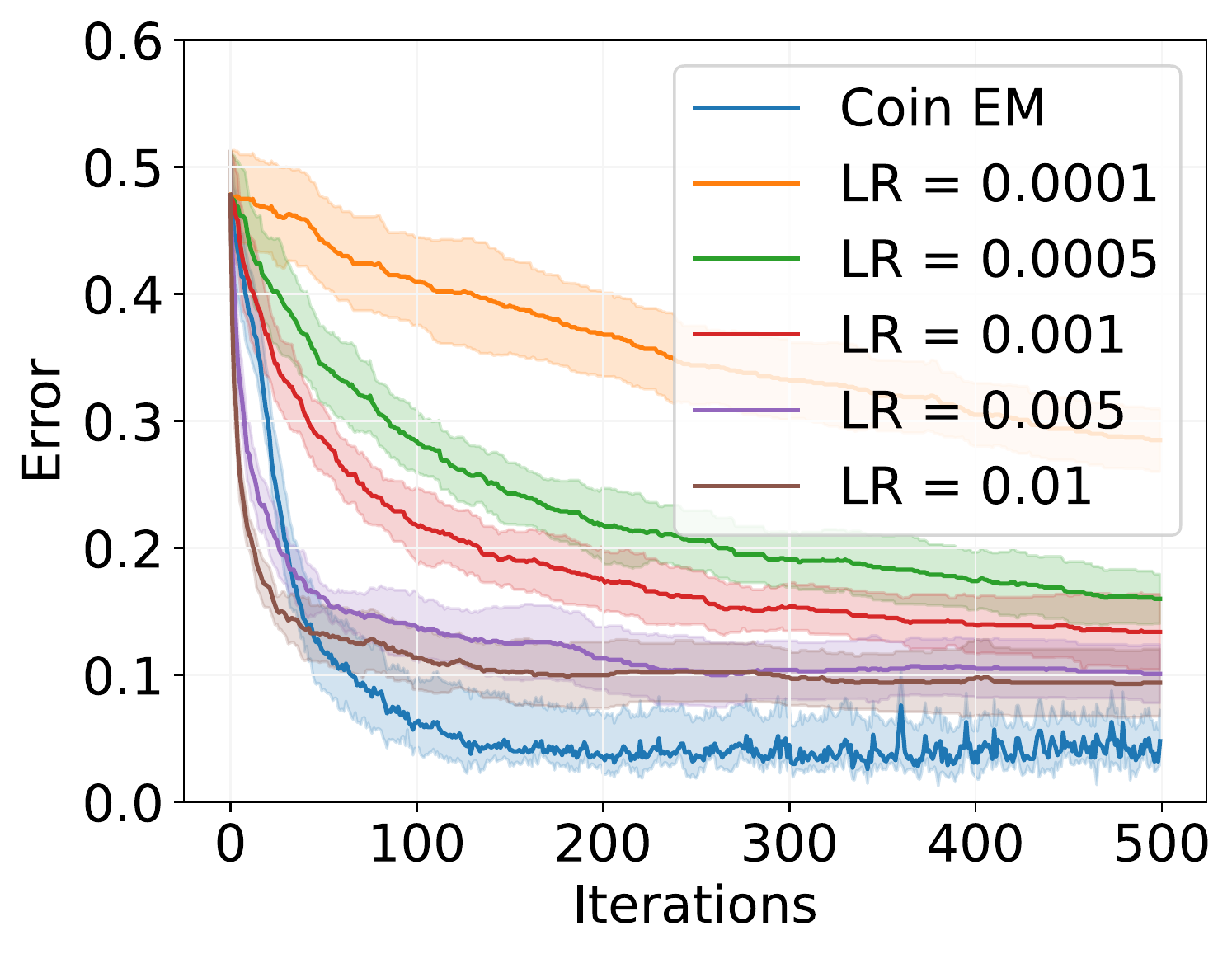}}
  \subfigure[SVGD EM' ($N=20$). \label{fig:bnn_mnist_compare_lr_svgd_h_N20}]{\includegraphics[width=0.25\textwidth, trim=0 0 0 0, clip]{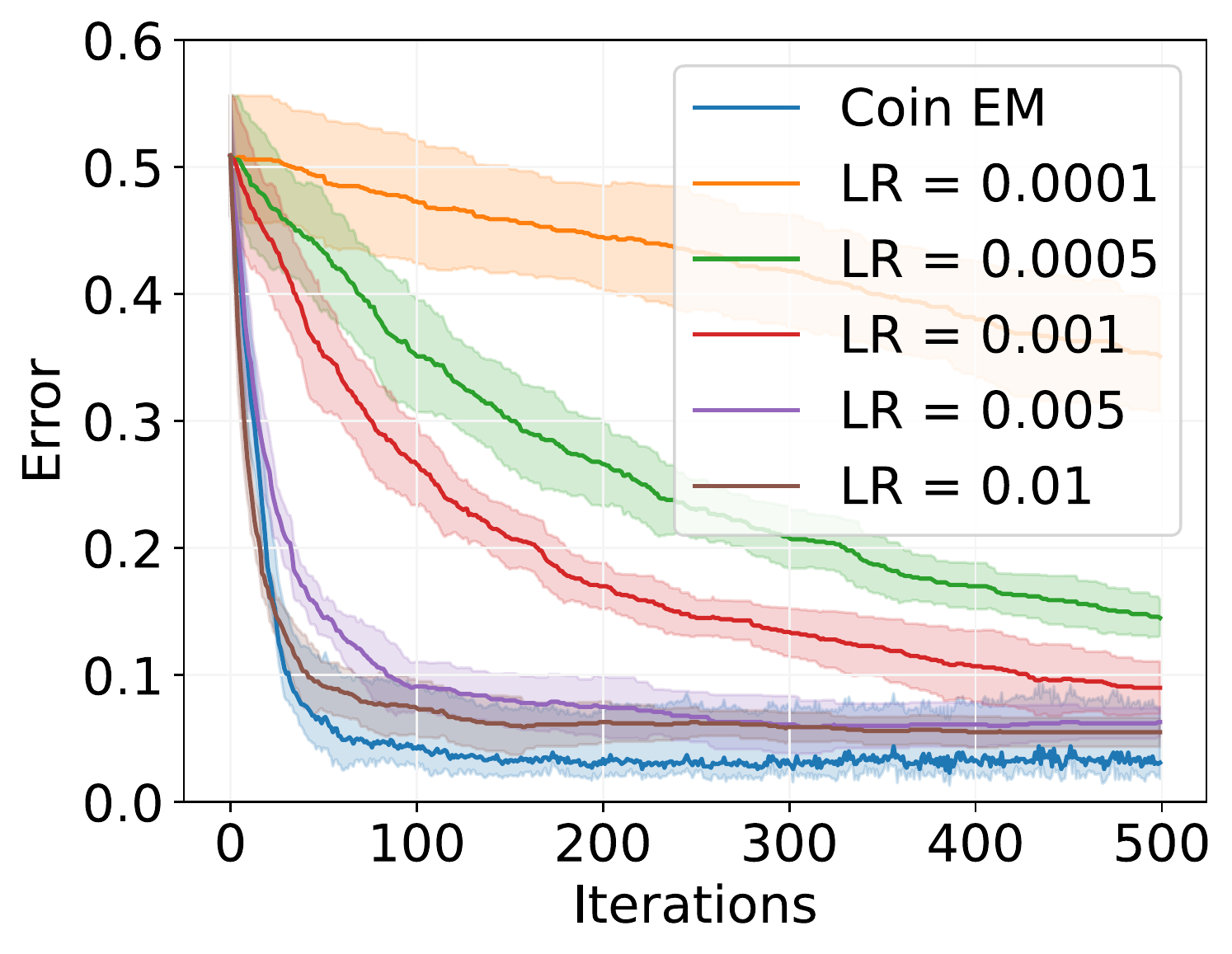}}
  \subfigure[SVGD EM' ($N=100$). \label{fig:bnn_mnist_compare_lr_svgd_h_N100}]{\includegraphics[width=0.25\textwidth, trim=0 0 0 0, clip]{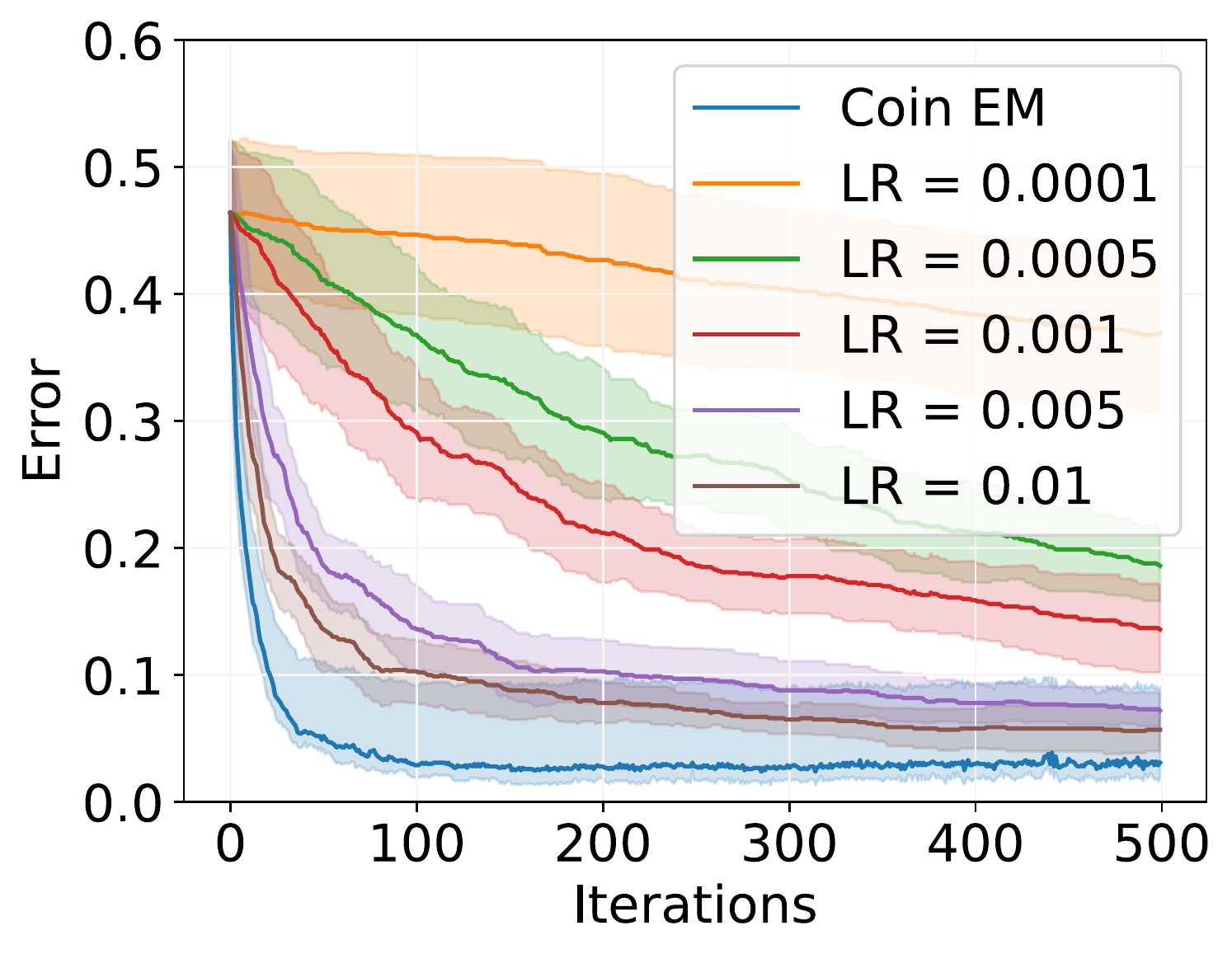}}
  \vspace{-1mm}
  \vspace{-1mm}
  \subfigure[SOUL ($N=5$). \label{fig:bnn_mnist_compare_lr_soul_N5}]{\includegraphics[width=0.25\textwidth, trim=0 0 0 0, clip]{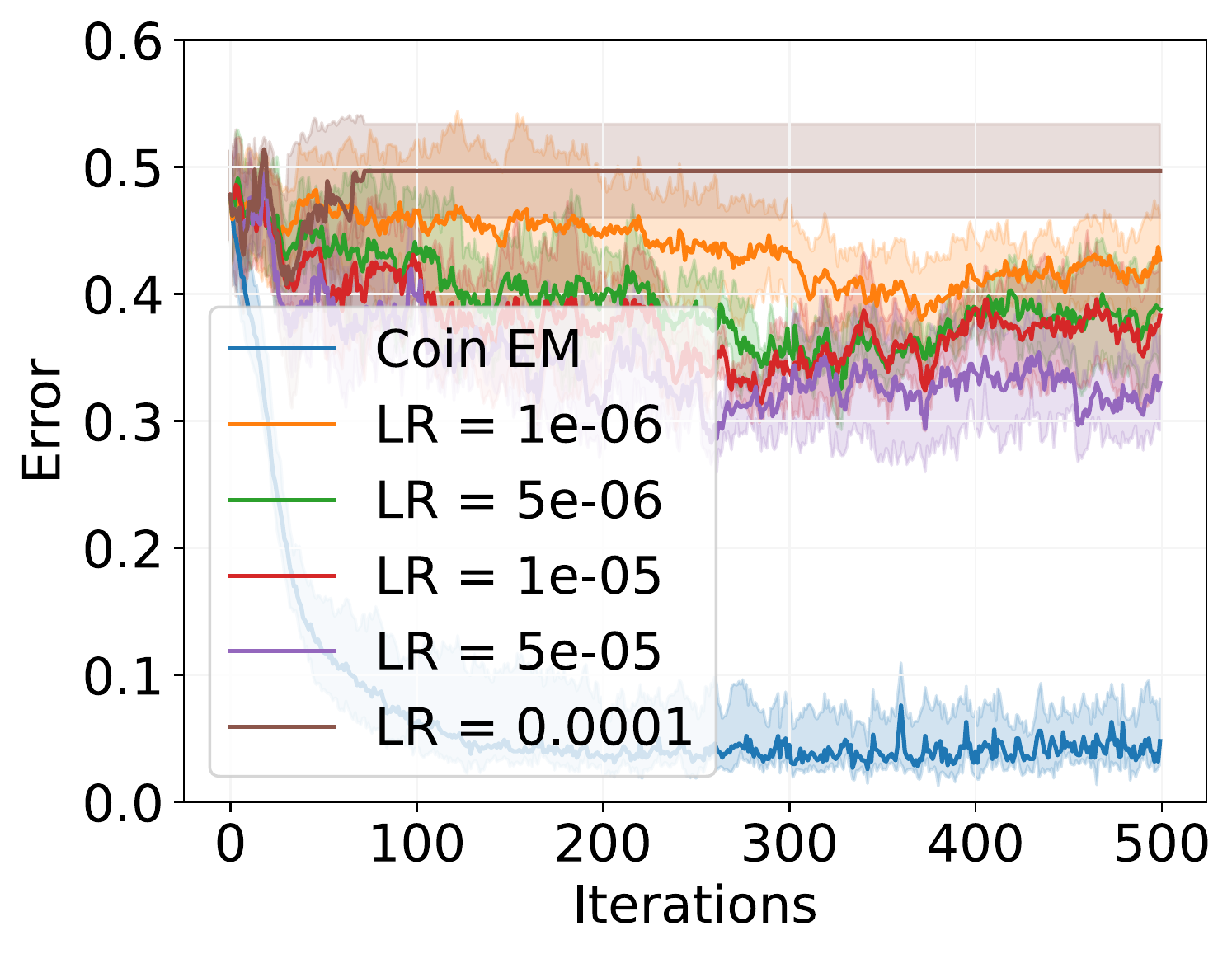}}
  \subfigure[SOUL ($N=20$). \label{fig:bnn_mnist_compare_lr_soul_N20}]{\includegraphics[width=0.25\textwidth, trim=0 0 0 0, clip]{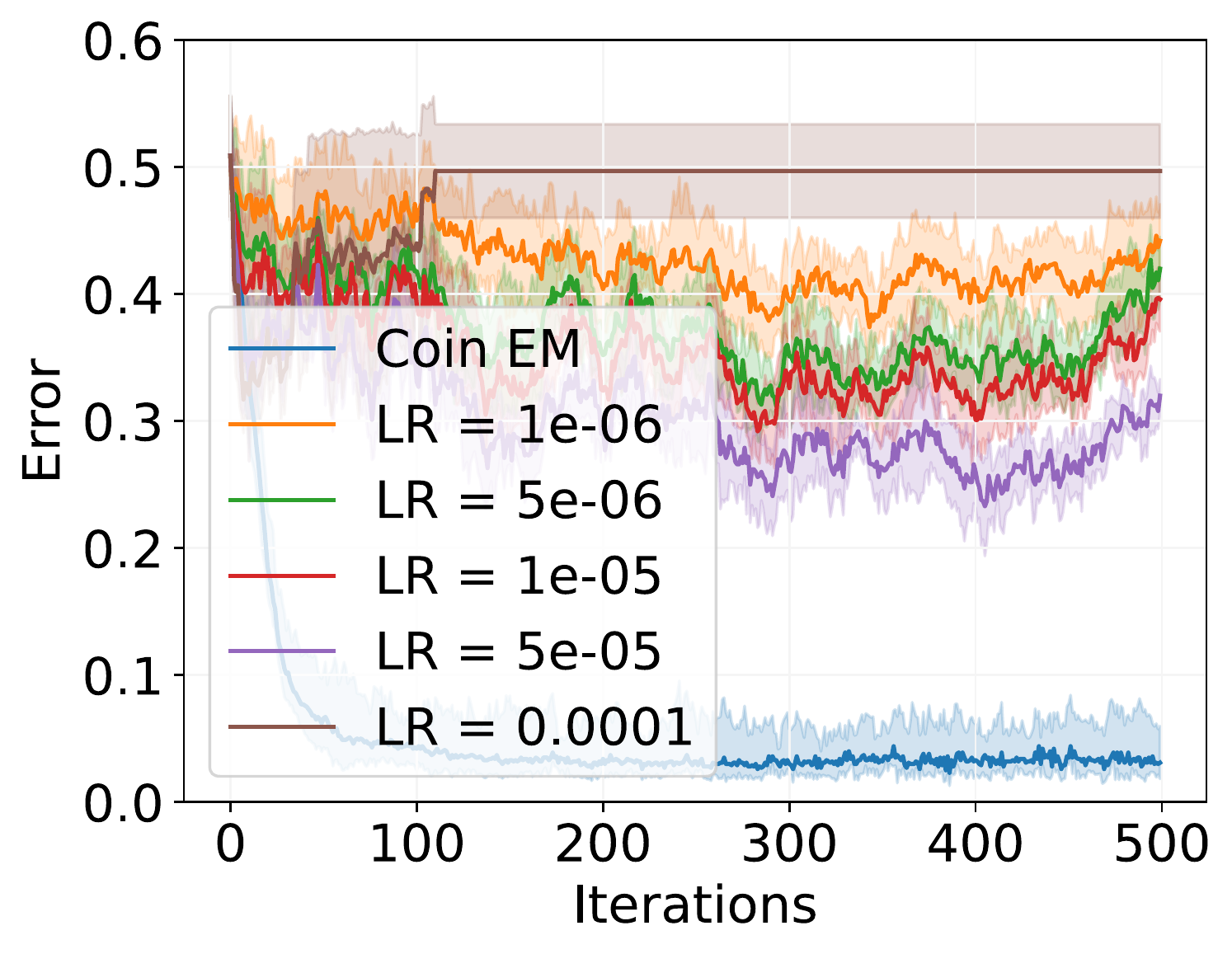}}
  \subfigure[SOUL ($N=100$). \label{fig:bnn_mnist_compare_lr_soul_N100}]{\includegraphics[width=0.25\textwidth, trim=0 0 0 0, clip]{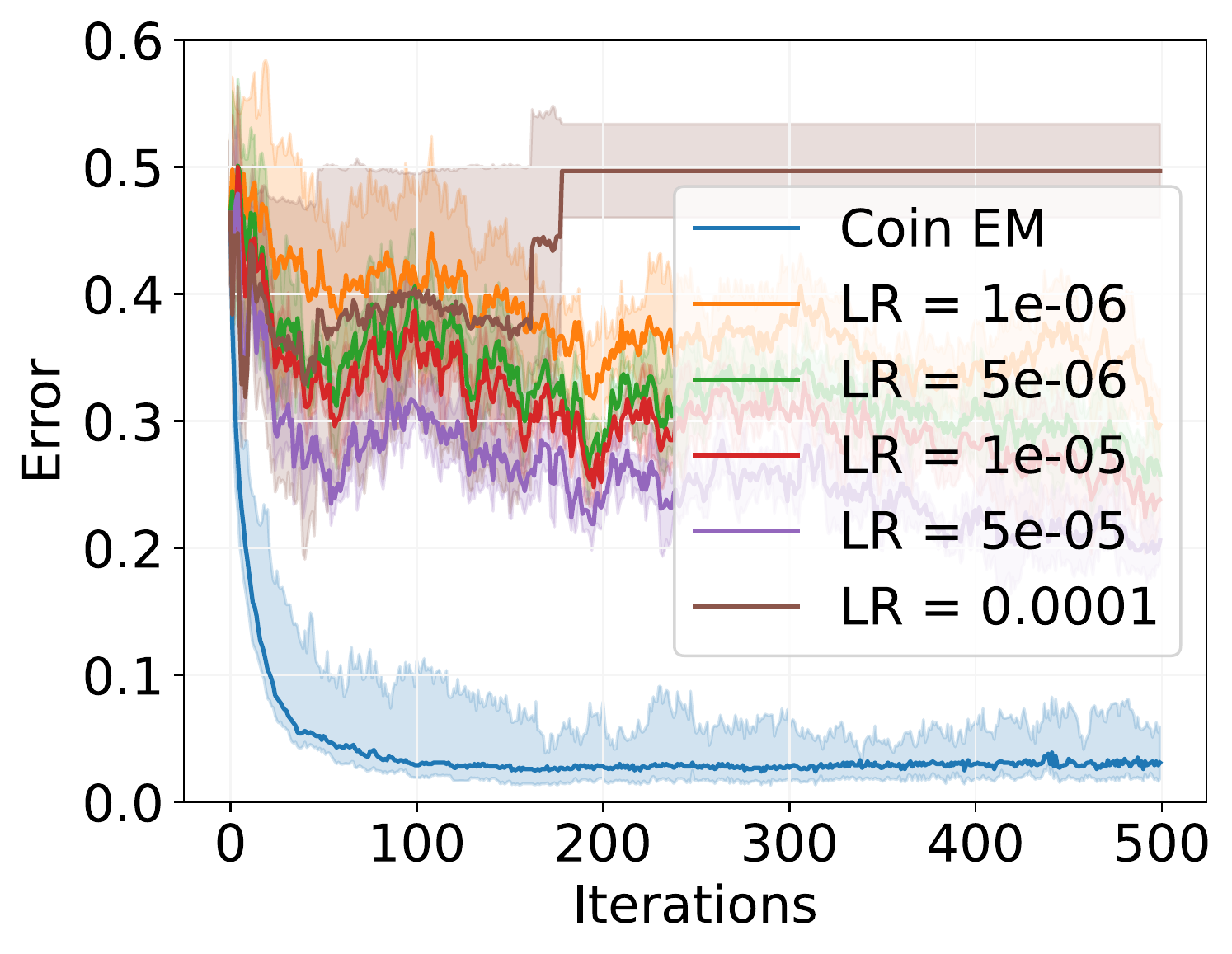}}
  \vspace{-1mm}
  \vspace{-1mm}
  \subfigure[SOUL' ($N=5$). \label{fig:bnn_mnist_compare_lr_soul_h_N5}]{\includegraphics[width=0.25\textwidth, trim=0 0 0 0, clip]{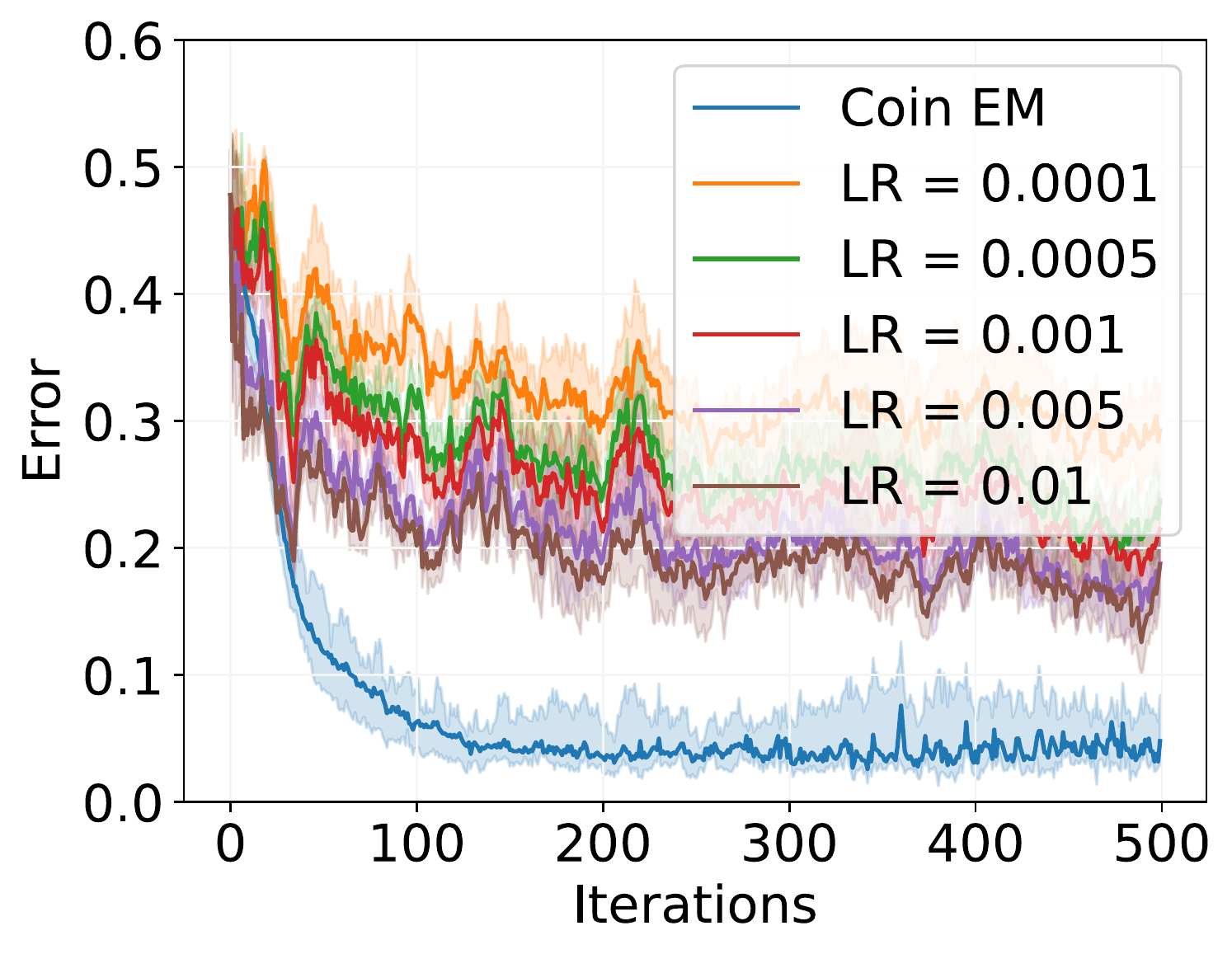}}
  \subfigure[SOUL' ($N=20$). \label{fig:bnn_mnist_compare_lr_soul_h_N20}]{\includegraphics[width=0.25\textwidth, trim=0 0 0 0, clip]{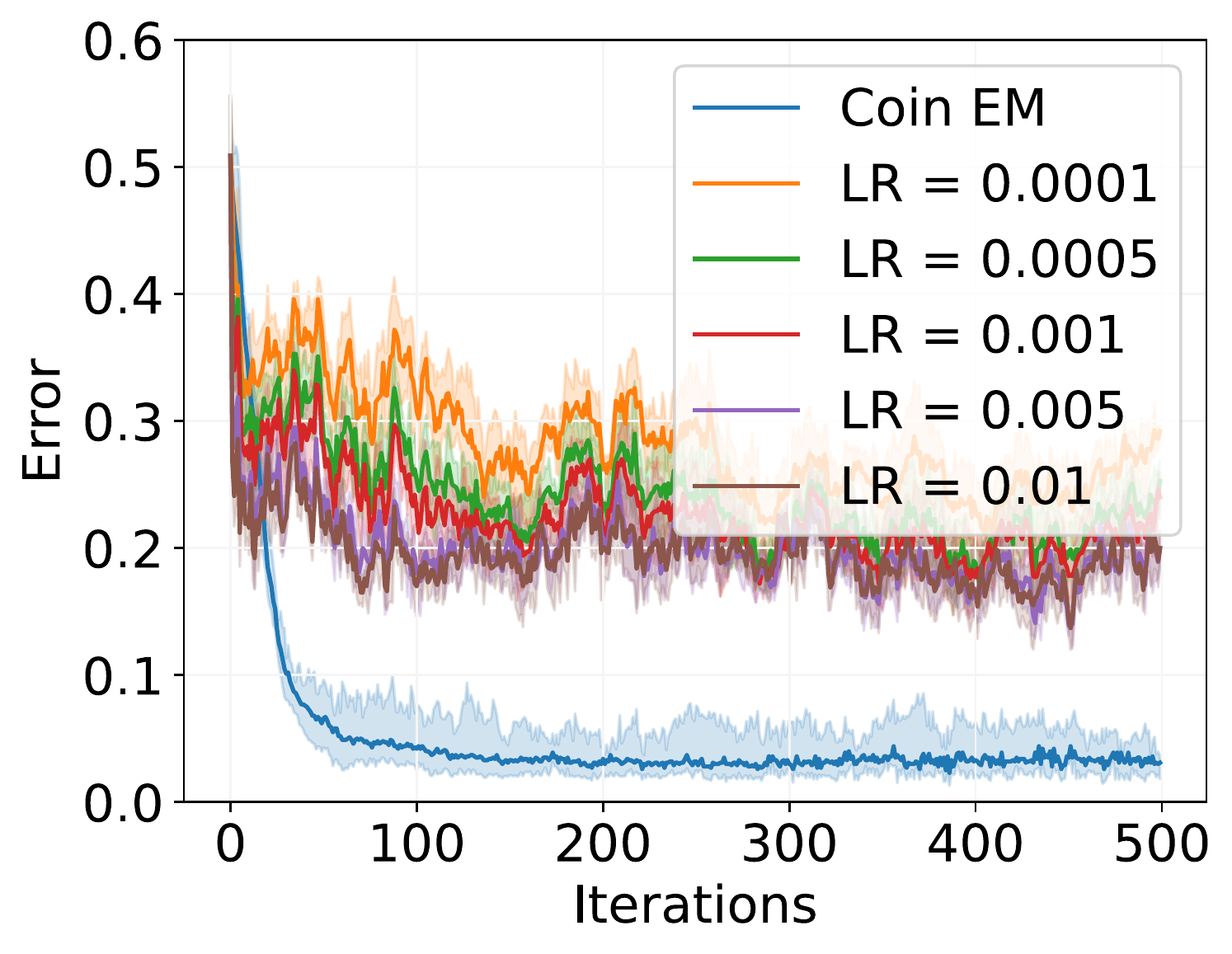}}
  \subfigure[SOUL' ($N=100$). \label{fig:bnn_mnist_compare_lr_soul_h_N100}]{\includegraphics[width=0.25\textwidth, trim=0 0 0 0, clip]{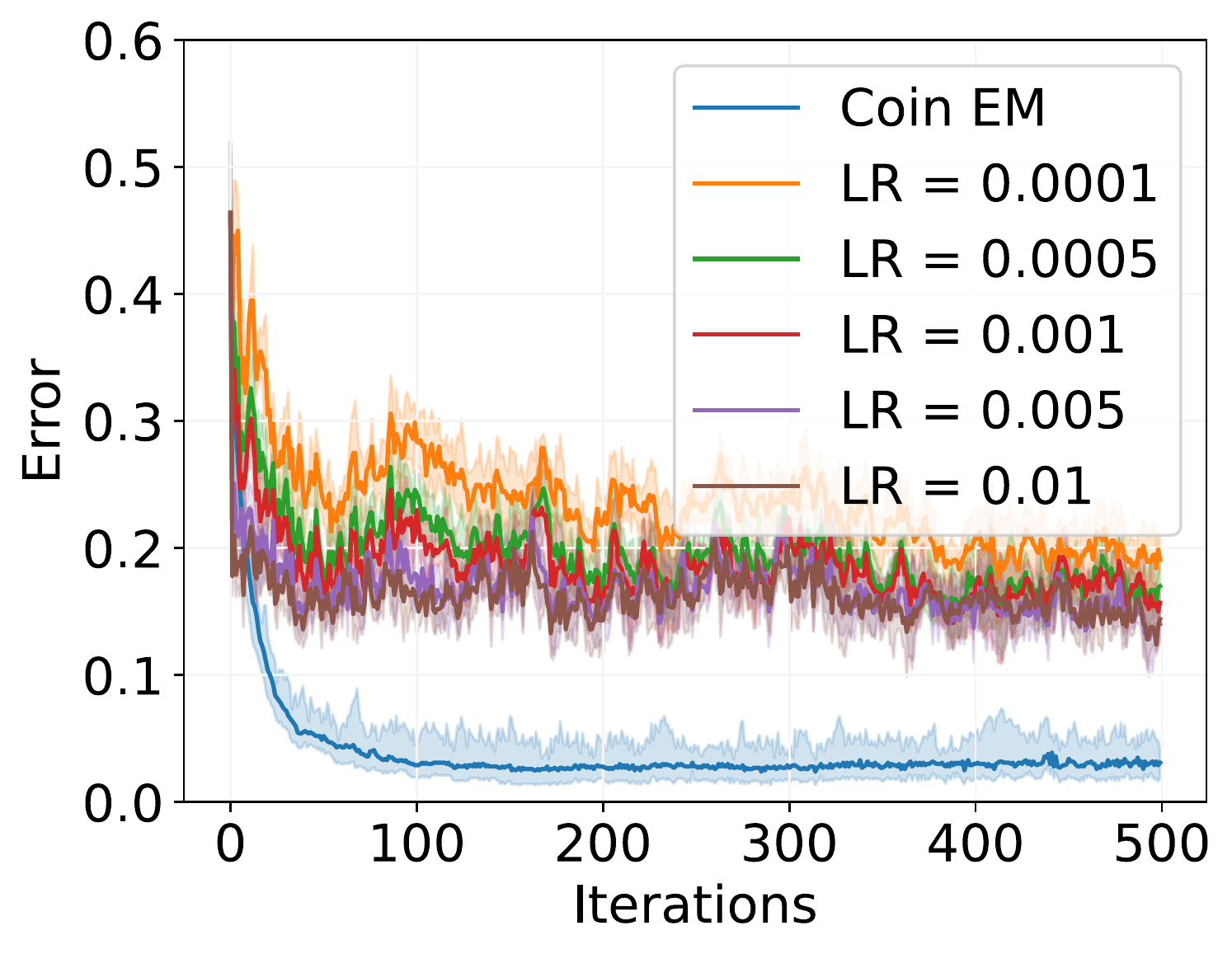}}
  \vspace{-3mm}
  \caption{\textbf{Additional results for the Bayesian neural network model in Sec. \ref{sec:bayes_nn_results}: learning rate comparison.} Test error achieved over $T=500$ training iterations, for different learning rates.}  
  \label{fig:bnn_mnist_compare_lr}
\end{figure}
}

In Fig. \ref{fig:bnn_mnist_compare_N}, we further investigate the performance of each algorithm as we vary the number of particles. In this case, each method is improved by increasing the number of particles. Interestingly, Coin EM is robust to the number of particles, achieving good predictive performance even when the number is small (Fig. \ref{fig:bnn_mnist_compare_N_coin}). The same is true, to a lesser extent, for SVGD EM (with or without the heuristic), which performs relatively well across experiments (Fig. \ref{fig:bnn_mnist_compare_N_SVGD}, Fig. \ref{fig:bnn_mnist_compare_N_SVGD_h}). PGD observes a more significant performance increase as the number of particles is increased (Fig. \ref{fig:bnn_mnist_compare_N_PGD}), consistent with the observations in \cite[][Sec. 3.2]{Kuntz2023}, although performs relatively well even for small particular numbers with the heuristic (Fig. \ref{fig:bnn_mnist_compare_N_PGD_h}). Finally, even for large numbers of particles, and when using the heuristic, SOUL struggles to compete with the other methods (Fig. \ref{fig:bnn_mnist_compare_N_SOUL}, Fig. \ref{fig:bnn_mnist_compare_N_SOUL_h}). 

\begin{figure}[H]
  \centering
  \subfigure[Coin EM. \label{fig:bnn_mnist_compare_N_coin}]{\includegraphics[width=0.26\textwidth, trim=0 0 0 0, clip]{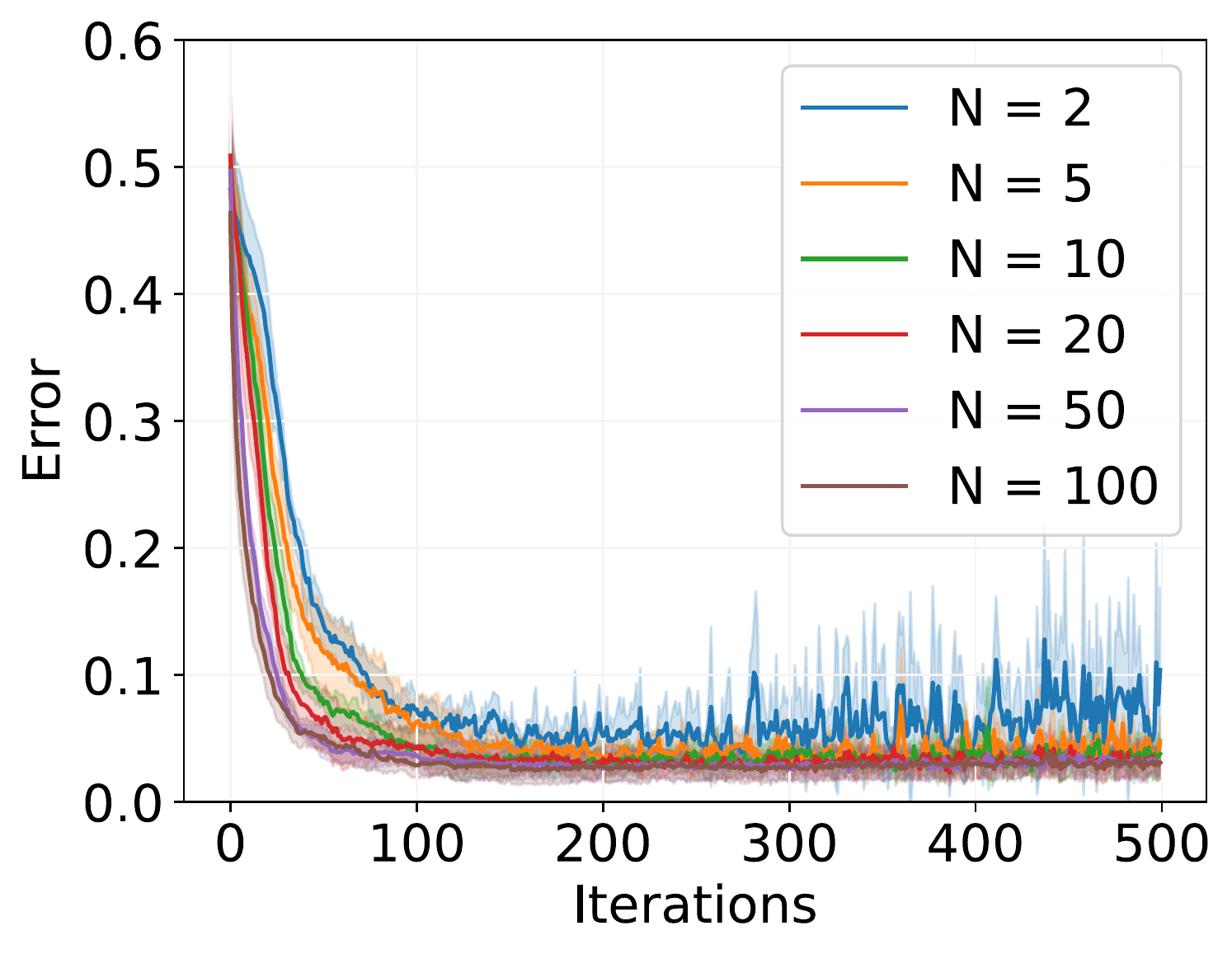}}
  \subfigure[SVGD EM. \label{fig:bnn_mnist_compare_N_SVGD}]{\includegraphics[width=0.26\textwidth, trim=0 0 0 0, clip]{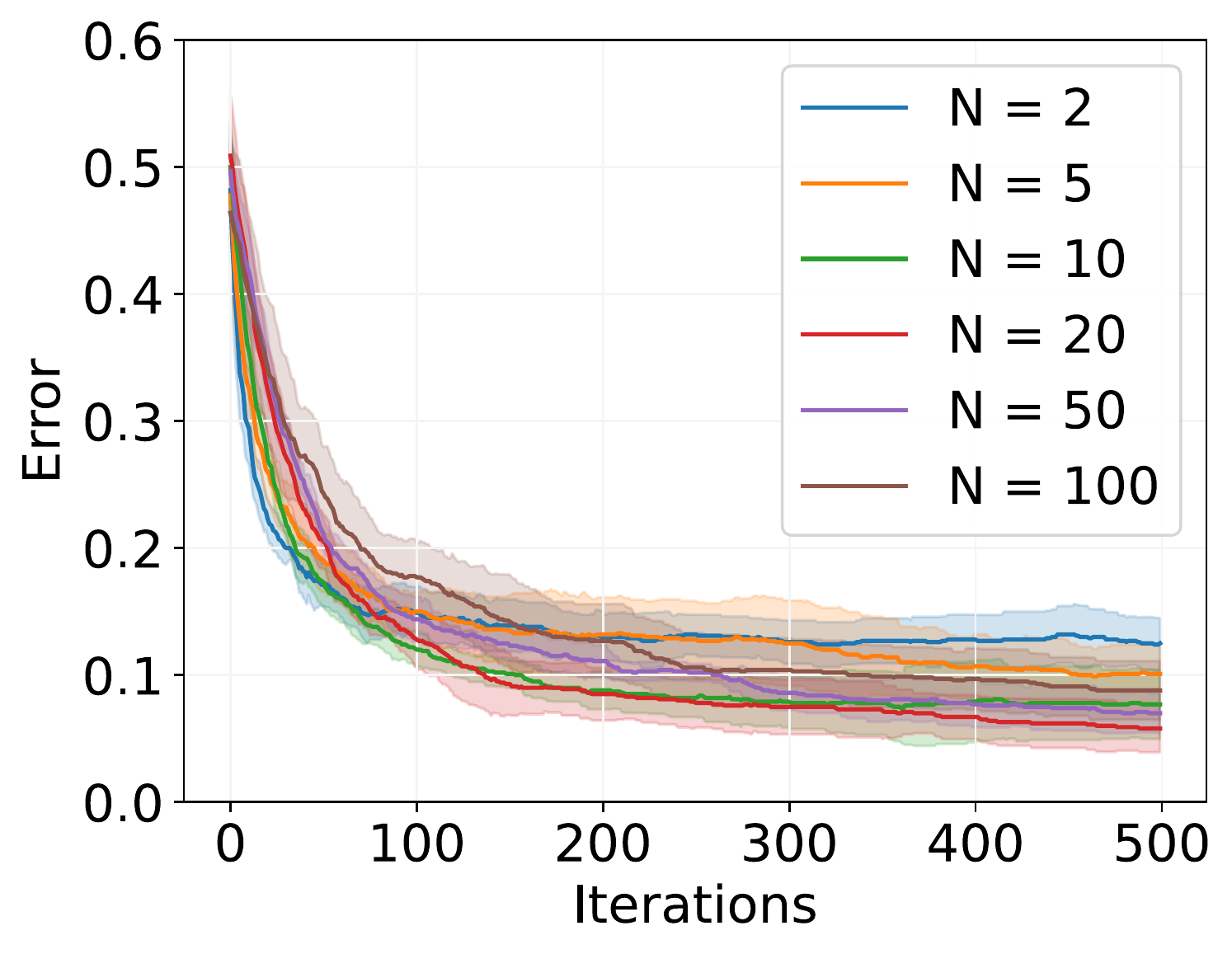}}
  \subfigure[SVGD EM'. \label{fig:bnn_mnist_compare_N_SVGD_h}]{\includegraphics[width=0.26\textwidth, trim=0 0 0 0, clip]{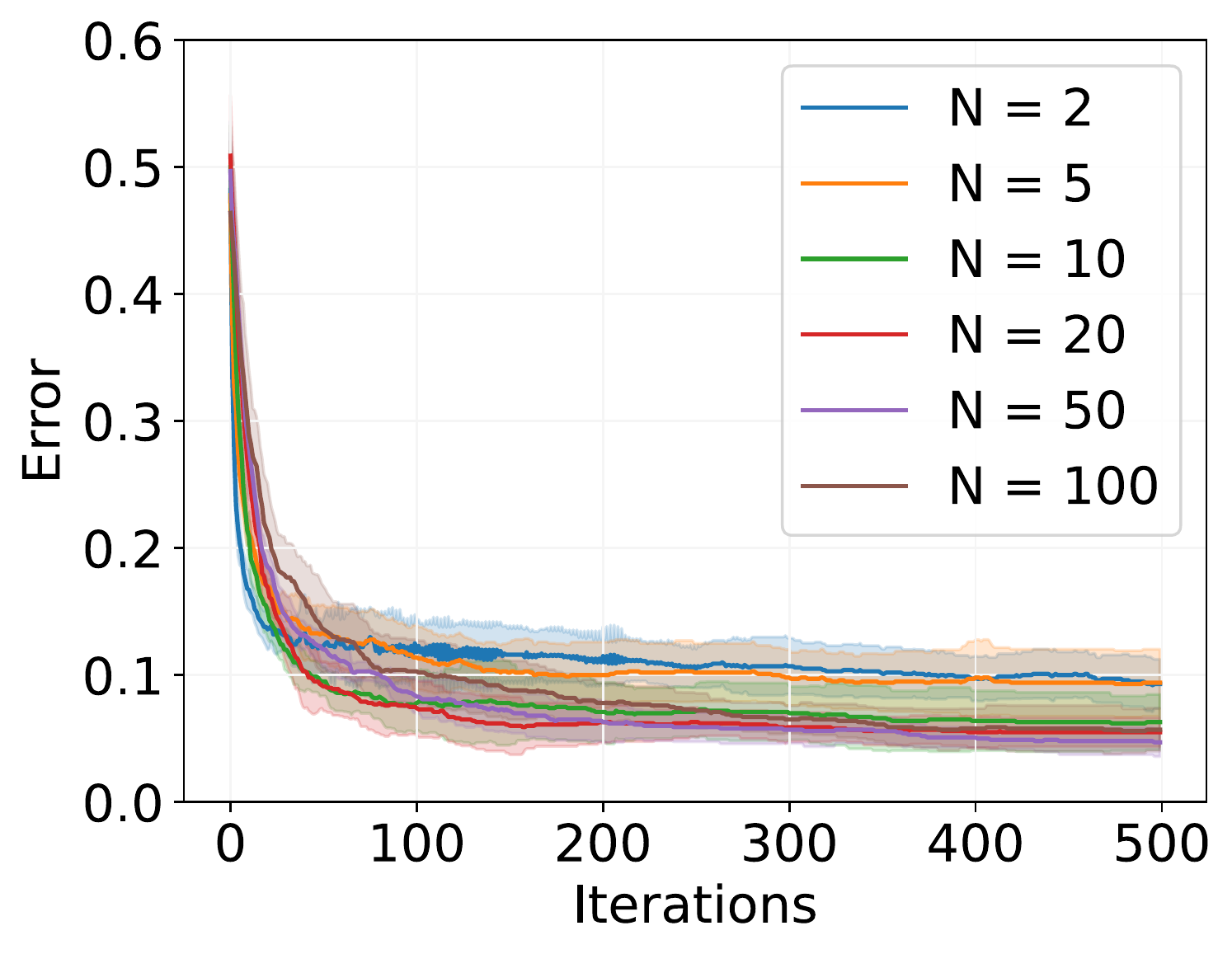}} \\
  \subfigure[PGD. \label{fig:bnn_mnist_compare_N_PGD}]{\includegraphics[width=0.24\textwidth, trim=0 0 0 0, clip]{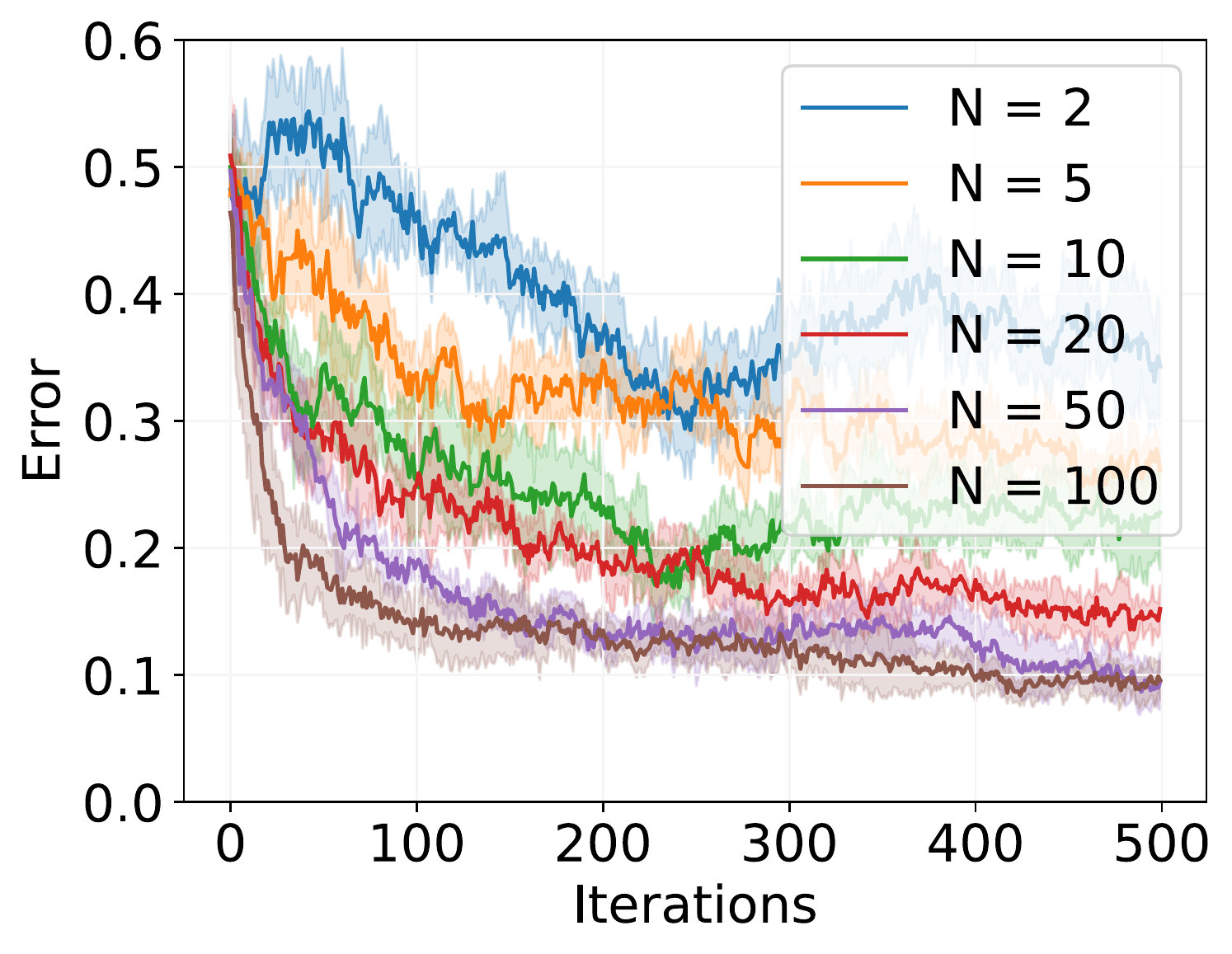}}
  \subfigure[PGD'. \label{fig:bnn_mnist_compare_N_PGD_h}]{\includegraphics[width=0.24\textwidth, trim=0 0 0 0, clip]{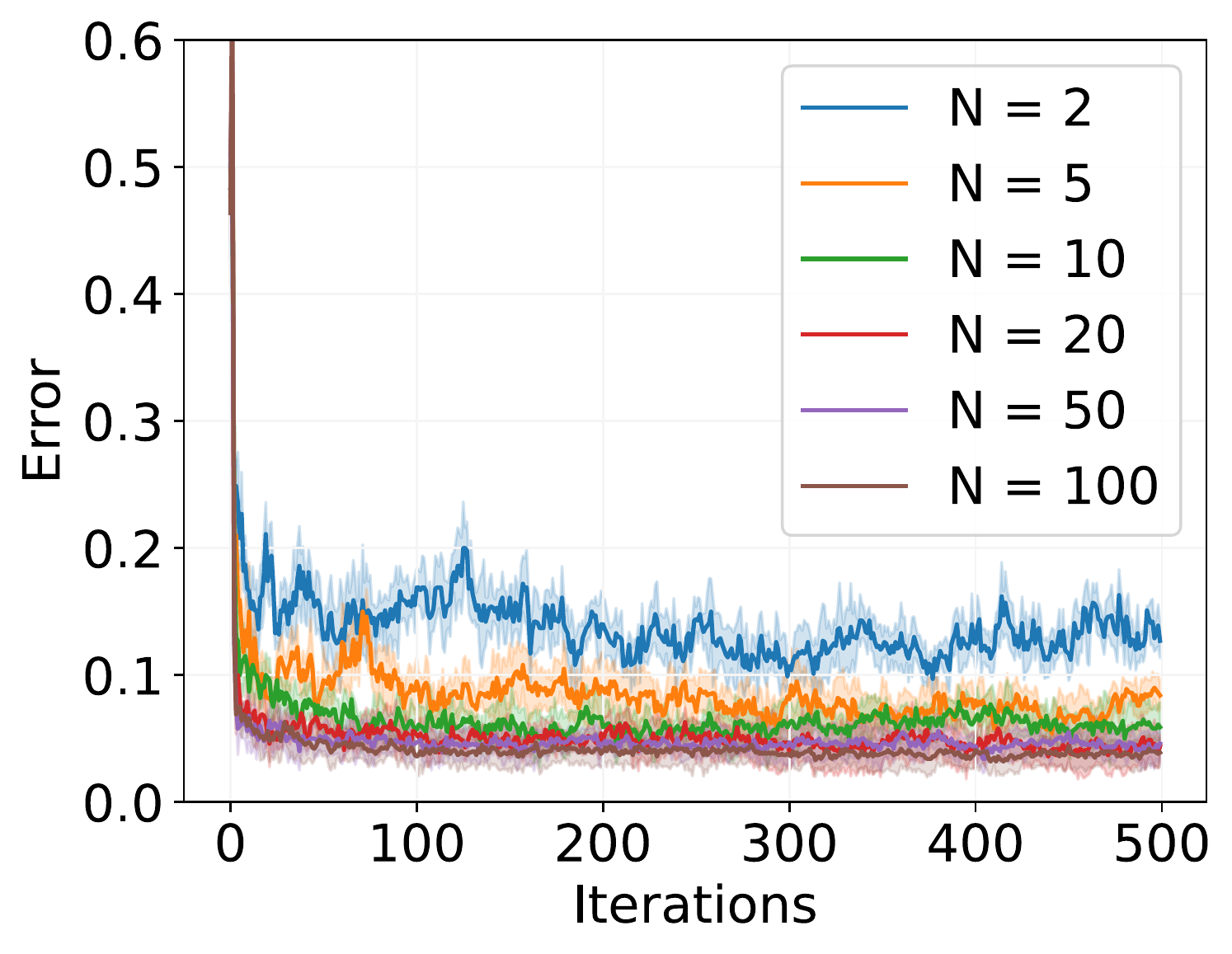}}
  \subfigure[SOUL. \label{fig:bnn_mnist_compare_N_SOUL}]{\includegraphics[width=0.24\textwidth, trim=0 0 0 0, clip]{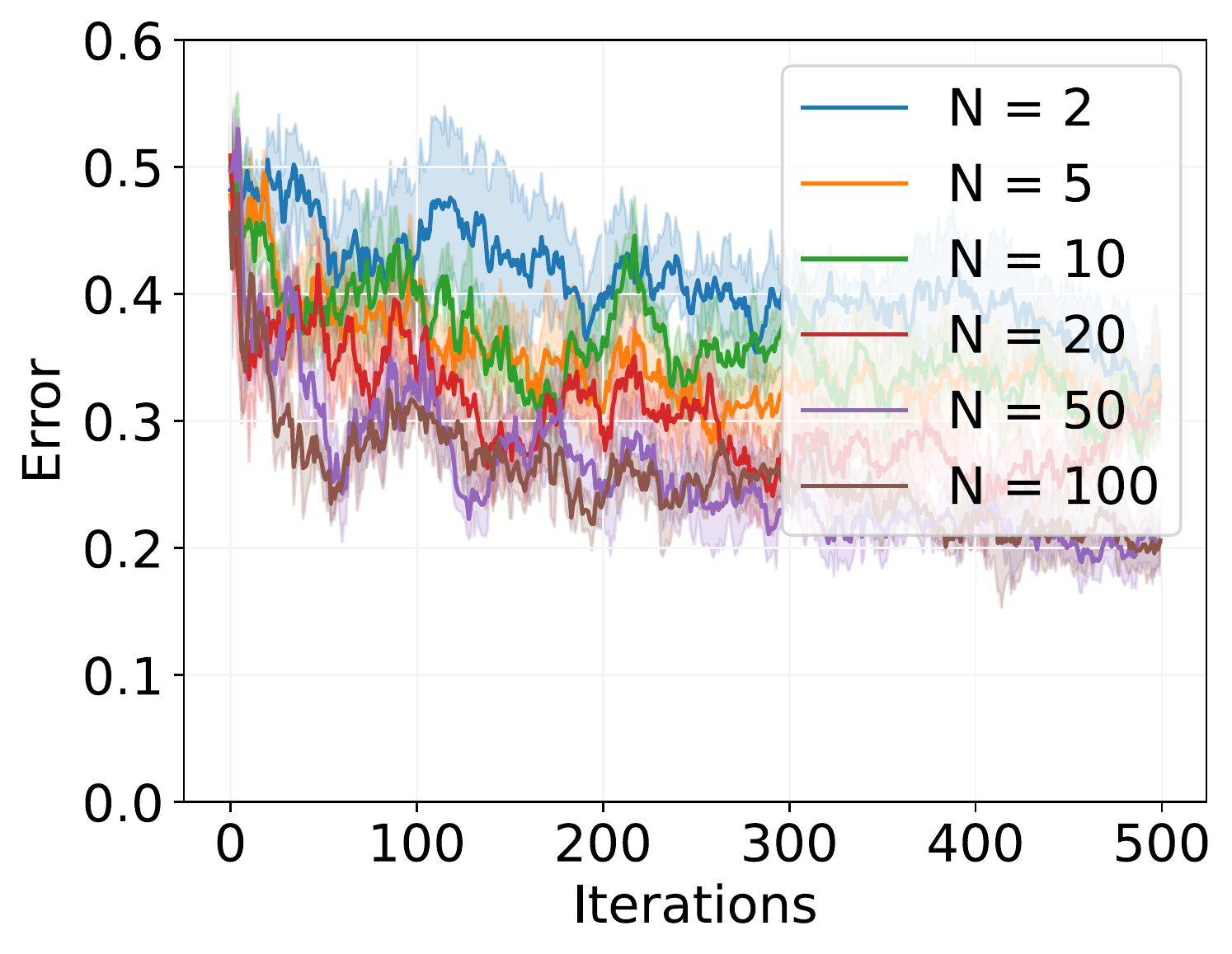}}
  \subfigure[SOUL'. \label{fig:bnn_mnist_compare_N_SOUL_h}]{\includegraphics[width=0.24\textwidth, trim=0 0 0 0, clip]{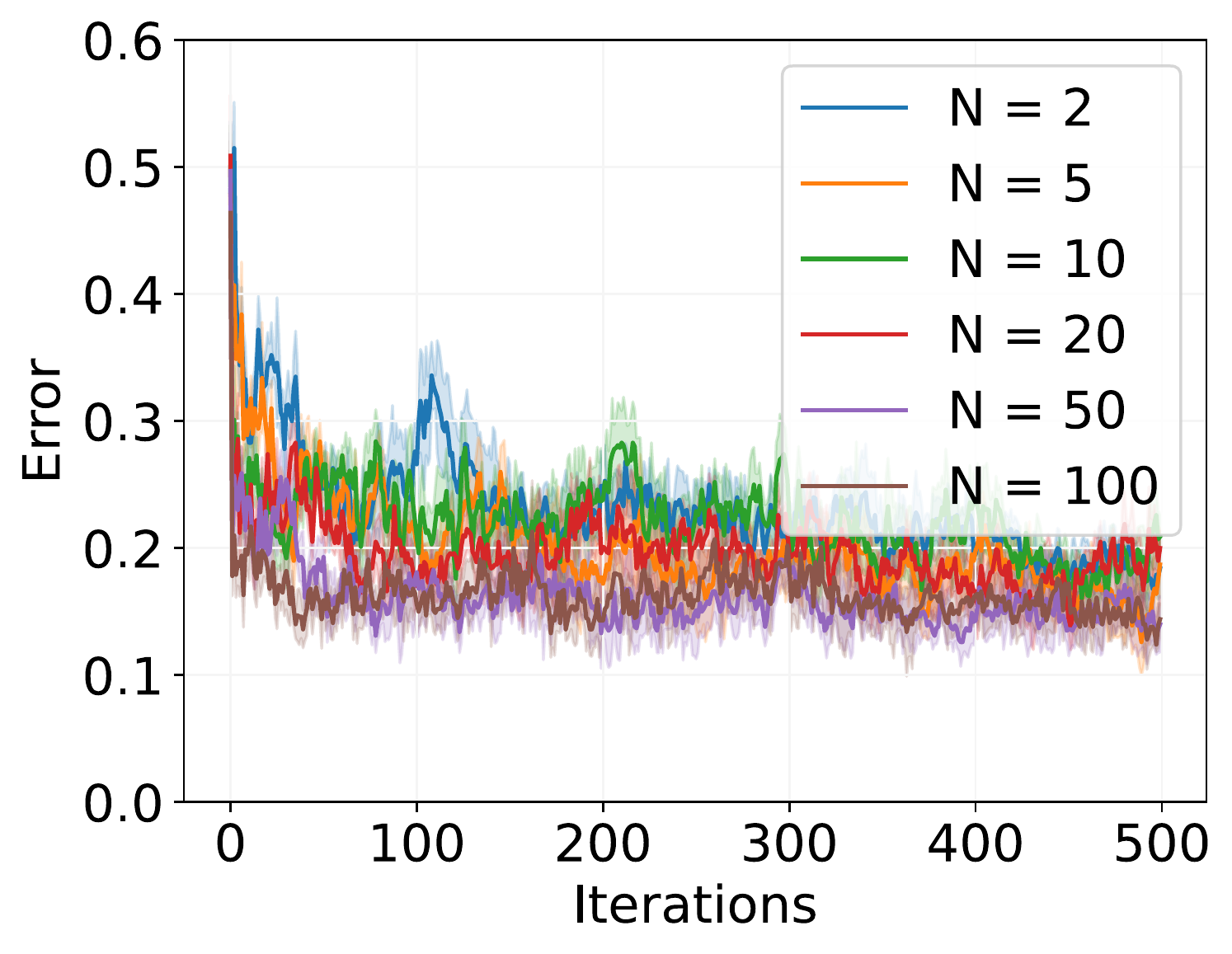}}
  \caption{\textbf{Additional results for the Bayesian neural network model in Sec. \ref{sec:bayes_nn_results}: particle number comparison.} Test error over $T=500$ training iterations, for different numbers of particles. For all learning-rate dependent methods, we use the best learning rate as determined by Fig. \ref{fig:bnn_mnist_compare_lr}.}
  \label{fig:bnn_mnist_compare_N}
\end{figure}

Finally, in Fig. \ref{fig:bayes_nn_mnist}, we provide a comparison of the test accuracy achieved by Coin EM, SVGD EM, PGD, and SOUL, now over a much finer grid of learning rates.

\begin{figure}[ht]
  \centering
  \subfigure[Standard]{\includegraphics[width=0.4\textwidth]{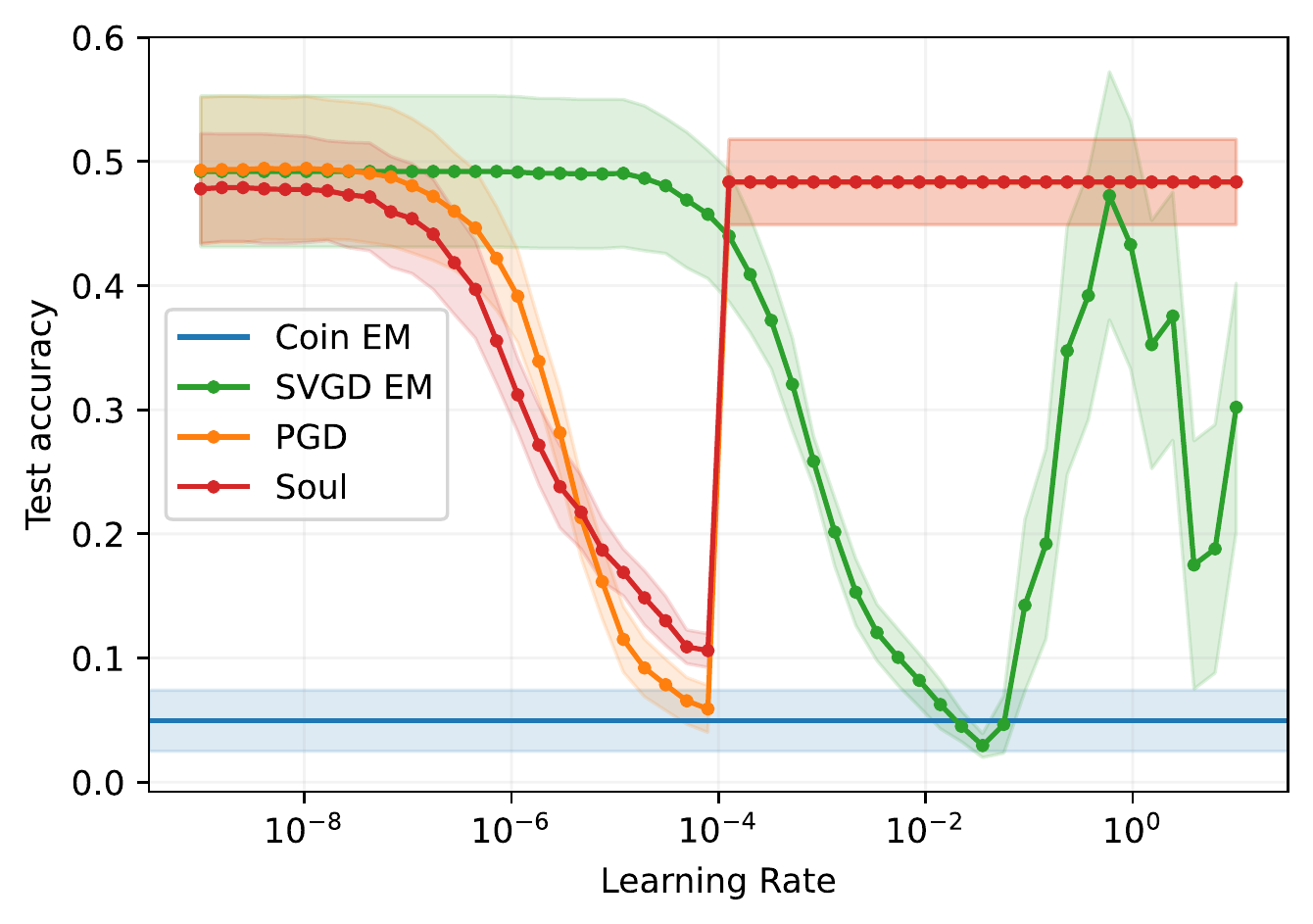}} \hspace{4mm}
  \subfigure[Heuristic]{\includegraphics[width=0.4\textwidth]{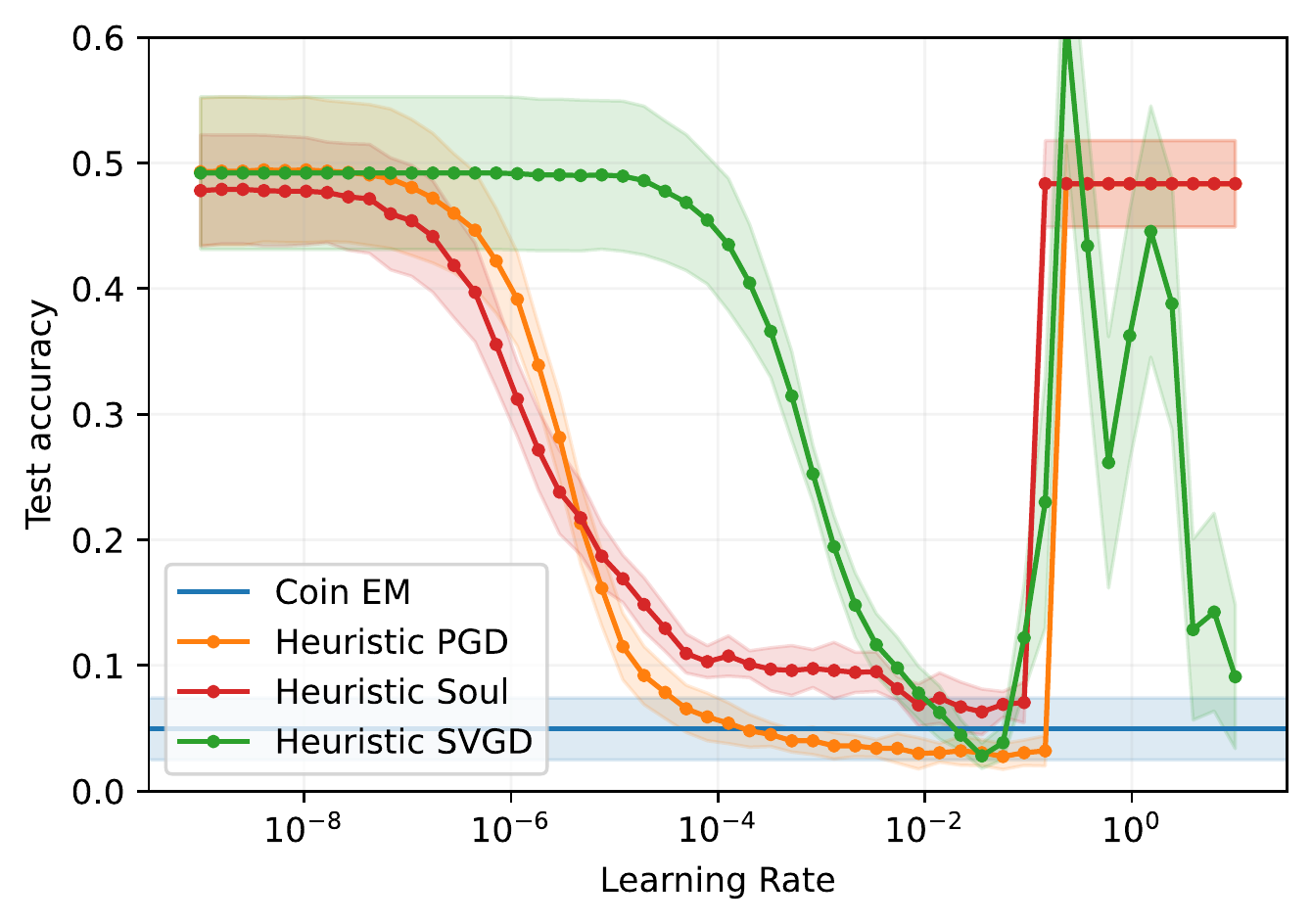}}
  \caption{\textbf{Additional results for the Bayesian neural network in Sec. \ref{sec:bayes_nn_results}}. Accuracy as a function of the learning rate, for the MNIST dataset, averaged over ten random test-train splits.}
  \label{fig:bayes_nn_mnist}
\end{figure}

\subsection{Bayesian neural network (alternative model)}
\label{sec:bayes-nn-add-results}

We now present numerical results for the alternative Bayesian neural network model described in App. \ref{sec:bayes-nn-details}. In Fig. \ref{fig:bayes_nn_additional}, we compare the performance of Coin EM and SVGD EM on several UCI benchmarks. Once again, we compare our algorithms with PGD and SOUL. As noted in Sec. \ref{sec:bayes_nn_results}, in this case the performance of the three learning-rate-dependent algorithms (SVGD EM, PGD, SOUL) is highly sensitive to the choice of learning rate. In particular, there is generally a very small range of learning rates for which these algorithms outperform Coin EM. Given an optimal choice of learning rate, the best predictive performance of SVGD EM is generally comparable with the best predictive performance of SOUL, and better than the best predictive performance of PGD, though there are cases in which Coin EM and SVGD EM significantly outperform the competing methods (see Fig. \ref{fig:bayes_nn_additional_power}). 

\begin{figure}[ht]
  \centering
\subfigure[Concrete]{\includegraphics[width=0.32\textwidth]{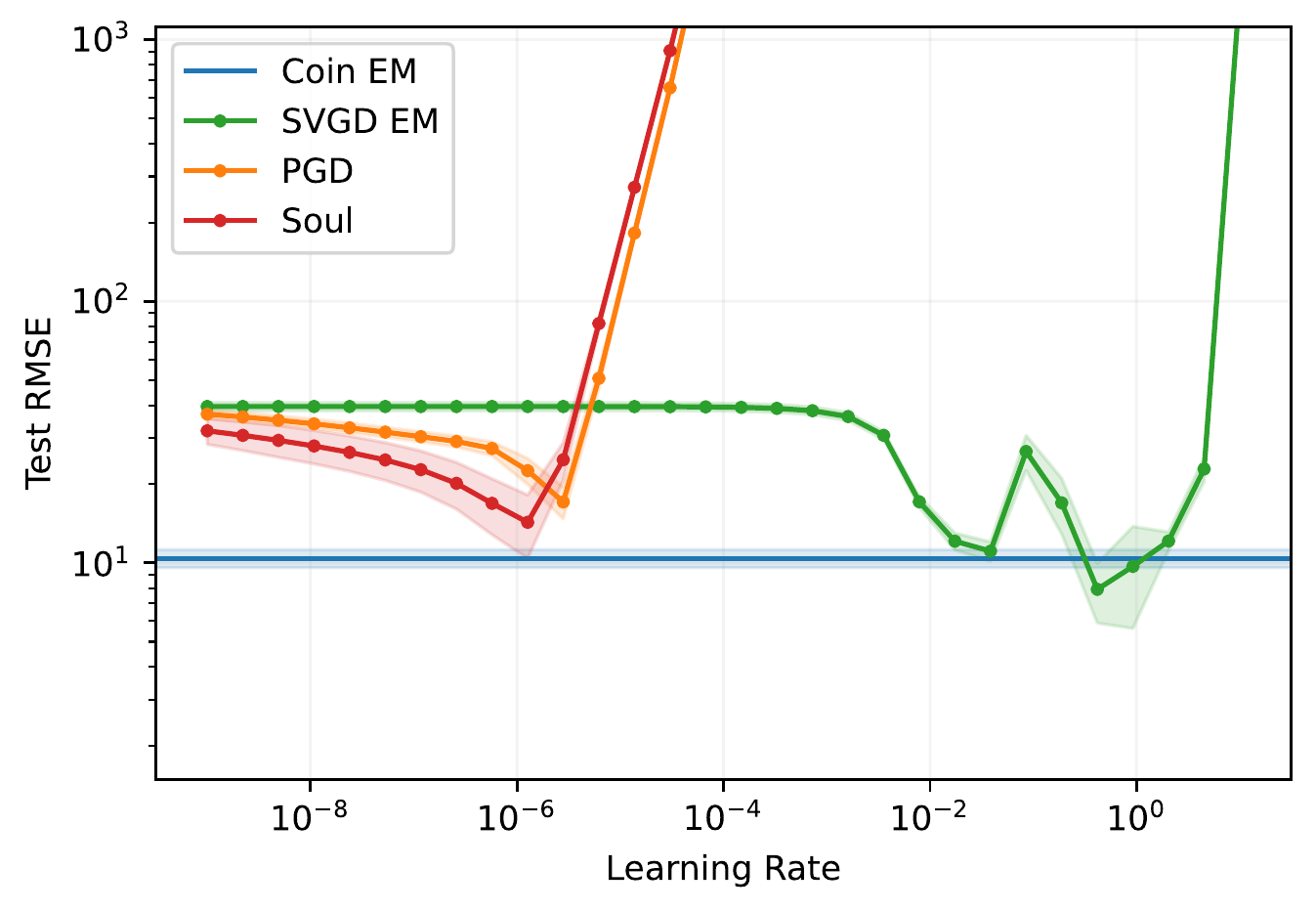}}
  \subfigure[Energy]{\includegraphics[width=0.32\textwidth]{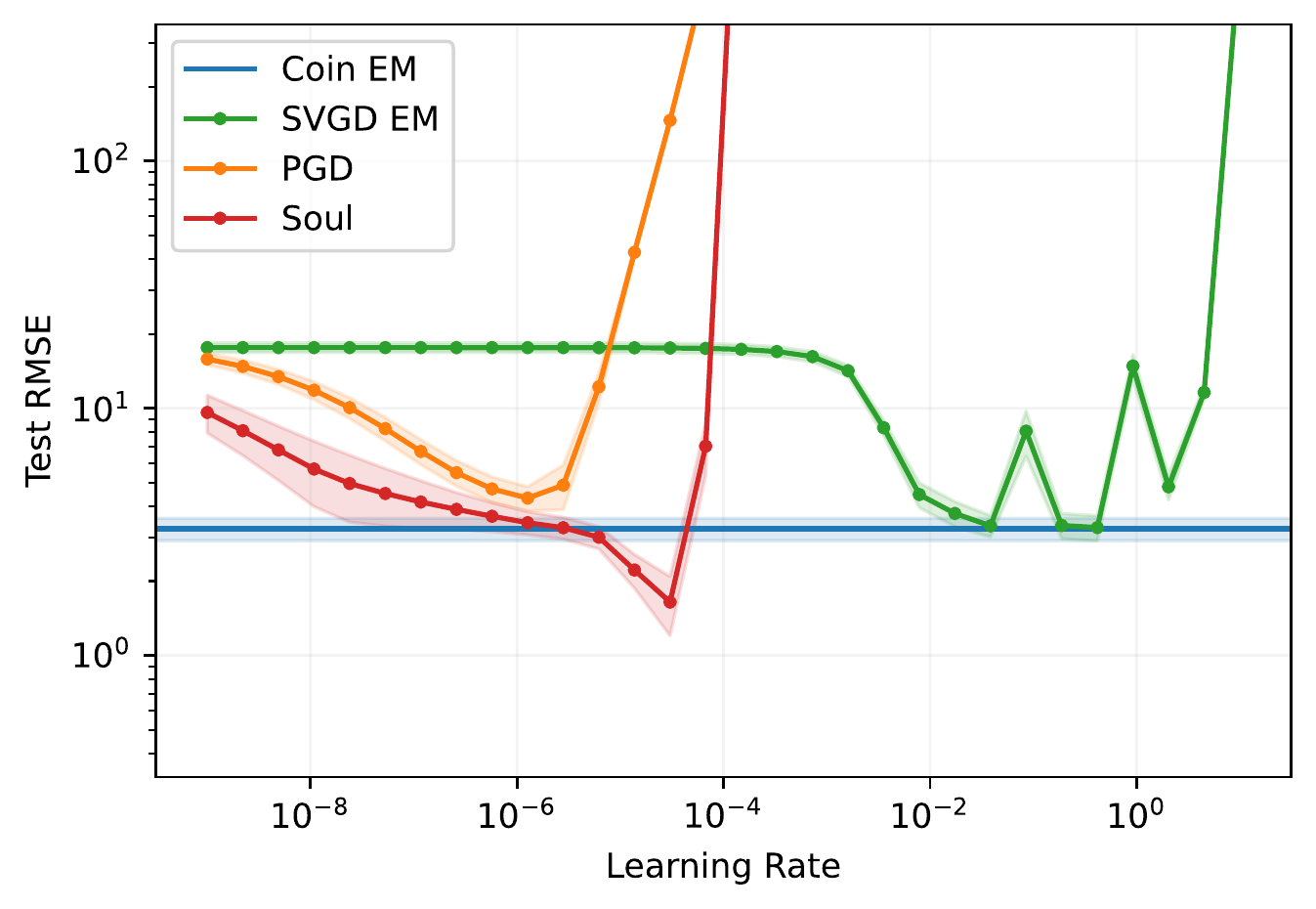}}
  \subfigure[Kin8nm]{\includegraphics[width=0.32\textwidth]{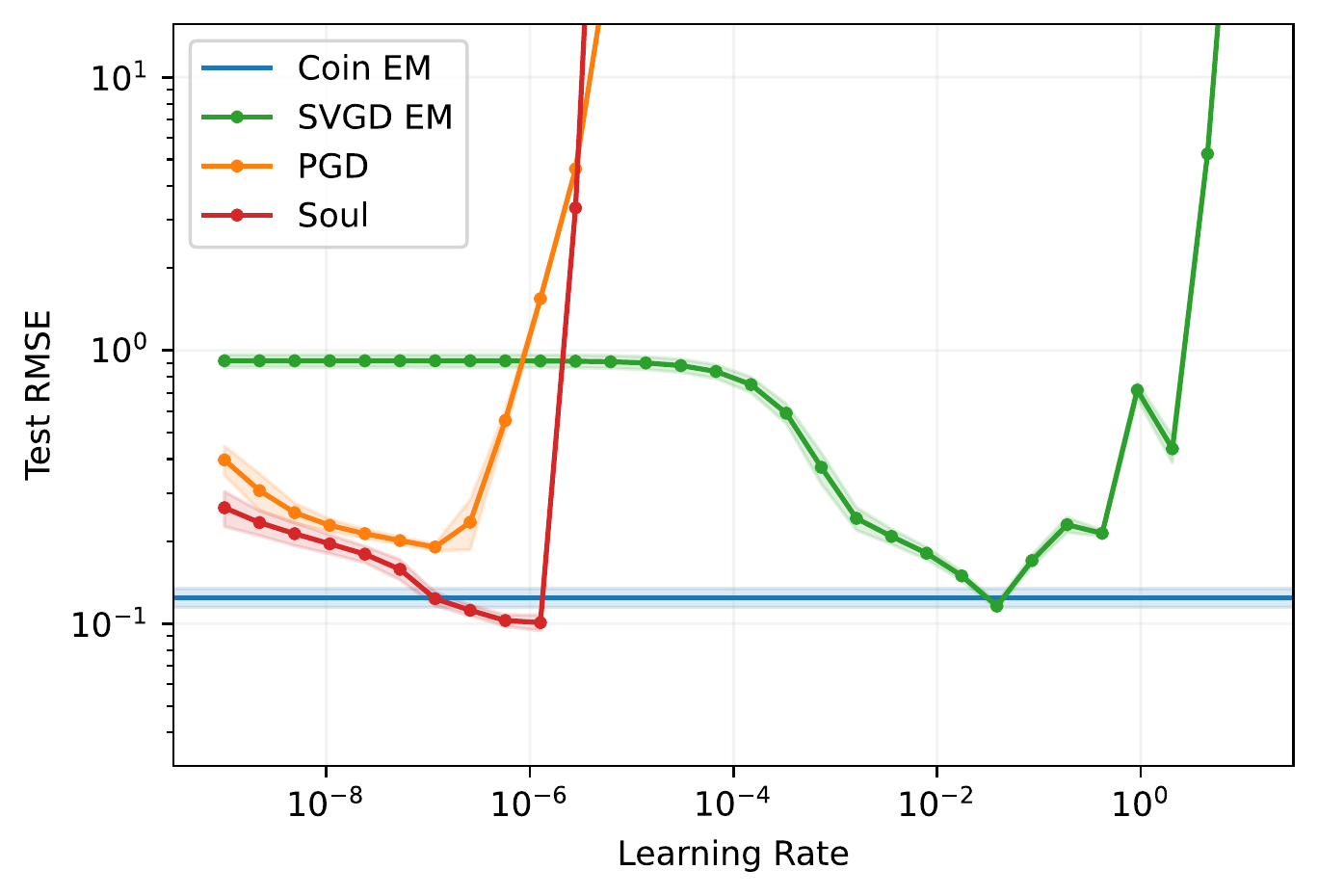}}
  \subfigure[Boston]{\includegraphics[width=0.32\textwidth]{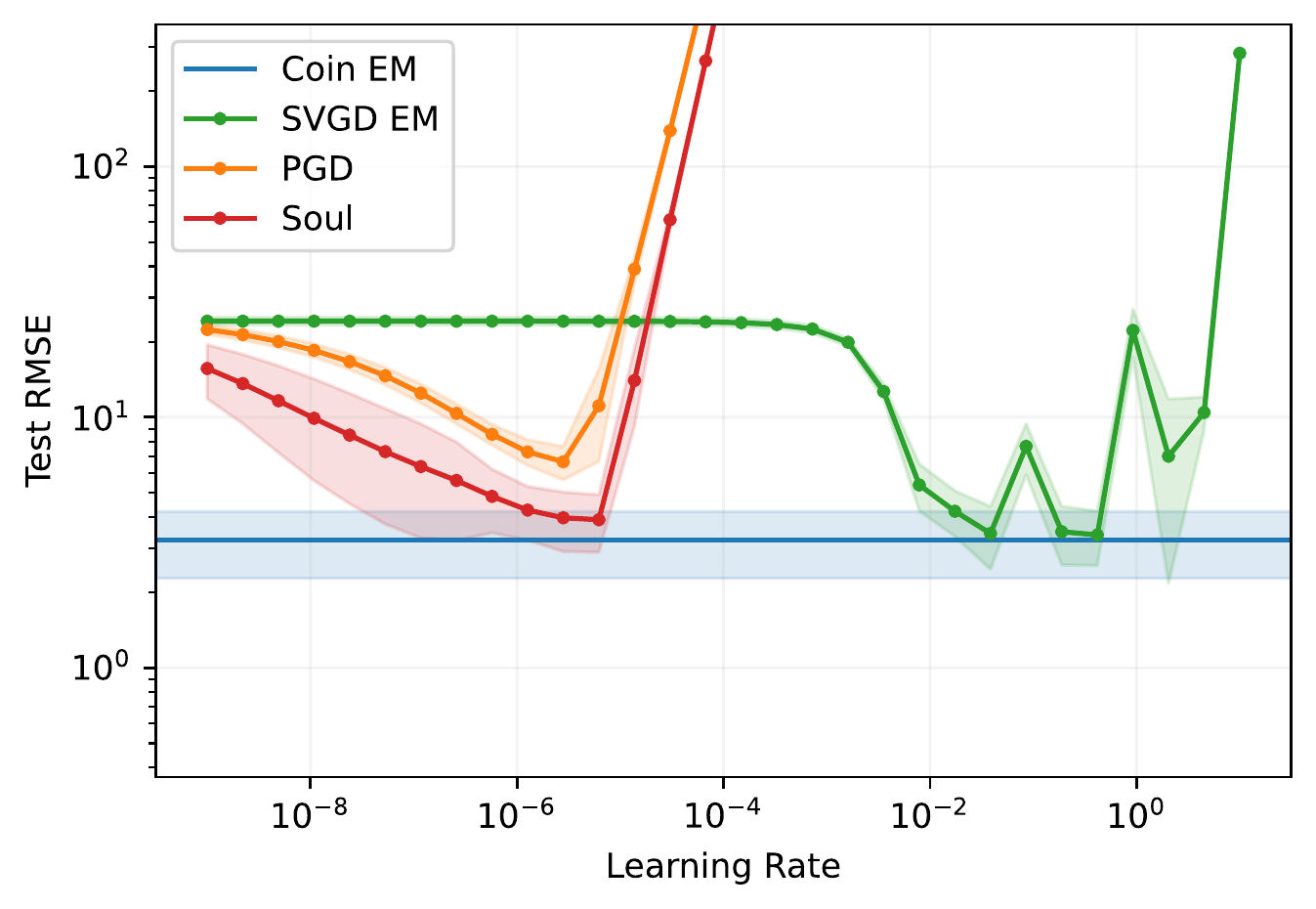}}
  \subfigure[Naval]{\includegraphics[width=0.32\textwidth]{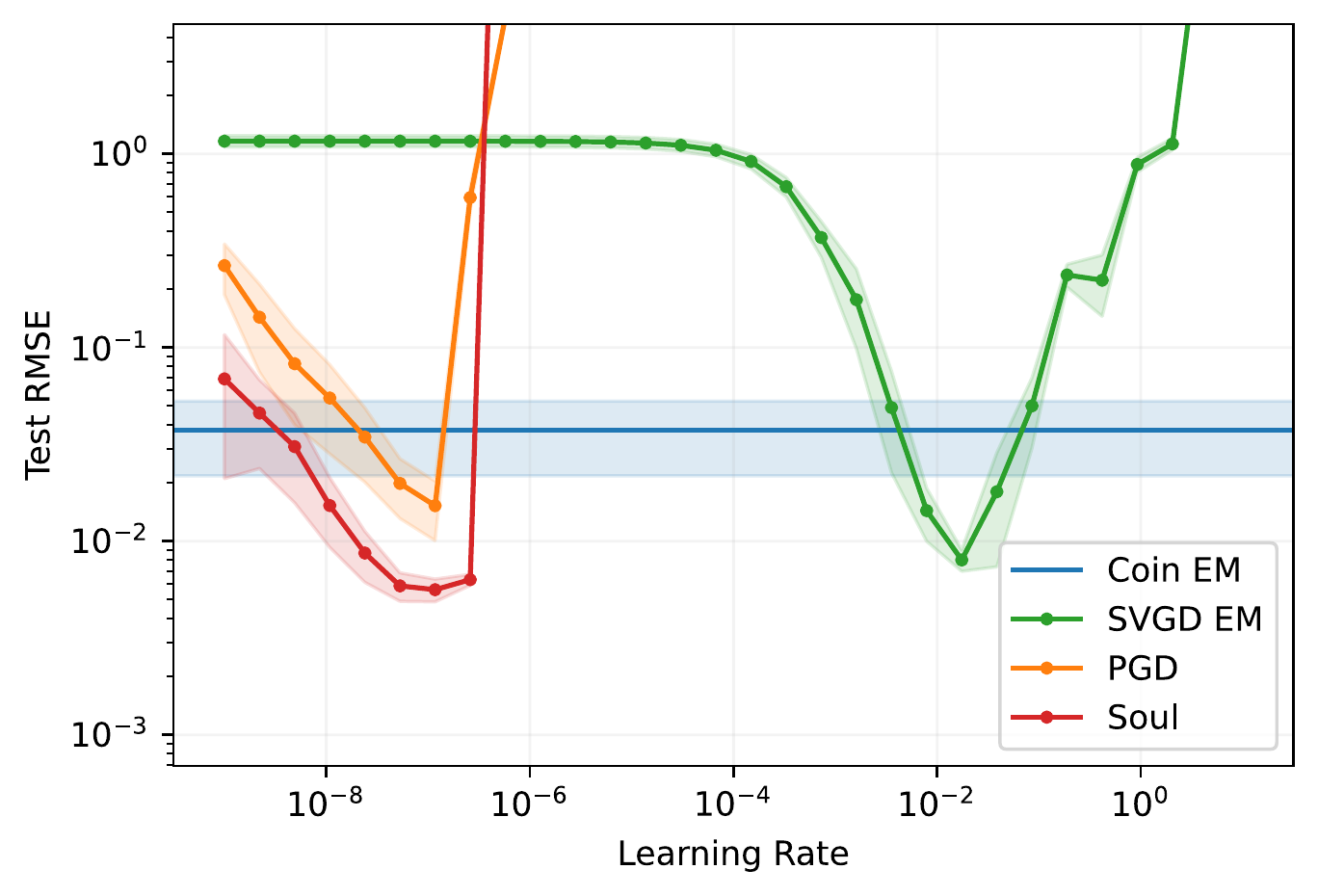}}
  \subfigure[Power. \label{fig:bayes_nn_additional_power}]{\includegraphics[width=0.32\textwidth]{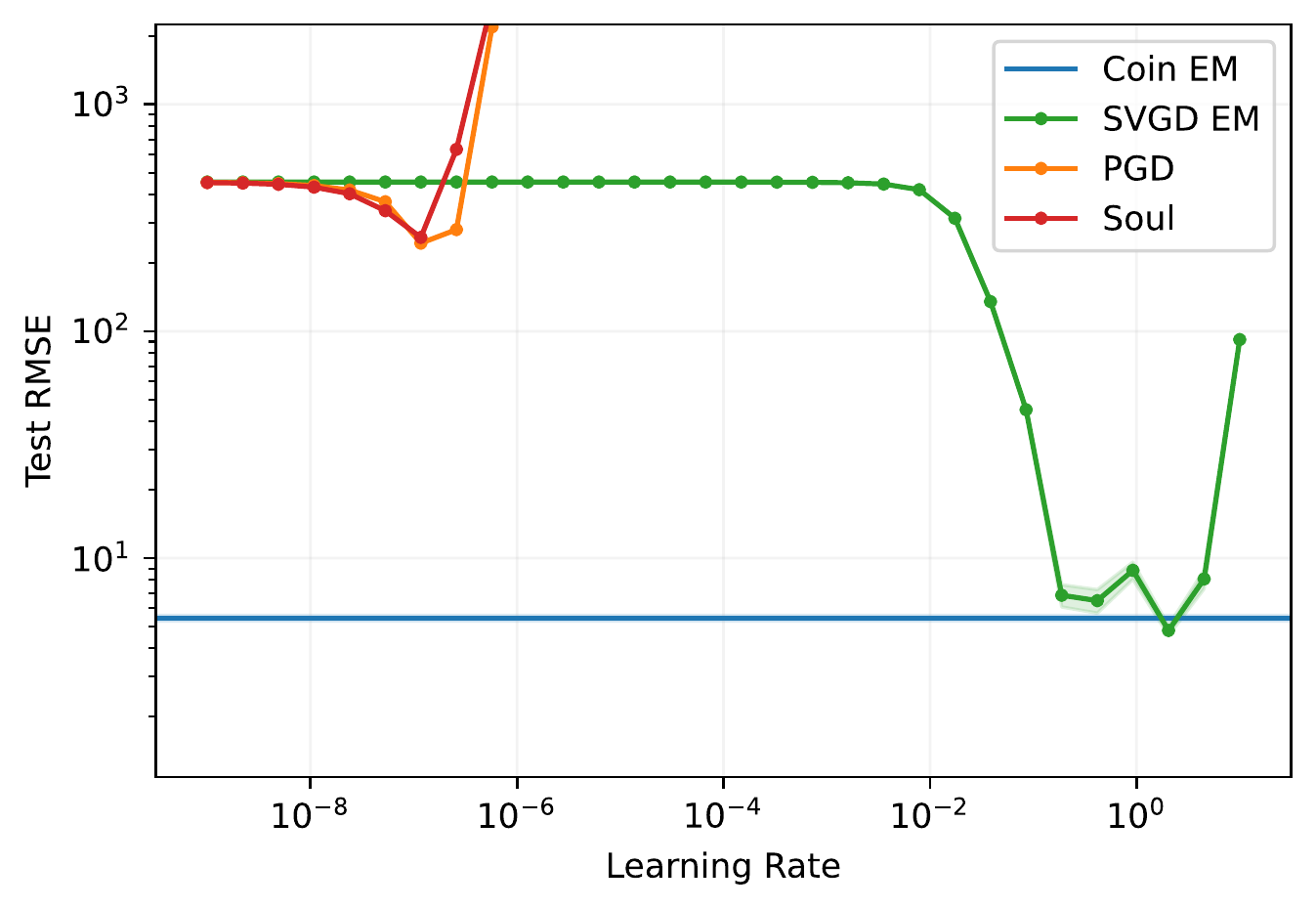}}
  \caption{\textbf{Results for the alternative Bayesian neural network in App. \ref{sec:bayes-nn-details}}. Root mean-squared-error (RMSE) as a function of the learning rate, for several UCI datasets, averaged over ten random test-train splits.}
  \label{fig:bayes_nn_additional}
\end{figure}

\subsection{Network model}
\label{sec:network-add-results}

In this section, we provide additional results for latent space network model studied in Sec. \ref{sec:latent_space_model}. In Fig. \ref{fig:coin_network_model} - Fig. \ref{fig:soul_network_model}, we plot the mean of the particles $\{z_{T}^{i}\}_{i=1}^N$ output by Coin EM (Fig. \ref{fig:coin_network_model}), PGD (Fig. \ref{fig:pgd_network_model}), and SOUL (Fig. \ref{fig:soul_network_model}), using $N=10$ particles and $T=500$ iterations. In this case, each latent variable represents a network node, which corresponds to a Game of Thrones character. In each plot, Fig. (a) - (d) correspond to Series 1 - 4 of the TV series, respectively. The nodes are colour coded according to their cluster assignment from running DBScan \cite{Ester1996}. 

In this case, only Coin EM is able to successfully capture the subjectively correct relationships between characters. 
We experimented with various learning rates for PGD and SOUL and were not able to improve the latent representation of the nodes to infer useful clusters among the characters. If additional covariate information were available, e.g. house labels such as Targaryen, Lannister, etc. then it is possible that the additional information could improve the latent representation for the PGD and SOUL algorithms.   

\begin{figure}[htb]
\vspace{-2mm}
  \centering
  \subfigure[Season 1]{\includegraphics[width=0.35\textwidth, trim = 20 20 20 40, clip]{figs/type1/network/got_network_latent_variables_season_1_coin.pdf}}
  \subfigure[Season 2]{\includegraphics[width=0.35\textwidth, trim = 20 20 20 40, clip]{figs/type1/network/got_network_latent_variables_season_2_coin.pdf}}
  \subfigure[Season 3]{\includegraphics[width=0.35\textwidth, trim = 20 20 20 40, clip]{figs/type1/network/got_network_latent_variables_season_3_coin.pdf}}
   \subfigure[Season 4]{\includegraphics[width=0.35\textwidth, trim = 20 20 20 40, clip]{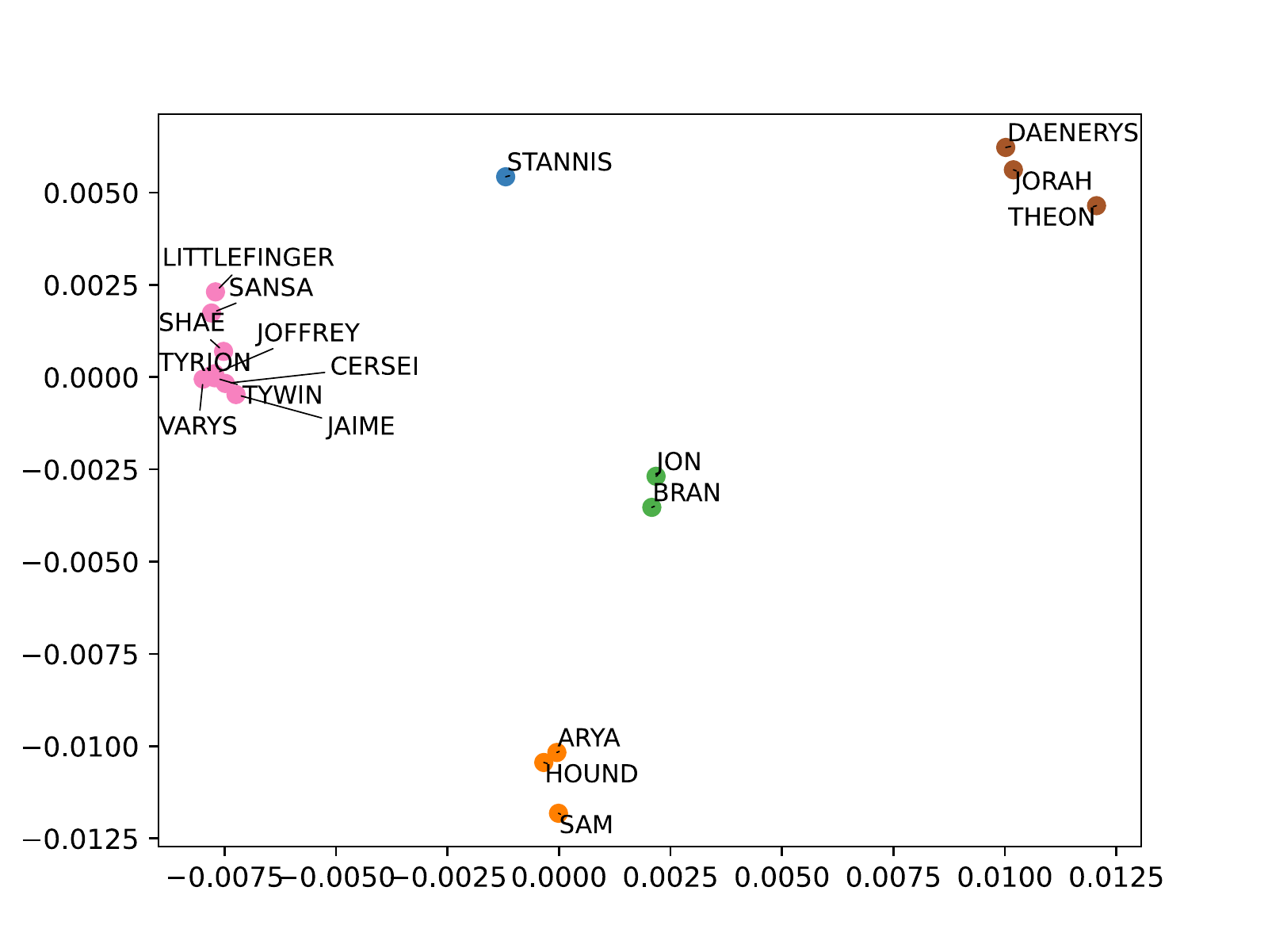}}
  \caption{\textbf{Additional results for the network model in Sec. \ref{sec:latent_space_model}}. Posterior mean $\frac{1}{N}\sum_{j=1}^N z_{T}^{j}$ of the particles generated by {Coin EM} after $T=500$ iterations.}
  \label{fig:coin_network_model}
  \vspace{-3mm}
\end{figure}

\begin{figure}[htb]
  \centering
  \subfigure[Season 1]{\includegraphics[width=0.35\textwidth, trim = 20 20 20 40, clip]{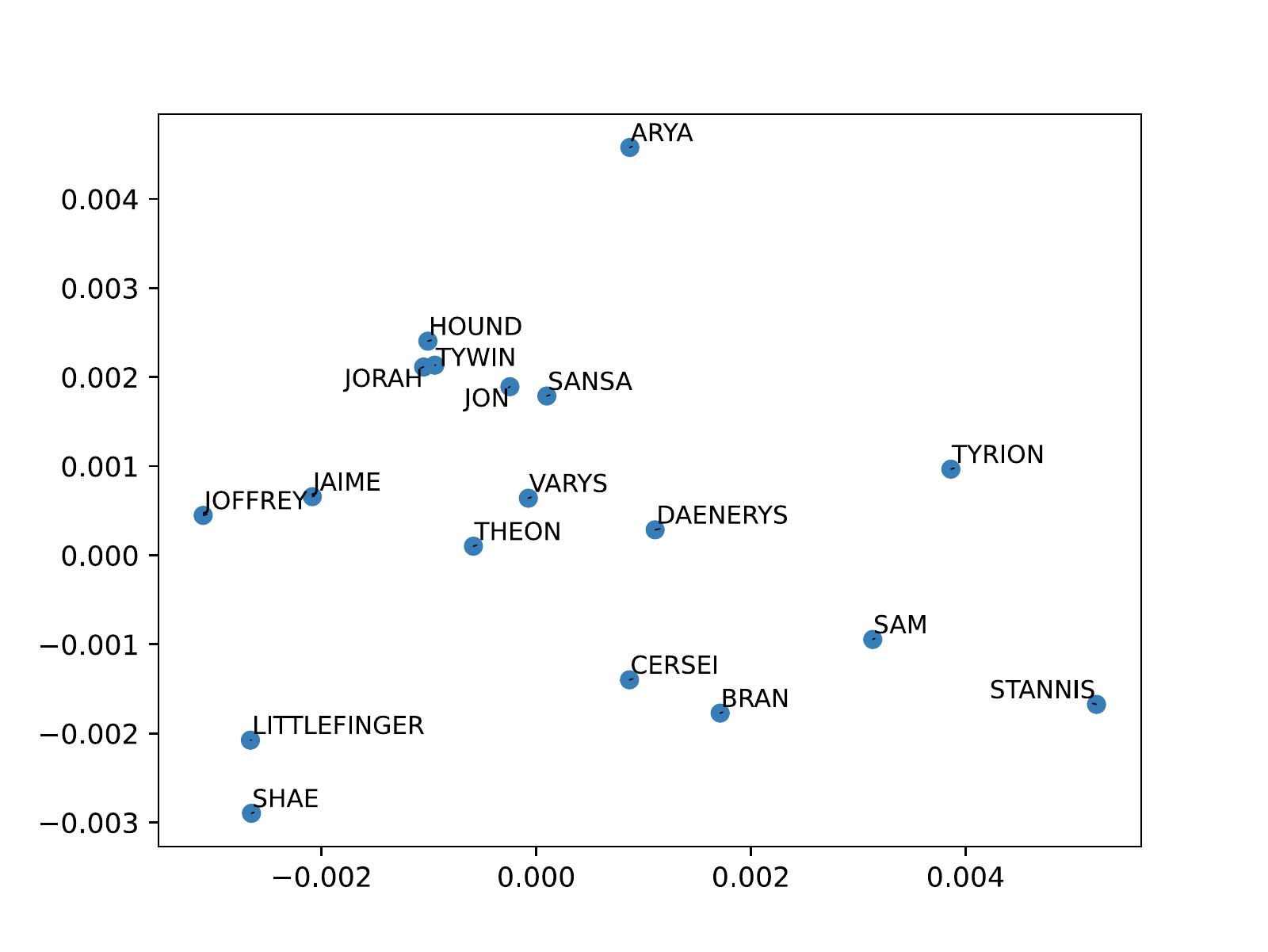}}
  \subfigure[Season 2]{\includegraphics[width=0.35\textwidth, trim = 20 20 20 40, clip]{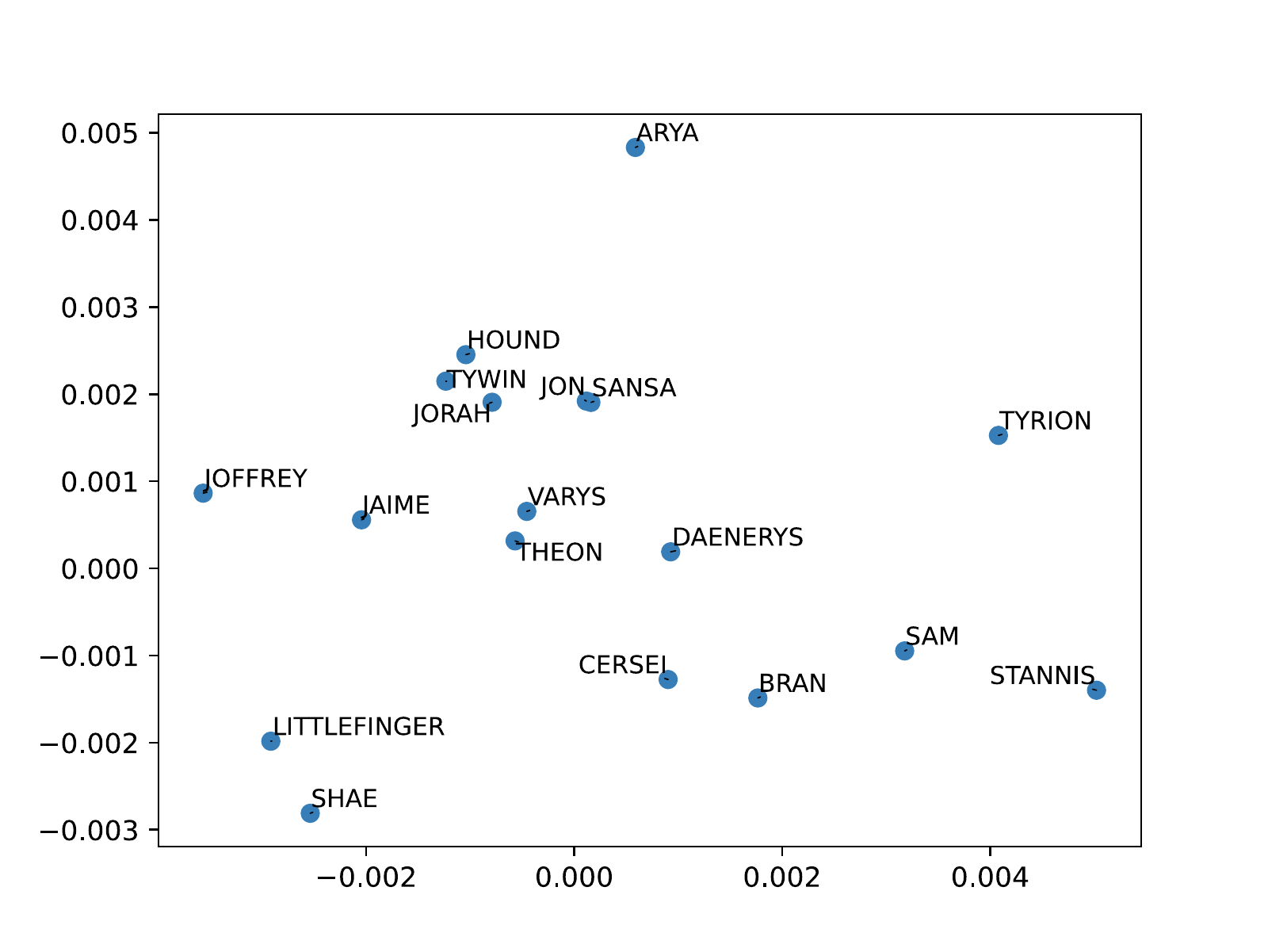}}
  \subfigure[Season 3]{\includegraphics[width=0.35\textwidth, trim = 20 20 20 40, clip]{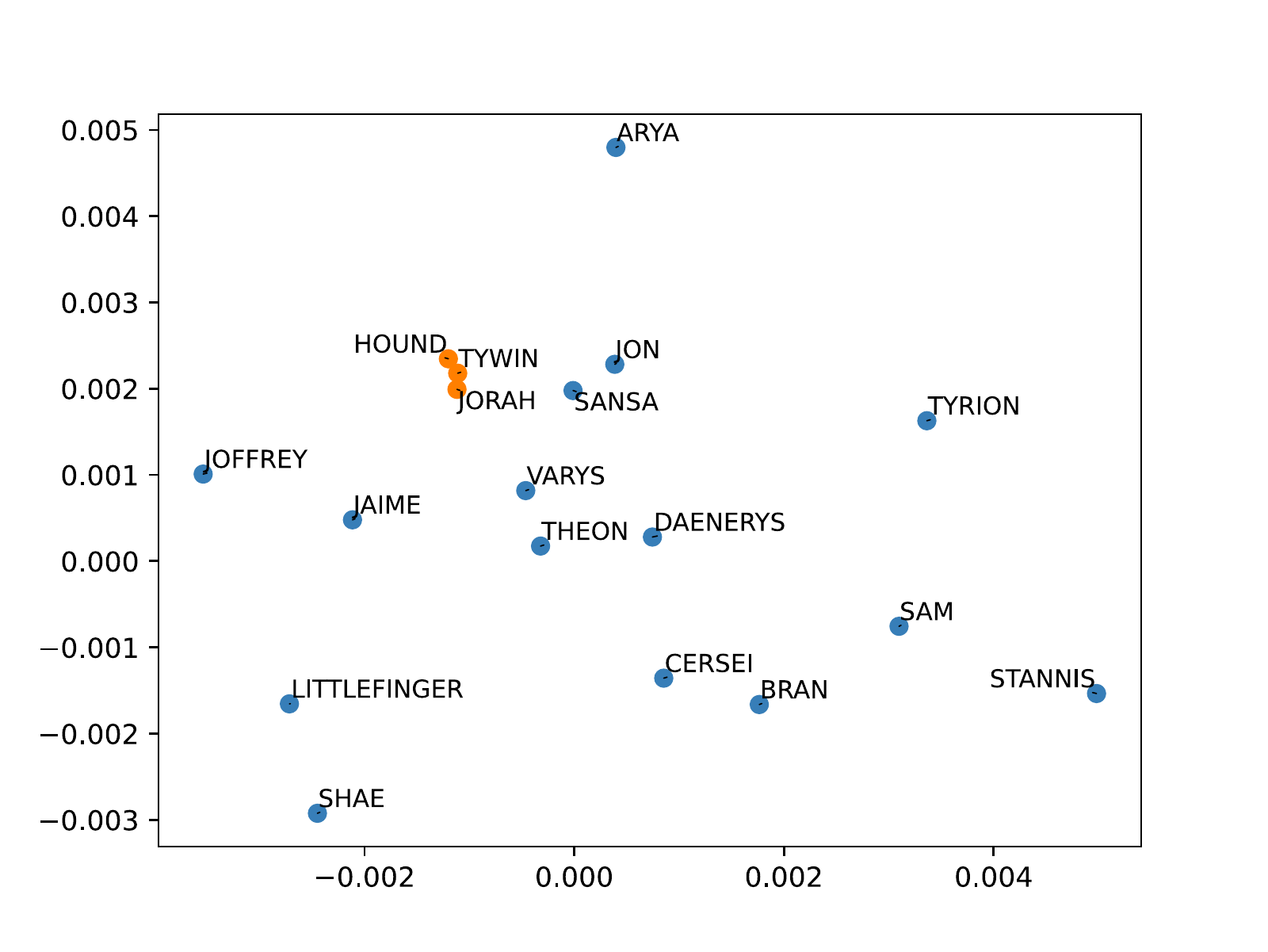}}
   \subfigure[Season 4]{\includegraphics[width=0.35\textwidth, trim = 20 20 20 40, clip]{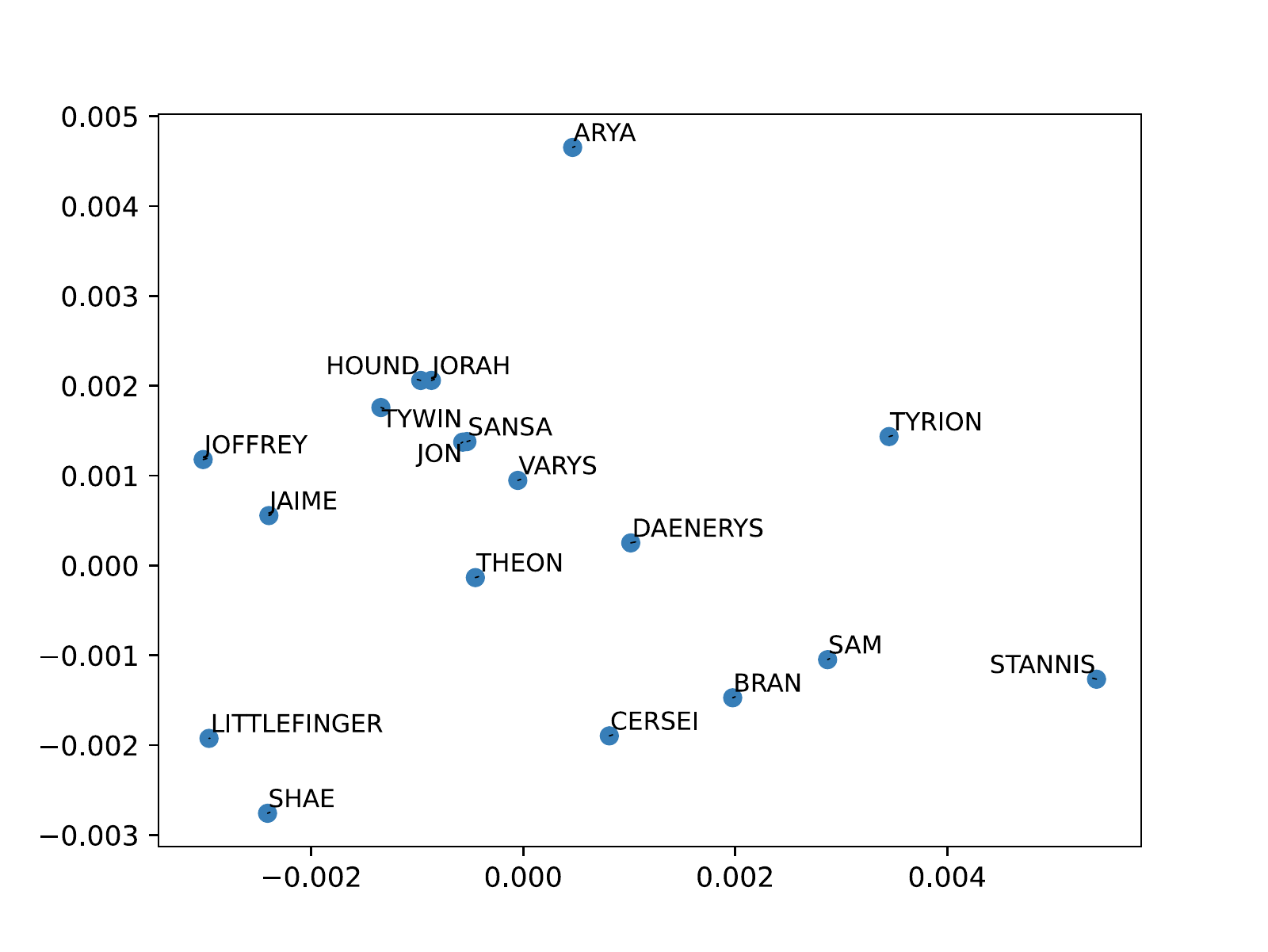}}
  \caption{\textbf{Additional results for the network model in Sec. \ref{sec:latent_space_model}}. Posterior mean $\frac{1}{N}\sum_{j=1}^N z_{T}^{j}$ of the particles generated by {PGD} after $T=500$ iterations.}
  \label{fig:pgd_network_model}
\vspace{-5mm}
\end{figure}

\begin{figure}[htb]
  \centering
  \subfigure[Season 1]{\includegraphics[width=0.35\textwidth, trim = 20 20 20 40, clip]{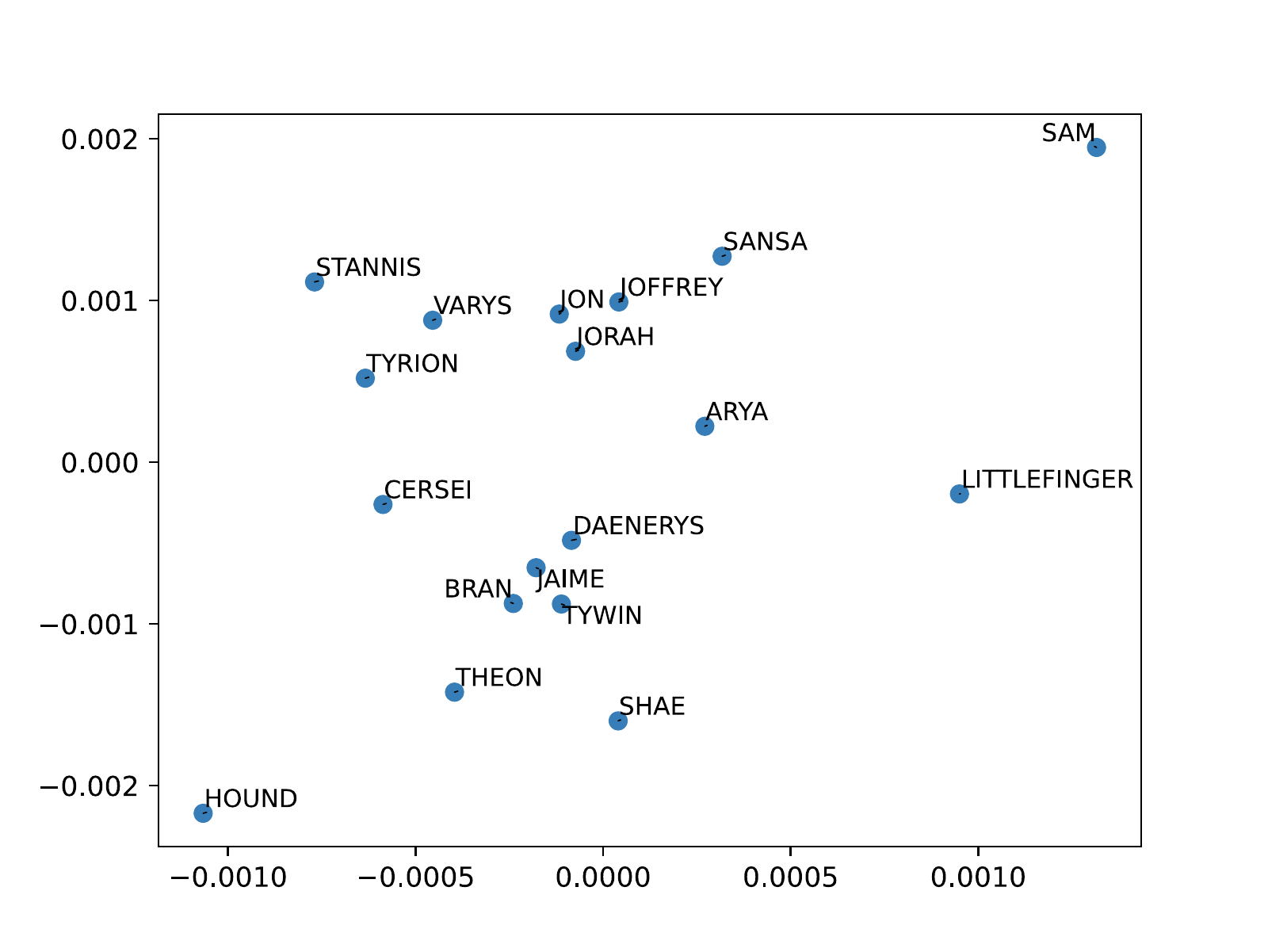}}
  \subfigure[Season 2]{\includegraphics[width=0.35\textwidth, trim = 20 20 20 40, clip]{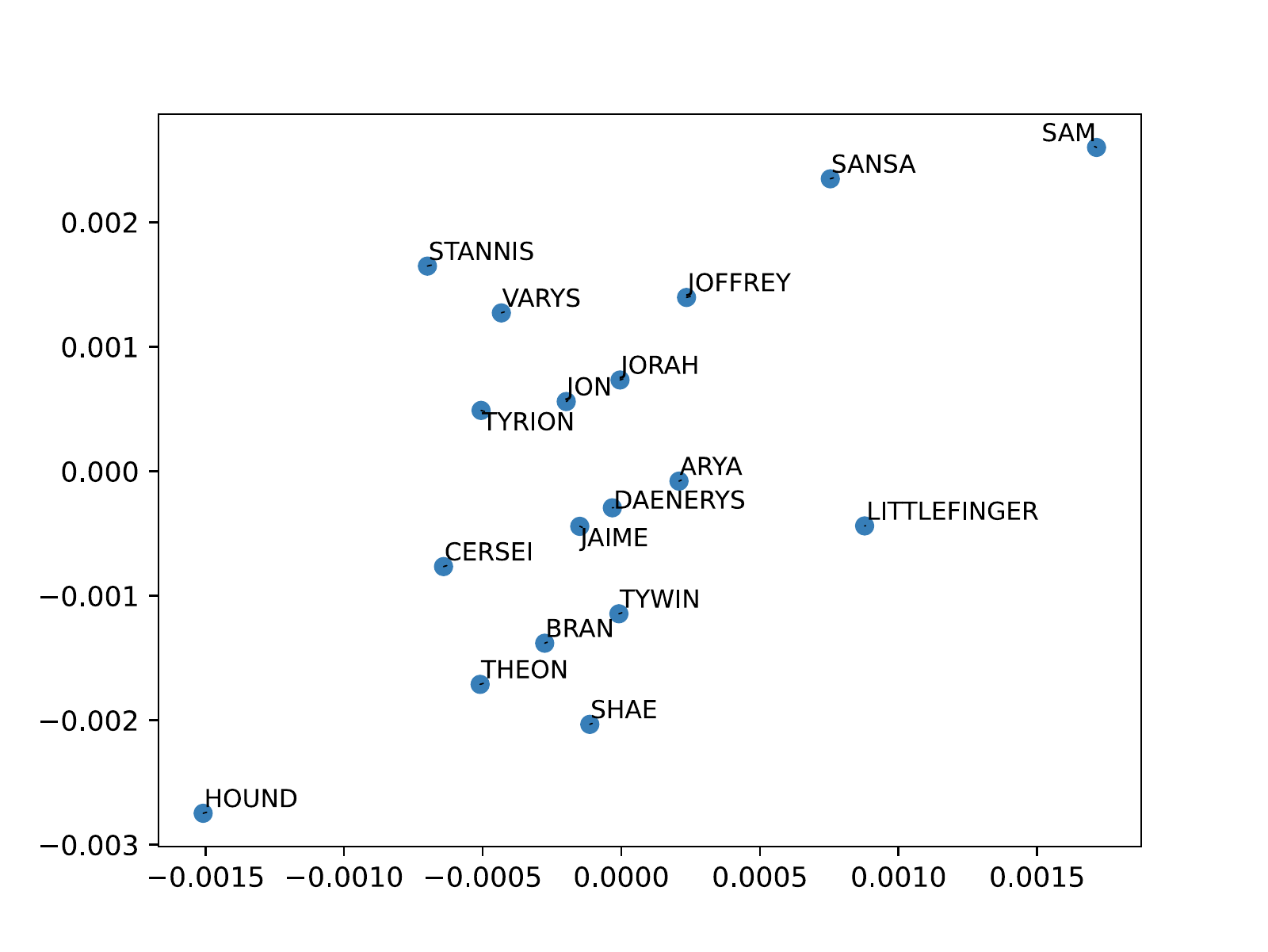}}
  \subfigure[Season 3]{\includegraphics[width=0.35\textwidth, trim = 20 20 20 40, clip]{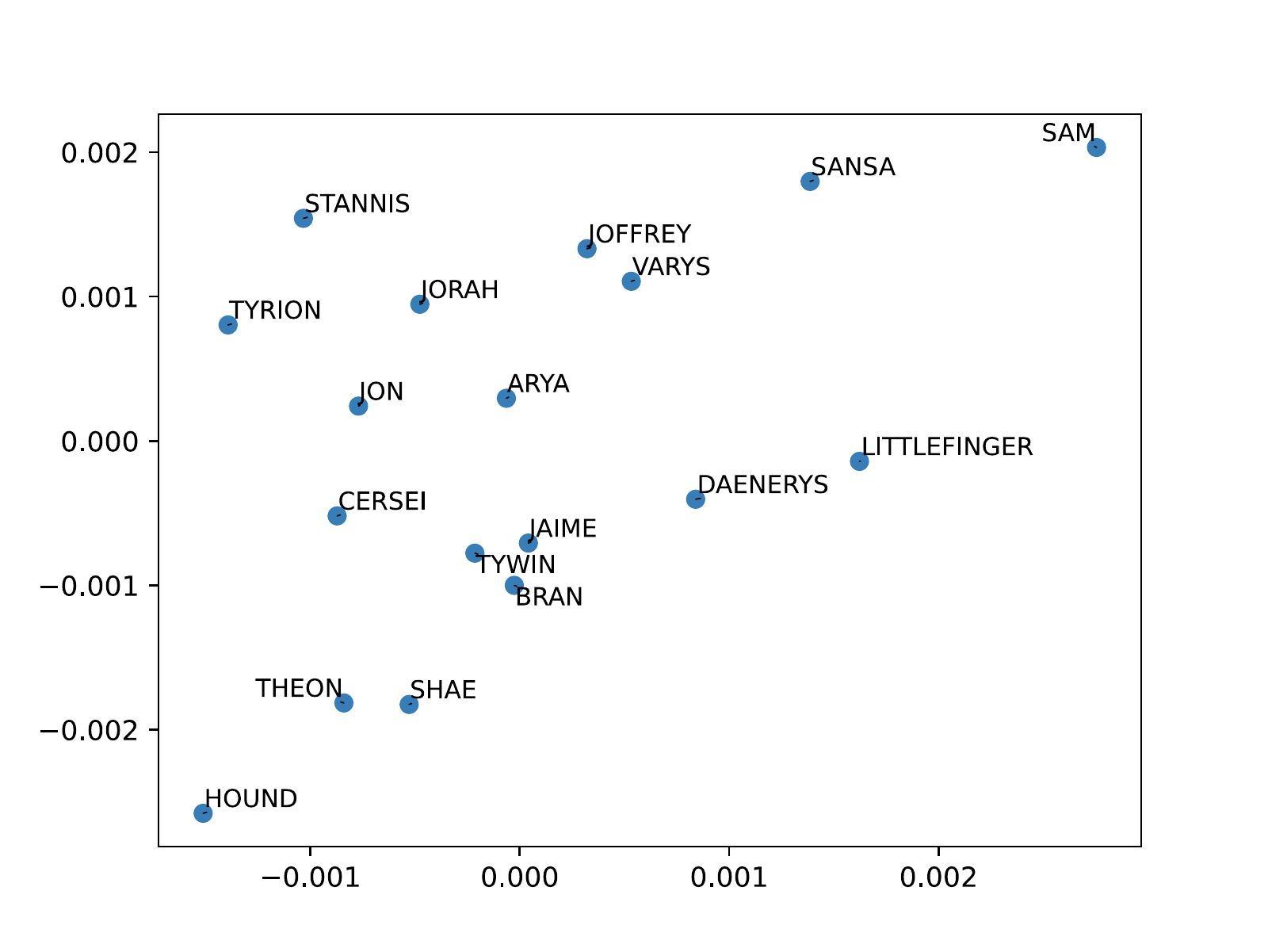}}
   \subfigure[Season 4]{\includegraphics[width=0.35\textwidth, trim = 20 20 20 40, clip]{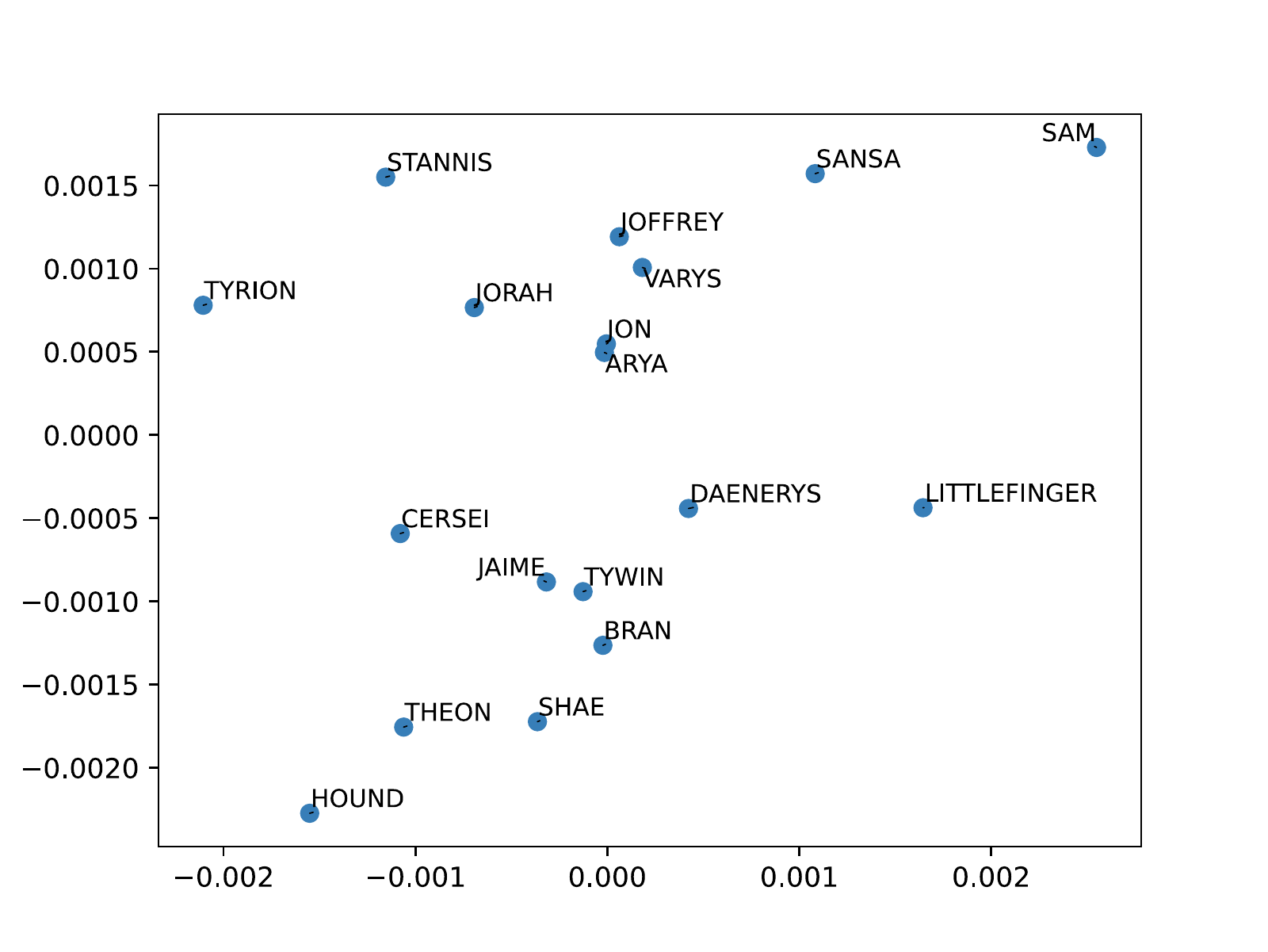}}
  \caption{\textbf{Additional results for the network model in Sec. \ref{sec:latent_space_model}}. Posterior mean $\frac{1}{N}\sum_{j=1}^N z_{T}^{j}$ of the particles generated by {SOUL} after $T=500$ iterations.}
  \label{fig:soul_network_model}
  \vspace{-4mm}
\end{figure}

\vfill

\end{document}